\documentclass{article}

\usepackage{amstext,amssymb,latexsym,alltt}
\usepackage{graphicx,tabularx} 
\usepackage{amsmath,colonequals}
\usepackage{array} 
\usepackage[colorlinks=true,%
            linkcolor=blue,%
            citecolor=blue,%
            urlcolor=blue]{hyperref}           

\usepackage{fancyhdr}           
\newenvironment{proof}{\begin{trivlist}\item[]{\bf Proof}}{\hspace*{\fill} $\blacksquare$ \end{trivlist}}
\usepackage{ifthen}
\usepackage[T1]{fontenc}
\fontencoding{T1}\selectfont
\makeatletter
\DeclareRobustCommand*\cal{\@fontswitch\relax\mathcal}
\makeatother

\usepackage{amssymb,txfonts,latexsym,itemlist,epic,eepic}

\newcommand{\dwedge}[2]%
{\overset{#1}{\underset{#2}{\mbox{$\bigwedge\hspace{-2ex}\bigwedge$}}}}
\newcommand{\dvee}[2]%
{\overset{#1}{\underset{#2}{\mbox{$\bigvee\hspace{-2ex}\bigvee$}}}}

\newcommand{\jset}[1]{\langle {#1}\rangle}

\newcommand{\inMath}[1]{\relax\ifmmode{#1}%
\else{\mbox{$#1$}}\fi}

\newcommand{\vdot}%
{\unitlength0.4mm
\begin{picture}(8,10)
\put(4,0){.}
\put(0,10){.}
\put(8,10){.}
\put(2,5){.}
\put(6,5){.}
\end{picture}
}

\newcommand{\ignore}[1]{}

\newcommand{\romanref}[1]{%
\if \ref{#1}\empty {\setcounter{romanrefcounter}0} \else 
{\setcounter{romanrefcounter}{\ref{#1}}}\fi%
{\it \roman{romanrefcounter}}}

\newcommand{\TO}{\twoheadrightarrow}

\newcommand{\half}{\mbox{$\frac{1}{2}$}}

\newcommand{\BBB}{\mbox{\(\Bbb B\)}}

\newcommand{\CC}{\mbox{$\mathcal{C}$}}
\newcommand{\BBC}{\mbox{\(\Bbb C\)}}
\newcommand{\BD}{{\bf D}}
\newcommand{\CD}{\mbox{$\mathcal{D}$}}

\newcommand{\BE}{{\bf E}}
\newcommand{\Be}{{\bf e}}

\newcommand{\BBF}{\mbox{\(\mathbb{F}\)}}
\newcommand{\Bf}{{\bf f}}

\newcommand{\BH}{{\bf H}}

\newcommand{\BK}{{\bf K}}

\newcommand{\BL}{{{\bf L}}}
\newcommand{\CL}{{\mathcal{L}}}
\newcommand{\BBL}{\mbox{\(\Bbb L\)}}

\newcommand{\BBM}{\mbox{\(\Bbb M\)}}
\newcommand{\Bm}{{\bf m}}

\newcommand{\Bn}{{\bf n}}
\newcommand{\BBN}{\mbox{\(\Bbb N\)}}

\newcommand{\BBR}{\mbox{\(\Bbb R\)}}
\newcommand{\CS}{{\mathcal{S}}}

\newcommand{\Bt}{{\bf t}}

\newcommand{\Bu}{{\bf u}}

\newlength{\circlen}
\newlength{\symblen}

\ifx\text\undefined
  \newcommand{\text}[1]{\relax
    \ifmmode\mathchoice
      {\hbox{\the\textfont0\relax#1}}%
      {\hbox{\the\textfont0\relax#1}}%
      {\hbox{\the\scriptfont0\relax#1}}%
      {\hbox{\the\scriptscriptfont0\relax#1}}%
    \else{\relax#1}\fi}
  \fi

\newcommand{\defsub}[1]{}

\newcommand{\proves}{\mbox{${\mid\!\sim}$}}

\newcommand{\To}{\Rightarrow}

\def\twoheaddownarrow{\rlap{$\downarrow$}\raise-.5ex\hbox{$\downarrow$}}
\def\twoheaduparrow{\rlap{$\uparrow$}\raise.5ex\hbox{$\uparrow$}}

\def\texturespicture #1 by #2 (#3){
\vbox to #2 {\hrule width #1 height 0pt depth 0pt
}}

\def\scaledpicture #1 by #2 (#3 scaled #4){{\dimen0=#1 \dimen1=#2
\divide\dimen0 by 1000 \multiply \dimen0 by #4
\divide\dimen1 by 1000 \multiply \dimen1 by #4
\texturespicture\dimen0 by \dimen1 (#3 scaled #4)}}

 {\begin{tabular}[t]{llr} && \fbox{#1}\\}%
 {\end{tabular}}

{\end{list}}

{\end{tabular}%
\end{trivlist}}

\def\introrule{{\cal I}}\def\elimrule{{\cal E}}

\newdimen\proofrulebreadth \proofrulebreadth=.05em
\newdimen\proofdotseparation \proofdotseparation=1.25ex
\newdimen\proofrulebaseline \proofrulebaseline=2ex
\newcount\proofdotnumber \proofdotnumber=3
\let\then\relax
\def\hfi{\hskip0pt plus.0001fil}
\mathchardef\squigto="3A3B
\newif\ifinsideprooftree\insideprooftreefalse
\newif\ifonleftofproofrule\onleftofproofrulefalse
\newif\ifproofdots\proofdotsfalse
\newif\ifdoubleproof\doubleprooffalse
\let\wereinproofbit\relax
\newdimen\shortenproofleft
\newdimen\shortenproofright
\newdimen\proofbelowshift
\newbox\proofabove
\newbox\proofbelow
\newbox\proofrulename
\def\shiftproofbelow{\let\next\relax\afterassignment\setshiftproofbelow\dimen0 }
\def\shiftproofbelowneg{\def\next{\multiply\dimen0 by-1 }%
\afterassignment\setshiftproofbelow\dimen0 }
\def\setshiftproofbelow{\next\proofbelowshift=\dimen0 }
\def\setproofrulebreadth{\proofrulebreadth}

\def\prooftree{
\ifnum  \lastpenalty=1
\then   \unpenalty
\else   \onleftofproofrulefalse
\fi
\ifonleftofproofrule
\else   \ifinsideprooftree
        \then   \hskip.5em plus1fil
        \fi
\fi
\bgroup
\setbox\proofbelow=\hbox{}\setbox\proofrulename=\hbox{}%
\let\justifies\proofover\let\leadsto\proofoverdots\let\Justifies\proofoverdbl
\let\using\proofusing\let\[\prooftree
\ifinsideprooftree\let\]\endprooftree\fi
\proofdotsfalse\doubleprooffalse
\let\thickness\setproofrulebreadth
\let\shiftright\shiftproofbelow \let\shift\shiftproofbelow
\let\shiftleft\shiftproofbelowneg
\let\ifwasinsideprooftree\ifinsideprooftree
\insideprooftreetrue
\setbox\proofabove=\hbox\bgroup$\displaystyle 
\let\wereinproofbit\prooftree
\shortenproofleft=0pt \shortenproofright=0pt \proofbelowshift=0pt
\onleftofproofruletrue\penalty1
}

\def\eproofbit{
\ifx    \wereinproofbit\prooftree
\then   \ifcase \lastpenalty
        \then   \shortenproofright=0pt  
        \or     \unpenalty\hfil         
        \or     \unpenalty\unskip       
        \else   \shortenproofright=0pt  
        \fi
\fi
\global\dimen0=\shortenproofleft
\global\dimen1=\shortenproofright
\global\dimen2=\proofrulebreadth
\global\dimen3=\proofbelowshift
\global\dimen4=\proofdotseparation
\global\count255=\proofdotnumber
$\egroup  
\shortenproofleft=\dimen0
\shortenproofright=\dimen1
\proofrulebreadth=\dimen2
\proofbelowshift=\dimen3
\proofdotseparation=\dimen4
\proofdotnumber=\count255
}

\def\proofover{
\eproofbit 
\setbox\proofbelow=\hbox\bgroup 
\let\wereinproofbit\proofover
$\displaystyle
}%
\def\proofoverdbl{
\eproofbit 
\doubleprooftrue
\setbox\proofbelow=\hbox\bgroup 
\let\wereinproofbit\proofoverdbl
$\displaystyle
}%
\def\proofoverdots{
\eproofbit 
\proofdotstrue
\setbox\proofbelow=\hbox\bgroup 
\let\wereinproofbit\proofoverdots
$\displaystyle
}%
\def\proofusing{
\eproofbit 
\setbox\proofrulename=\hbox\bgroup 
\let\wereinproofbit\proofusing
\kern0.3em$
}

\def\endprooftree{
\eproofbit 
  \dimen5 =0pt
\dimen0=\wd\proofabove \advance\dimen0-\shortenproofleft
\advance\dimen0-\shortenproofright
\dimen1=.5\dimen0 \advance\dimen1-.5\wd\proofbelow
\dimen4=\dimen1
\advance\dimen1\proofbelowshift \advance\dimen4-\proofbelowshift
\ifdim  \dimen1<0pt
\then   \advance\shortenproofleft\dimen1
        \advance\dimen0-\dimen1
        \dimen1=0pt
        \ifdim  \shortenproofleft<0pt
        \then   \setbox\proofabove=\hbox{%
                        \kern-\shortenproofleft\unhbox\proofabove}%
                \shortenproofleft=0pt
        \fi
\fi
\ifdim  \dimen4<0pt
\then   \advance\shortenproofright\dimen4
        \advance\dimen0-\dimen4
        \dimen4=0pt
\fi
\ifdim  \shortenproofright<\wd\proofrulename
\then   \shortenproofright=\wd\proofrulename
\fi
\dimen2=\shortenproofleft \advance\dimen2 by\dimen1
\dimen3=\shortenproofright\advance\dimen3 by\dimen4
\ifproofdots
\then
        \dimen6=\shortenproofleft \advance\dimen6 .5\dimen0
        \setbox1=\vbox to\proofdotseparation{\vss\hbox{$\cdot$}\vss}%
        \setbox0=\hbox{%
                \advance\dimen6-.5\wd1
                \kern\dimen6
                $\vcenter to\proofdotnumber\proofdotseparation
                        {\leaders\box1\vfill}$%
                \unhbox\proofrulename}%
\else   \dimen6=\fontdimen22\the\textfont2 
        \dimen7=\dimen6
        \advance\dimen6by.5\proofrulebreadth
        \advance\dimen7by-.5\proofrulebreadth
        \setbox0=\hbox{%
                \kern\shortenproofleft
                \ifdoubleproof
                \then   \hbox to\dimen0{%
                        $\mathsurround0pt\mathord=\mkern-6mu%
                        \cleaders\hbox{$\mkern-2mu=\mkern-2mu$}\hfill
                        \mkern-6mu\mathord=$}%
                \else   \vrule height\dimen6 depth-\dimen7 width\dimen0
                \fi
                \unhbox\proofrulename}%
        \ht0=\dimen6 \dp0=-\dimen7
\fi
\let\doll\relax
\ifwasinsideprooftree
\then   \let\VBOX\vbox
\else   \ifmmode\else$\let\doll=$\fi
        \let\VBOX\vcenter
\fi
\VBOX   {\baselineskip\proofrulebaseline \lineskip.2ex
        \expandafter\lineskiplimit\ifproofdots0ex\else-0.6ex\fi
        \hbox   spread\dimen5   {\hfi\unhbox\proofabove\hfi}%
        \hbox{\box0}%
        \hbox   {\kern\dimen2 \box\proofbelow}}\doll%
\global\dimen2=\dimen2
\global\dimen3=\dimen3
\egroup 
\ifonleftofproofrule
\then   \shortenproofleft=\dimen2
\fi
\shortenproofright=\dimen3
\onleftofproofrulefalse
\ifinsideprooftree
\then   \hskip.5em plus 1fil \penalty2
\fi
}

\newcommand\strikethrough[1]{{\setbox0=\hbox{$#1$}
\hrule height.75ex depth-.65ex width\wd0 \kern-\wd0\box0}}

\mathchardef\gt="313E \mathchardef\lt="313C

\input{rotate.tex}
\newsavebox{\iotabox}
\savebox{\iotabox}{$\iota$}
\newsavebox{\updowniota}
\savebox{\updowniota}{\rotu\iotabox}

\newsavebox{\ampersandbox}
\savebox{\ampersandbox}{$\&$}
\newsavebox{\updownampersand}
\savebox{\updownampersand}{\rotu\ampersandbox}

\newcommand{\comma}{,\ldots ,}

\newtheorem{theorem}{Theorem}[section]
 \newtheorem{definition}[theorem]{Definition}
\newtheorem{example}[theorem]{Example}
\newtheorem{remark}[theorem]{Remark}

\newtheorem{lemma}[theorem]{Lemma}

\newcommand{\tO}{\twoheadrightarrow}
\renewcommand{\max}{\rm max}

\newcommand{\LtO}{\twoheadleftarrow\!\!\twoheadrightarrow}
\newcommand{\vare}{\varepsilon}
\newcommand{\Bi}{\bf i}

\begin{document}

\title{Theory of Semi-Instantiation in Abstract Argumentation}
\author{D. M. Gabbay\\
Bar Ilan University, Israel\\
King's College London, UK\\
University of Luxembourg, Luxembourg\\
University of Manchester, UK\\
{\tt dov.gabbay@kcl.ac.uk}\thanks{Research  supported by the Israel Science Foundation Project 1321/10: Integrating Logic and Networks.}}


\date{}

\maketitle

\begin{abstract}
We study instantiated abstract argumentation frames of the form \linebreak $(S,R,I)$, where $(S,R)$ is an abstract argumentation frame and where the arguments $x$ of $S$ are instantiated by $I(x)$ as well formed formulas of a well known logic, for example as Boolean formulas or as predicate logic formulas or as modal logic formulas.  We use the method of conceptual analysis to derive the properties of our proposed system.  We seek to define the notion of  complete extensions for such  systems and provide algorithms for finding such extensions. We further develop a theory of instantiation in the abstract, using the framework of Boolean attack formations and of conjunctive and disjunctive attacks. We discuss applications and compare critically with the existing related literature.
\end{abstract}

\maketitle

\section{Motivation and Orientation}
This paper studies semi-instantiated argumentation network of the form $(S, R, I)$, where $(S, R)$, $R\subseteq S\times S$ is an abstract argumentation network and $I$ is an instantiation function, giving for each $x\in S$ a formula $I(x)$ of some logic \BL.

The attack relation $R$ is not instantiated and remains abstract.  We are not told,  in terms of the logic \BL, why there is an attack.

There are several possibilities for such a system to arise.

\paragraph{Option 1.}  We can view such a system as semi-ASPIC like instantiation. The APPIC approach (see \cite{509-31}) will start with a theory $\Delta$ in a logic, \BL\ define the notion of $\BL$-proofs for $\Delta$ and will elt $S$ be the set of all possible such proofs and will further define the attack relation $R \subseteq S\times S$ in terms of relationships among these proofs.   In our case we just take $\Delta$ as the arguments (no proof theory avaiable) and simply tell you abstractly what is supposed to attack what.

For example, let $\Delta = \{A, A\to \bot\}$ (with $\to$ being implication, and $\bot$ being falsity) and the attack relation be from $A$ to $A\to \bot$ but {\em not} from $A\to \bot$ to $A$.  We are not explaining why this attack relation is defined so.  There are logics, like the Lambek calculus where modus ponens works from the left but not from right.  So $A, A\to B\vdash B$ holds, but $A\to B, A\not\vdash B$.  Thus if $X\tO Y$ means $X, Y\vdash \bot$, then we get that $A\tO (A\to \bot)$ holds (``$\tO$'' is attack) but $(A\to \bot)$ does not attack $A$.

\paragraph{Option 2.}  Such instantiated systems can also arise from general methodolgoical considerations.  Let us ask ourselves a very simple question:

\paragraph{Question.}  {\em What is the added value of abstract argumentation networks over, say, classical propositional logic?  }

Obviously they have the same expressive power. Many papers by various authors have demonstrated such equivalence. My favourite is my own paper \cite{509-37}, showing that the attack relation is really the Peirce--Quine dagger connective ($x\downarrow y =\mbox{ def. } \neg x \wedge\neg y$) of classical logic.  So, we ask, what is the added value of such networks?  My answer to this is that in these networks we bring some meta-level features into the object level Dung argumentation networks, expand classical propositional logic with the meta-predicate ``$x$ is false''. When we write $z\tO x$ (i.e., $z$ attacks $x$) we are saying $z=$ ``$x$ is false'', or $z\leftrightarrow$ ``$x$ is false'' So the liar paradox becomes $x\tO x$, ``I am false''.  So the added value of abstract argumentation networks to classical propositional logic is the meta-predicate False($x$).  So the language now has $(\neg x, x\wedge y, x\vee y, x\to y)$ and the additional connective False $(x)$.  

Now the minute we accept this view we must also allow and address expressions like 
\[
x\tO A
\]
where $A$ is a wff, i.e., $x=$ ``$A$ is false'' and we then must also allow 
\[
x\wedge y\tO A
\]
meaning that $x$ and $y$ together say that $A$ is false, and now, of course, once we go this far we must also address
\[
B\tO A.
\]
The latter is nothing but the equivalence
\[
B\leftrightarrow \mbox{ False} (A).
\]
If you think about it, once we add to any logic a new connective ``$C(x)$'', we must be able to address $y\leftrightarrow C(x)$, it being just another wff.

Having established some interest in semi-instantiated argumentation networks, let us now get to business and describe the machinery and problems involved.

Let $(S, R)$ be an abstract argumentation frame. This means that $S$ is a non-empty set and $R\subseteq S\times S$. 

Let \BL\ be a logic, with a set of well formed formulas WFF$(\BL)$ and let $\mu$ be either semantics or proof theory for this logic. Assume that we have models for this logic which we denote by $\{\Bm\}$, and/or theories for this logic which we denote by $\{\Delta\}$.  We assume that a notion $\Vdash_{\BL}$ for this logic is available such that for each $\Bm$ or $\Delta$ and for each $\Phi \in {\rm WFF}(\BL)$ the relation $\Delta\Vdash_{\BL}\Phi$ or $\Bm\Vdash_{\BL}\Phi$ can get 3 answers.  Yes $(=1)$, no $(=0)$ or undecided $(=\half)$.  

As an example of a logic let us take intuitionisitc propositional logic \BH, with consequence $\vdash_{\BH}$ we can have:

\[\begin{array}{l}
\Delta\Vdash_{\BH}\Phi \mbox{ is } 1 \mbox{ if } \Delta\vdash_{\BH}\Phi\mbox{ holds}\\
\Delta\Vdash_{\BH}\Phi \mbox{ is } 0 \mbox{ if } \Delta\vdash_{\BH}\neg\Phi\mbox{ holds}\\
\Delta\Vdash_{\BH}\Phi \mbox{ is } \half \mbox{ if neither } \Delta\vdash_{\BH}\Phi \mbox{ or } \Delta\vdash_{\BH}\neg\Phi \mbox{ holds}\\
\end{array}
\]

Similarly if \Bm\ is a propositional Kripke model for \BH, we can have
\[\begin{array}{l}
\Bm\Vdash_{\BH}\Phi \mbox{ is } 1 \mbox{ if } \Bm\Vdash\Phi\mbox{ holds}\\
\Bm\Vdash_{\BH}\Phi \mbox{ is } 0 \mbox{ if } \Bm\Vdash\neg \Phi\mbox{ holds}\\
\Bm\Vdash_{\BH}\Phi \mbox{ is } \half \mbox{ if neither holds}
\end{array}
\]

Another example is 3-valued classical propositional logic with the  \
Kleene truth table for $\{0, \half, 1\}$. Call it \BK. See \cite{509-18} and Definition \ref{509-BD3}.

Given a model assignment \Bm\ to the atoms, we have $\Bm\Vdash_{\BK}\Phi$ is the value that \Bm\ gives to $\Phi$, denoted by $\Bm(\Phi)$. It is a value in $\{0,1,\half\}$.

Let $(S, R)$ be a network and let \BL\ be a logic with $\Vdash_{\BL}$. Consider the instantiation function, $I: S\mapsto {\rm WFF}(\BL)$.

Consider $(S, R, I)$. This is an instantiated argumentation network. We seek to define the notion of complete extensions for $(S, R, I)$ and give algoirthms for finding such extensions.

After performing a conceptual analysis of this problem, we reached the following definition.  A model \Bm\ (or a theory $\Delta$) of \BL\ generates an extension for $(S, R, I)$ if the function $\lambda_{\Bm}$ (or $\lambda_\Delta$) defined on $S$ below is a legitimate Caminada labelling giving rise to a complete extension on $(S, R)$.

The function $\lambda$ is:
\begin{quote}
$\lambda_{\Bm}(x)=$ value of $(\Bm\Vdash_{\BL}I(x))$\\
$\lambda_{\Delta}(x)=$ value of $(\Delta\Vdash_{\BL}I(x))$
\end{quote}

The problem is how to identify/compute, using purely argumentation methods, such extensions  for $(S, R, I)$.  This is the task of this paper.\footnote{The reader should note that we are  not defining, as a stipulated  technical definition,  the complete  extensions of $(S,R,I)$ as those legitimate Caminada labellings arising from models or theories of the logic. We are deriving this definition from conceptual analysis of the idea of instantiation. 

To make the point absolutely  clear, suppose we instantiate the elements of $S$ by names of Chinese restaurants in London. We can define by stipulation extensions for such Chinese systems  as those legitimate Caminada extensions $\lambda$  such that  if
\begin{itemize}
\item $\lambda(x)=$ in  then the Chinese restaurant associated with $x$ made a profit in 2014
\item $ \lambda(x)=$ out,    then the Chinese restaurant associated with $x$ made a loss in 2014
\item $\lambda(x)=$ undecided,  then the Chinese restaurant associated with $x$ came out even  in 2014
\end{itemize}
The above stipulation has nothing to do with a Chinese restaurant attacking another, and is nothing more than means of restricting the Caminada labellings on $(S,R)$. }  Note that the emphasis is on using geometric syntactical argumentation methods to find the extensions of $(S,R,I)$. What we can do and we do not want to do is to systematically generate all models \Bm\ of the logic or all theories $\Delta$ of the logic and check whether $\lambda_\Bm$   or $\lambda_\Delta$  generate a legitimate Caminada extension. We want to  syntactically transform $(S,R,I)$ into an argumentation network. Put differently, we want to identify and use the argumentation network meaning of the logic.

We consider three main logics.
\begin{enumerate}
\item Classical propositional logic based on 3 valued Kleene truth table.
\item Monadic predicate logic without equality based on Kleene table.
\item Modal logic S5 based on Kleene table.
\end{enumerate}

This paper solves the problem. However, many of the results are technical and are done in the Appendices.  

The methodological schema is simple:

Given $(S, R, I)$ with $I$ being an instantiation into WFF$(\BL)$ we follow the steps below:
\paragraph{Step 1.}  Rewrite any wff $\Phi$ of \BL\ into an equivalent formula (in \BL) which is argumentation friendly and convenient form.

Finding the right friendly form is not immediate and requires some analysis and trial and error. Once we find a convenient form for any $\Phi \in {\rm WFF}(\BL)$ we need to prove the equivalence. This may involve some technical manipulations and is therefore done in an appendix.

\paragraph{Step 2.}  The instantiation $I(x)=\Phi_x$ can now be assumed to be in this special form. When $x$ attacks $y$ in $(S, R)$, we get after instantiation that $\Phi_x$ attacks $\Phi_y$. This attack between formulas of \BL, needs to be given a meaning. The two formulas may be consistent in \BL, so what does it mean that one attacks the other?  So we transform any $\Phi$ of \BL\ into an argumentation template called a Boolean attack formation ``equivalent'' to $\Phi$ which we denote by $\BBB\BBF(\Phi)$. Such formations can attack each other because they are defined as argumentation systems with input output nodes.

Note that we {\em do not} want to say something like ``$\Phi$ attacks $\Psi$ if $\{\Phi,\Psi\}$ is not consistent in $\BL$'' because we do not want to use \BL.  We turn $\Phi$ and $\Psi$ syntactically into argumentation networks and remain solidly within the argumentation framework world. 

We now have $\BBB\BBF(\Phi_x)$ attacking $\BBB\BBF(\Phi_y)$ where each $\BBB\BBF$ is an argumentation network with input and output nodes.

The attack formation associated with $\Phi$, encodes the logical meaning of $\Phi$.  This has to be defined and proved.  Because of the technical complexity of the transformation from $\Phi$ to $\BBB\BBF(\Phi)$, this is also done in an appendix.

\paragraph{Step 3.}  We now have the original $(S, R)$ network instantiated by Boolean attack formations. So this is a system of network of networks. Thus steps 1 and 2 reduced the problem of instantiating a network $(S, R)$ by wffs of a logic \BL, into the problem of instantiating $(S, R)$ by some special argumentation networks called $\BBB\BBF$ (Boolean attack formations) derived for \BL.  Such a purely argumentation problem is of interest in its own right.  We have a complex system $(S, R, I)$, where $(S, R)$ is an argumentation network and for each $x\in S, I(x) = (S_x, R_x)$ is an argumentation network.

We seek to turn this system into one big master argumentation network $(S^M, R^M)$. We now want to define the notion of complete extensions for this network $(S^M, R^M)$.

Remember that this network is not arbitrary but was constructed to do the job of finding extensions for instantiations of $(S, R)$ into WFF$(\BL)$.

We find by conceptual analysis that the traditional notions of extensions for neworks is not adequate for the job and that we need to develop and motivate our own new notion of extension (which we call {\em non-toxic truth intervention extensions}).

This new notion of extensions is, in fact, a paper in its own right and because of its complexity it is done in an appendix.

\paragraph{Step 4.}  We now have $(S^M, R^M)$ and a new method of finding extensions for it. We generate these extensions and from the extensions we generate models for the logic \BL.  The models for \BL\ give us extensions for $(S, R, I)$, with $I$ an instantiation into WFF$(\BL)$.

\medskip
The above step by step workflow schema is quite simple but is rather technical. So most of the work is done in the appendices so as not to burden the reader.

It would be of value to illustrate the steps of the workflow on an example.  

Start with $(S, R)$ being a two point loop $\{x,y\}$ with $x$ attacking $y$ and $y$ attacking $x$. See Figure \ref{509-SF1}.

\begin{figure}
\centering
\setlength{\unitlength}{0.00083333in}
\begingroup\makeatletter\ifx\SetFigFont\undefined%
\gdef\SetFigFont#1#2#3#4#5{%
  \reset@font\fontsize{#1}{#2pt}%
  \fontfamily{#3}\fontseries{#4}\fontshape{#5}%
  \selectfont}%
\fi\endgroup%
{\renewcommand{\dashlinestretch}{30}
\begin{picture}(1696,939)(0,-10)
\path(121,132)(61,227)
\blacken\path(150.444,141.561)(61.000,227.000)(99.714,109.522)(150.444,141.561)
\path(1511,802)(1571,737)
\blacken\path(1467.562,804.828)(1571.000,737.000)(1511.650,845.525)(1467.562,804.828)
\path(31,612)(32,613)(35,614)
	(40,616)(47,619)(58,624)
	(72,631)(89,639)(110,648)
	(135,659)(162,671)(193,684)
	(226,698)(261,713)(297,728)
	(336,743)(375,759)(416,774)
	(457,789)(499,804)(542,818)
	(586,831)(631,844)(677,857)
	(724,868)(772,878)(821,888)
	(872,896)(923,903)(976,908)
	(1029,911)(1081,912)(1143,910)
	(1201,905)(1254,898)(1302,888)
	(1345,877)(1384,864)(1419,849)
	(1451,833)(1480,817)(1506,799)
	(1531,781)(1553,762)(1573,743)
	(1592,724)(1609,706)(1624,688)
	(1637,672)(1649,657)(1659,644)
	(1666,634)(1672,625)(1681,612)
\blacken\path(1588.029,693.587)(1681.000,612.000)(1637.361,727.739)(1588.029,693.587)
\path(1681,387)(1680,386)(1677,385)
	(1672,383)(1664,379)(1653,374)
	(1638,367)(1620,359)(1598,348)
	(1572,336)(1542,323)(1510,308)
	(1474,293)(1436,276)(1396,259)
	(1354,241)(1311,223)(1266,205)
	(1221,187)(1175,169)(1129,152)
	(1082,135)(1034,119)(986,103)
	(938,88)(888,75)(839,62)
	(788,50)(737,39)(686,30)
	(634,23)(582,17)(531,13)
	(481,12)(428,14)(379,18)
	(334,25)(294,35)(259,46)
	(228,59)(200,74)(176,89)
	(155,106)(136,124)(120,143)
	(106,162)(93,182)(83,202)
	(73,223)(65,243)(58,263)
	(52,283)(47,301)(43,318)
	(40,334)(37,348)(35,359)
	(34,369)(32,376)(31,387)
\blacken\path(71.741,270.209)(31.000,387.000)(11.987,264.777)(71.741,270.209)
\put(31,462){\makebox(0,0)[b]{\smash{{\SetFigFont{10}{12.0}{\rmdefault}{\mddefault}{\updefault}$x$}}}}
\put(1681,462){\makebox(0,0)[b]{\smash{{\SetFigFont{10}{12.0}{\rmdefault}{\mddefault}{\updefault}$y$}}}}
\end{picture}
}

\caption{}\label{509-SF1}
\end{figure}
We have three  extensions:
\begin{quote}
$E_1: x=$ ``in'', $y=$ ``out''\\
$E_2: x=$ ``out'', $y=$ ``in''\\
$E_3: x=y=$ ``undecided''
\end{quote}

Instantiate with
\begin{quote}
$I(x)=$ someone is tall\\
$I(y)=$ Dov is tall
\end{quote}
We get Figure \ref{509-SF2}.

\begin{figure}
\centering
\setlength{\unitlength}{0.00083333in}
\begingroup\makeatletter\ifx\SetFigFont\undefined%
\gdef\SetFigFont#1#2#3#4#5{%
  \reset@font\fontsize{#1}{#2pt}%
  \fontfamily{#3}\fontseries{#4}\fontshape{#5}%
  \selectfont}%
\fi\endgroup%
{\renewcommand{\dashlinestretch}{30}
\begin{picture}(1696,939)(0,-10)
\path(121,132)(61,227)
\blacken\path(150.444,141.561)(61.000,227.000)(99.714,109.522)(150.444,141.561)
\path(1511,802)(1571,737)
\blacken\path(1467.562,804.828)(1571.000,737.000)(1511.650,845.525)(1467.562,804.828)
\path(31,612)(32,613)(35,614)
	(40,616)(47,619)(58,624)
	(72,631)(89,639)(110,648)
	(135,659)(162,671)(193,684)
	(226,698)(261,713)(297,728)
	(336,743)(375,759)(416,774)
	(457,789)(499,804)(542,818)
	(586,831)(631,844)(677,857)
	(724,868)(772,878)(821,888)
	(872,896)(923,903)(976,908)
	(1029,911)(1081,912)(1143,910)
	(1201,905)(1254,898)(1302,888)
	(1345,877)(1384,864)(1419,849)
	(1451,833)(1480,817)(1506,799)
	(1531,781)(1553,762)(1573,743)
	(1592,724)(1609,706)(1624,688)
	(1637,672)(1649,657)(1659,644)
	(1666,634)(1672,625)(1681,612)
\blacken\path(1588.029,693.587)(1681.000,612.000)(1637.361,727.739)(1588.029,693.587)
\path(1681,387)(1680,386)(1677,385)
	(1672,383)(1664,379)(1653,374)
	(1638,367)(1620,359)(1598,348)
	(1572,336)(1542,323)(1510,308)
	(1474,293)(1436,276)(1396,259)
	(1354,241)(1311,223)(1266,205)
	(1221,187)(1175,169)(1129,152)
	(1082,135)(1034,119)(986,103)
	(938,88)(888,75)(839,62)
	(788,50)(737,39)(686,30)
	(634,23)(582,17)(531,13)
	(481,12)(428,14)(379,18)
	(334,25)(294,35)(259,46)
	(228,59)(200,74)(176,89)
	(155,106)(136,124)(120,143)
	(106,162)(93,182)(83,202)
	(73,223)(65,243)(58,263)
	(52,283)(47,301)(43,318)
	(40,334)(37,348)(35,359)
	(34,369)(32,376)(31,387)
\blacken\path(71.741,270.209)(31.000,387.000)(11.987,264.777)(71.741,270.209)
\put(31,462){\makebox(0,0)[rb]{\smash{{\SetFigFont{10}{12.0}{\rmdefault}{\mddefault}{\updefault}$\exists x T(x)$}}}}
\put(1681,462){\makebox(0,0)[lb]{\smash{{\SetFigFont{10}{12.0}{\rmdefault}{\mddefault}{\updefault}$T(d)$}}}}
\end{picture}
}
\caption{}\label{509-SF2}
\end{figure}

We are looking for (Kleene three valued) predicate models which would give values to $\exists x T(x)$ and $T(d)$ which will form an extension.

Our knowledge of logic tells us that there  is only a model for $\exists x T(x)\wedge\neg T(d)$, e.g.\ Dov and Lydia, with Dov not tall and Lydia tall.  There is no model with $T(d)\wedge\neg\exists xT(x)$.  Thus our knowledge of logic tells us that the above instantiation has only one extension in classical two valued monadic logic. In Kleene 3 valued logic, we can also have a model with $\exists x T(x) =T(d)=\half$.

However, we are not supposed to use our knowledge of logic but we must use an algorithm to find the model. So as far as the algorithm is concerned we have in this example a case of the  instantiation of Figure \ref{509-SF3}

\begin{figure}
\centering
\setlength{\unitlength}{0.00083333in}
\begingroup\makeatletter\ifx\SetFigFont\undefined%
\gdef\SetFigFont#1#2#3#4#5{%
  \reset@font\fontsize{#1}{#2pt}%
  \fontfamily{#3}\fontseries{#4}\fontshape{#5}%
  \selectfont}%
\fi\endgroup%
{\renewcommand{\dashlinestretch}{30}
\begin{picture}(1696,939)(0,-10)
\path(121,132)(61,227)
\blacken\path(150.444,141.561)(61.000,227.000)(99.714,109.522)(150.444,141.561)
\path(1511,802)(1571,737)
\blacken\path(1467.562,804.828)(1571.000,737.000)(1511.650,845.525)(1467.562,804.828)
\path(31,612)(32,613)(35,614)
	(40,616)(47,619)(58,624)
	(72,631)(89,639)(110,648)
	(135,659)(162,671)(193,684)
	(226,698)(261,713)(297,728)
	(336,743)(375,759)(416,774)
	(457,789)(499,804)(542,818)
	(586,831)(631,844)(677,857)
	(724,868)(772,878)(821,888)
	(872,896)(923,903)(976,908)
	(1029,911)(1081,912)(1143,910)
	(1201,905)(1254,898)(1302,888)
	(1345,877)(1384,864)(1419,849)
	(1451,833)(1480,817)(1506,799)
	(1531,781)(1553,762)(1573,743)
	(1592,724)(1609,706)(1624,688)
	(1637,672)(1649,657)(1659,644)
	(1666,634)(1672,625)(1681,612)
\blacken\path(1588.029,693.587)(1681.000,612.000)(1637.361,727.739)(1588.029,693.587)
\path(1681,387)(1680,386)(1677,385)
	(1672,383)(1664,379)(1653,374)
	(1638,367)(1620,359)(1598,348)
	(1572,336)(1542,323)(1510,308)
	(1474,293)(1436,276)(1396,259)
	(1354,241)(1311,223)(1266,205)
	(1221,187)(1175,169)(1129,152)
	(1082,135)(1034,119)(986,103)
	(938,88)(888,75)(839,62)
	(788,50)(737,39)(686,30)
	(634,23)(582,17)(531,13)
	(481,12)(428,14)(379,18)
	(334,25)(294,35)(259,46)
	(228,59)(200,74)(176,89)
	(155,106)(136,124)(120,143)
	(106,162)(93,182)(83,202)
	(73,223)(65,243)(58,263)
	(52,283)(47,301)(43,318)
	(40,334)(37,348)(35,359)
	(34,369)(32,376)(31,387)
\blacken\path(71.741,270.209)(31.000,387.000)(11.987,264.777)(71.741,270.209)
\put(31,462){\makebox(0,0)[b]{\smash{{\SetFigFont{10}{12.0}{\rmdefault}{\mddefault}{\updefault}$\Phi$}}}}
\put(1681,462){\makebox(0,0)[b]{\smash{{\SetFigFont{10}{12.0}{\rmdefault}{\mddefault}{\updefault}$\Psi$}}}}
\end{picture}
}
\caption{}\label{509-SF3}
\end{figure}

\begin{figure}
\centering
\setlength{\unitlength}{0.00083333in}
\begingroup\makeatletter\ifx\SetFigFont\undefined%
\gdef\SetFigFont#1#2#3#4#5{%
  \reset@font\fontsize{#1}{#2pt}%
  \fontfamily{#3}\fontseries{#4}\fontshape{#5}%
  \selectfont}%
\fi\endgroup%
{\renewcommand{\dashlinestretch}{30}
\begin{picture}(3044,1851)(0,-10)
\put(1965,1587){\makebox(0,0)[rb]{\smash{{\SetFigFont{10}{12.0}{\rmdefault}{\mddefault}{\updefault}$\Psi$:}}}}
\path(2432,1824)(1832,924)(2432,24)
	(3032,924)(2432,1824)
\put(615,1437){\makebox(0,0)[b]{\smash{{\SetFigFont{10}{12.0}{\rmdefault}{\mddefault}{\updefault}{\bf in}}}}}
\put(615,237){\makebox(0,0)[b]{\smash{{\SetFigFont{10}{12.0}{\rmdefault}{\mddefault}{\updefault}{\bf out}}}}}
\put(2415,1437){\makebox(0,0)[b]{\smash{{\SetFigFont{10}{12.0}{\rmdefault}{\mddefault}{\updefault}{\bf in}}}}}
\put(2415,237){\makebox(0,0)[b]{\smash{{\SetFigFont{10}{12.0}{\rmdefault}{\mddefault}{\updefault}{\bf out}}}}}
\put(615,837){\makebox(0,0)[b]{\smash{{\SetFigFont{10}{12.0}{\rmdefault}{\mddefault}{\updefault}$\Phi$}}}}
\put(2490,837){\makebox(0,0)[b]{\smash{{\SetFigFont{10}{12.0}{\rmdefault}{\mddefault}{\updefault}$\Psi$}}}}
\put(15,1587){\makebox(0,0)[rb]{\smash{{\SetFigFont{10}{12.0}{\rmdefault}{\mddefault}{\updefault}$\Phi$:}}}}
\path(615,1812)(15,912)(615,12)
	(1215,912)(615,1812)
\end{picture}
}
\caption{}\label{509-SF4}
\end{figure}

We turn $\Phi$ and $\Psi$ into equivalent networks of Boolean attack formation which we draw as in Figure \ref{509-SF4}
 and Figure \ref{509-SF3} becomes Figure \ref{509-SF5} which is a master network representing the instantiation. We find extensiosn for the master network and these extensions will generate models.

\begin{figure}
\centering
\setlength{\unitlength}{0.00083333in}
\begingroup\makeatletter\ifx\SetFigFont\undefined%
\gdef\SetFigFont#1#2#3#4#5{%
  \reset@font\fontsize{#1}{#2pt}%
  \fontfamily{#3}\fontseries{#4}\fontshape{#5}%
  \selectfont}%
\fi\endgroup%
{\renewcommand{\dashlinestretch}{30}
\begin{picture}(3326,2417)(0,-10)
\put(2772,1362){\makebox(0,0)[b]{\smash{{\SetFigFont{10}{12.0}{\rmdefault}{\mddefault}{\updefault}$\Psi$}}}}
\path(2714,2349)(2114,1449)(2714,549)
	(3314,1449)(2714,2349)
\path(632,2312)(732,2327)
\blacken\path(617.778,2279.531)(732.000,2327.000)(608.877,2338.867)(617.778,2279.531)
\path(2447,2357)(2542,2357)
\blacken\path(2422.000,2327.000)(2542.000,2357.000)(2422.000,2387.000)(2422.000,2327.000)
\path(897,537)(898,536)(901,534)
	(905,530)(912,524)(922,516)
	(935,506)(952,493)(971,479)
	(993,463)(1017,445)(1044,427)
	(1073,408)(1103,389)(1134,371)
	(1166,353)(1198,338)(1231,324)
	(1265,312)(1299,303)(1333,297)
	(1367,294)(1401,296)(1436,302)
	(1472,312)(1507,329)(1543,352)
	(1579,381)(1614,418)(1647,462)
	(1672,501)(1695,544)(1715,588)
	(1734,634)(1749,679)(1762,724)
	(1772,767)(1780,810)(1785,850)
	(1789,889)(1790,926)(1789,961)
	(1787,995)(1783,1027)(1779,1058)
	(1773,1088)(1767,1117)(1760,1146)
	(1753,1175)(1746,1203)(1740,1232)
	(1734,1262)(1728,1292)(1724,1324)
	(1720,1358)(1718,1393)(1718,1430)
	(1720,1470)(1723,1511)(1729,1555)
	(1738,1602)(1749,1650)(1762,1701)
	(1779,1753)(1799,1806)(1821,1859)
	(1845,1911)(1872,1962)(1904,2015)
	(1937,2064)(1971,2108)(2005,2147)
	(2039,2181)(2074,2211)(2108,2238)
	(2141,2260)(2175,2279)(2208,2296)
	(2241,2309)(2274,2320)(2307,2329)
	(2339,2336)(2372,2342)(2403,2346)
	(2435,2349)(2465,2350)(2495,2351)
	(2523,2351)(2550,2350)(2575,2349)
	(2598,2348)(2619,2346)(2638,2344)
	(2653,2343)(2667,2341)(2677,2340)
	(2685,2339)(2697,2337)
\blacken\path(2573.701,2327.136)(2697.000,2337.000)(2583.565,2386.320)(2573.701,2327.136)
\path(2697,537)(2696,536)(2694,535)
	(2690,532)(2685,527)(2676,521)
	(2665,512)(2651,502)(2634,489)
	(2614,474)(2592,457)(2566,438)
	(2537,417)(2507,395)(2474,373)
	(2439,349)(2402,325)(2364,300)
	(2325,276)(2284,251)(2242,227)
	(2199,204)(2155,181)(2109,159)
	(2062,138)(2014,118)(1963,100)
	(1911,82)(1857,66)(1801,52)
	(1743,39)(1682,28)(1619,20)
	(1555,14)(1489,12)(1422,12)
	(1356,16)(1291,22)(1230,30)
	(1172,40)(1119,51)(1070,61)
	(1026,72)(986,83)(951,93)
	(920,102)(893,111)(869,119)
	(848,127)(830,134)(813,141)
	(799,148)(785,155)(772,162)
	(759,169)(746,177)(733,186)
	(718,197)(703,208)(685,222)
	(665,237)(643,255)(618,275)
	(591,299)(561,325)(528,355)
	(493,389)(455,427)(416,468)
	(376,513)(336,561)(297,612)
	(261,665)(228,718)(198,771)
	(171,823)(147,872)(126,918)
	(108,961)(92,1001)(79,1038)
	(67,1072)(58,1102)(50,1130)
	(43,1156)(37,1180)(32,1202)
	(28,1223)(25,1242)(22,1262)
	(20,1281)(18,1301)(16,1322)
	(15,1344)(14,1367)(13,1392)
	(12,1420)(12,1450)(12,1482)
	(13,1518)(15,1557)(18,1598)
	(22,1643)(28,1689)(35,1738)
	(45,1788)(57,1838)(72,1887)
	(93,1939)(117,1987)(143,2031)
	(172,2071)(202,2106)(234,2138)
	(267,2165)(301,2190)(336,2211)
	(371,2230)(407,2247)(444,2261)
	(481,2273)(518,2284)(555,2294)
	(592,2302)(629,2309)(665,2315)
	(699,2320)(732,2324)(762,2327)
	(790,2330)(815,2332)(836,2334)
	(854,2335)(869,2336)(880,2336)(897,2337)
\blacken\path(778.969,2300.005)(897.000,2337.000)(775.445,2359.902)(778.969,2300.005)
\put(897,1962){\makebox(0,0)[b]{\smash{{\SetFigFont{10}{12.0}{\rmdefault}{\mddefault}{\updefault}{\bf in}}}}}
\put(897,762){\makebox(0,0)[b]{\smash{{\SetFigFont{10}{12.0}{\rmdefault}{\mddefault}{\updefault}{\bf out}}}}}
\put(2697,1962){\makebox(0,0)[b]{\smash{{\SetFigFont{10}{12.0}{\rmdefault}{\mddefault}{\updefault}{\bf in}}}}}
\put(2697,762){\makebox(0,0)[b]{\smash{{\SetFigFont{10}{12.0}{\rmdefault}{\mddefault}{\updefault}{\bf out}}}}}
\put(897,1362){\makebox(0,0)[b]{\smash{{\SetFigFont{10}{12.0}{\rmdefault}{\mddefault}{\updefault}$\Phi$}}}}
\path(897,2337)(297,1437)(897,537)
	(1497,1437)(897,2337)
\end{picture}
}
\caption{}\label{509-SF5}
\end{figure}

The perceptive reader might ask why not look at $\Phi \wedge\neg \Psi$ and $\Psi\wedge\neg\Phi$ and find models for them?  After all, we know the logic \BL?  The answer is that we are not using the logic at all. The transformation  $\Phi$ to $\BBB\BBF(\Phi)$ is pure syntax. It is the argumentation extensions which find the models for $\{I(x)|x\in S\}$ which respect the constraints imposed by the original $(S, R)$.

To continue analysing this example from the argumentation point of view, assume the universe has only two people, Dov and Lydia.  So $\exists x T(x)$ becomes $T(d)\vee T(l)$.

Figure \ref{509-SF2} becomes Figure \ref{509-SF2-1}.

\begin{figure}
\centering
\setlength{\unitlength}{0.00083333in}
\begingroup\makeatletter\ifx\SetFigFont\undefined%
\gdef\SetFigFont#1#2#3#4#5{%
  \reset@font\fontsize{#1}{#2pt}%
  \fontfamily{#3}\fontseries{#4}\fontshape{#5}%
  \selectfont}%
\fi\endgroup%
{\renewcommand{\dashlinestretch}{30}
\begin{picture}(1696,939)(0,-10)
\path(121,132)(61,227)
\blacken\path(150.444,141.561)(61.000,227.000)(99.714,109.522)(150.444,141.561)
\path(1511,802)(1571,737)
\blacken\path(1467.562,804.828)(1571.000,737.000)(1511.650,845.525)(1467.562,804.828)
\path(31,612)(32,613)(35,614)
	(40,616)(47,619)(58,624)
	(72,631)(89,639)(110,648)
	(135,659)(162,671)(193,684)
	(226,698)(261,713)(297,728)
	(336,743)(375,759)(416,774)
	(457,789)(499,804)(542,818)
	(586,831)(631,844)(677,857)
	(724,868)(772,878)(821,888)
	(872,896)(923,903)(976,908)
	(1029,911)(1081,912)(1143,910)
	(1201,905)(1254,898)(1302,888)
	(1345,877)(1384,864)(1419,849)
	(1451,833)(1480,817)(1506,799)
	(1531,781)(1553,762)(1573,743)
	(1592,724)(1609,706)(1624,688)
	(1637,672)(1649,657)(1659,644)
	(1666,634)(1672,625)(1681,612)
\blacken\path(1588.029,693.587)(1681.000,612.000)(1637.361,727.739)(1588.029,693.587)
\path(1681,387)(1680,386)(1677,385)
	(1672,383)(1664,379)(1653,374)
	(1638,367)(1620,359)(1598,348)
	(1572,336)(1542,323)(1510,308)
	(1474,293)(1436,276)(1396,259)
	(1354,241)(1311,223)(1266,205)
	(1221,187)(1175,169)(1129,152)
	(1082,135)(1034,119)(986,103)
	(938,88)(888,75)(839,62)
	(788,50)(737,39)(686,30)
	(634,23)(582,17)(531,13)
	(481,12)(428,14)(379,18)
	(334,25)(294,35)(259,46)
	(228,59)(200,74)(176,89)
	(155,106)(136,124)(120,143)
	(106,162)(93,182)(83,202)
	(73,223)(65,243)(58,263)
	(52,283)(47,301)(43,318)
	(40,334)(37,348)(35,359)
	(34,369)(32,376)(31,387)
\blacken\path(71.741,270.209)(31.000,387.000)(11.987,264.777)(71.741,270.209)
\put(31,462){\makebox(0,0)[rb]{\smash{{\SetFigFont{10}{12.0}{\rmdefault}{\mddefault}{\updefault}$T(d)\vee T(l)$}}}}
\put(1681,462){\makebox(0,0)[lb]{\smash{{\SetFigFont{10}{12.0}{\rmdefault}{\mddefault}{\updefault}$T(d)$}}}}
\end{picture}
}
\caption{}\label{509-SF2-1}
\end{figure}

The reader can see now the logic behind this instantiation. We have to define correctly what it means to attack a disjunction and what it means for a disjunction to mount an attack.  Once we do that (this is done later in tis paper) we get that Figure \ref{509-SF2-1} is equivalent to Figure \ref{509-SF2-2}

\begin{figure}
\centering
\setlength{\unitlength}{0.00083333in}
\begingroup\makeatletter\ifx\SetFigFont\undefined%
\gdef\SetFigFont#1#2#3#4#5{%
  \reset@font\fontsize{#1}{#2pt}%
  \fontfamily{#3}\fontseries{#4}\fontshape{#5}%
  \selectfont}%
\fi\endgroup%
{\renewcommand{\dashlinestretch}{30}
\begin{picture}(858,2043)(0,-10)
\put(412,60){\makebox(0,0)[b]{\smash{{\SetFigFont{10}{12.0}{\rmdefault}{\mddefault}{\updefault}$T(l)$}}}}
\path(557,410)(502,330)
\blacken\path(545.262,445.881)(502.000,330.000)(594.705,411.889)(545.262,445.881)
\path(242,1225)(307,1300)
\blacken\path(251.079,1189.669)(307.000,1300.000)(205.738,1228.965)(251.079,1189.669)
\path(412,210)(410,212)(406,217)
	(398,225)(386,238)(370,256)
	(349,279)(326,306)(299,336)
	(271,368)(242,402)(213,437)
	(185,471)(158,505)(134,538)
	(111,569)(91,598)(74,626)
	(58,652)(45,677)(35,701)
	(26,724)(20,746)(15,768)
	(13,789)(12,810)(13,831)
	(15,852)(20,874)(26,896)
	(35,919)(45,943)(58,968)
	(74,994)(91,1022)(111,1051)
	(134,1082)(158,1115)(185,1149)
	(213,1183)(242,1218)(271,1252)
	(299,1284)(326,1314)(349,1341)
	(370,1364)(386,1382)(412,1410)
\blacken\path(352.330,1301.651)(412.000,1410.000)(308.362,1342.478)(352.330,1301.651)
\path(412,1410)(414,1408)(417,1403)
	(423,1395)(433,1381)(446,1363)
	(462,1341)(481,1314)(502,1283)
	(525,1251)(548,1216)(571,1181)
	(594,1147)(615,1113)(635,1080)
	(653,1049)(669,1019)(683,991)
	(695,965)(705,940)(714,916)
	(721,893)(726,871)(729,849)
	(731,828)(732,807)(731,786)
	(729,764)(726,743)(721,721)
	(714,698)(705,674)(695,649)
	(683,623)(669,595)(653,566)
	(635,535)(615,503)(594,470)
	(571,435)(548,401)(525,367)
	(502,335)(481,305)(462,278)
	(446,256)(433,238)(412,210)
\blacken\path(460.000,324.000)(412.000,210.000)(508.000,288.000)(460.000,324.000)
\path(412,1635)(414,1638)(417,1644)
	(423,1656)(431,1671)(442,1691)
	(455,1714)(468,1738)(482,1761)
	(495,1784)(508,1806)(520,1825)
	(532,1843)(543,1858)(554,1873)
	(565,1886)(576,1898)(587,1910)
	(600,1923)(614,1935)(629,1947)
	(644,1958)(660,1969)(676,1979)
	(692,1988)(709,1996)(725,2003)
	(740,2008)(755,2013)(769,2015)
	(781,2016)(793,2016)(803,2014)
	(812,2010)(819,2006)(826,2000)
	(831,1993)(836,1985)(840,1975)
	(843,1965)(845,1953)(846,1941)
	(846,1928)(845,1915)(843,1901)
	(840,1887)(835,1873)(830,1860)
	(824,1846)(817,1834)(809,1822)
	(800,1810)(789,1799)(778,1788)
	(765,1777)(750,1766)(733,1755)
	(713,1743)(691,1731)(667,1718)
	(640,1705)(611,1691)(583,1678)
	(555,1665)(531,1654)(487,1635)
\blacken\path(585.274,1710.114)(487.000,1635.000)(609.061,1655.030)(585.274,1710.114)
\put(412,1485){\makebox(0,0)[b]{\smash{{\SetFigFont{10}{12.0}{\rmdefault}{\mddefault}{\updefault}$T(d)$}}}}
\path(717,1750)(632,1700)
\blacken\path(720.222,1786.700)(632.000,1700.000)(750.643,1734.984)(720.222,1786.700)
\end{picture}
}
\caption{}\label{509-SF2-2}
\end{figure}

The only extensions are
\begin{enumerate}
\item $T(l)=$ in, $T(d)=$ out (corresponding to $\exists x T(x)\wedge\neg T(d)$) and
\item $T(l)=T(d)=$ undecided (corresponding to $\exists x T(x) =T(d)=\half$).
\end{enumerate}

Note that we can use the above as an argumentation theorem prover to check the consistency of $\Phi\wedge\neg\Psi$ and $\neg\Phi\wedge\Psi$.\footnote{Existing machines for finding extensions for argumentation networks push the problem to logical provers. So we are not suggesting our reduction as a practical theorem prover for logic but only to highlight that we are operating purely in the argumentation world.}

Let us conclude this preliminary orientation by saying a few words about our deductive and expositional approach.  Our approach is new.  We arrive at our proposed system through a conceptual analysis  (using common sense) of the components needed for an abstract theory of instantiation in general and for the specific instance of, for example,  predicate and modal argumentation and following this analysis we define our system.  So, as we are writing the present lines, we do not know yet the full details of what kind of system we will get.

We choose to deal, for the sake of simplicity,  with the classical Boolean propositional calculus and with  monadic classical predicate logic without equality and with modal logic S5.  See Appendix A.  We shall deal with more complex logics in a subsequent paper.

There exists in the literature the instantiated approach to argumentation, also known as ASPIC, see \cite{509-13,509-14,509-16}. This approach is related but not the same as our approach. See Appendix E for full comparison and discussion.

\subsection{Structure of our paper}
Our program for this paper has the following methodological structure:
\paragraph{Starting point.}  We assume as our given starting point Dung theory of atomic finite propositional abstract argumentation frames.  Namely frames of the form $(S, R)$, where $S$ a finite set of atomic arguments, $R\subseteq S\times S$ is the attack relation together with the traditional notion of complete extensions $E\subseteq S$ (to be recalled and defined in the next section).

\paragraph{Objective.}  Extend what is given to classical propositional calculus  and to monadic predicate logic and to modal logic, namely allow the elements of $S$ to be instantiated as formulas of classical propositional calculus or respectively as formulas of (monadic) predicate logic or respectively as formulas of modal logic S5 and define the concept of a complete extensions for this case in a natural and completely syntactical and combinatorial way, without using any logical notions.\footnote{By completely syntactical and combinatorial way we mean we use only the geometry of the graph $(S, R)$ and possibly a new concepts of attack based on the geometry of the graph and on the  simple syntactical structure of the arguments. So we assume a system of the form $(S,R,I)$, where $S$ is a set of atomic arguments, $R$ is the attack relation, and $I$ is an instantiation function giving for each $x \in S$ a formula $I(x)$ of propositional logic or of predicate logic or resp. modal logic. We look for extensions  respecting  the instantiation $I$, but such extensions are to be defined purely syntactically.  So defining $ARB$ iff $A=\neg B$ or $B =\neg A$ is acceptable but defining $ARB$ iff $A, B\vdash\bot$ is not!}

\paragraph{Methodological approach.} To achieve our objective we use the method of common sense conceputal analysis, a well known method in philosophy circles.

\paragraph{Conclusion.}  We produce an abstract theory of instantiation of argumentation frameworks in general and for the specific cases of Boolean, predicate, and modal  instantiations, and  compare and discuss what we get with systems proposed in the literature.

\medskip\noindent
Let us now begin this rather unusual approach.

\subsection{General methodological remarks}
\begin{enumerate}
\item First, we remark that there are two major approaches to extending any propositional system to a predicate system, in this case extending propositional abstract argumentation to predicate or to modal S5 argumentation:
\begin{enumerate}
\renewcommand{\labelenumii}{*\arabic{enumii}.}
\item Look at applications areas of the (propositional) system and see its shortcomings and seek to extend it accordingly (to a predicate system or to a modal system).  The needs of the applications will dictate what kind of generalisation to adopt.
\item Look formally at the (propositional) system and its components and use theoretical considerations to extend it by adding quantifiers or modalities. 
\end{enumerate}
Personally we believe in the (*1) approach.  In our particular case, however, the (*2) approach is just as good, because adding predicates and quantifiers or modalities  to a propositional system is universally done for many logics. This is a well trodden path and it can lead to good extensions which will do well with applications, as long as our conceputal analysis is done with good common sense and care!
\item Second, let us recall a general methodological remedy for fixing any system which does not behave properly. 

Suppose we have an input/output system of some kind, as represented in Figure \ref{509-F2}.

\begin{figure}
\centering
\setlength{\unitlength}{0.00083333in}
\begingroup\makeatletter\ifx\SetFigFont\undefined%
\gdef\SetFigFont#1#2#3#4#5{%
  \reset@font\fontsize{#1}{#2pt}%
  \fontfamily{#3}\fontseries{#4}\fontshape{#5}%
  \selectfont}%
\fi\endgroup%
{\renewcommand{\dashlinestretch}{30}
\begin{picture}(3906,810)(0,-10)
\put(3370,333){\makebox(0,0)[b]{\smash{{\SetFigFont{10}{12.0}{\rmdefault}{\mddefault}{\updefault}Output}}}}
\put(495,383){\ellipse{974}{750}}
\path(1570,783)(2470,783)(2470,33)
	(1570,33)(1570,783)
\put(1270,333){\makebox(0,0)[b]{\smash{{\SetFigFont{10}{12.0}{\rmdefault}{\mddefault}{\updefault}$\To$}}}}
\put(2695,333){\makebox(0,0)[b]{\smash{{\SetFigFont{10}{12.0}{\rmdefault}{\mddefault}{\updefault}$\To$}}}}
\put(2020,333){\makebox(0,0)[b]{\smash{{\SetFigFont{10}{12.0}{\rmdefault}{\mddefault}{\updefault}System}}}}
\put(445,333){\makebox(0,0)[b]{\smash{{\SetFigFont{10}{12.0}{\rmdefault}{\mddefault}{\updefault}Input}}}}
\put(3411,398){\ellipse{974}{750}}
\end{picture}
}
\caption{}\label{509-F2}
\end{figure}

Suppose the output is problematic and not to our liking. How do we remedy the situation?  There are three pure traditional approaches and many options using their various combinations

\begin{enumerate}
\renewcommand{\labelenumii}{r\arabic{enumii}.}
\item Restrict the input to make sure the output is acceptable.
\item Fix the system.
\item Modify the output to make it acceptable.
\end{enumerate}
To ilustrate, suppose we write a program for adding two numbers $x$ and $y$, to get $x \oplus y$.
\[
\{x, y\}\tO x\oplus y.
\]
Suppose we get the correct answer for 
\[
0 \leq x, y < 100,
\]
but for $x \geq 100$ or $y \geq 100$ we get the answer 
\[
x \oplus y = x+y+1
\]
where ``$+$'' is the correct addition.

(r1) says do not use numbers $\geq 100$.

(r2) says fix the program for $\oplus$

(r3) says subtract 1 from the result for the case that one of the input numbers is $\geq 100$ and you will get the correct answer.
\end{enumerate}

\subsection{Conceputal analysis for instantiated Boolean or  predicate argumentation}
Let us first analyse some characteristics of the propositinal case and list their conceptual significance.  

An abstract argumentation network \cite{509-15} has the form $(S, R)$, where $S\neq \varnothing$ is the set of abstract arguments and $R\subseteq S\times S$ is the attack relation. (We also write $x \tO y$ to denote $(x, y)\in R$, reading $x$ attacks $y$, especially in figures.)

The formal machinery associates with each $(S, R)$ several types of extensions. It is convenient for us in this paper to use the Caminada labelling approach to extensions.  See Caminada--Gabbay survey paper \cite{509-11}.  The Caminada labelling has the form $\lambda: S\mapsto \{1, 0, \half\}$ where $\lambda(x) =1$ means $x$ is ``in'', $\lambda (x) =0$ means $x$ is ``out'' and $\lambda(x)=\half$ means $x$ is ``undecided''.

The exact definitions and background will be given in the next section in Definition \ref{509-2DM1}. Here, in the orientation section, let us just note that because we regard the elements of $S$ as atomic  symbols, we can also view them as interpreted into any logic with $\neg$,  as atomic propositions of that logic, and the set 
\[
T_\lambda =\{q|\lambda (q) =1\} \cup \{\neg q |\lambda(q)=0\}
\]
is always consistent in any such logic, giving us no logical problems whatsoever.

This fact makes any coherent process giving rise to a $\lambda$ and $T_\lambda$ an acceptable process.

Let us assume we have a process for finding such acceptable functions $\lambda$.  Our theoretical considerations for extending the propositional case to a predicate or modal extension  can follow several tracks.

\paragraph{Track (t1): Substitution.}
Substitute for elements of $S$ predicate wffs and apply the process and see what happens and when in difficulty offer suitable remedies.
\paragraph{Track (t2): Translation.}
Look at translations of the propositional theory into other predicate systems (translation into logic programming or into classical logic or into modal logic) and see the behaviour of the image of the  source in the context of the target and find out how to import predicate logic argumentation or modal argumentation from the translation.

\medskip\noindent
At this point of our deliberations, Track (t2) seems more difficult than Track (t1). Let us therefore  begin with (t1).

\subsection{The substitution track}
Problems arise when we instantiate  the elements of $S$ and give them internal structure, such as a wff of predicate logic. The elements of $T_\lambda$, when having an internal structure, may clash with one another and render $T_\lambda$ inconsistent, in whatever logic we happen to be using.\footnote{Let 
\[\begin{array}{l}
E_\lambda^1 = \{q| \lambda(q) = 1\}\\[1.5ex]
E_\lambda^0 = \{q| \lambda(q) = 0\}.
\end{array}\]
These two sets are disjoint  and therefore  when we  interpret the union of these two sets  in a logic via the special instantiation 
\[
I(q) = q , \mbox{ for } q \in E_\lambda^1
\]
and
\[
I(q) =\neg q , \mbox{ for } q \in E_\lambda^0
\]
   to form $T_\lambda$,  we get a consistent set in the logic.

However for any other possible interpretation/instantiation $I^*$, we may get inconsistency in the logic.} This means that the simplistic approach of allowing propositional or predicate wffs to be substituted for the atoms in $S$ will most likely be problematic and may require a remedy. As part of our conceptual analysis we will proceed along this path and seek to identify possible remedies.

Even the simplest possible instantiation can be problematic. Suppose we take an argumentation network  $(S , R)$ and choose a single $y$ in $S$ and instantiate just this single $y$ as the propositional constant $\top$. What happens?  $\top$ is always true, so this is equivalent to saying that $y$ should always be ``in". This means that any $x$ attacking $y$ should be ``out". This may sounds simple but it is not,  because it changes the rules of the game:  Firstly  there may not be extensions where $y$ is ``in". So we have to say that we are dealing with networks which could be without semantics. Secondly the direction of attacks is no longer only following the arrows, but if there is an arrow leading onto  $\top$, the attack is directed opposite the direction of the arrows.  Dealing with such instantiation looks simple but it is not and so we postpone any further  discussion to Appendix B.  

In this section we want to illustrate our conceptual analysis on a less subtle case,   so we deal with predicate instantiation.

Figures \ref{509-F1}a and \ref{509-F1}b present us with a simple network and a predicate instantiation for it, to serve as a simple example for our conceputal analysis:

\begin{figure}
\centering 
\setlength{\unitlength}{0.00083333in}
\begingroup\makeatletter\ifx\SetFigFont\undefined%
\gdef\SetFigFont#1#2#3#4#5{%
  \reset@font\fontsize{#1}{#2pt}%
  \fontfamily{#3}\fontseries{#4}\fontshape{#5}%
  \selectfont}%
\fi\endgroup%
{\renewcommand{\dashlinestretch}{30}
\begin{picture}(4679,742)(0,-10)
\put(2102,68){\makebox(0,0)[lb]{\smash{{\SetFigFont{10}{12.0}{\rmdefault}{\mddefault}{\updefault}(b)}}}}
\path(535,495)(610,495)
\blacken\path(490.000,465.000)(610.000,495.000)(490.000,525.000)(490.000,465.000)
\path(1060,495)(1510,495)
\blacken\path(1390.000,465.000)(1510.000,495.000)(1390.000,525.000)(1390.000,465.000)
\path(1285,495)(1360,495)
\blacken\path(1240.000,465.000)(1360.000,495.000)(1240.000,525.000)(1240.000,465.000)
\path(15,715)(1790,715)(1790,275)
	(15,275)(15,715)
\path(2515,480)(2965,480)
\blacken\path(2845.000,450.000)(2965.000,480.000)(2845.000,510.000)(2845.000,450.000)
\path(3685,485)(4135,485)
\blacken\path(4015.000,455.000)(4135.000,485.000)(4015.000,515.000)(4015.000,455.000)
\path(2725,488)(2822,480)
\blacken\path(2699.940,459.965)(2822.000,480.000)(2704.872,519.762)(2699.940,459.965)
\path(3895,495)(3985,495)
\blacken\path(3865.000,465.000)(3985.000,495.000)(3865.000,525.000)(3865.000,465.000)
\path(2020,713)(4667,713)(4667,278)
	(2020,278)(2020,713)
\put(1640,425){\makebox(0,0)[b]{\smash{{\SetFigFont{10}{12.0}{\rmdefault}{\mddefault}{\updefault}$z$}}}}
\put(175,440){\makebox(0,0)[b]{\smash{{\SetFigFont{10}{12.0}{\rmdefault}{\mddefault}{\updefault}$x$}}}}
\put(925,435){\makebox(0,0)[b]{\smash{{\SetFigFont{10}{12.0}{\rmdefault}{\mddefault}{\updefault}$y$}}}}
\put(80,55){\makebox(0,0)[lb]{\smash{{\SetFigFont{10}{12.0}{\rmdefault}{\mddefault}{\updefault}(a)}}}}
\put(2115,425){\makebox(0,0)[lb]{\smash{{\SetFigFont{10}{12.0}{\rmdefault}{\mddefault}{\updefault}$A(J)$}}}}
\put(3015,425){\makebox(0,0)[lb]{\smash{{\SetFigFont{10}{12.0}{\rmdefault}{\mddefault}{\updefault}$\exists xA(x)$}}}}
\put(4205,430){\makebox(0,0)[lb]{\smash{{\SetFigFont{10}{12.0}{\rmdefault}{\mddefault}{\updefault}$A(M)$}}}}
\path(310,495)(760,495)
\blacken\path(640.000,465.000)(760.000,495.000)(640.000,525.000)(640.000,465.000)
\end{picture}
}

\caption{}\label{509-F1}
\end{figure}

The extension $\lambda$ for Figure \ref{509-F1}a is $\lambda_a(x) =\lambda_a(z)=1$ and $\lambda_a(y)=0$.

Thus $T_{\lambda_a} =\{x,\neg y, z\}$. 

Upon instantiation in Figure \ref{509-F1}b we get
\[
T_{\lambda_b} =\{A(J),\neg\exists x A(x), A(M)\}.
\]
This theory is inconsistent. Here $A$ is a unary predicate and $J$ and $M$ are elements in the predicate universe.

We now recall the general methodological remedies (r1)--(r3) disucssed in subsection 1.2.

We have here an input output system

\paragraph{Inputs:} Instantiated $(S, R)$ with predicate formulas
\paragraph{Outputs:} Predicate extensions
\paragraph{Problem:} The output may be inconsistent.

\medskip
\noindent The systems needs a remedy.

The (r1) option would restrict the input, maybe to some fragment of predicate logic. This is a very simple case, there is nothing to restrict.

The (r2) option would mean that we change the process of finding extensions.

The (r3) option means that we somehow revise the result and render it consistent.

For example we can use an idea of Sanjay Modgil from his PhD thesis and regard the output predicate theory $T_\lambda$ as a theory in a defeasible predicate logic. Thus $\neg \exists x A(x)$ is a defeasible rule with $A(J)$ and $A(M)$ as exceptions.  However, we do want $T_\lambda$ to be in classical logic, so (r3) is not an option for us.

It seems the simplest remedy for us is to use (r2). Revise the process.  

Seeking a remedy, let us simplify and assume that our predicate universe contains only two elements $J$ and $M$ and see whether this simplifications helps us get some ideas. Thus we have 
\begin{itemize}
\item $\exists x A(x) \equiv A(J)\vee A(M)$
\item $\forall x A(x)\equiv A(J)\wedge A(M)$.
\end{itemize}

Figures \ref{509-F1}b becomes Figure \ref{509-F1c}

\begin{figure}[ht]
\centering
\setlength{\unitlength}{0.00083333in}
\begingroup\makeatletter\ifx\SetFigFont\undefined%
\gdef\SetFigFont#1#2#3#4#5{%
  \reset@font\fontsize{#1}{#2pt}%
  \fontfamily{#3}\fontseries{#4}\fontshape{#5}%
  \selectfont}%
\fi\endgroup%
{\renewcommand{\dashlinestretch}{30}
\begin{picture}(2655,215)(0,-10)
\path(2265,140)(2415,140)
\blacken\path(2295.000,110.000)(2415.000,140.000)(2295.000,170.000)(2295.000,110.000)
\path(2040,140)(2565,140)
\blacken\path(2445.000,110.000)(2565.000,140.000)(2445.000,170.000)(2445.000,110.000)
\path(615,140)(765,140)
\blacken\path(645.000,110.000)(765.000,140.000)(645.000,170.000)(645.000,110.000)
\path(390,140)(915,140)
\blacken\path(795.000,110.000)(915.000,140.000)(795.000,170.000)(795.000,110.000)
\put(15,60){\makebox(0,0)[lb]{\smash{{\SetFigFont{10}{12.0}{\rmdefault}{\mddefault}{\updefault}$A(J)$}}}}
\put(990,65){\makebox(0,0)[lb]{\smash{{\SetFigFont{10}{12.0}{\rmdefault}{\mddefault}{\updefault}$A(J)\vee A(M)$}}}}
\put(2640,65){\makebox(0,0)[lb]{\smash{{\SetFigFont{10}{12.0}{\rmdefault}{\mddefault}{\updefault}$A(M)$}}}}
\end{picture}
}
\caption{}\label{509-F1c}
\end{figure}

To continue our conceputal analysis we need to deal with, and seek a remedy for, the problems of Figure \ref{509-F1c}.  In other words we need to solve the problem of instantiation into the classical propositional calculus, where $I(x)$ can give an arbitrary formula of classical propositional logic as values for $x$. 

Well, looking at Figure \ref{509-F2}, let us give meaning to attacks on disjunctions and a meaning to disjunctions attacking other elements.

The case of $x$ attacking the disjunction $y \vee z$ is complicated to express and it will be dealt with in Appendix C. Basically we have to express the  Boolean equation $\neg x \leftrightarrow (y \vee z)$, and for this we need disjunctive attacks.  

The case of a disjunction $y \vee z$ attacking $x$ is simple to express.

\begin{itemize}
\item $y\vee z\tO x$ means $y\tO x$ and $z\tO x$.
\end{itemize}

The reason being the meaning of $a\tO b$.  It means 
\begin{itemize}
\item if $a=$ ``in'', then $b=$ ``out''.
\end{itemize}
So 
\begin{itemize}
\item if $y\vee z$ is ``in'' then $x$ is ``out''
\end{itemize}
is equivalent to 
\begin{itemize}
\item if [$y$ is ``in'' or $z$ is ``in''] then $x$ is ``out''
\end{itemize}
which is equivalent to the conjunction of [if $y$  is ``in'' then $x$ is ``out''] and [if $z$ is ``in'' then $x$ is ``out''].

Thus Figure \ref{509-F1c} becomes Figure \ref{509-F1d}, which is the same as Figure \ref{509-F1e}

\begin{figure}[ht]
\centering
\setlength{\unitlength}{0.00083333in}
\begingroup\makeatletter\ifx\SetFigFont\undefined%
\gdef\SetFigFont#1#2#3#4#5{%
  \reset@font\fontsize{#1}{#2pt}%
  \fontfamily{#3}\fontseries{#4}\fontshape{#5}%
  \selectfont}%
\fi\endgroup%
{\renewcommand{\dashlinestretch}{30}
\begin{picture}(2304,1635)(0,-10)
\put(1077,60){\makebox(0,0)[b]{\smash{{\SetFigFont{10}{12.0}{\rmdefault}{\mddefault}{\updefault}$A(M)$}}}}
\path(1292,1470)(2277,960)
\blacken\path(2156.643,988.534)(2277.000,960.000)(2184.230,1041.816)(2156.643,988.534)
\path(12,765)(802,135)
\blacken\path(689.475,186.363)(802.000,135.000)(726.885,233.273)(689.475,186.363)
\path(1357,130)(2292,735)
\blacken\path(2207.549,644.623)(2292.000,735.000)(2174.954,694.997)(2207.549,644.623)
\path(2042,1090)(2137,1030)
\blacken\path(2019.522,1068.714)(2137.000,1030.000)(2051.561,1119.444)(2019.522,1068.714)
\path(2077,600)(2172,650)
\blacken\path(2079.782,567.563)(2172.000,650.000)(2051.837,620.658)(2079.782,567.563)
\path(607,290)(687,235)
\blacken\path(571.119,278.262)(687.000,235.000)(605.111,327.705)(571.119,278.262)
\path(692,1345)(797,1395)
\blacken\path(701.555,1316.322)(797.000,1395.000)(675.759,1370.494)(701.555,1316.322)
\put(27,810){\makebox(0,0)[b]{\smash{{\SetFigFont{10}{12.0}{\rmdefault}{\mddefault}{\updefault}$A(J)$}}}}
\put(2277,810){\makebox(0,0)[b]{\smash{{\SetFigFont{10}{12.0}{\rmdefault}{\mddefault}{\updefault}$A(M)$}}}}
\put(1077,1485){\makebox(0,0)[b]{\smash{{\SetFigFont{10}{12.0}{\rmdefault}{\mddefault}{\updefault}$A(J)$}}}}
\path(47,960)(932,1470)
\blacken\path(843.007,1384.091)(932.000,1470.000)(813.049,1436.077)(843.007,1384.091)
\end{picture}
}
\caption{}\label{509-F1d}
\end{figure}

\begin{figure}
\centering
\setlength{\unitlength}{0.00083333in}
\begingroup\makeatletter\ifx\SetFigFont\undefined%
\gdef\SetFigFont#1#2#3#4#5{%
  \reset@font\fontsize{#1}{#2pt}%
  \fontfamily{#3}\fontseries{#4}\fontshape{#5}%
  \selectfont}%
\fi\endgroup%
{\renewcommand{\dashlinestretch}{30}
\begin{picture}(2479,777)(0,-10)
\put(2176,80){\makebox(0,0)[b]{\smash{{\SetFigFont{10}{12.0}{\rmdefault}{\mddefault}{\updefault}$A(M)$}}}}
\path(1616,125)(1741,125)
\blacken\path(1621.000,95.000)(1741.000,125.000)(1621.000,155.000)(1621.000,95.000)
\path(61,400)(91,325)
\blacken\path(18.579,425.275)(91.000,325.000)(74.287,447.559)(18.579,425.275)
\path(1996,450)(2031,385)
\blacken\path(1947.694,476.434)(2031.000,385.000)(2000.522,504.880)(1947.694,476.434)
\path(351,200)(353,202)(357,206)
	(364,212)(374,222)(387,235)
	(402,250)(419,268)(436,287)
	(454,307)(470,329)(486,351)
	(500,374)(512,398)(522,423)
	(529,449)(532,477)(531,505)
	(526,529)(517,551)(508,570)
	(497,586)(487,600)(477,612)
	(467,621)(458,629)(449,636)
	(440,643)(430,648)(420,654)
	(409,660)(397,666)(383,672)
	(367,679)(348,685)(327,690)
	(305,694)(281,695)(258,693)
	(236,687)(216,681)(199,673)
	(184,666)(171,659)(161,652)
	(152,645)(143,639)(136,632)
	(129,626)(122,618)(114,610)
	(106,599)(98,587)(88,573)
	(79,556)(70,536)(62,514)
	(56,490)(53,459)(56,429)
	(62,400)(71,373)(83,348)
	(96,325)(111,302)(127,280)
	(143,260)(158,243)(171,227)(196,200)
\blacken\path(92.458,267.669)(196.000,200.000)(136.484,308.434)(92.458,267.669)
\path(2286,255)(2288,257)(2292,261)
	(2299,267)(2309,277)(2322,290)
	(2337,305)(2354,323)(2371,342)
	(2389,362)(2405,384)(2421,406)
	(2435,429)(2447,453)(2457,478)
	(2464,504)(2467,532)(2466,560)
	(2461,584)(2452,606)(2443,625)
	(2432,641)(2422,655)(2412,667)
	(2402,676)(2393,684)(2384,691)
	(2375,698)(2365,703)(2355,709)
	(2344,715)(2332,721)(2318,727)
	(2302,734)(2283,740)(2262,745)
	(2240,749)(2216,750)(2193,748)
	(2171,742)(2151,736)(2134,728)
	(2119,721)(2106,714)(2096,707)
	(2087,700)(2078,694)(2071,687)
	(2064,681)(2057,673)(2049,665)
	(2041,654)(2033,642)(2023,628)
	(2014,611)(2005,591)(1997,569)
	(1991,545)(1988,514)(1991,484)
	(1997,455)(2006,428)(2018,403)
	(2031,380)(2046,357)(2062,335)
	(2078,315)(2093,298)(2106,282)(2131,255)
\blacken\path(2027.458,322.669)(2131.000,255.000)(2071.484,363.434)(2027.458,322.669)
\put(236,60){\makebox(0,0)[b]{\smash{{\SetFigFont{10}{12.0}{\rmdefault}{\mddefault}{\updefault}$A(J)$}}}}
\path(536,125)(1891,130)
\blacken\path(1771.112,99.557)(1891.000,130.000)(1770.890,159.557)(1771.112,99.557)
\end{picture}
}
\caption{}\label{509-F1e}
\end{figure}

The need for attacks from and attacks to conjunctions and disjunctions of atoms has already been considered by us in 2009 in connection with fibring argumentation networks. Figure \ref{509-FF1} explains our notation from 2009 (full details are given in Appendix C).

\begin{figure}
\begin{center}
\setlength{\unitlength}{0.00083333in}
\begingroup\makeatletter\ifx\SetFigFont\undefined%
\gdef\SetFigFont#1#2#3#4#5{%
  \reset@font\fontsize{#1}{#2pt}%
  \fontfamily{#3}\fontseries{#4}\fontshape{#5}%
  \selectfont}%
\fi\endgroup%
{\renewcommand{\dashlinestretch}{30}
\begin{picture}(2580,1855)(0,-10)
\put(2565,430){\makebox(0,0)[b]{\smash{{\SetFigFont{10}{12.0}{\rmdefault}{\mddefault}{\updefault}$y$}}}}
\path(2265,1630)(2265,1030)(1965,580)
\blacken\path(2006.603,696.487)(1965.000,580.000)(2056.526,663.205)(2006.603,696.487)
\path(2265,1030)(2565,580)
\blacken\path(2473.474,663.205)(2565.000,580.000)(2523.397,696.487)(2473.474,663.205)
\path(615,1630)(390,1030)
\path(390,845)(385,725)
\blacken\path(360.022,846.145)(385.000,725.000)(419.970,843.647)(360.022,846.145)
\path(2105,790)(2055,710)
\blacken\path(2093.160,827.660)(2055.000,710.000)(2144.040,795.860)(2093.160,827.660)
\path(2425,805)(2485,700)
\blacken\path(2399.416,789.305)(2485.000,700.000)(2451.511,819.073)(2399.416,789.305)
\put(1965,430){\makebox(0,0)[b]{\smash{{\SetFigFont{10}{12.0}{\rmdefault}{\mddefault}{\updefault}$x$}}}}
\put(390,430){\makebox(0,0)[b]{\smash{{\SetFigFont{10}{12.0}{\rmdefault}{\mddefault}{\updefault}$z$}}}}
\put(390,55){\makebox(0,0)[b]{\smash{{\SetFigFont{10}{12.0}{\rmdefault}{\mddefault}{\updefault}(a)}}}}
\put(2340,55){\makebox(0,0)[b]{\smash{{\SetFigFont{10}{12.0}{\rmdefault}{\mddefault}{\updefault}(b)}}}}
\put(15,1705){\makebox(0,0)[b]{\smash{{\SetFigFont{10}{12.0}{\rmdefault}{\mddefault}{\updefault}$x$}}}}
\put(615,1705){\makebox(0,0)[b]{\smash{{\SetFigFont{10}{12.0}{\rmdefault}{\mddefault}{\updefault}$y$}}}}
\put(2265,1705){\makebox(0,0)[b]{\smash{{\SetFigFont{10}{12.0}{\rmdefault}{\mddefault}{\updefault}$z$}}}}
\path(15,1630)(390,1030)(390,580)
\blacken\path(360.000,700.000)(390.000,580.000)(420.000,700.000)(360.000,700.000)
\end{picture}
}
\end{center}
\begin{enumerate}
\item [(a)] $x,y$ jointly attack $z$. If $x=y=1$ then $z=0$.
\item [(b)] $z$ disjunctively attacks $x,y$. If $z=1$ then either $x=0$ or $y=0$ or both equal 0.
\end{enumerate}
\caption{}\label{509-FF1}
\end{figure}

It appears that we may now have a plan for a remedy of how to deal with monadic predicate logic substitutions $I(x)$ for nodes $x$ in argumentation networks:
\begin{enumerate}
\item  Eliminate the quantifiers in terms of  conjunctions and disjunctions
\item  Develop a theory of attacks involving  disjunctions and conjunctions.
\end{enumerate}

However, the reduction of the quantifiers to disjunctions and conjunctions and the (yet to be described) argumentation networks with joint and disjunctive attacks    does not solve our problem. Even for finite models the number of elements is not bounded and so we cannot replace $\exists x A(x)$ by a finite disjunction. We can attempt to say  use a closed world assumption and use all the names mentioned explicitly in the network.  This may work but not easily. We may have $J_1\comma J_k$ as all the names with $A(J_i), i=1\comma k$ being ``in'', but nevertheless $\exists x A(x)$ being out.  Worse still, if we allow predicate formulas for the form $\forall x \exists y \Phi (x,y)$, we may be forced to have an infinite number of elements. So this is not the way to go, at least not as a first attempt at the problem. 

Looking again at Figure \ref{509-F1}b, we suddenly make a surprising realisation. $A(J)$ is being essentially attacked by $\exists x A(x)$ when $\exists x A(x)$ is ``out''!

We would expect, as in the case of propositional networks, that when an argument is out, then it is ``dead''.  It has no effect. In the case of $\exists x A(x)$ when it is out it has an effect. This means, when taken to its full conclusions, that being ``in'', ``out'', ``undecided'' is not a value but it is a {\em state}, from which an argument can mount attacks.  Figure \ref{509-F3} illustrates this new point of view.

\begin{figure}
\centering
\setlength{\unitlength}{0.00083333in}
\begingroup\makeatletter\ifx\SetFigFont\undefined%
\gdef\SetFigFont#1#2#3#4#5{%
  \reset@font\fontsize{#1}{#2pt}%
  \fontfamily{#3}\fontseries{#4}\fontshape{#5}%
  \selectfont}%
\fi\endgroup%
{\renewcommand{\dashlinestretch}{30}
\begin{picture}(2723,1681)(0,-10)
\put(2708,158){\makebox(0,0)[lb]{\smash{{\SetFigFont{10}{12.0}{\rmdefault}{\mddefault}{\updefault}$x_3$}}}}
\put(908,833){\ellipse{1800}{450}}
\put(908,233){\ellipse{1800}{450}}
\path(1958,1433)(2558,1433)
\blacken\path(2438.000,1403.000)(2558.000,1433.000)(2438.000,1463.000)(2438.000,1403.000)
\path(1958,833)(2558,833)
\blacken\path(2438.000,803.000)(2558.000,833.000)(2438.000,863.000)(2438.000,803.000)
\path(1958,233)(2558,233)
\blacken\path(2438.000,203.000)(2558.000,233.000)(2438.000,263.000)(2438.000,203.000)
\path(2258,1433)(2408,1433)
\blacken\path(2288.000,1403.000)(2408.000,1433.000)(2288.000,1463.000)(2288.000,1403.000)
\path(2258,833)(2408,833)
\blacken\path(2288.000,803.000)(2408.000,833.000)(2288.000,863.000)(2288.000,803.000)
\path(2258,233)(2408,233)
\blacken\path(2288.000,203.000)(2408.000,233.000)(2288.000,263.000)(2288.000,203.000)
\put(158,1358){\makebox(0,0)[lb]{\smash{{\SetFigFont{10}{12.0}{\rmdefault}{\mddefault}{\updefault}$a$ is in state ``in''}}}}
\put(158,758){\makebox(0,0)[lb]{\smash{{\SetFigFont{10}{12.0}{\rmdefault}{\mddefault}{\updefault}$a$ is in state ``out''}}}}
\put(158,158){\makebox(0,0)[lb]{\smash{{\SetFigFont{10}{12.0}{\rmdefault}{\mddefault}{\updefault}$a$ is in state ``undec''}}}}
\put(2708,1358){\makebox(0,0)[lb]{\smash{{\SetFigFont{10}{12.0}{\rmdefault}{\mddefault}{\updefault}$x_1$}}}}
\put(2708,758){\makebox(0,0)[lb]{\smash{{\SetFigFont{10}{12.0}{\rmdefault}{\mddefault}{\updefault}$x_2$}}}}
\put(908,1433){\ellipse{1800}{450}}
\end{picture}
}
\caption{}\label{509-F3}
\end{figure}

In fact, to be completely coherent, we must allow for attacks of the form 
\[
(a\mbox{ is in state }\xi_1)\tO (b\mbox{ is in state }\xi_2)
\]

This brings us to the idea of what we call {\em state argumentation networks}, a new concept, which once made precise, can help us represent our original goal, that of predicate argumentation nework. We can possibly transform Figure \ref{509-F2}b into Figure \ref{509-F4}.

\begin{figure}
\centering
\setlength{\unitlength}{0.00083333in}
\begingroup\makeatletter\ifx\SetFigFont\undefined%
\gdef\SetFigFont#1#2#3#4#5{%
  \reset@font\fontsize{#1}{#2pt}%
  \fontfamily{#3}\fontseries{#4}\fontshape{#5}%
  \selectfont}%
\fi\endgroup%
{\renewcommand{\dashlinestretch}{30}
\begin{picture}(3025,2217)(0,-10)
\put(1058,246){\makebox(0,0)[b]{\smash{{\SetFigFont{10}{12.0}{\rmdefault}{\mddefault}{\updefault}$A(M)=$ in}}}}
\put(570,1896){\ellipse{1124}{600}}
\put(2455,1881){\ellipse{1124}{600}}
\put(2455,1056){\ellipse{1124}{600}}
\put(1105,306){\ellipse{1124}{600}}
\path(1133,1896)(1883,1896)
\blacken\path(1763.000,1866.000)(1883.000,1896.000)(1763.000,1926.000)(1763.000,1866.000)
\path(1583,1896)(1733,1896)
\blacken\path(1613.000,1866.000)(1733.000,1896.000)(1613.000,1926.000)(1613.000,1866.000)
\path(1958,1146)(983,1671)
\blacken\path(1102.880,1640.522)(983.000,1671.000)(1074.434,1587.694)(1102.880,1640.522)
\path(1958,846)(1518,516)
\blacken\path(1596.000,612.000)(1518.000,516.000)(1632.000,564.000)(1596.000,612.000)
\path(1223,1541)(1108,1591)
\blacken\path(1230.010,1570.665)(1108.000,1591.000)(1206.087,1515.641)(1230.010,1570.665)
\path(1753,696)(1653,611)
\blacken\path(1725.003,711.576)(1653.000,611.000)(1763.862,665.860)(1725.003,711.576)
\path(1773,1596)(1833,1651)
\blacken\path(1764.813,1547.798)(1833.000,1651.000)(1724.270,1592.028)(1764.813,1547.798)
\path(1778,1356)(1838,1306)
\blacken\path(1726.608,1359.775)(1838.000,1306.000)(1765.019,1405.869)(1726.608,1359.775)
\put(533,1821){\makebox(0,0)[b]{\smash{{\SetFigFont{10}{12.0}{\rmdefault}{\mddefault}{\updefault}$A(J)=$ in}}}}
\put(2483,1821){\makebox(0,0)[b]{\smash{{\SetFigFont{10}{12.0}{\rmdefault}{\mddefault}{\updefault}$\exists x A(x)=$ in}}}}
\put(2483,996){\makebox(0,0)[b]{\smash{{\SetFigFont{10}{12.0}{\rmdefault}{\mddefault}{\updefault}$\exists x A(x)=$ out}}}}
\put(1990.481,1483.963){\arc{486.158}{1.6016}{4.6198}}
\blacken\path(1859.404,1245.908)(1983.000,1241.000)(1876.254,1303.493)(1859.404,1245.908)
\blacken\path(1865.258,1657.122)(1968.000,1726.000)(1844.933,1713.574)(1865.258,1657.122)
\end{picture}
}
\caption{}\label{509-F4}
\end{figure}

A reader might say why not use negation as in Figure \ref{509-F5}?

\begin{figure}
\centering
\setlength{\unitlength}{0.00083333in}
\begingroup\makeatletter\ifx\SetFigFont\undefined%
\gdef\SetFigFont#1#2#3#4#5{%
  \reset@font\fontsize{#1}{#2pt}%
  \fontfamily{#3}\fontseries{#4}\fontshape{#5}%
  \selectfont}%
\fi\endgroup%
{\renewcommand{\dashlinestretch}{30}
\begin{picture}(2307,1185)(0,-10)
\put(1815,1035){\makebox(0,0)[lb]{\smash{{\SetFigFont{10}{12.0}{\rmdefault}{\mddefault}{\updefault}$\exists x A(x)$}}}}
\path(1815,1035)(465,210)
\blacken\path(551.750,298.172)(465.000,210.000)(583.037,246.976)(551.750,298.172)
\blacken\path(2295.000,840.000)(2265.000,960.000)(2235.000,840.000)(2295.000,840.000)
\path(2265,960)(2265,210)
\blacken\path(2235.000,330.000)(2265.000,210.000)(2295.000,330.000)(2235.000,330.000)
\blacken\path(2295.000,690.000)(2265.000,810.000)(2235.000,690.000)(2295.000,690.000)
\path(2265,810)(2265,360)
\blacken\path(2235.000,480.000)(2265.000,360.000)(2295.000,480.000)(2235.000,480.000)
\path(1515,1110)(1665,1110)
\blacken\path(1545.000,1080.000)(1665.000,1110.000)(1545.000,1140.000)(1545.000,1080.000)
\path(690,360)(615,285)
\blacken\path(678.640,391.066)(615.000,285.000)(721.066,348.640)(678.640,391.066)
\put(15,1035){\makebox(0,0)[lb]{\smash{{\SetFigFont{10}{12.0}{\rmdefault}{\mddefault}{\updefault}$A(J)$}}}}
\put(15,60){\makebox(0,0)[lb]{\smash{{\SetFigFont{10}{12.0}{\rmdefault}{\mddefault}{\updefault}$A(M)$}}}}
\put(1815,60){\makebox(0,0)[lb]{\smash{{\SetFigFont{10}{12.0}{\rmdefault}{\mddefault}{\updefault}$\neg \exists A(x)$}}}}
\path(540,1110)(1815,1110)
\blacken\path(1695.000,1080.000)(1815.000,1110.000)(1695.000,1140.000)(1695.000,1080.000)
\end{picture}
}
\caption{}\label{509-F5}
\end{figure}

We can do that, but in general, an argument can have more than two states. State argumentation networks is a more general concept and we may wish to continue and develop it in this paper. Let us define it intuitively. 

\begin{definition}\label{509-D1}
Let $(S, R)$ be a network and assume that $S=S_1\cup \ldots \cup S_k$ with $S_i\neq 0$ and $S_i\cap S_j =\varnothing, i\neq j$. Also assume that for each $i$ and each $x, y \in S_i$, such that $x\neq y$ we have that $xRy$ hold. Under these conditions we can regard $(S, R)$ as a state argumentation network, where the elements are $\{S_i\}$ and each $x\in S_i$ is a different state of $S_i$.
\end{definition}

Figure \ref{509-F1}b can become Figure \ref{509-F1bb}.  $\top$ is attacking all $\neg x$ nodes where $x$ is not attacked in the original figure.

\begin{figure}[h]
\centering
\setlength{\unitlength}{0.00083333in}
\begingroup\makeatletter\ifx\SetFigFont\undefined%
\gdef\SetFigFont#1#2#3#4#5{%
  \reset@font\fontsize{#1}{#2pt}%
  \fontfamily{#3}\fontseries{#4}\fontshape{#5}%
  \selectfont}%
\fi\endgroup%
{\renewcommand{\dashlinestretch}{30}
\begin{picture}(3657,2251)(0,-10)
\put(3615,2089){\makebox(0,0)[b]{\smash{{\SetFigFont{10}{12.0}{\rmdefault}{\mddefault}{\updefault}$S_3$}}}}
\blacken\path(3645.000,1444.000)(3615.000,1564.000)(3585.000,1444.000)(3645.000,1444.000)
\path(3615,1564)(3615,1189)
\blacken\path(3585.000,1309.000)(3615.000,1189.000)(3645.000,1309.000)(3585.000,1309.000)
\blacken\path(195.000,1669.000)(165.000,1789.000)(135.000,1669.000)(195.000,1669.000)
\path(165,1789)(165,1039)
\blacken\path(135.000,1159.000)(165.000,1039.000)(195.000,1159.000)(135.000,1159.000)
\blacken\path(195.000,1519.000)(165.000,1639.000)(135.000,1519.000)(195.000,1519.000)
\path(165,1639)(165,1189)
\blacken\path(135.000,1309.000)(165.000,1189.000)(195.000,1309.000)(135.000,1309.000)
\blacken\path(1695.000,1669.000)(1665.000,1789.000)(1635.000,1669.000)(1695.000,1669.000)
\path(1665,1789)(1665,1039)
\blacken\path(1635.000,1159.000)(1665.000,1039.000)(1695.000,1159.000)(1635.000,1159.000)
\blacken\path(1695.000,1519.000)(1665.000,1639.000)(1635.000,1519.000)(1695.000,1519.000)
\path(1665,1639)(1665,1189)
\blacken\path(1635.000,1309.000)(1665.000,1189.000)(1695.000,1309.000)(1635.000,1309.000)
\path(465,1864)(1515,1864)
\blacken\path(1395.000,1834.000)(1515.000,1864.000)(1395.000,1894.000)(1395.000,1834.000)
\path(1215,1864)(1365,1864)
\blacken\path(1245.000,1834.000)(1365.000,1864.000)(1245.000,1894.000)(1245.000,1834.000)
\path(2415,1864)(3315,1864)
\blacken\path(3195.000,1834.000)(3315.000,1864.000)(3195.000,1894.000)(3195.000,1834.000)
\path(3015,1864)(3165,1864)
\blacken\path(3045.000,1834.000)(3165.000,1864.000)(3045.000,1894.000)(3045.000,1834.000)
\path(1665,1039)(3540,1714)
\blacken\path(3437.255,1645.127)(3540.000,1714.000)(3416.932,1701.580)(3437.255,1645.127)
\path(1665,1039)(240,1714)
\blacken\path(361.291,1689.742)(240.000,1714.000)(335.606,1635.518)(361.291,1689.742)
\path(3250,1609)(3395,1669)
\blacken\path(3295.589,1595.397)(3395.000,1669.000)(3272.647,1650.838)(3295.589,1595.397)
\path(465,1609)(375,1659)
\blacken\path(494.468,1626.948)(375.000,1659.000)(465.330,1574.498)(494.468,1626.948)
\path(165,289)(165,814)
\blacken\path(195.000,694.000)(165.000,814.000)(135.000,694.000)(195.000,694.000)
\path(1495,1119)(1650,1044)
\blacken\path(1528.914,1069.263)(1650.000,1044.000)(1555.048,1123.272)(1528.914,1069.263)
\path(1390,1174)(1505,1114)
\blacken\path(1384.733,1142.910)(1505.000,1114.000)(1412.487,1196.105)(1384.733,1142.910)
\path(1830,1099)(1690,1049)
\blacken\path(1792.919,1117.613)(1690.000,1049.000)(1813.099,1061.108)(1792.919,1117.613)
\path(1975,1154)(1845,1099)
\blacken\path(1943.827,1173.386)(1845.000,1099.000)(1967.205,1118.128)(1943.827,1173.386)
\put(15,1789){\makebox(0,0)[lb]{\smash{{\SetFigFont{10}{12.0}{\rmdefault}{\mddefault}{\updefault}$A(J)$}}}}
\put(15,889){\makebox(0,0)[lb]{\smash{{\SetFigFont{10}{12.0}{\rmdefault}{\mddefault}{\updefault}$\neg A(J)$}}}}
\put(1515,1789){\makebox(0,0)[lb]{\smash{{\SetFigFont{10}{12.0}{\rmdefault}{\mddefault}{\updefault}$\exists xA(x)$}}}}
\put(1515,889){\makebox(0,0)[lb]{\smash{{\SetFigFont{10}{12.0}{\rmdefault}{\mddefault}{\updefault}$\neg\exists x A(x)$}}}}
\put(3315,1789){\makebox(0,0)[lb]{\smash{{\SetFigFont{10}{12.0}{\rmdefault}{\mddefault}{\updefault}$A(M)$}}}}
\put(3315,889){\makebox(0,0)[lb]{\smash{{\SetFigFont{10}{12.0}{\rmdefault}{\mddefault}{\updefault}$\neg A(M)$}}}}
\put(165,64){\makebox(0,0)[b]{\smash{{\SetFigFont{10}{12.0}{\rmdefault}{\mddefault}{\updefault}$\top$}}}}
\put(165,2089){\makebox(0,0)[b]{\smash{{\SetFigFont{10}{12.0}{\rmdefault}{\mddefault}{\updefault}$S_1$}}}}
\put(1665,2089){\makebox(0,0)[b]{\smash{{\SetFigFont{10}{12.0}{\rmdefault}{\mddefault}{\updefault}$S_2$}}}}
\blacken\path(3645.000,1594.000)(3615.000,1714.000)(3585.000,1594.000)(3645.000,1594.000)
\path(3615,1714)(3615,1039)
\blacken\path(3585.000,1159.000)(3615.000,1039.000)(3645.000,1159.000)(3585.000,1159.000)
\end{picture}
}
\caption{}\label{509-F1bb}
\end{figure}

\begin{definition}\label{509-D2}
A two state argumentation network has the form 
\[
(S\cup S^\neg \cup \{\top \},R)
\]
where $S$ is a set of atoms, $S^\neg =\{\neg q | q \in S\}$ and $\top$ is top.  We have
\begin{itemize}
\item $\neg \exists x (xR\top)$
\item $\forall x (xR\neg x)$
\item $\forall x(\neg xRx)$
\end{itemize}
\end{definition}

\begin{lemma}\label{509-L3}
Every $(S, R)$ is equivlaent to $(S^*, R^*)$ where $S^* = S\cup S^\neg \cup \{\top\}$ and $R^*$ is 
\[
R\cup \{(\top, \neg x)\mid x \in S\}\cup \{(x, \neg x ), (\neg x, x)\}.
\]
The extensions $E^*$ of $(S^*, R^*)$ are exactlly the extension $E$ of $(S, R)$ augmented by $\{\top\}$.
\end{lemma}
\begin{proof}
Start with $(S, R)$. Create $S^\neg =\{\neg x | x \in S\}$ and assume $\top \not\in S$. Let $S^* =S\cup S^\neg \cup \{\top\}$.  Let $R^*$ be defined as $R^* =R\cup \{(x,\neg x), (\neg x, x)|x\in S\}\cup \{(\top,\neg x|x\in S\}$.  It is clear that $\top$ in $S^*$ attacks all the new points of $S^\neg$ which we added. Thus any extension $E$ of $(S, R)$ becomes the extension $E\cup \{\top\}$ of $(S^*, R^*)$ and vice versae. 
\end{proof}

\subsection{Summary of our plan so far for monadic predicate instantiation}
We propose, at this stage of our conceptual analysis, the following plan.

We are given an abstract argumentation network $(S,R)$, with an instantiation function $I$, giving for each $x \in S$ a formula of monadic predicate logic. We want  to deal with it.
\begin{enumerate}
\renewcommand{\labelenumi}{\alph{enumi}.}
\item First we prove some theorems that the input from predicate logic can be restricted, without loss of generality, to an argumentation friendly form.
\item  Assuming the input is of this form, we use its syntactical form together with $R$, to move to a new abstract argumentation network $(S^*, R^*, I^*)$, with $S$ a subset of $S^*$ and $R$ a subset of $R^*$.

We take extensions $E^*$ for the new network and look at $E^*\cap  S$. We declare these  as the sought for extensions for the original $(S,R,I)$. 
\item  Hopefully we will prove that $\{ I(x) \mid x \in  E^*\cap S \}$ is consistent.
\end{enumerate}
To achieve this we need some technical results.

The following is the list:
\begin{enumerate}
\item Define the notion of a 2-state argumentation network. Show that such networks are  a special case of abstract argumentation network in the sense that they can be identified by special properties on the attack relation $R$.
\item Show that every argumentation network $(S,R)$ can be embedded in a larger 2 state argumentation network $(S^*,R^*)$ in a critical way. This means that $(S,R)$ preserves all it properties  even though it is part of the larger network. Thus we can say that every argumentation network is  equivalent to a 2-state argumentation network. The equivalence is shown by a linear general transformation.
\item Show that every formula $\Phi$ of monadic predicate logic is classically equivalent to a formula $\bar{\Phi}$ in a standard argumentation friendly form, to be defined and to be convenient for our objective.
\item For every ordinary $(S, R)$, define a 2-state $(S^*, R^*)$ called the associate of $(S, R)$, by 
\[\begin{array}{l}
S^*=S\cup S^\neg \cup \{\top\}\\
R^*=\{(x,\neg x), (\neg x, x) \mid x\in S\} \cup \{\top, y\mid y \mbox{ not attacked in }(S, R)\}\cup R
\end{array}
\]
This is not the embedding described in (2) above.
\item Given an $(S, R)$ form the associate $(S^*, R^*)$. In order to instantiate $(S, R)$ with predicate formulas $x\mapsto \Phi_x$, for $x\in S$, use instead  $(S^*, R^*)$ of (4) above  and instantiate the nodes with standard form formulas, $x\mapsto \bar{\Phi}_x$ and $\neg x \to \neg\bar{\Phi}_x$.

Let $I(z)$ for $z\in S^*$ be the instantiation function. We look at $(S^*, R^*, I)$ and rewrite (transform) the network in an easy and purely syntactical way into a new network $(S^*, R^{**}, I)$ and then take extensions. In this way, we hope, the correct consistent extensions are obtained. Thus the extensions $E^{**}$ thus obtained restricted to $S$ shall be declared as the predicate extensions of $(S, R, I)$.
\end{enumerate}

It is useful to give an example.

\begin{example}\label{509-E4}{\ }
\begin{enumerate}
\item Start with the network $(S, R)$ of Figure \ref{509-F1}a.
\item Transform it to the equivalent Figure \ref{509-F1aa}.

\begin{figure}
\centering
\setlength{\unitlength}{0.00083333in}
\begingroup\makeatletter\ifx\SetFigFont\undefined%
\gdef\SetFigFont#1#2#3#4#5{%
  \reset@font\fontsize{#1}{#2pt}%
  \fontfamily{#3}\fontseries{#4}\fontshape{#5}%
  \selectfont}%
\fi\endgroup%
{\renewcommand{\dashlinestretch}{30}
\begin{picture}(2495,2014)(0,-10)
\put(648,64){\makebox(0,0)[b]{\smash{{\SetFigFont{10}{12.0}{\rmdefault}{\mddefault}{\updefault}$\top$}}}}
\path(1398,1939)(2298,1939)
\blacken\path(2178.000,1909.000)(2298.000,1939.000)(2178.000,1969.000)(2178.000,1909.000)
\path(648,214)(48,964)
\blacken\path(146.389,889.037)(48.000,964.000)(99.537,851.555)(146.389,889.037)
\blacken\path(1276.376,1743.606)(1248.000,1864.000)(1216.381,1744.416)(1276.376,1743.606)
\path(1248,1864)(1238,1124)
\blacken\path(1209.624,1244.394)(1238.000,1124.000)(1269.619,1243.584)(1209.624,1244.394)
\blacken\path(79.376,1738.606)(51.000,1859.000)(19.381,1739.416)(79.376,1738.606)
\path(51,1859)(41,1119)
\blacken\path(12.624,1239.394)(41.000,1119.000)(72.619,1238.584)(12.624,1239.394)
\blacken\path(2479.376,1718.606)(2451.000,1839.000)(2419.381,1719.416)(2479.376,1718.606)
\path(2451,1839)(2441,1099)
\blacken\path(2412.624,1219.394)(2441.000,1099.000)(2472.619,1218.584)(2412.624,1219.394)
\blacken\path(79.393,1589.357)(48.000,1709.000)(19.397,1588.659)(79.393,1589.357)
\path(48,1709)(53,1279)
\blacken\path(21.607,1398.643)(53.000,1279.000)(81.603,1399.341)(21.607,1398.643)
\blacken\path(1283.000,1594.000)(1253.000,1714.000)(1223.000,1594.000)(1283.000,1594.000)
\path(1253,1714)(1253,1289)
\blacken\path(1223.000,1409.000)(1253.000,1289.000)(1283.000,1409.000)(1223.000,1409.000)
\blacken\path(2483.000,1574.000)(2453.000,1694.000)(2423.000,1574.000)(2483.000,1574.000)
\path(2453,1694)(2453,1254)
\blacken\path(2423.000,1374.000)(2453.000,1254.000)(2483.000,1374.000)(2423.000,1374.000)
\path(198,779)(143,839)
\blacken\path(246.202,770.813)(143.000,839.000)(201.972,730.270)(246.202,770.813)
\put(48,1864){\makebox(0,0)[b]{\smash{{\SetFigFont{10}{12.0}{\rmdefault}{\mddefault}{\updefault}$x$}}}}
\put(1248,1864){\makebox(0,0)[b]{\smash{{\SetFigFont{10}{12.0}{\rmdefault}{\mddefault}{\updefault}$y$}}}}
\put(2448,1864){\makebox(0,0)[b]{\smash{{\SetFigFont{10}{12.0}{\rmdefault}{\mddefault}{\updefault}$z$}}}}
\put(48,1039){\makebox(0,0)[b]{\smash{{\SetFigFont{10}{12.0}{\rmdefault}{\mddefault}{\updefault}$\neg x$}}}}
\put(1248,1039){\makebox(0,0)[b]{\smash{{\SetFigFont{10}{12.0}{\rmdefault}{\mddefault}{\updefault}$\neg y$}}}}
\put(2448,1039){\makebox(0,0)[b]{\smash{{\SetFigFont{10}{12.0}{\rmdefault}{\mddefault}{\updefault}$\neg z$}}}}
\path(198,1939)(1098,1939)
\blacken\path(978.000,1909.000)(1098.000,1939.000)(978.000,1969.000)(978.000,1909.000)
\end{picture}
}

\caption{}\label{509-F1aa}
\end{figure}
\item Substitute/instantiate:
\[\begin{array}{l}
I(x) =A(J)\\
I(y) =\exists x A(x)\\
I(z) = A(M)
\end{array}\]
and adjust the figure by adding attacks from any $\neg \exists x P(x)$ onto any $P(y)$ and from any $\forall x P(x)$ onto any $\neg P(y)$.

Note that this is done purely syntactically without taking into consideration any logical meaning of the instantiations formulas.  We get Figure \ref{509-F1bb}.
\item Calculate traditional extensions.  We get the extensions 
\[
\{\neg A(M), \neg \exists x A(x)\}
\]
We declare these as the extensions of $(S, R, I)$.
\end{enumerate}
\end{example}

\begin{remark}\label{509-R5}
The perceptive reader looking at the extensions obtained in Example \ref{509-E4} for the network of Figure \ref{509-F1}b might justifiably ask, what is the intuition behind this?  Our answer is  that in this case there we should not expect much intuition. We took an arbitrary abstract network coming from nowhere and substituted arbitrary predicate formulas into it. What kind of result would you expect, beyond that it is consistent?  Nevertheless, let us look at the result from an AGM revision point of view.  Our original theory was $\{A(J), \neg \exists xA(x), A(M)\}$ and we offer the  revision option
\[
\{\neg  A(M),\neg \exists xA(x)\}.
\]
 This is a maximal consistent sub-theory, and it makes sense if priority is given to   $\neg\exists xA(x)$.   

The real test for our intuition,  however, in the case where    we take a set of arbitrary predicate formulas, is to regard them as a network (i.e. with the empty attack domain) and apply the process to them. Do we get all maximally consistent subsets as the family of all extensions? This is the real test.

To be more precise, let  our starting network $(S,R,I)$ be with $R$ empty, i.e. no attacks, and apply our process to it. See what we get. This is the intuitiveness  test.\footnote{This type of network (i.e. arbitrary $S$, empty attack relation and any instantiation into classical propositional calculus) can be viewed as an Abstract Dialectical Frame of Brewka and Woltran, see \cite{509-2}. }
\end{remark}

Discussion and comparison with the literature will follow in the appropriate later section.


\section{Abstract instantiated argumentation frames (AIAF)}
This is a more formal section which will deal with several types of argumentation frames where the abstract arguments are instantiated by formulas of some logic. We consider classical propositional logic, classical monadic predicate logic without equality and modal logic S5.  We also discuss other possibilitites such as instantiating with Boolean Attack Formations (BAFs, see Appendix C2). We examine our options, then propose a more general theory, and compare with the literature.

\subsection{Instantiating with formulas of propositional logic}
This sub-section is a case study, leading to to the next subsection 2.2, which will give concrete definitions.
\begin{definition}\label{509-2D1}
The classical propositional calculus is built up syntactically as follows:
\begin{enumerate}
\item A set $Q$ of atomic propositions.
\item The classical connectives $\{\neg, \wedge,\vee, \to\}$, which are used to define the traditional notion of a formula, (wff).
\item The traditional notion of ``the formula $\Phi$ of classical propositional logic is consistent'', defined either semantically or proof-theoretically. We do not care how it is done. We just need to use it.
\end{enumerate}
\end{definition}

\begin{definition}\label{509-2D2}{\ }
\begin{enumerate}
\item An abstract argumentation frame has the form $(S, R)$, where $S$ is a set of atomic symbols (we use for $S$ distinct symbols from those we use for $Q$ of Definition \ref{509-2D1}), and $R\subseteq S \times S$ is the attack relation. We also denote $xRy$ by $x\tO y$.
\item We follow Dung \cite{509-15} and define the notion of complete extension $E\subseteq S$ as a set satisfying the following:
\begin{enumerate}
\item $E$ is {\em conflict free}, namely for no $x, y \in E$ do we have $xRy$.
\item $E$ {\em protects} its members, where ``$E$ protects $x$'' means that $\forall  z (zRx\to \exists y \in E(yRz))$.
\item $E$ is {\em complete}, namely if $E$ protects $x$ then $x\in E$, for any $x\in S$.
\end{enumerate}
\item It is well known that complete extensions always exist, though they might be empty.
\end{enumerate}
\end{definition}

\begin{definition}\label{509-2DM1}
Let $(S, R)$ be an argumentation frame. Let $\lambda: S\mapsto \{0, \half, 1\}$ be a labelling function. We say $\lambda$ is a legitimate Caminada labelling iff the following holds:
\begin{enumerate}
\item If $\neg \exists y (yRx)$ then $\lambda (x) =1$
\item If for all $y$ s.t. $yRx$ we have $\lambda (y) =0$, then $\lambda (x) =1$.
\item If for some $y$ such that $yRx$ we have $\lambda (y) =1$ then $\lambda (x) =0$.
\item If (a) for all $y$ such that $uRx$ we have $\lambda (y) < 1$ and (b) for some $y$ such that $yRx$ we have $\lambda (y)=\half$ then $\lambda (x) =\half$.
\end{enumerate}
\end{definition}

\begin{lemma}\label{509-2LM2}
Let $(S, R)$ be an argumentation network. Let $E$ be a complete extension as defined in Definition \ref{509-2D2}. Let $\lambda_E$ be defined by 
\[
\lambda_E (x) =
\left\{
\begin{array}{l}
1 \mbox{ if } x \in E\\
0 \mbox{ if } \exists y \in E yRx\\
\half \mbox{ otherwise}
\end{array}
\right.
\]
Then $\lambda_E$ is a legitimate Caminada labelling.
\end{lemma}
\begin{proof}{\ }
\begin{enumerate}
\item If $\neg \exists y (yRx)$ then $x$ is in $E$ and hence $\lambda_E (x)=1$.
\item If for all $y$ such that $(yRx)$ we have $\lambda _E(y) =0$ then for all such $y$ there is a $z$ such that $zRy$ and $z\in E$. Thus $E$ protects $x$, hence $x\in E$, hence $\lambda (x) =1$.
\item If for some $y$ such that $yRx$ we have $\lambda_E (y) =1$, then $y\in E$ and hence $\lambda_E (x) =0$ by definition.
\item Assume 
\begin{enumerate}
\item for all $y$ such that $yRx$ we have $\lambda_E (y) < 1$
\item for some $y_0, y_0Rx$ we have $\lambda_E (y_0) =\half$. 
\end{enumerate}
We want to show that $\lambda (x) =\half$.

From (a) we get that $\neg \exists y \in E (yRx)$. Thus $\lambda_E (x) \neq 0$. 

We show that $\lambda _E (x) \neq 1$, i.e.\ $x\not\in E$.  From (b) above, $y_0Rx$ and $\lambda (y_0) =\half$. Hence $\neg \exists z \in E (zRy_0)$. This means that $x$ is not protected by $E$ and hence $x \not\in E$. 
\end{enumerate}
\end{proof}

\begin{lemma}\label{509-2LM3}
Let $\lambda$ be a legitimate Caminada labelling and let $E_\lambda =\{x|\lambda (x) = 1\}$, then $E_\lambda$ is a complete extension.
\end{lemma}

\begin{proof}
We show $E_\lambda$ is a complete extension:
\begin{enumerate}
\item If $\neg \exists y (xRy)$ then $\lambda (x) =1$ and so $x\in E_\lambda$.
\item $E_\lambda$ is conflict fee because if $\lambda (x) =1$ and $xRy$ then $\lambda (y) =0$ and $y \not\in E_\lambda$.
\item If $\lambda (x) =1$ and $yRx$ then $\lambda (y) < 1$. We must also have for such $y$ that $\lambda (y) =0$ for otherwise we would get $\lambda (x) =\half$.  But $\lambda (y) =0$ for all such $y$ means that $\forall y (yRx \to \exists z (zRy\wedge \lambda (z) =1)$. This means that $E_\lambda$ protects all of its members.
\item Suppose $E_\lambda$ protects $x$. This means that $E_\lambda$ attacks all of the attackers of $x$. This means $\forall y (y Rx\to \lambda (y) =0)$ Therefore $\lambda(x) =1$ and so $x\in E_\lambda$.
\end{enumerate}
\end{proof}

\begin{definition}\label{509-2DM4}
Let $(S_i, R_i)$ for $i=1,2$ be argumentation frames. Let $I_{1,2}$ be a function from $S_1$ to $S_2$. We form the network $(S_{1,2}, R_{1,2})$ called {\em the abstract instantiation of } $(S_1, R_1)$ {\em by} $(S_2, R_2)$ {\em using} $I_{1,2}$ as follows:

\[\begin{array}{lcl}
S_{1,2} &=&\{I_{1,2} (x) | x\in S_1\}\\
R_{1,2} &=&\{(y,z)|\mbox{for some } a, b\in S_1, I_{2,3} (a) = y\mbox{ and } I_{1,2} (b) =z\\
&&\mbox{and } (y,z) \in R_1\}\cup R_2\upharpoonright S_{1,2}.
\end{array}
\]
\end{definition}

\begin{example}\label{509-2EM5}{\ }
\begin{enumerate}
\item Let \BL\ be a logical system. Let $\Phi, \Psi$ be two wffs of \BL.  Let $\rho (\Phi, \Psi)$ mean that $\Phi$, when ``added'' to $\Psi$ causes ``incompatibility''.  (We are not saying ``$\{\Phi, \Psi\}$ is inconsistent'' because we do not say anything about the logic and $\rho$ may not even be symmetrical. In some logics $\Phi \circledwedge \Psi$ is not the same as $\Psi\circledwedge \Phi$). Let WFF(\BL) be all the well formed formulas of the logic. We can consider the network (WFF$(\BL),\rho))$.
\item Given $(S, R)$ and a logic \BL, we can instantiate it by (WFF(\BL),$\rho$) as defined in Definition \ref{509-2DM4}. This will include classical logic instantiation.

For logics \BL\ which have a negation symbol $\neg$ (e.g.\ classical modal, monadic, or intuitionistic logics) we can require $(S, R)$ to be a 2-state network as in Definition \ref{509-D2} and require the instantiation function $I: S\mapsto \mbox{WFF}(\BL)$ to satisfy
\[
I(\neg x) =\neg I(x).
\] 
\end{enumerate}
\end{example}

\begin{definition}\label{509-2D3}{\ }
\begin{enumerate}
\item An abstract instantiated Boolean argumentation frame $(B-AIAF)$ has the form $(S, R, I)$ where $(S, R)$ is an abstract argumentation frame and $I$ is a function, giving for each $x\in S$, a formula $\Phi_x(\{q_1\comma q_n\})$ of the classical propositional calculus.
\item We write 
\[
I(x)=\Phi_x(\{q_1\comma q_m\})
\]
indicating that the classical propositional atoms $\{q_1\comma q_n\}$ are exactly those that appear in $I(x)=\Phi_x$.

The symbols $\{q_1\comma q_n\}$ are distinct from the atomic symbols of $S$.
\end{enumerate}
\end{definition}

We want to view $(S, R, I)$ as a more general network than $(S, R)$ and would like to define a sensible notion of complete extensions for it. We begin by explaining our strategy:

We are given an instantiated system of the form $(S,R,I)$, where $I$ is a function associating  entities  of the form \Be\ from some space \BE. We assume these entities can interact among themselves and that as a result of the interaction we can get truth values in, say, valued in the unit interval $[0,1]$.\\
{\bf Strategy 1:} Regard the attack relation $R$ as stimulating interaction among the instantiated entities \Be\ and use the interaction to obtain values in [0,1] for the nodes of $S$.\footnote{There are more details in the beginning of appendix E1. See Example \ref{509-EE0} and the paragraph preceding it.}\\
{\bf Strategy 2:} Use the relationships among the entities \BE\ to change $(S,R,I)$ into a new $(S^*,R^*,I^*)$ and proceed with Strategy 1 for the new system.

There are several options and to explain the differences we need to be very precise.

\begin{definition}\label{509-2D4} {\ }
\begin{enumerate}
\item Let $S$ be a set of arguments.  Let $Q_S$ be a corresponding set of atomic propositions of the classical propositional logic of the form 
\[
Q_S =\{q_x| x\in S\}
\]
where $q_x$ are distinct symbols, i.e., $x\neq y \To q_x\neq q_y$ and $S\cap Q_X=\varnothing$.  Note that when there is no possibility of confusion, we abuse notation and write ``$x$'' in place of ``$q_x$''.
\item Let $I_{id}$ be the function $I_{id}(x) =q_x$, or, by abuse of notation, $I(x)=x$.
\end{enumerate}
\end{definition}

Our objective is to define the notion of complete extensions for $B-AIAF$s of the form $(S, R, I)$. There are several main views we can take:

\paragraph{View 1: The $(E<I)$ view.}
First take a traditional complete extension $E$ of $(S, R)$ and then instantiate the element of $E$.  We get a set $T_E$ of classical propositional formulas.  We need this set to be logically consistent.

This is what we did with the system of Figures \ref{509-F1}a and \ref{509-F1}b.

We needed the remedies discussed in Section 1, which essentially abandoned this view in favour of the next view $(I < E)$ which says, first instantiate then take extensions.   We shall discuss this view next.

\paragraph{View 2: The $(I < E)$ view.} 
We follow Strategy 1 here and keep $(S,R,I)$ intact as is. We view $R$ as stimulating interactions among the instantiating entities and try to extract numbers from the interaction. We use the Equational approach to do that.
 This view instantiates first and then takes extensions. Obviously we need to make use of the instantiation when we consider how to define the extensions. Let us explain our options using our experience with the examples we have from Section 1.

\begin{figure}
\centering
\setlength{\unitlength}{0.00083333in}
\begingroup\makeatletter\ifx\SetFigFont\undefined%
\gdef\SetFigFont#1#2#3#4#5{%
  \reset@font\fontsize{#1}{#2pt}%
  \fontfamily{#3}\fontseries{#4}\fontshape{#5}%
  \selectfont}%
\fi\endgroup%
{\renewcommand{\dashlinestretch}{30}
\begin{picture}(2130,2026)(0,-10)
\put(1065,64){\makebox(0,0)[b]{\smash{{\SetFigFont{10}{12.0}{\rmdefault}{\mddefault}{\updefault}$y_1\comma y_n$ are all the attackers of $x$.}}}}
\path(2115,1789)(1065,664)
\blacken\path(1124.947,772.196)(1065.000,664.000)(1168.810,731.257)(1124.947,772.196)
\path(860,889)(955,774)
\blacken\path(855.445,847.409)(955.000,774.000)(901.703,885.622)(855.445,847.409)
\path(1275,894)(1180,794)
\blacken\path(1240.900,901.662)(1180.000,794.000)(1284.400,860.337)(1240.900,901.662)
\put(15,1864){\makebox(0,0)[b]{\smash{{\SetFigFont{10}{12.0}{\rmdefault}{\mddefault}{\updefault}$y_1:I(y_1)$}}}}
\put(2115,1864){\makebox(0,0)[b]{\smash{{\SetFigFont{10}{12.0}{\rmdefault}{\mddefault}{\updefault}$y_n:I(y_n)$}}}}
\put(1065,1864){\makebox(0,0)[b]{\smash{{\SetFigFont{10}{12.0}{\rmdefault}{\mddefault}{\updefault}$\comma$}}}}
\put(1065,514){\makebox(0,0)[b]{\smash{{\SetFigFont{10}{12.0}{\rmdefault}{\mddefault}{\updefault}$x:I(x)$}}}}
\path(15,1789)(1065,664)
\blacken\path(961.190,731.257)(1065.000,664.000)(1005.053,772.196)(961.190,731.257)
\end{picture}
}
\caption{}\label{509-2F6}
\end{figure}

Consider $(S,R,I)$ and consider the situation in Figure \ref{509-2F6}.  The nodes $\{x, y_1\comma y_n\}$ are all from $(S, R)$, where $y_1\comma y_n$ are all the attackers of $x$. The figure shows also the instantiation $I(x)$ and $I(y_1)\comma I(y_n)$.

We need to give meaning to the statement $(\sharp 1)$
$$
I(y_1)\comma I(y_n)\mbox{ attack } I(x)
\eqno (\sharp 1)
$$
The basic meaning of the atomic statement 
$$y_1\comma y_n\mbox{ attack }x
\eqno (\sharp 1 \mbox{ atomic})$$ 
is that 

\begin{center}
$x=1$ (``in'') iff 
all $y_1\comma y_n$ are equal 
0 = (``out'').
\end{center}

We consider several options for understanding such attacks.

\paragraph{View 2 - option 1. The equational approach.}
We follow Strategy 2 here and rewrite $(S,R)$ in a more convenient form, using the structure of the instantiated entities. In this case we simplify/eliminate negation, by moving to a two state networks.  This view was proposed in my paper on the equational approach to contrary to duty obligations \cite{509-1}.  The first section of that paper is general theory and regards Figure \ref{509-2F6} as generating the Boolean (or real valued $[0,1]$) equation 
\[
Eq(x): I(x) \leftrightarrow [\bigwedge^n_{i=1} \neg I(y_i)].
\]
The system of equations $\{Eq(x) | x\in S\}$ may or may not have a solution in the space $\{0, \half, 1\}$.

Any such solution $\Bf$ is considered a complete extension for $(S, R, I)$.

Note that the solution \Bf\ gives values to the atoms of the logic, i.e.\ it is a 3-valued model of propositional logic. These values can be propagated to the formulas of the form $I(x), x \in S$, and the value of $I(x)$ under \Bf\  can be viewed as the argumentation value of $x$ under the complete extension \Bf.

For example, suppose our network contains  Figure \ref{509-2F6} and that  we have $I(y_i) = \neg q$ and $I(x) = q$. Then the equation for node $x$  is $q =\neg q$ and  therefore any solution  \Bf\  to the overall system of equations (for the network in which Figure \ref{509-2F6} resides)  will have to  give  the three valued assignment  $q = \half$  to $q$. The value of $\neg q$ is then $(1- \half) = \half$ and hence  the value of $I(x) =\half$.

Therefore the node  $x$ is considered undecided in $(S,R,I)$ under the solution (complete extension) \Bf.

This is a sweeping general option. Let us see what it does to Figure \ref{509-F1c}.

We have the following equations:
\begin{itemize}
\item $A(J) =1$
\item $A(J)\vee A(M) =\neg A(J)$
\item $A(M)=\neg(A(J)\vee A(M))$
\end{itemize}
There is no solution to these equations. We are not surprised. The equational approach agrees with and generalises the traditional approach and so the inconsistencies and problems remain.

We need the remedies hinted at in Section 1.

\paragraph{View 2 - option 2. the two state equational approach.}
We follow Strategy 2 here and rewrite $(S,R)$ in a more convenient form, using the structure of the instantiated entities. In this case we simplify/eliminate negation, by moving to a two state networks.

This is the approach we adopted in Example \ref{509-E4}, executed within the equational framework.  To show what it does, we modify first Figure \ref{509-F1c} into a two-state associated figure and then use equations. We get Figure \ref{509-2F7}.

\begin{figure}[hb]
\centering
\setlength{\unitlength}{0.00083333in}
\begingroup\makeatletter\ifx\SetFigFont\undefined%
\gdef\SetFigFont#1#2#3#4#5{%
  \reset@font\fontsize{#1}{#2pt}%
  \fontfamily{#3}\fontseries{#4}\fontshape{#5}%
  \selectfont}%
\fi\endgroup%
{\renewcommand{\dashlinestretch}{30}
\begin{picture}(2934,2389)(0,-10)
\put(2892,889){\makebox(0,0)[b]{\smash{{\SetFigFont{10}{12.0}{\rmdefault}{\mddefault}{\updefault}$\neg z: \neg A(M)$}}}}
\drawline(1392,1939)(1392,1939)
\drawline(1392,1939)(1392,1939)
\blacken\path(1422.000,1819.000)(1392.000,1939.000)(1362.000,1819.000)(1422.000,1819.000)
\path(1392,1939)(1392,1339)
\blacken\path(1362.000,1459.000)(1392.000,1339.000)(1422.000,1459.000)(1362.000,1459.000)
\blacken\path(72.000,1969.000)(42.000,2089.000)(12.000,1969.000)(72.000,1969.000)
\path(42,2089)(42,1189)
\blacken\path(12.000,1309.000)(42.000,1189.000)(72.000,1309.000)(12.000,1309.000)
\blacken\path(2922.000,1969.000)(2892.000,2089.000)(2862.000,1969.000)(2922.000,1969.000)
\path(2892,2089)(2892,1189)
\blacken\path(2862.000,1309.000)(2892.000,1189.000)(2922.000,1309.000)(2862.000,1309.000)
\blacken\path(72.000,1819.000)(42.000,1939.000)(12.000,1819.000)(72.000,1819.000)
\path(42,1939)(42,1339)
\blacken\path(12.000,1459.000)(42.000,1339.000)(72.000,1459.000)(12.000,1459.000)
\blacken\path(2922.000,1819.000)(2892.000,1939.000)(2862.000,1819.000)(2922.000,1819.000)
\path(2892,1939)(2892,1339)
\blacken\path(2862.000,1459.000)(2892.000,1339.000)(2922.000,1459.000)(2862.000,1459.000)
\blacken\path(1479.464,1276.464)(1392.000,1189.000)(1510.334,1225.015)(1479.464,1276.464)
\path(1392,1189)(2892,2089)
\blacken\path(2804.536,2001.536)(2892.000,2089.000)(2773.666,2052.985)(2804.536,2001.536)
\blacken\path(1275.513,1230.603)(1392.000,1189.000)(1308.795,1280.526)(1275.513,1230.603)
\path(1392,1189)(42,2089)
\blacken\path(158.487,2047.397)(42.000,2089.000)(125.205,1997.474)(158.487,2047.397)
\blacken\path(1145.248,1324.852)(1262.000,1284.000)(1178.207,1374.988)(1145.248,1324.852)
\path(1262,1284)(182,1994)
\blacken\path(298.752,1953.148)(182.000,1994.000)(265.793,1903.012)(298.752,1953.148)
\blacken\path(1624.464,1361.464)(1537.000,1274.000)(1655.334,1310.015)(1624.464,1361.464)
\path(1537,1274)(2762,2009)
\blacken\path(2674.536,1921.536)(2762.000,2009.000)(2643.666,1972.985)(2674.536,1921.536)
\path(787,204)(27,839)
\blacken\path(138.322,785.081)(27.000,839.000)(99.852,739.037)(138.322,785.081)
\path(342,2314)(792,2314)
\blacken\path(672.000,2284.000)(792.000,2314.000)(672.000,2344.000)(672.000,2284.000)
\path(1992,2314)(2517,2314)
\blacken\path(2397.000,2284.000)(2517.000,2314.000)(2397.000,2344.000)(2397.000,2284.000)
\path(567,2314)(642,2314)
\blacken\path(522.000,2284.000)(642.000,2314.000)(522.000,2344.000)(522.000,2284.000)
\path(2292,2314)(2367,2314)
\blacken\path(2247.000,2284.000)(2367.000,2314.000)(2247.000,2344.000)(2247.000,2284.000)
\put(42,2239){\makebox(0,0)[b]{\smash{{\SetFigFont{10}{12.0}{\rmdefault}{\mddefault}{\updefault}$x:A(J)$}}}}
\put(1392,2239){\makebox(0,0)[b]{\smash{{\SetFigFont{10}{12.0}{\rmdefault}{\mddefault}{\updefault}$y:A(J)\vee A(M)$}}}}
\put(2892,2239){\makebox(0,0)[b]{\smash{{\SetFigFont{10}{12.0}{\rmdefault}{\mddefault}{\updefault}$z:A(M)$}}}}
\put(1392,889){\makebox(0,0)[b]{\smash{{\SetFigFont{10}{12.0}{\rmdefault}{\mddefault}{\updefault}$\neg y: \neg A(J)\wedge\neg A(M)$}}}}
\put(42,889){\makebox(0,0)[b]{\smash{{\SetFigFont{10}{12.0}{\rmdefault}{\mddefault}{\updefault}$\neg x:\neg A(J)$}}}}
\put(792,64){\makebox(0,0)[b]{\smash{{\SetFigFont{10}{12.0}{\rmdefault}{\mddefault}{\updefault}$\top$}}}}
\blacken\path(1422.000,1969.000)(1392.000,2089.000)(1362.000,1969.000)(1422.000,1969.000)
\path(1392,2089)(1392,1189)
\blacken\path(1362.000,1309.000)(1392.000,1189.000)(1422.000,1309.000)(1362.000,1309.000)
\end{picture}
}

\caption{}\label{509-2F7}
\end{figure}

The equations are
\begin{itemize}
\item $\top = 1$
\item $\neg A(J)=\bot\wedge\neg A(J)$
\item $A(J) =\neg\neg A(J)\wedge\neg(\neg A(J)\wedge\neg A(M))$.
\item $\neg A(J) \wedge\neg A(M)=\neg (A(J)\vee A(M))$
\item $A(J)\vee A(M) =\neg A(J)\wedge\neg(\neg A(J)\wedge\neg A(M))$
\item $A(M)=\neg (A(J)\vee A(M))\wedge\neg\neg A(M)\wedge\neg (\neg A()\wedge\neg A(M)$
\item $\neg A(M = \neg A(M)$.
\end{itemize}

There is one solution $A(J) =A(M)=0$.

We need to explain how and why, for example, $\neg y: \neg A(J)\wedge\neg A(M)$ is attacking $x:A(J)$ and $z:A(M)$ and vice versae.  We added these attacks. The reason for this addition should be syntactical, not because we use logic. Formal definitions need to be given for the syntactical pattern matching.

To explain how this is done, let us do again the analysis of Figure \ref{509-F1c}.  This time we write all wffs in disjunctive normal form. We get Figure \ref{509-2F8}.

\begin{figure}[ht]
\centering
\setlength{\unitlength}{0.0007in}
\begingroup\makeatletter\ifx\SetFigFont\undefined%
\gdef\SetFigFont#1#2#3#4#5{%
  \reset@font\fontsize{#1}{#2pt}%
  \fontfamily{#3}\fontseries{#4}\fontshape{#5}%
  \selectfont}%
\fi\endgroup%
{\renewcommand{\dashlinestretch}{30}
\begin{picture}(957,1410)(1000,-10)
\path(915,435)(915,360)
\blacken\path(885.000,480.000)(915.000,360.000)(945.000,480.000)(885.000,480.000)
\path(915,585)(915,210)
\blacken\path(885.000,330.000)(915.000,210.000)(945.000,330.000)(885.000,330.000)
\path(915,1035)(915,960)
\blacken\path(885.000,1080.000)(915.000,960.000)(945.000,1080.000)(885.000,1080.000)
\path(915,1185)(915,810)
\blacken\path(885.000,930.000)(915.000,810.000)(945.000,930.000)(885.000,930.000)
\put(15,1260){\makebox(0,0)[lb]{\smash{{\SetFigFont{10}{12.0}{\rmdefault}{\mddefault}{\updefault}$x: (A(J)\wedge A(M))\vee (A(J)\wedge\neg A(M))$}}}}
\put(15,60){\makebox(0,0)[lb]{\smash{{\SetFigFont{10}{12.0}{\rmdefault}{\mddefault}{\updefault}$z: (A(J)\wedge A(M)) \vee(\neg A(J)\wedge A(M))$}}}}
\put(15,660){\makebox(0,0)[lb]{\smash{{\SetFigFont{10}{12.0}{\rmdefault}{\mddefault}{\updefault}$y: (A(J)\wedge A(M))\vee (A(J) \wedge\neg A(M)) \vee(\neg A(J) \wedge A(M))$}}}}
\end{picture}
}
\caption{}\label{509-2F8}
\end{figure}

Rewriting as a two state network we get Figure \ref{509-2F9}

 \begin{figure}[ht]
\centering
\setlength{\unitlength}{0.00070in}
\begingroup\makeatletter\ifx\SetFigFont\undefined%
\gdef\SetFigFont#1#2#3#4#5{%
  \reset@font\fontsize{#1}{#2pt}%
  \fontfamily{#3}\fontseries{#4}\fontshape{#5}%
  \selectfont}%
\fi\endgroup%
{\renewcommand{\dashlinestretch}{30}
\begin{picture}(1932,2272)(0,-10)
\put(1917,46){\makebox(0,0)[lb]{\smash{{\SetFigFont{8}{9.6}{\rmdefault}{\mddefault}{\updefault}$\neg z:(A(J)\wedge\neg A(M))\vee (\neg A(J)\wedge\neg AM))$}}}}
\blacken\path(1287.000,1501.000)(1167.000,1471.000)(1287.000,1441.000)(1287.000,1501.000)
\path(1167,1471)(1842,1471)
\blacken\path(1722.000,1441.000)(1842.000,1471.000)(1722.000,1501.000)(1722.000,1441.000)
\blacken\path(1287.000,826.000)(1167.000,796.000)(1287.000,766.000)(1287.000,826.000)
\path(1167,796)(1842,796)
\blacken\path(1722.000,766.000)(1842.000,796.000)(1722.000,826.000)(1722.000,766.000)
\blacken\path(1287.000,151.000)(1167.000,121.000)(1287.000,91.000)(1287.000,151.000)
\path(1167,121)(1842,121)
\blacken\path(1722.000,91.000)(1842.000,121.000)(1722.000,151.000)(1722.000,91.000)
\blacken\path(1437.000,1501.000)(1317.000,1471.000)(1437.000,1441.000)(1437.000,1501.000)
\path(1317,1471)(1692,1471)
\blacken\path(1572.000,1441.000)(1692.000,1471.000)(1572.000,1501.000)(1572.000,1441.000)
\blacken\path(1437.000,826.000)(1317.000,796.000)(1437.000,766.000)(1437.000,826.000)
\path(1317,796)(1692,796)
\blacken\path(1572.000,766.000)(1692.000,796.000)(1572.000,826.000)(1572.000,766.000)
\blacken\path(1437.000,151.000)(1317.000,121.000)(1437.000,91.000)(1437.000,151.000)
\path(1317,121)(1692,121)
\blacken\path(1572.000,91.000)(1692.000,121.000)(1572.000,151.000)(1572.000,91.000)
\path(42,1396)(42,871)
\blacken\path(12.000,991.000)(42.000,871.000)(72.000,991.000)(12.000,991.000)
\path(42,1171)(42,1021)
\blacken\path(12.000,1141.000)(42.000,1021.000)(72.000,1141.000)(12.000,1141.000)
\path(292,1781)(172,1746)
\blacken\path(278.800,1808.400)(172.000,1746.000)(295.600,1750.800)(278.800,1808.400)
\put(1092,2146){\makebox(0,0)[b]{\smash{{\SetFigFont{8}{9.6}{\rmdefault}{\mddefault}{\updefault}$\top$}}}}
\put(1092,1471){\makebox(0,0)[rb]{\smash{{\SetFigFont{8}{9.6}{\rmdefault}{\mddefault}{\updefault}$x:(A(J)\wedge A(M))\vee(A(J)\wedge\neg A(M))$}}}}
\put(1092,721){\makebox(0,0)[rb]{\smash{{\SetFigFont{8}{9.6}{\rmdefault}{\mddefault}{\updefault}$y:(A(J)\wedge A(M))\vee (A(J)\wedge \neg A(M))\vee(\neg A(J)\wedge A(M))$}}}}
\put(1092,46){\makebox(0,0)[rb]{\smash{{\SetFigFont{8}{9.6}{\rmdefault}{\mddefault}{\updefault}$z:(A(J)\wedge A(M))\vee(\neg A(J)\wedge A(M))$}}}}
\put(1917,721){\makebox(0,0)[lb]{\smash{{\SetFigFont{8}{9.6}{\rmdefault}{\mddefault}{\updefault}$\neg y: \neg A(K)\wedge\neg A(M)$}}}}
\put(1917,1471){\makebox(0,0)[lb]{\smash{{\SetFigFont{8}{9.6}{\rmdefault}{\mddefault}{\updefault}$\neg x: (\neg A(J)\wedge A(M))\vee (\neg A(K)\wedge\neg A(M))$}}}}
\path(1092,2071)(42,1696)
\blacken\path(144.919,1764.613)(42.000,1696.000)(165.099,1708.108)(144.919,1764.613)
\end{picture}
}
\caption{}\label{509-2F9}
\end{figure}

We can now add attacks to Figure \ref{509-2F9} using pattern recognition.
\begin{itemize}
\item $e_1: \bigvee \alpha_i$ attacks $e_2: \bigvee \beta_j$ if they don't have any $\gamma$ in common, i.e.\ it is not the case that for some $i, j,\gamma=\alpha_i=\beta_j$.
\end{itemize}
We get Figure \ref{509-2F10}. Note that $\top $ is an exception to this rule.

\begin{figure}
\centering
\setlength{\unitlength}{0.00083333in}
\begingroup\makeatletter\ifx\SetFigFont\undefined%
\gdef\SetFigFont#1#2#3#4#5{%
  \reset@font\fontsize{#1}{#2pt}%
  \fontfamily{#3}\fontseries{#4}\fontshape{#5}%
  \selectfont}%
\fi\endgroup%
{\renewcommand{\dashlinestretch}{30}
\begin{picture}(2307,2397)(0,-10)
\put(42,60){\makebox(0,0)[b]{\smash{{\SetFigFont{10}{12.0}{\rmdefault}{\mddefault}{\updefault}$z$}}}}
\path(42,1485)(42,1035)
\blacken\path(12.000,1155.000)(42.000,1035.000)(72.000,1155.000)(12.000,1155.000)
\path(42,735)(42,285)
\blacken\path(12.000,405.000)(42.000,285.000)(72.000,405.000)(12.000,405.000)
\blacken\path(312.000,165.000)(192.000,135.000)(312.000,105.000)(312.000,165.000)
\path(192,135)(2142,135)
\blacken\path(2022.000,105.000)(2142.000,135.000)(2022.000,165.000)(2022.000,105.000)
\blacken\path(312.000,915.000)(192.000,885.000)(312.000,855.000)(312.000,915.000)
\path(192,885)(2142,885)
\blacken\path(2022.000,855.000)(2142.000,885.000)(2022.000,915.000)(2022.000,855.000)
\blacken\path(2019.229,900.077)(2142.000,885.000)(2040.768,956.078)(2019.229,900.077)
\path(2142,885)(192,1635)
\blacken\path(314.771,1619.923)(192.000,1635.000)(293.232,1563.922)(314.771,1619.923)
\blacken\path(312.000,1665.000)(192.000,1635.000)(312.000,1605.000)(312.000,1665.000)
\path(192,1635)(2142,1635)
\blacken\path(2022.000,1605.000)(2142.000,1635.000)(2022.000,1665.000)(2022.000,1605.000)
\blacken\path(2040.768,813.922)(2142.000,885.000)(2019.229,869.923)(2040.768,813.922)
\path(2142,885)(192,135)
\blacken\path(293.232,206.078)(192.000,135.000)(314.771,150.077)(293.232,206.078)
\blacken\path(467.000,165.000)(347.000,135.000)(467.000,105.000)(467.000,165.000)
\path(347,135)(1997,135)
\blacken\path(1877.000,105.000)(1997.000,135.000)(1877.000,165.000)(1877.000,105.000)
\blacken\path(1885.524,749.271)(1987.000,820.000)(1864.178,805.346)(1885.524,749.271)
\path(1987,820)(332,190)
\blacken\path(433.476,260.729)(332.000,190.000)(454.822,204.654)(433.476,260.729)
\blacken\path(476.626,906.457)(357.000,875.000)(477.356,846.461)(476.626,906.457)
\path(357,875)(2002,895)
\blacken\path(1882.374,863.543)(2002.000,895.000)(1881.644,923.539)(1882.374,863.543)
\blacken\path(1864.272,960.426)(1987.000,945.000)(1885.970,1016.365)(1864.272,960.426)
\path(1987,945)(337,1585)
\blacken\path(459.728,1569.574)(337.000,1585.000)(438.030,1513.635)(459.728,1569.574)
\path(1832,1640)(1982,1640)
\blacken\path(1862.000,1610.000)(1982.000,1640.000)(1862.000,1670.000)(1862.000,1610.000)
\path(472,1645)(337,1645)
\blacken\path(457.000,1675.000)(337.000,1645.000)(457.000,1615.000)(457.000,1675.000)
\path(422,1760)(322,1710)
\blacken\path(415.915,1790.498)(322.000,1710.000)(442.748,1736.833)(415.915,1790.498)
\path(42,1290)(37,1200)
\blacken\path(13.703,1321.479)(37.000,1200.000)(73.610,1318.151)(13.703,1321.479)
\path(52,540)(52,450)
\blacken\path(22.000,570.000)(52.000,450.000)(82.000,570.000)(22.000,570.000)
\put(1092,2235){\makebox(0,0)[b]{\smash{{\SetFigFont{10}{12.0}{\rmdefault}{\mddefault}{\updefault}$\top$}}}}
\put(42,1560){\makebox(0,0)[b]{\smash{{\SetFigFont{10}{12.0}{\rmdefault}{\mddefault}{\updefault}$x$}}}}
\put(42,810){\makebox(0,0)[b]{\smash{{\SetFigFont{10}{12.0}{\rmdefault}{\mddefault}{\updefault}$y$}}}}
\put(2292,1560){\makebox(0,0)[lb]{\smash{{\SetFigFont{10}{12.0}{\rmdefault}{\mddefault}{\updefault}$\neg x$}}}}
\put(2292,810){\makebox(0,0)[lb]{\smash{{\SetFigFont{10}{12.0}{\rmdefault}{\mddefault}{\updefault}$\neg y$}}}}
\put(2292,60){\makebox(0,0)[lb]{\smash{{\SetFigFont{10}{12.0}{\rmdefault}{\mddefault}{\updefault}$\neg z$}}}}
\path(1092,2085)(192,1635)
\blacken\path(285.915,1715.498)(192.000,1635.000)(312.748,1661.833)(285.915,1715.498)
\end{picture}
}
\caption{}\label{509-2F10}
\end{figure}

The perceptive reader might complain that we are nevertheless using logic, i.e.\ we are using resolution theorem proving. The answer is that we are not. Resolution is a discipline of sequencing various pattern matchings. Just comparing two disjunctive normal forms does not, in itself, make a resolution theorem prover.

So the steps in View 2, option 2 are as follows:
\begin{enumerate}
\item Start with $(S, R, I)$.
\item Take the associated $(S^*, R^*, I^*)$ explained above (formal definitions to come later).
\item Use the equational approach to find the extensions.
\end{enumerate}

\paragraph{View 3. Direct computation approach.} In this approach we develop  the concept of semantics and extensions directly on the instantiated network by translating it (with the help of additional arguments) into traditional Dung networks or into a modified/generalised such network. This requires as a by product the translation of the entities \BE\ into abstract argumentation, either directly, or indirectly.

It also means that we are turning the instantiation problem into a fibring problem in the sense of \cite{509-3}. 

\subsection{Concrete classical propositional instantiations}
We are going to give a progression of challenges for instantiations from the classical propositional calculus.  Many of these instantiations have been dealt with in the Appendices.  Here we summarise the big picture.

\subsubsection*{Challenge 1: Instantiation with $\top$}
This has been defined and given semantics in Appendix B.

\subsubsection*{Challenge 2: Instantiate with conjunctions of atomic propositions}
Given a traditional network $(S, R)$ we instantiate with a function $I$ defined on $S$, giving for each $x\in S$ a conjunction $\Psi_x$ of atomic propositions of the form $\Psi_x=\bigwedge^{n(x)}_{i=1} q^x_i$.  We also write $\Psi_x$ as a set $\{q^x_1\comma q^x_{n(x)}\}$.  The basic geometrical position we get is as in Figure \ref{509-2F11}, which should be compared with Figures \ref{509-2F6} and \ref{509-BBF3}.

\begin{figure}
\centering
\setlength{\unitlength}{0.00083333in}
\begingroup\makeatletter\ifx\SetFigFont\undefined%
\gdef\SetFigFont#1#2#3#4#5{%
  \reset@font\fontsize{#1}{#2pt}%
  \fontfamily{#3}\fontseries{#4}\fontshape{#5}%
  \selectfont}%
\fi\endgroup%
{\renewcommand{\dashlinestretch}{30}
\begin{picture}(2430,2101)(0,-10)
\put(2415,64){\makebox(0,0)[b]{\smash{{\SetFigFont{10}{12.0}{\rmdefault}{\mddefault}{\updefault}$z_k:\Psi_{z_k}$}}}}
\path(15,1864)(1290,1114)
\blacken\path(1171.357,1148.984)(1290.000,1114.000)(1201.778,1200.700)(1171.357,1148.984)
\path(2415,214)(1290,889)
\blacken\path(1408.334,852.985)(1290.000,889.000)(1377.464,801.536)(1408.334,852.985)
\path(15,214)(1290,889)
\blacken\path(1197.982,806.340)(1290.000,889.000)(1169.909,859.367)(1197.982,806.340)
\path(1060,1249)(1145,1189)
\blacken\path(1029.663,1233.693)(1145.000,1189.000)(1064.264,1282.711)(1029.663,1233.693)
\path(1510,1264)(1425,1219)
\blacken\path(1517.018,1301.660)(1425.000,1219.000)(1545.091,1248.633)(1517.018,1301.660)
\path(1510,754)(1420,809)
\blacken\path(1538.037,772.024)(1420.000,809.000)(1506.750,720.828)(1538.037,772.024)
\path(1055,764)(1155,819)
\blacken\path(1064.312,734.883)(1155.000,819.000)(1035.397,787.456)(1064.312,734.883)
\put(1215,964){\makebox(0,0)[b]{\smash{{\SetFigFont{10}{12.0}{\rmdefault}{\mddefault}{\updefault}$x: \Psi_x$}}}}
\put(15,1939){\makebox(0,0)[b]{\smash{{\SetFigFont{10}{12.0}{\rmdefault}{\mddefault}{\updefault}$y_1:\Psi_{y_1}$}}}}
\put(2415,1939){\makebox(0,0)[b]{\smash{{\SetFigFont{10}{12.0}{\rmdefault}{\mddefault}{\updefault}$y_n:\Psi_{y_n}$}}}}
\put(15,64){\makebox(0,0)[b]{\smash{{\SetFigFont{10}{12.0}{\rmdefault}{\mddefault}{\updefault}$z_1: \Psi_{z_1}$}}}}
\put(1215,64){\makebox(0,0)[b]{\smash{{\SetFigFont{10}{12.0}{\rmdefault}{\mddefault}{\updefault}\ldots}}}}
\put(1215,1939){\makebox(0,0)[b]{\smash{{\SetFigFont{10}{12.0}{\rmdefault}{\mddefault}{\updefault}\ldots}}}}
\path(2415,1864)(1290,1114)
\blacken\path(1373.205,1205.526)(1290.000,1114.000)(1406.487,1155.603)(1373.205,1205.526)
\end{picture}
}
\caption{}\label{509-2F11}
\end{figure}

We offer semantics for the instantiated system $(S, R, I)$ by using instead of $\Psi_z$ the $\BBB\BBF_z$ of Figure \ref{509-BBF5} and Figure \ref{509-BBF14} of Appendix C2.  

We then implement the instantiation of the $\BBB\BBF_z$ (i.e.\ $I'(z) =\BBB\BBF_z, z\in S$) as proposed in Definition \ref{509-BBD19} of Appendix C3.

We adopt the semantics of option (iv) of Appendix B as discussed there.

\subsubsection*{Challenge 3: Instantiating with disjunctions}
When we instantiate with wffs containing disjunctions, we get attacks of the form 
\[
(A\vee B)\tO (C\vee D)
\]
a disjunction attacking another disjunction. The form $A\vee B\tO z$ is equivalent to $A\tO z$ and $B\tO z$. So we need to deal with the case of the form $x\tO C\vee D$ and try and eliminate or implement the disjuction. The semantic consition $x\tO C\vee D$ from the equational point of view is
\[
C\vee D \equiv \neg x
\]
or
\[
\neg C\wedge\neg D\equiv x
\]
or
\[
\neg C\wedge\neg D\equiv\neg(\neg x).
\]

This means that $x\tO C\vee D$ is equivalent to $\neg x\tO \neg C\wedge\neg D$.

We already know how to attack conjunctions from Challenge 2.  So we need to deal with negation.  Once we do that we will be able to deal with full Boolean instantiation.

\subsubsection*{Challenge 4: Instantiating with negated formulas}
Let us start with $(S, R)$. We move to $(S^*, R^*)$ as defined in Definition 
\ref{509-D2} and for and for which Lemma \ref{509-L3} holds.  $(S^*, R^*)$ is a two state network. For every $x\in S^*$, there is a node $\neg x \in S^*$, with $x\tO \neg x$ and $\neg x \tO x$, and where $\neg\neg x$ is $x$.  Thus any instantiation $I(x)=\Phi$ for $x\in S$ becomes the double instantiation $I^*$ on $S^*$ where $I^* (x) =\Phi$ and $I^*(\neg x) =\neg \Phi$. 

We now have a system $(S^*, R^*, I^*)$ of instantiated Boolean net. We use Appendices C2 and C3, Example \ref{509-BBE7} and Definition \ref{509-BBD19} to replace the wff $I(x)$ for $x\in S^*$ with a Boolean formulation $\BBB\BBF_x$ and then instantiate $(S^*, R^*)$ with $x\mapsto \BBB\BBF_x$. The resulting system is a $\top$-net as discussed in Appendix B and we can calculate option (iv) extensions for it.

\begin{remark}\label{509-BBR21}
Let us summarise how to get and what it means to be extensions when we instantiate any network $(S, R)$ with formulas from classical propositional logic, using any of the above challenges (depending on the wffs used in the instantiations).

Let the instantiating function be $I: S\mapsto$ wffs, where wffs belong to a language with atoms $Q$.  We replace $I$ by $I^*$ giving each $x$ a $\BBB\BBF_x$ which basically does the same job (says the same) as $I(x)=\Phi_x$. These $\BBB\BBF$s were done in Appendix C. $\BBB\BBF_x$ contains the atoms appearing in $\Phi_x$ (atoms from $Q$) as well as many auxiliary atoms. 

$(S, R)$ is replaced by $(S^*, R^*)$ being a network of instantiated $\BBB\BBF$s containing the propositional atoms of all the $\Phi$s $\{\Phi_x|x\in S\}$ (i.e.\ all the atoms of $Q$), as well as many auxiliary atoms. We use appendix B to find extensions for $(S^*, R^*)$. These are functions $\lambda$, giving values in $\{0, 1, \half\}$ to all atoms in $(S^*, R^*)$ and thus giving values to all the atoms $Q$ of the propositional language. Once we have values for the atoms we have a 3-valued propositional model and we get values for all the wffs $\Phi_x, x\in S$. Define an extension $\lambda^*$ on $(S, R)$ by $\lambda^*(x) =\lambda (\Phi_x)$. We need to show that this is an extension in the sense of Definition \ref{509-2DM1}. This follows from the way we set up the entire process. It does require proof but I will not do it now.  

We thus got a 3-valued model for the language $Q$ and an extension $\lambda^*$ out of the instantiation $I$ for $(S, R)$.

So given a Boolean instantiation $(S, R, I)$ what does it mean to have an extension for it?

It means a 3-valued model $\lambda$ for the Boolean instantiation language such that $\lambda^*(x)=\lambda(\Phi_x)$ is an extension of $(S, R)$.
\end{remark}

\subsection{Instantiating with  monadic wffs and modal S5 wffs}
Given $(S, R)$ we want to instantiate with wffs $\Phi$ of monadic logic. At the first instance we assume monadic logic with $\{P_1\comma P_n\}$ and we assume that $\Phi$ has no free variables. We make use of Appendix A.

Let $I$ be an instantiation function giving each $x\in S$ a closed formula $I(x)=\Phi_x$ of the monadic predicate logic based on $\{P_1\comma P_n\}$.

We seek extensions for the instantiated system $(S, R, I)$.

We use Lemma \ref{509-AL5}, which says that every wff $\Phi$ without free variables is a Boolean combination of ``atomic'' wffs of the form $q_\vare =\exists x \alpha_\vare(x)$, where $\alpha_\vare(x)$ has the form $\bigwedge^n_{i=1} P^{e_i}_i(x)$ where $\vare=(e_1\comma e_n)\in 2^n$.  

We can thus pretend we are dealing with classical propositional logic with $2^n$ atomic formulas of the form $q_\vare, \vare\in 2^n$.

Associate with each closed formula $\Phi =\Phi(\exists x\alpha_\vare(x))$ the propositional formula 
\[
\Phi^*=\Phi(\exists x\alpha_\vare(x) / q_\vare).
\]

We now instantiate  $(S, R)$ with the formulas $\Phi^*$ instead of $\Phi$. I.e.\ we look at $(S, R, I^*)$ where $I^*(x)=\Phi^*_x$.  Get an extension $\lambda$ on the atoms $\{q_\vare\}$. From it  get an extension $\lambda^*$ on $S$. Also since we have $\lambda$ values for $\{q_\vare\}$ we get values for all the predicate ``atomic" formulas of the form $\exists x \alpha_\vare(x), \vare\in 2^n$ and thus get a 3-valued predicate model ``instantiating'' the net $(S, R, I)$. 

The case where the formula has free variables , we regard them as constants and proceed using Remark \ref{509-AR5}.  The case of modal S5 is treated similarly in view of  Remark \ref{509-AR6}.

\subsection{Beyond predicate instantiation}
Our methodological approach, so far as discussed in Section 1.2, was to use theoretical considerations in expanding our argumentation theory to predicate instantiation. We simply substituted formulas of monadic predicate logic into an abstract argumentation network $(S, R)$ and asked how can we deal with it on theoretical grounds.

We now want to check what kind of predicate networks are required by day to day practical applications. We use Example \ref{509-PE1} as a starting point. This type of examples arise in Talmudic logic, see \cite{509-21}.
\begin{example}\label{509-PE1}
At home I have several sinks and 3 toilets. If a sink is blocked, I can handle it myself. If a toilet is blocked, it is reasonable that I call a plumber.  Two opposing principles come to play here.  A plumber costs money (about US\$ 100 just to visit in addition to any other charges depending on the job). So it makes sense for me to try and do the job myself if I can.  So with a simple case like a blocked sink I can do it myself, but with a blocked toilet I had better call a plumber.

Let us write these rules:
\begin{enumerate}
\item $B(s) \to \neg \exists x P(x,s)$
\item $B(t) \to \exists P(x,t)$
\end{enumerate}
where $B(z)$ reads $z$ is blocked and $P(x,z)$ reads $x$ is called to repair $z$.  The $\to$ is ordinary implication.  Figure \ref{509-PF2} shows these rules in argumentation form.

\begin{figure}
\centering
\setlength{\unitlength}{0.00083333in}
\begingroup\makeatletter\ifx\SetFigFont\undefined%
\gdef\SetFigFont#1#2#3#4#5{%
  \reset@font\fontsize{#1}{#2pt}%
  \fontfamily{#3}\fontseries{#4}\fontshape{#5}%
  \selectfont}%
\fi\endgroup%
{\renewcommand{\dashlinestretch}{30}
\begin{picture}(1584,1410)(0,-10)
\put(42,60){\makebox(0,0)[b]{\smash{{\SetFigFont{10}{12.0}{\rmdefault}{\mddefault}{\updefault}$\exists xP(x,s)$}}}}
\path(1542,1185)(1542,285)
\blacken\path(1512.000,405.000)(1542.000,285.000)(1572.000,405.000)(1512.000,405.000)
\path(42,585)(42,435)
\blacken\path(12.000,555.000)(42.000,435.000)(72.000,555.000)(12.000,555.000)
\path(1542,585)(1542,435)
\blacken\path(1512.000,555.000)(1542.000,435.000)(1572.000,555.000)(1512.000,555.000)
\put(42,1260){\makebox(0,0)[b]{\smash{{\SetFigFont{10}{12.0}{\rmdefault}{\mddefault}{\updefault}$B(s)$}}}}
\put(1542,60){\makebox(0,0)[b]{\smash{{\SetFigFont{10}{12.0}{\rmdefault}{\mddefault}{\updefault}$\neg \exists xP(x,t)$}}}}
\put(1542,1260){\makebox(0,0)[b]{\smash{{\SetFigFont{10}{12.0}{\rmdefault}{\mddefault}{\updefault}$B(t)$}}}}
\path(42,1185)(42,285)
\blacken\path(12.000,405.000)(42.000,285.000)(72.000,405.000)(12.000,405.000)
\end{picture}
}
\caption{}\label{509-PF2}
\end{figure}
Now let us check what happens if both the sink and the toilet are blocked. Common sense dictates, that since I have to call a plumber to do the toilet anyway, I may as well ask him to do the sink. I am paying the \$ 100 for the visit anyway.  In comparison, the strict logical solution to the problem is that I do the sink and the plumber does the toilet!  (This is applying principles (1) and (2).)

Furthermore, if two of my toilets are blocked, the above rules (1) and (2) allow me to called different plumbers, one for each toilet rather than call the same plumber to do all jobs.

So what we need formally in the predicate network is that if $\exists x A(x)$ is ``in'' because of $x=a$, then every other $\exists x D(x)$ which is ``in'' must be instantiated by the same $x=a$!

So let us forget about plumbers and just look at the network of Figure \ref{509-PF3}

\begin{figure}
\centering
\setlength{\unitlength}{0.00083333in}
\begingroup\makeatletter\ifx\SetFigFont\undefined%
\gdef\SetFigFont#1#2#3#4#5{%
  \reset@font\fontsize{#1}{#2pt}%
  \fontfamily{#3}\fontseries{#4}\fontshape{#5}%
  \selectfont}%
\fi\endgroup%
{\renewcommand{\dashlinestretch}{30}
\begin{picture}(1083,1989)(0,-10)
\put(1067,387){\makebox(0,0)[lb]{\smash{{\SetFigFont{10}{12.0}{\rmdefault}{\mddefault}{\updefault}$\neg \exists xD(x)$}}}}
\path(122,1322)(82,1372)
\blacken\path(180.389,1297.037)(82.000,1372.000)(133.537,1259.555)(180.389,1297.037)
\path(987,767)(1027,702)
\blacken\path(938.559,788.476)(1027.000,702.000)(989.658,819.922)(938.559,788.476)
\path(132,117)(82,167)
\blacken\path(188.066,103.360)(82.000,167.000)(145.640,60.934)(188.066,103.360)
\path(17,1662)(18,1663)(21,1665)
	(27,1668)(35,1673)(46,1680)
	(61,1689)(80,1700)(102,1712)
	(126,1727)(154,1742)(183,1759)
	(214,1776)(247,1793)(280,1811)
	(315,1828)(350,1845)(385,1862)
	(421,1878)(458,1893)(495,1907)
	(532,1919)(571,1931)(610,1941)
	(649,1950)(689,1956)(729,1961)
	(767,1962)(809,1960)(847,1954)
	(880,1946)(909,1935)(935,1921)
	(956,1907)(975,1890)(991,1872)
	(1004,1854)(1016,1834)(1026,1814)
	(1034,1793)(1041,1773)(1048,1753)
	(1053,1734)(1057,1717)(1060,1702)
	(1063,1689)(1064,1678)(1067,1662)
\blacken\path(1015.399,1774.416)(1067.000,1662.000)(1074.372,1785.473)(1015.399,1774.416)
\path(12,550)(13,551)(16,553)
	(22,556)(30,561)(41,568)
	(56,577)(75,588)(97,600)
	(121,615)(149,630)(178,647)
	(209,664)(242,681)(275,699)
	(310,716)(345,733)(380,750)
	(416,766)(453,781)(490,795)
	(527,807)(566,819)(605,829)
	(644,838)(684,844)(724,849)
	(762,850)(804,848)(842,842)
	(875,834)(904,823)(930,809)
	(951,795)(970,778)(986,760)
	(999,742)(1011,722)(1021,702)
	(1029,681)(1036,661)(1043,641)
	(1048,622)(1052,605)(1055,590)
	(1058,577)(1059,566)(1062,550)
\blacken\path(1010.399,662.416)(1062.000,550.000)(1069.372,673.473)(1010.399,662.416)
\path(1067,1512)(1066,1511)(1063,1509)
	(1057,1505)(1048,1499)(1036,1490)
	(1021,1480)(1002,1467)(981,1452)
	(956,1436)(929,1418)(900,1400)
	(870,1381)(839,1362)(806,1343)
	(773,1325)(739,1307)(704,1291)
	(669,1275)(632,1260)(594,1247)
	(556,1236)(515,1226)(474,1219)
	(433,1214)(392,1212)(346,1214)
	(303,1221)(266,1231)(232,1244)
	(202,1259)(176,1276)(153,1295)
	(133,1315)(114,1336)(98,1358)
	(83,1381)(69,1403)(57,1425)
	(47,1445)(38,1463)(31,1479)
	(26,1491)(17,1512)
\blacken\path(91.845,1413.520)(17.000,1512.000)(36.696,1389.885)(91.845,1413.520)
\path(1067,312)(1066,311)(1063,309)
	(1057,305)(1048,299)(1036,290)
	(1021,280)(1002,267)(981,252)
	(956,236)(929,218)(900,200)
	(870,181)(839,162)(806,143)
	(773,125)(739,107)(704,91)
	(669,75)(632,60)(594,47)
	(556,36)(515,26)(474,19)
	(433,14)(392,12)(346,14)
	(303,21)(266,31)(232,44)
	(202,59)(176,76)(153,95)
	(133,115)(114,136)(98,158)
	(83,181)(69,203)(57,225)
	(47,245)(38,263)(31,279)
	(26,291)(17,312)
\blacken\path(91.845,213.520)(17.000,312.000)(36.696,189.885)(91.845,213.520)
\put(17,1512){\makebox(0,0)[rb]{\smash{{\SetFigFont{10}{12.0}{\rmdefault}{\mddefault}{\updefault}$\exists xA(x)$}}}}
\put(1067,1512){\makebox(0,0)[lb]{\smash{{\SetFigFont{10}{12.0}{\rmdefault}{\mddefault}{\updefault}$\neg \exists xA(x)$}}}}
\put(17,387){\makebox(0,0)[rb]{\smash{{\SetFigFont{10}{12.0}{\rmdefault}{\mddefault}{\updefault}$\exists xD(x)$}}}}
\path(997,1862)(1042,1812)
\blacken\path(939.425,1881.126)(1042.000,1812.000)(984.023,1921.264)(939.425,1881.126)
\end{picture}
}
\caption{}\label{509-PF3}
\end{figure}

We require that if one of $\exists xA(x)$ or $\exists xD(x)$ is ``in'', the other is also ``in'' and because of the same element $a$.  (This reflects our analysis that if you have many sinks and toilets and you need to call a plumber for one of the toilets then you ask this same plumber to do everything.)
\end{example}

\paragraph{General problem.}
We are given a predicate network $(S, R, I)$ and a subset $S_0\subseteq S$ such that for all $s\in S_0$, $I(s)=\exists x D_s (x)$ (with $D_s$ a unary predicate depending on $s$).  We want to implement the additional constraint
\begin{itemlist}{cccc}
\item [$\BBC$:] If for some $s\in S_0$, $I(s)$ is ``in'' and for some $a, D_s (a)=
\top$ then for all $s\in S_0$ we have that $I(s)=$ ``in'' and furthermore for that same $a$, $D_s(a) =\top$.

\end{itemlist}

Implementing that part of $\BBC$ which says
\begin{quote}
if $I(s_1)$ is ``in'' then also $I(s_2)$ is ``in'' 
\end{quote}
is not simple.  We might think that we can implement that by letting  $I(s_1)$ attack $\neg I(s_2)$. But this is correct only because $I(s_1)$ and $I(s_2)$ are instantiated as existential statements. So we cannot write the attack in the geometry of $(S, R)$. It must occur/activate after the instantiation. How do we do that?

To get a better idea, let us assume that our universe of elements has only two elemnts $\{a,b\}$. This would allow us to rewrite any $\exists x E(x)$ as $E(a)\vee E(b)$, and so Figure \ref{509-PF3} becomes Figure \ref{509-PF4}.

\begin{figure}
\centering
\setlength{\unitlength}{0.00083333in}
\begingroup\makeatletter\ifx\SetFigFont\undefined%
\gdef\SetFigFont#1#2#3#4#5{%
  \reset@font\fontsize{#1}{#2pt}%
  \fontfamily{#3}\fontseries{#4}\fontshape{#5}%
  \selectfont}%
\fi\endgroup%
{\renewcommand{\dashlinestretch}{30}
\begin{picture}(1083,1989)(0,-10)
\put(1067,387){\makebox(0,0)[lb]{\smash{{\SetFigFont{10}{12.0}{\rmdefault}{\mddefault}{\updefault}$\neg D(a)\wedge\neg D(b)$}}}}
\path(122,1322)(82,1372)
\blacken\path(180.389,1297.037)(82.000,1372.000)(133.537,1259.555)(180.389,1297.037)
\path(987,767)(1027,702)
\blacken\path(938.559,788.476)(1027.000,702.000)(989.658,819.922)(938.559,788.476)
\path(132,117)(82,167)
\blacken\path(188.066,103.360)(82.000,167.000)(145.640,60.934)(188.066,103.360)
\path(17,1662)(18,1663)(21,1665)
	(27,1668)(35,1673)(46,1680)
	(61,1689)(80,1700)(102,1712)
	(126,1727)(154,1742)(183,1759)
	(214,1776)(247,1793)(280,1811)
	(315,1828)(350,1845)(385,1862)
	(421,1878)(458,1893)(495,1907)
	(532,1919)(571,1931)(610,1941)
	(649,1950)(689,1956)(729,1961)
	(767,1962)(809,1960)(847,1954)
	(880,1946)(909,1935)(935,1921)
	(956,1907)(975,1890)(991,1872)
	(1004,1854)(1016,1834)(1026,1814)
	(1034,1793)(1041,1773)(1048,1753)
	(1053,1734)(1057,1717)(1060,1702)
	(1063,1689)(1064,1678)(1067,1662)
\blacken\path(1015.399,1774.416)(1067.000,1662.000)(1074.372,1785.473)(1015.399,1774.416)
\path(12,550)(13,551)(16,553)
	(22,556)(30,561)(41,568)
	(56,577)(75,588)(97,600)
	(121,615)(149,630)(178,647)
	(209,664)(242,681)(275,699)
	(310,716)(345,733)(380,750)
	(416,766)(453,781)(490,795)
	(527,807)(566,819)(605,829)
	(644,838)(684,844)(724,849)
	(762,850)(804,848)(842,842)
	(875,834)(904,823)(930,809)
	(951,795)(970,778)(986,760)
	(999,742)(1011,722)(1021,702)
	(1029,681)(1036,661)(1043,641)
	(1048,622)(1052,605)(1055,590)
	(1058,577)(1059,566)(1062,550)
\blacken\path(1010.399,662.416)(1062.000,550.000)(1069.372,673.473)(1010.399,662.416)
\path(1067,1512)(1066,1511)(1063,1509)
	(1057,1505)(1048,1499)(1036,1490)
	(1021,1480)(1002,1467)(981,1452)
	(956,1436)(929,1418)(900,1400)
	(870,1381)(839,1362)(806,1343)
	(773,1325)(739,1307)(704,1291)
	(669,1275)(632,1260)(594,1247)
	(556,1236)(515,1226)(474,1219)
	(433,1214)(392,1212)(346,1214)
	(303,1221)(266,1231)(232,1244)
	(202,1259)(176,1276)(153,1295)
	(133,1315)(114,1336)(98,1358)
	(83,1381)(69,1403)(57,1425)
	(47,1445)(38,1463)(31,1479)
	(26,1491)(17,1512)
\blacken\path(91.845,1413.520)(17.000,1512.000)(36.696,1389.885)(91.845,1413.520)
\path(1067,312)(1066,311)(1063,309)
	(1057,305)(1048,299)(1036,290)
	(1021,280)(1002,267)(981,252)
	(956,236)(929,218)(900,200)
	(870,181)(839,162)(806,143)
	(773,125)(739,107)(704,91)
	(669,75)(632,60)(594,47)
	(556,36)(515,26)(474,19)
	(433,14)(392,12)(346,14)
	(303,21)(266,31)(232,44)
	(202,59)(176,76)(153,95)
	(133,115)(114,136)(98,158)
	(83,181)(69,203)(57,225)
	(47,245)(38,263)(31,279)
	(26,291)(17,312)
\blacken\path(91.845,213.520)(17.000,312.000)(36.696,189.885)(91.845,213.520)
\put(17,1512){\makebox(0,0)[rb]{\smash{{\SetFigFont{10}{12.0}{\rmdefault}{\mddefault}{\updefault}$A(a)\vee A(b)$}}}}
\put(1067,1512){\makebox(0,0)[lb]{\smash{{\SetFigFont{10}{12.0}{\rmdefault}{\mddefault}{\updefault}$\neg A(a)\wedge\neg A(b)$}}}}
\put(17,387){\makebox(0,0)[rb]{\smash{{\SetFigFont{10}{12.0}{\rmdefault}{\mddefault}{\updefault}$D(a)\vee D(b)$}}}}
\path(997,1862)(1042,1812)
\blacken\path(939.425,1881.126)(1042.000,1812.000)(984.023,1921.264)(939.425,1881.126)
\end{picture}
}
\caption{}\label{509-PF4}
\end{figure}

We know how to handle full propositional instantiations from Section 2.2. In this case we have the additional constraints:
\begin{enumerate}
\item $A(a) =1$ iff $D(a)=1$
\item $A(b)=1$ iff $D(b)=1$
\\
In general we want a condition like 
\item $\forall x(A(x)=1\mbox{ iff } D(x) =1)$.
\end{enumerate}
This means we add to Figure \ref{509-PF4} also the attacks of Figure \ref{509-PF4a}.
\begin{figure}
\centering
\setlength{\unitlength}{0.00083333in}
\begingroup\makeatletter\ifx\SetFigFont\undefined%
\gdef\SetFigFont#1#2#3#4#5{%
  \reset@font\fontsize{#1}{#2pt}%
  \fontfamily{#3}\fontseries{#4}\fontshape{#5}%
  \selectfont}%
\fi\endgroup%
{\renewcommand{\dashlinestretch}{30}
\begin{picture}(1083,1989)(0,-10)
\put(1067,387){\makebox(0,0)[lb]{\smash{{\SetFigFont{10}{12.0}{\rmdefault}{\mddefault}{\updefault}$\neg D(b)$}}}}
\path(122,1322)(82,1372)
\blacken\path(180.389,1297.037)(82.000,1372.000)(133.537,1259.555)(180.389,1297.037)
\path(987,767)(1027,702)
\blacken\path(938.559,788.476)(1027.000,702.000)(989.658,819.922)(938.559,788.476)
\path(132,117)(82,167)
\blacken\path(188.066,103.360)(82.000,167.000)(145.640,60.934)(188.066,103.360)
\path(17,1662)(18,1663)(21,1665)
	(27,1668)(35,1673)(46,1680)
	(61,1689)(80,1700)(102,1712)
	(126,1727)(154,1742)(183,1759)
	(214,1776)(247,1793)(280,1811)
	(315,1828)(350,1845)(385,1862)
	(421,1878)(458,1893)(495,1907)
	(532,1919)(571,1931)(610,1941)
	(649,1950)(689,1956)(729,1961)
	(767,1962)(809,1960)(847,1954)
	(880,1946)(909,1935)(935,1921)
	(956,1907)(975,1890)(991,1872)
	(1004,1854)(1016,1834)(1026,1814)
	(1034,1793)(1041,1773)(1048,1753)
	(1053,1734)(1057,1717)(1060,1702)
	(1063,1689)(1064,1678)(1067,1662)
\blacken\path(1015.399,1774.416)(1067.000,1662.000)(1074.372,1785.473)(1015.399,1774.416)
\path(12,550)(13,551)(16,553)
	(22,556)(30,561)(41,568)
	(56,577)(75,588)(97,600)
	(121,615)(149,630)(178,647)
	(209,664)(242,681)(275,699)
	(310,716)(345,733)(380,750)
	(416,766)(453,781)(490,795)
	(527,807)(566,819)(605,829)
	(644,838)(684,844)(724,849)
	(762,850)(804,848)(842,842)
	(875,834)(904,823)(930,809)
	(951,795)(970,778)(986,760)
	(999,742)(1011,722)(1021,702)
	(1029,681)(1036,661)(1043,641)
	(1048,622)(1052,605)(1055,590)
	(1058,577)(1059,566)(1062,550)
\blacken\path(1010.399,662.416)(1062.000,550.000)(1069.372,673.473)(1010.399,662.416)
\path(1067,1512)(1066,1511)(1063,1509)
	(1057,1505)(1048,1499)(1036,1490)
	(1021,1480)(1002,1467)(981,1452)
	(956,1436)(929,1418)(900,1400)
	(870,1381)(839,1362)(806,1343)
	(773,1325)(739,1307)(704,1291)
	(669,1275)(632,1260)(594,1247)
	(556,1236)(515,1226)(474,1219)
	(433,1214)(392,1212)(346,1214)
	(303,1221)(266,1231)(232,1244)
	(202,1259)(176,1276)(153,1295)
	(133,1315)(114,1336)(98,1358)
	(83,1381)(69,1403)(57,1425)
	(47,1445)(38,1463)(31,1479)
	(26,1491)(17,1512)
\blacken\path(91.845,1413.520)(17.000,1512.000)(36.696,1389.885)(91.845,1413.520)
\path(1067,312)(1066,311)(1063,309)
	(1057,305)(1048,299)(1036,290)
	(1021,280)(1002,267)(981,252)
	(956,236)(929,218)(900,200)
	(870,181)(839,162)(806,143)
	(773,125)(739,107)(704,91)
	(669,75)(632,60)(594,47)
	(556,36)(515,26)(474,19)
	(433,14)(392,12)(346,14)
	(303,21)(266,31)(232,44)
	(202,59)(176,76)(153,95)
	(133,115)(114,136)(98,158)
	(83,181)(69,203)(57,225)
	(47,245)(38,263)(31,279)
	(26,291)(17,312)
\blacken\path(91.845,213.520)(17.000,312.000)(36.696,189.885)(91.845,213.520)
\put(17,1512){\makebox(0,0)[rb]{\smash{{\SetFigFont{10}{12.0}{\rmdefault}{\mddefault}{\updefault}$A(a)$}}}}
\put(1067,1512){\makebox(0,0)[lb]{\smash{{\SetFigFont{10}{12.0}{\rmdefault}{\mddefault}{\updefault}$\neg D(a)$}}}}
\put(17,387){\makebox(0,0)[rb]{\smash{{\SetFigFont{10}{12.0}{\rmdefault}{\mddefault}{\updefault}$A(b)$}}}}
\path(997,1862)(1042,1812)
\blacken\path(939.425,1881.126)(1042.000,1812.000)(984.023,1921.264)(939.425,1881.126)
\end{picture}
}
\caption{}\label{509-PF4a}
\end{figure}
Even if we could implement this using various geometrical attacks, we still have two problems:
\begin{itemlist}{cccc}
\item [P1] It all depends on the instantiation.
\item [P2] It applies only to points in $S_0$.
\end{itemlist}
It looks like we need a new fresh breakthrough point of view,  otherwise we will have to restrict the acceptable complete extensions by conditions written in the metalevel and not by the object level properties of the network.

Our problems are not over yet. We saw we needed the condition 
\[
\forall x(A(x)=1\mbox{ iff } D(x)=1).
\]
This means in attack terms
\[
\forall x (A(x)\LtO \neg D(x)).
\]
But how do we implement the general principle like a universal quantification of the form $\forall x (A(x)\LtO \neg D(x)$?

We will have to add free variable attacks to networks as in Figure \ref{509-PF5}.

This is a new twist in our conceptual analysis of what is happening here.

\begin{figure}
\centering
\setlength{\unitlength}{0.00083333in}
\begingroup\makeatletter\ifx\SetFigFont\undefined%
\gdef\SetFigFont#1#2#3#4#5{%
  \reset@font\fontsize{#1}{#2pt}%
  \fontfamily{#3}\fontseries{#4}\fontshape{#5}%
  \selectfont}%
\fi\endgroup%
{\renewcommand{\dashlinestretch}{30}
\begin{picture}(930,226)(0,-10)
\path(15,139)(915,139)
\blacken\path(795.000,109.000)(915.000,139.000)(795.000,169.000)(795.000,109.000)
\path(540,139)(765,139)
\blacken\path(645.000,109.000)(765.000,139.000)(645.000,169.000)(645.000,109.000)
\put(15,64){\makebox(0,0)[rb]{\smash{{\SetFigFont{10}{12.0}{\rmdefault}{\mddefault}{\updefault}$E_1(x)$}}}}
\put(915,64){\makebox(0,0)[lb]{\smash{{\SetFigFont{10}{12.0}{\rmdefault}{\mddefault}{\updefault}$E_2(x)$}}}}
\end{picture}
}

$\alpha(x)$ is a constraint on $x$. Only certain values of $x$ can attack from $E_1$ to $E_2$. Those $x$ which satisfy $\alpha(x)$.
\caption{}\label{509-PF5}
\end{figure}

In Section 2.3 we treated free variables in an argumentation networks as constants. Now we see we need to treat them as parameters, and the attack arrows are annotated by these parameters.

This is a new game to be properly analysed.

\section*{Appendices}
\appendix
\section{Classical monadic predicate logic}
The purpose of this appendix is to show point 3 of Section 1.3, namely that every formula of monadic predicate logic is equivalent to a formulas in a (argumentation friendly) syntactic form, Lemma \ref{509-AL5} below.
\begin{definition}\label{509-AD1}{\ }
\begin{enumerate}
\item The language of $\BBL_n$, of the classical monadic predicate logic without equality has $n$ unary predicates $P_1\comma P_n$, variables the classical connectives $\{\neg, \wedge, \vee, \to\}$ and the quantifiers $\{\forall, \exists\}$ with their usual meaning.
\item A model for $\BBL_n$ has the form $\BBM = (D, D_1\comma D_n)$, where $D$ is a non-empty domain and $D_i\subseteq D$ is the extension of the predicate $P_i$.

\end{enumerate}
We define satisfaction in a model in the traditional way, using the traditional abuse of notation as follows:
\begin{itemize}
\item $\BBM \vDash P_i(d)$ iff $d\in D_i$
\item $\BBM \vDash \neg A$ iff $\BBM \not\vDash A$
\item $\BBM \vDash A\wedge B$ iff $\BBM \vDash A$ and $\BBM\vDash B$
\item $\BBM \vDash A\vee B$ iff $\BBM\vDash A$ or $\BBM\vDash B$
\item $\BBM \vDash A\to B$ iff $\BBM\vDash A$ implies $\BBM\vDash B$
\item $\BBM \vDash \exists x A(x)$ iff for some $d\in D, \BBM\vDash A(d)$
\item $\BBM \vDash \forall x A(x)$ iff for all $d\in D, \BBM\vDash A(d)$.
\end{itemize}
\end{definition}

\begin{lemma}\label{509-AL2}
Let $\BBM$ be a model. Then there exists a model $\BBM^*$ with at most $2^n$ elements which is equivalent to $\BBM$.
\end{lemma}

\begin{proof}
Define $\approx$ on $D$ by
\begin{itemize}
\item $x\approx y$ iff for all $i, P_i(x) \leftrightarrow P_i(y)$.
\end{itemize}
Then $\approx$ is an equivalence relation.  We show that 
$$\begin{array}{l}
\mbox{For every $A(x_1\comma x_m)$, if $x_j \approx y_j$, for $j=1\comma m$}
\\
\mbox{then $\BBM \vDash A(x_1\comma x_m) \leftrightarrow A(y_1\comma y_m)$.}
\end{array}
\eqno (*)
$$
\paragraph{Proof of (*).}
By induction on the syntactical structure of $A$.

\medskip\noindent 
Let $D^*$ be the set of $\approx$ equivalence classes of elements of $D$. Similarly let $D^*_i$ be the set of equivalence classes of elements of $D_i$.

Let $x^*$ be the equivalence class of $x$.  

Let $\BBM^*=(D^*, D^*_1\comma D^*_n)$.

We prove the following.

For any $x_j\in D, \BBM\vDash A(x_j)$ iff $\BBM^* \vDash A(x^*_j)$.

Proof is by induction on $A$.
\end{proof}

$\BBM^*$ has at most $2^n$ elements because there are at most $2^n$ possibilities for $(\pm P_1(x)\comma \pm P_m(x))$ for any single $x$.

\begin{definition}\label{509-AD3}
Two models $\BBM$ and $\BBN$ are said to be equivalent if $\BBM^* = \BBN^*$.
\end{definition}

\begin{definition}\label{509-AD4}{\ }
\begin{enumerate}
\item Let $\vare$ denote a vector of length $n$ of elements $e_i\in \{0,1\}$. Thus 
\[
\vare = (e_1\comma e_n)\in 2^n.\]
$\vare$ is called a type.
\item For any $\vare \in 2^n$, let $\alpha_\vare(x)$ be the wff
\[
\alpha_\vare(x)=\bigwedge^n_{i=1} P^{e_i}_i(x)
\]
where $P^0(x)=\neg P(x)$ and $P^1(x)=P(x)$.  

This means that $x$ is of type $\vare$.
\item Let $\Gamma$ be a non-empty set of types, i.e.\ $\vare$-vectors, $\Gamma\subseteq 2^n$.
\end{enumerate}
\end{definition}

\begin{lemma}\label{509-AL4}
Let $\BBM$ be a model of $\{P_1\comma P_n\}$. Then $\BBM$ is characterised by a formula of the form 
\[
\Phi_{\BBM} =\bigwedge_{\vare\in \Gamma_{\BBM}} \exists x \alpha_\vare (x)\wedge\bigwedge_{\vare\not\in\Gamma_{\BBM}} \neg\exists x \alpha_\vare(x).
\]
In words, $\BBM$ is characterised by the types it realises where $\Gamma_{\BBM}$ is defined as 
\[
\{\vare \mid \mbox{for some } d\in \BBM, \BBM\vDash \alpha_\vare (d)\}.
\]
\end{lemma}

\begin{proof}{\ }
\begin{enumerate}
\item Clearly $\BBM\vDash\Phi_{\BBM}$
\item Let $\BBN$ be a model such that $\BBN\vDash \Phi_{\BBM}$. Since $\Phi_{\BBM}$ says exactly for every $\vare$ whether $\exists x \alpha_\vare(x)$ holds or not, we get that $\BBN$ satisfies the same types as $\BBM$. Therefore we have $\BBM^* =\BBN^*$.
\end{enumerate}
\end{proof}

\begin{lemma}\label{509-AL5}
Let $\Phi$ be any formula of the language of $\{P_1\comma P_n\}$ without free variables. Then $\Phi$ is equivalent to a wff of the form 
\[
\Phi =\bigvee_j \Phi_{\Gamma_j}
\]
where $\Gamma_j \subseteq 2^n$ and $\Phi_{\Gamma_j}$ is defined as follows
\[
\Phi_{\Gamma_j} =\bigwedge_{\vare\in \Gamma_j} \exists x\alpha_\vare(x)\wedge\bigwedge_{\vare\not\in\Gamma_j}\neg \exists x \alpha_\vare (x).
\]
\end{lemma}

\begin{proof}
By Lemma \ref{509-AL2} it is sufficient to consider models of $\Phi$ of less than $2^n$ elements.  Let $\BBM_1\comma \BBM_k$ be all the models of $\Phi$. then the wff $\Phi$ is equivalent to $\bigvee^k_{j=1} \Phi_{\BBM_j}$ which has the same models as $\Phi$.
\end{proof}

\begin{remark}\label{509-AR5}
In case the language contains propositional constants $q_1\comma q_m$ as well as monadic predicates $P_1\comma P_m$ then a model $\BBM$ for this language is characterised by a conjunction of the form 
\[
\Phi_{\BBM} =\bigwedge_{\vare\in \Gamma_{\BBM}} \exists x \alpha_\vare (x)\wedge\bigwedge_{\vare\not\in \Gamma_{\BBM}} \neg\exists x \alpha_\vare(x)\wedge \beta_{\BBM}
\]
where $\Gamma_{\BBM}$ is as in Lemma \ref{509-AL4} and 
\[
\beta_{\BBM} =\bigwedge^m_{j=1} q^{e_j}_j
\]
where $\eta_{\BBM} =(e_1\comma e_m)$ is the vector in $2^m$ of the atoms $q_j$ or their negations which hold in the model $\BBM$.

Therefore any formula $\Phi$ of the language with the constants $\{q_j\}$ is equivalent to a disjunction 
\[
\bigvee_j\Phi_{\Gamma_j}\wedge \beta_j
\]
where $\Phi_{\Gamma_j}$ are as in Lemma \ref{509-AL5} and $\beta_j$ is a formula
\[
\beta_j=\bigwedge_r q^{e^j_r}_r
\]
as discussed above.

If $\Phi$ is a formula with free variables $x_1\comma x_k$, we regard the free variables as constants and regard $P_i(x_j)$ $i=1\comma n, j=1\comma k$ as propositional constants $q_{i,j}=P_i(x_j)$.

We thus can construct an equivalent formula as done above for monadic logic with constants $q_1\comma q_m$, $m=k\times n$.
\end{remark}

\begin{remark}\label{509-AR6}
The connection between monadic predicate logic with $\{P_1\comma P_n\}$ and the modal logic S5 based on the atomic propositions $\{P_1\comma P_n\} $ is well known. For an S5 Kripke model with a set of possible worlds $D$, let $d\vDash P_i$ be interpreted as $P_i(d)$, in the monadic theory based on $D$.

Following this correspondence, let $\beta_\vare$, for $\vare\in 2^n$ be $\lozenge (\bigwedge_i P^{e_i}_i)$ nd let $\beta'_\vare$ be $\bigwedge_j P^{e_i}$ where $\vare=(e_1\comma e_n)$.  We get therefore
\begin{enumerate}
\item [(*)]
Every wff of modal logic S5 with atoms $\{P_1\comma P_n\}$ is equivalent to a wff of the form 
\[
\bigvee_j(\beta'_{\vare_j}\wedge \bigwedge_{\vare\in \Gamma_j}\beta_\vare\wedge\bigwedge_{\vare\not\in \Gamma_j} \neg \beta_\vare)
\]
for some $\Gamma_1\comma \Gamma_k\subseteq 2^n$ and for some $\beta'_{\vare_j}$ in $\Gamma_i$  respectively.
 \end{enumerate}
(*) holds because  an S5 model is a set of worlds  containning the actual world. Each world is characterised by a conjuction of the form $\bigwedge_i P^{e_i}_i$, where  $\vare = (e_1\comma  e_n)$. We can identify a world with $\vare$. A model is characterised by the set of worlds $\Gamma$  it contains in conjuction with the set of worlds  it does not contain (the complement of $\Gamma$). Thus a model is a conjuction   $\beta'_{\vare_i}\wedge\bigwedge_{\vare\in\Gamma}\wedge\bigwedge_{\vare\not\in \Gamma_j}\neg \beta_\vare$).  A formula is equivalent to several sets of worlds, i.e. a disjunction  of several   of the formulas characterising models.
\end{remark}

The above observations will allow us to instantiate Dung argumentation frames with monadic  predicate formulas or with  S5 modal formulas.

\section{Instantiating with $\top$}
This appendix discusses the subtleties of instantiating a single node in argumentation network with just $\top$. We shall see that new concepts of semantics and extensions are required  for the proper handling of this seemingly innocent substitution. 

To show that finding extensions for instantiated Boolean network is not that simple a task, let us start with a  very simple example.  Consider the network $(\{x,y\}, x\tO y\}$.  This has the extension $x=$ `in'' and $y=$ ``out'', i.e.\ $x\wedge\neg y$. Let us instantiate $y=\top$. We get the network of Figure \ref{509-TF2}.  Let us refer to any network with atomic arguments $S$ which also containns $\top$ as $\top$-net.

\begin{figure}
\centering
\setlength{\unitlength}{0.00083333in}
\begingroup\makeatletter\ifx\SetFigFont\undefined%
\gdef\SetFigFont#1#2#3#4#5{%
  \reset@font\fontsize{#1}{#2pt}%
  \fontfamily{#3}\fontseries{#4}\fontshape{#5}%
  \selectfont}%
\fi\endgroup%
{\renewcommand{\dashlinestretch}{30}
\begin{picture}(84,1552)(0,-10)
\path(42,1339)(42,214)
\blacken\path(12.000,334.000)(42.000,214.000)(72.000,334.000)(12.000,334.000)
\path(42,514)(42,364)
\blacken\path(12.000,484.000)(42.000,364.000)(72.000,484.000)(12.000,484.000)
\put(42,1414){\makebox(0,0)[b]{\smash{{\SetFigFont{10}{12.0}{\rmdefault}{\mddefault}{\updefault}$x$}}}}
\put(42,64){\makebox(0,0)[b]{\smash{{\SetFigFont{10}{12.0}{\rmdefault}{\mddefault}{\updefault}$\top$}}}}
\end{picture}
}

\caption{}\label{509-TF2}
\end{figure}

So Figure \ref{509-TF2} is a $\top$-net. We have a problem with this network.   $x$ is not attacked and hence $x=$ ``in''.  But $x$ attacks $\top$ and $\top$ cannot be ``out'', it has to be ``in''.  So what shall we do?

Further reflection shows that the problem is more serious than it seems at first sight. Traditional Dung extensions can be constructed using the geometrical directionality of the attack. Consider Figure \ref{509-TF2a}.

\begin{figure}
\centering
\setlength{\unitlength}{0.00083333in}
\begingroup\makeatletter\ifx\SetFigFont\undefined%
\gdef\SetFigFont#1#2#3#4#5{%
  \reset@font\fontsize{#1}{#2pt}%
  \fontfamily{#3}\fontseries{#4}\fontshape{#5}%
  \selectfont}%
\fi\endgroup%
{\renewcommand{\dashlinestretch}{30}
\begin{picture}(1087,2060)(0,-10)
\put(533,1850){\ellipse{1050}{374}}
\put(554,195){\ellipse{1050}{374}}
\path(533,1662)(533,387)
\blacken\path(503.000,507.000)(533.000,387.000)(563.000,507.000)(503.000,507.000)
\path(533,687)(533,537)
\blacken\path(503.000,657.000)(533.000,537.000)(563.000,657.000)(503.000,657.000)
\put(533,1812){\makebox(0,0)[b]{\smash{{\SetFigFont{10}{12.0}{\rmdefault}{\mddefault}{\updefault}$S_1$}}}}
\put(533,162){\makebox(0,0)[b]{\smash{{\SetFigFont{10}{12.0}{\rmdefault}{\mddefault}{\updefault}$S_2$}}}}
\end{picture}
}
\caption{}\label{509-TF2a}
\end{figure}

We have $S = S_1\cup S_2$ and attacks emanate from $S_1$ into $S_2$ and there are no attacks from $S_2$  into $S_1$. Thus we can find an appropriate inital extension $E_1$ for $S_1$ and propagate the attacks from $E_1$ into $S_2$ to complete the extension into $S_2$ and get a complete extension for $S_1\cup S_2$.  This is the directionality of the attack. We are always guaranteed an extension.

The situation we have now is that with $\top\in S_2$, the directionality no longer exists. Any attack from $S_1$ onto $\top \in S_2$, will force us to reconsider/change the extension $E_1$ of $S_1$. It is as if there are attacks from $S_2$ into $S_1$.

So what are our options in dealing with this?  Let us go back to the network of Figure \ref{509-TF2} and try and apply general principles to it giving us several options:
\paragraph{Option (i): Truth intervention view.}  We can say this network has no extensions. 
This position is perfectly acceptable.  
It is legitimate and reasonable to take a traditional network $(S, R)$, 
pick a $y \in S$ and demand an extension $E$ with $y\in E$. 
This is enforcing truth on $y$. We may find that no such extension can be found. 
So letting $y=\top$ amounts to saying that we want only extensions containing $y$.\footnote{We mention here reference 
\cite{509-23}, where they define the notion of constraint argumentation networks. 
Given a formula $\Phi$ of propositional logic, and a network $(S,R)$, we accept only those extensions  $\lambda$ such that $T_\lambda$ of Section 1.3, $( T_\lambda = \{q|\lambda (q) = 1\} \cup  \{\neg q|\lambda (q) = 0\})$, satisfies $\Phi$.  In our case the formula $\Phi$ is $y=\top$.} 

Let us call this option (i),  the $\top$-intervention approach,  because we are intervening and  forcing some nodes to be true = ``in''.

\paragraph{Option (ii): Counter attack view.}  We can say that any attacker $z$ of $\top$ is immediately attacked back by $\top$, and so the above Figure \ref{509-TF2} is actually Figure \ref{509-TF3}.  The extension is $\top=$ ``in'' and $x=$ ``out''.

\begin{figure}
\centering
\setlength{\unitlength}{0.00083333in}
\begingroup\makeatletter\ifx\SetFigFont\undefined%
\gdef\SetFigFont#1#2#3#4#5{%
  \reset@font\fontsize{#1}{#2pt}%
  \fontfamily{#3}\fontseries{#4}\fontshape{#5}%
  \selectfont}%
\fi\endgroup%
{\renewcommand{\dashlinestretch}{30}
\begin{picture}(1948,934)(0,-10)
\path(1758,782)(1828,722)
\blacken\path(1717.365,777.317)(1828.000,722.000)(1756.413,822.873)(1717.365,777.317)
\path(93,142)(48,242)
\blacken\path(124.601,144.880)(48.000,242.000)(69.886,120.258)(124.601,144.880)
\path(58,607)(59,607)(62,608)
	(66,610)(74,613)(84,617)
	(98,623)(115,630)(136,638)
	(160,647)(188,658)(219,670)
	(252,682)(288,696)(326,710)
	(366,724)(407,738)(450,753)
	(493,767)(538,781)(583,795)
	(629,809)(676,822)(724,834)
	(773,846)(822,857)(873,867)
	(925,877)(979,885)(1033,893)
	(1089,899)(1146,903)(1202,906)
	(1258,907)(1323,906)(1385,901)
	(1441,895)(1493,886)(1540,876)
	(1583,864)(1622,852)(1657,838)
	(1689,823)(1718,807)(1745,790)
	(1770,773)(1792,756)(1813,738)
	(1833,721)(1850,704)(1866,687)
	(1881,672)(1893,657)(1904,645)
	(1913,634)(1920,625)(1925,618)(1933,607)
\blacken\path(1838.157,686.403)(1933.000,607.000)(1886.681,721.693)(1838.157,686.403)
\path(1933,382)(1932,382)(1929,381)
	(1925,379)(1918,376)(1907,373)
	(1893,367)(1876,361)(1855,353)
	(1830,344)(1801,334)(1769,323)
	(1733,310)(1695,296)(1654,282)
	(1611,267)(1565,252)(1518,236)
	(1470,220)(1421,204)(1371,188)
	(1320,173)(1269,157)(1217,142)
	(1166,127)(1114,113)(1061,100)
	(1008,87)(955,74)(902,63)
	(848,52)(793,43)(739,34)
	(684,27)(630,20)(576,16)
	(524,13)(473,12)(417,13)
	(365,17)(318,24)(277,32)
	(241,42)(209,54)(181,66)
	(157,80)(137,95)(120,111)
	(105,128)(93,146)(83,164)
	(75,182)(69,201)(64,220)
	(60,239)(57,258)(56,276)
	(55,293)(54,309)(54,324)
	(55,337)(55,349)(56,359)
	(56,367)(57,373)(58,382)
\blacken\path(74.565,259.421)(58.000,382.000)(14.932,266.047)(74.565,259.421)
\put(58,457){\makebox(0,0)[b]{\smash{{\SetFigFont{10}{12.0}{\rmdefault}{\mddefault}{\updefault}$\top$}}}}
\put(1933,457){\makebox(0,0)[b]{\smash{{\SetFigFont{10}{12.0}{\rmdefault}{\mddefault}{\updefault}$x$}}}}
\end{picture}
}
\caption{}\label{509-TF3}
\end{figure}

The advantage of this view is that the usual traditional machinery for defining and finding complete extensions can be used, with the additional understanding that $\top$ is always ``in'' in any extension.
\paragraph{Option (iii):  New concept of extension view.}  The third option is to give a new definition of abstract networks with truth constant $\top$ as follows:
\begin{enumerate}
\item $(S, R)$ is a network with $\top$ if $\top \in S$ and $\neg \exists y (yR\top)$.
\item A Caminada legitimate $\top$-labelling for a network with $\top$ satisfies the following:
\begin{enumerate}
\item $\lambda (\top) =1$
\item $\lambda (x) =1$ if $\neg \exists y (yRx)$ and $\neg xR\top$.
\item $\lambda (x) =0$ if $xR\top$
\item $\lambda (x) =1$ if for all $y$ such that $yRx$ we have $\lambda (x) =0$ and $\neg xR \top$
\item $\lambda (x) =0$ if $xR\top$ or if for some $y, yRx$ and $\lambda (y) =1$.
\item Otherwise $\lambda (x) =\half$.
\end{enumerate}
\end{enumerate}

There is the question of whether every $\top$-net have an $\top$-extension (i.e.\ a legitimate Caminada $\top$-labelling).  The answer is yes, it does.  Let $(S, R)$ be a $\top$-net.  Let $T =\{y\in S|\top Ry\vee yR\top\}$.  We know all of these points should be out.  Let $*$ be a point not in $S$ and consider 
\[
\begin{array}{l}
S^*=S\cup \{*\}-\{\top\}\\
R^*=(R\upharpoonright S -\{\top\})\cup \{*\} \times  T
\end{array}
\]

In other words, we take $\top$ out and include $*$ which attacks all points in $T$.
$(S^*, R^*)$ is an ordinary network and has extensions.  Let $E$ be such an extension. Then $(E -\{*\})\cup\{\top\}$ is an extension of the $\top$-net $(S, R)$.  So for example the network $(S,R)$  of Figure \ref{509-TF2} becomes the network with $S^* = \{*,x\}$ with   $\{(*,x)\} = R^*$.

\paragraph{Option (iv): The non-toxic truth intervetion  view.} Let us adopt the option (i) view for a given $\top$-net (i.e.\ a network $(S, R)$ with a node $y=\top \in S$) and seek only extensions containing $y=\top$.  However, we shall adopt option (i) not alone on its own but adopt it  in conjunction with another new principle, which we shall call the principle of {\em maximal non-toxic extensions} ({\em NTE principle} in short).

Examples \ref{509-TE4} and \ref{509-TE6} shall explain it.

\begin{example}\label{509-TE4}
Consider Figure \ref{509-TF5}

\begin{figure}
\centering
\setlength{\unitlength}{0.00083333in}
\begingroup\makeatletter\ifx\SetFigFont\undefined%
\gdef\SetFigFont#1#2#3#4#5{%
  \reset@font\fontsize{#1}{#2pt}%
  \fontfamily{#3}\fontseries{#4}\fontshape{#5}%
  \selectfont}%
\fi\endgroup%
{\renewcommand{\dashlinestretch}{30}
\begin{picture}(2566,2203)(0,-10)
\put(2100,1967){\makebox(0,0)[b]{\smash{{\SetFigFont{10}{12.0}{\rmdefault}{\mddefault}{\updefault}$S_1$}}}}
\put(2108,1094){\ellipse{900}{374}}
\put(1208,194){\ellipse{900}{374}}
\put(458,1994){\ellipse{900}{374}}
\path(2100,1817)(2100,1292)
\blacken\path(2070.000,1412.000)(2100.000,1292.000)(2130.000,1412.000)(2070.000,1412.000)
\path(2100,917)(1200,392)
\blacken\path(1288.537,478.378)(1200.000,392.000)(1318.770,426.551)(1288.537,478.378)
\path(450,1817)(1200,392)
\blacken\path(1117.563,484.218)(1200.000,392.000)(1170.658,512.163)(1117.563,484.218)
\path(1080,622)(1140,497)
\blacken\path(1061.027,592.201)(1140.000,497.000)(1115.118,618.165)(1061.027,592.201)
\path(1410,507)(1325,447)
\blacken\path(1405.736,540.711)(1325.000,447.000)(1440.337,491.693)(1405.736,540.711)
\path(2095,1557)(2095,1437)
\blacken\path(2065.000,1557.000)(2095.000,1437.000)(2125.000,1557.000)(2065.000,1557.000)
\put(450,1967){\makebox(0,0)[b]{\smash{{\SetFigFont{10}{12.0}{\rmdefault}{\mddefault}{\updefault}$S_3$}}}}
\put(2100,1067){\makebox(0,0)[b]{\smash{{\SetFigFont{10}{12.0}{\rmdefault}{\mddefault}{\updefault}$S_2$}}}}
\put(1200,167){\makebox(0,0)[b]{\smash{{\SetFigFont{10}{12.0}{\rmdefault}{\mddefault}{\updefault}$S_4$}}}}
\put(2108,1994){\ellipse{900}{374}}
\end{picture}
}
\caption{}\label{509-TF5}
\end{figure}

This figure describes a network with four sub-networks $S = S_1\cup S_2\cup S_3\cup S_4$. The networks $S_i$ are pairwise disjoint and $S_2$ does not attack $S_1$, $S_4$ does not attack $S_2$ nor $S_3$.  $S_2$ however, contains $\top$.  Suppose for some specific extension $E_1$ of $S_1$, it is not possible to extend $E_1$ to an extension for $S_2$.  This makes $S_2$ toxic for $E_1$. $S_2$ forces us to say that $S$ has no extensions $E$ with $E\cap S_1 =E_1$.

However, if we ignore the toxic $S_2$, we may get extensions $E\supseteq E_1$ for $S_1\cup S_3\cup S_4$. It makes sense to do that and present $E$ as a maximal non-toxic extension for $S$ containing $E_1$ for $S_1$.

To motivate the logical sense of doing that, consider a perfectly nice network $(S, R)$ not containing the letter $x$ nor the symbol $\top$.  Add to $S$ the letters $x$ and $\top$ and augment $R$ with $x\tO \top$ (i.e.\ add the toxic Figure \ref{509-TF2} to $(S, R)$).  The resulting network has no exensions because of the toxic part. It does make sense, however, to say that if we ignore the toxic part, we can get extensions for $(S, R)$.

The perceptive reader might ask what if we add to $(S, R)$ a disjoint 3-cycle?  Can we similarly ignore it?  The anser is that we do not need to, because the three cycle does have the empty extension (all undecided) an so the traditional machinery works.

Let us give an example from real life.  Consider a couple going through a divorce in the UK.  They want to settle the financial part amicably and so they go to an accountant.  By UK law, if the accountant learns in the process of advising them of any tax evasion scheme he/she has to report it to the authorities. So the divorcing couple decide that some of their business is  ``toxic'' and better not tell the accountant.

Similarly in a court case both prosecutor and defence lawyers may decide to drop some charges because it is too complicated/toxic for each side to address, each for their own respective reasons.
\end{example}

\begin{example}\label{509-TE6}
Consider the $\top$-net of Figure \ref{509-TF7}.  We want to examine the semantics for it according to option (iv), non-toxic truth intervention.

\begin{figure}
\centering
\setlength{\unitlength}{0.00083333in}
\begingroup\makeatletter\ifx\SetFigFont\undefined%
\gdef\SetFigFont#1#2#3#4#5{%
  \reset@font\fontsize{#1}{#2pt}%
  \fontfamily{#3}\fontseries{#4}\fontshape{#5}%
  \selectfont}%
\fi\endgroup%
{\renewcommand{\dashlinestretch}{30}
\begin{picture}(2320,2114)(0,-10)
\put(2293,1714){\makebox(0,0)[b]{\smash{{\SetFigFont{10}{12.0}{\rmdefault}{\mddefault}{\updefault}$y$}}}}
\path(1918,1489)(1243,214)
\blacken\path(1272.633,334.091)(1243.000,214.000)(1325.660,306.018)(1272.633,334.091)
\path(43,1639)(43,889)
\blacken\path(13.000,1009.000)(43.000,889.000)(73.000,1009.000)(13.000,1009.000)
\path(43,1159)(53,1039)
\blacken\path(13.138,1156.094)(53.000,1039.000)(72.931,1161.077)(13.138,1156.094)
\path(593,1024)(673,1009)
\blacken\path(549.527,1001.628)(673.000,1009.000)(560.584,1060.601)(549.527,1001.628)
\path(233,529)(138,559)
\blacken\path(261.464,551.472)(138.000,559.000)(243.396,494.257)(261.464,551.472)
\path(1078,379)(1143,314)
\blacken\path(1036.934,377.640)(1143.000,314.000)(1079.360,420.066)(1036.934,377.640)
\path(1358,429)(1323,364)
\blacken\path(1353.478,483.880)(1323.000,364.000)(1406.306,455.434)(1353.478,483.880)
\path(1758,1534)(1668,1569)
\blacken\path(1790.714,1553.467)(1668.000,1569.000)(1768.967,1497.546)(1790.714,1553.467)
\path(2118,2039)(2198,1999)
\blacken\path(2077.252,2025.833)(2198.000,1999.000)(2104.085,2079.498)(2077.252,2025.833)
\path(43,889)(45,890)(49,892)
	(56,896)(67,902)(82,910)
	(100,920)(121,931)(144,942)
	(170,955)(196,967)(224,979)
	(253,990)(283,1001)(314,1011)
	(347,1020)(381,1028)(418,1034)
	(455,1038)(493,1039)(529,1038)
	(563,1034)(593,1028)(620,1020)
	(644,1011)(665,1001)(684,990)
	(702,979)(718,967)(732,955)
	(745,942)(757,931)(767,920)
	(775,910)(782,902)(793,889)
\blacken\path(692.585,961.228)(793.000,889.000)(738.389,999.985)(692.585,961.228)
\path(793,664)(791,663)(787,661)
	(780,657)(769,651)(754,643)
	(736,633)(715,622)(692,611)
	(666,598)(640,586)(612,574)
	(583,563)(553,552)(522,542)
	(489,533)(455,525)(418,519)
	(381,515)(343,514)(307,515)
	(273,519)(243,525)(216,533)
	(192,542)(171,552)(152,563)
	(134,574)(118,586)(104,598)
	(91,611)(79,622)(69,633)
	(61,643)(54,651)(43,664)
\blacken\path(143.415,591.772)(43.000,664.000)(97.611,553.015)(143.415,591.772)
\path(1558,1894)(1560,1895)(1564,1897)
	(1571,1901)(1582,1907)(1597,1915)
	(1615,1925)(1636,1936)(1659,1947)
	(1685,1960)(1711,1972)(1739,1984)
	(1768,1995)(1798,2006)(1829,2016)
	(1862,2025)(1896,2033)(1933,2039)
	(1970,2043)(2008,2044)(2044,2043)
	(2078,2039)(2108,2033)(2135,2025)
	(2159,2016)(2180,2006)(2199,1995)
	(2217,1984)(2233,1972)(2247,1960)
	(2260,1947)(2272,1936)(2282,1925)
	(2290,1915)(2297,1907)(2308,1894)
\blacken\path(2207.585,1966.228)(2308.000,1894.000)(2253.389,2004.985)(2207.585,1966.228)
\path(2308,1669)(2306,1668)(2302,1666)
	(2295,1662)(2284,1656)(2269,1648)
	(2251,1638)(2230,1627)(2207,1616)
	(2181,1603)(2155,1591)(2127,1579)
	(2098,1568)(2068,1557)(2037,1547)
	(2004,1538)(1970,1530)(1933,1524)
	(1896,1520)(1858,1519)(1822,1520)
	(1788,1524)(1758,1530)(1731,1538)
	(1707,1547)(1686,1557)(1667,1568)
	(1649,1579)(1633,1591)(1619,1603)
	(1606,1616)(1594,1627)(1584,1638)
	(1576,1648)(1569,1656)(1558,1669)
\blacken\path(1658.415,1596.772)(1558.000,1669.000)(1612.611,1558.015)(1658.415,1596.772)
\put(43,739){\makebox(0,0)[b]{\smash{{\SetFigFont{10}{12.0}{\rmdefault}{\mddefault}{\updefault}$a$}}}}
\put(793,739){\makebox(0,0)[b]{\smash{{\SetFigFont{10}{12.0}{\rmdefault}{\mddefault}{\updefault}$b$}}}}
\put(43,1714){\makebox(0,0)[b]{\smash{{\SetFigFont{10}{12.0}{\rmdefault}{\mddefault}{\updefault}$x$}}}}
\put(1243,64){\makebox(0,0)[b]{\smash{{\SetFigFont{10}{12.0}{\rmdefault}{\mddefault}{\updefault}$\top$}}}}
\put(1543,1714){\makebox(0,0)[b]{\smash{{\SetFigFont{10}{12.0}{\rmdefault}{\mddefault}{\updefault}$z$}}}}
\path(793,664)(1243,214)
\blacken\path(1136.934,277.640)(1243.000,214.000)(1179.360,320.066)(1136.934,277.640)
\end{picture}
}

\caption{}\label{509-TF7}
\end{figure}

This figure has the form of Figure \ref{509-TF5} with $S_1 =\{x\}, S_2=\{a,b\}, S_3=\{z, y\}$ and $S_4=\{\top\}$.

$S_1$ has one extension $\lambda^1_1$ with $\lambda ^1_1(x)=1$.  This can be extended to $S_2$ in only one way, namely $\lambda^2_1$ with $\lambda^2_1(a) =0$ and $\lambda ^2_1(b)=1$.

However, because of $\top \in S_4$, we cannot have $\lambda^2_1(b)=1$, because it should be 0.  Thus $S_4=\{\top\}$ is toxic for $\lambda^1_1$ and $\lambda^2_1$. So we abandon $S_4$ for  this sequence, (i.e.\ $S_1\tO S_2 \tO S_4$) and get the extension $\lambda^{1,2} =\lambda^1_1\cup \lambda^2_1$, with $\lambda^{1,2}(x)=1,\lambda^{1,2} (a)=0$ and $\lambda^{2,1}(b)=1$.  We ignore $\top$.

Now let us look at $S_3$.  It has two extensions $\lambda^3_1,\lambda^3_2$ with $\lambda^3_1(z)=1$ and $\lambda^3_1(y)=0$ and $\lambda^3_2(z)=0$ and $\lambda^3_2(y)=1$.  The extension $\lambda^3_1$ is not possible because $z\tO \top$.  So we have only the extension $\lambda^3_2$.

The final extension for the $\top$-net of Figure \ref{509-TF7} is $\lambda =\lambda^1_1\cup \lambda^2_1\cup \lambda^3_2$ namely 
\[
\lambda(x)=1, \lambda(a)=0, \lambda(b)=1, \lambda(z)=0 \mbox{ and }\lambda(y)=1.
\]
Note that we ignore $S_4=\{\top\}$ only for the evaluation of the path $S_1\tO S_2 \tO S_4$. 

For the path $S_3\tO S_4$ we do not ignore $S_4=\{\top\}$ because there is an extension for $S_3$ which is OK.

Thus the following $\lambda'$ is {\em not} an extension for $S_1\cup S_2\cup S_3$: 
\[
\lambda'(x)=1, \lambda'(a)=0, \lambda'(b)=1, \lambda'(z)=1 \mbox{ and } \lambda'(y)=0.
\]

We can achieve the same result if we say that we abandon the attack $b\tO \top$, rather than the set $\{\top\}$. This view is better because we are not ``touching'' the toxic part in some cases, rather than deleting it.

The next example \ref{509-TE6a} shows how to do this.
\end{example}

\begin{example}\label{509-TE6a}
Let us do the option (iv) semantics for Figure \ref{509-TF7} in a different way from the way done in Example \ref{509-TE6}.  

We proceed as follows:
\begin{enumerate}
\item We are given a $\top$-net $(S, R)$ with $\top \in S$. In this case it is the net of Figure \ref{509-TF7}.
\item Replace every occurrence of $\top \in S$ by a new node letter $\tau$, and get $(S_\tau, R_\tau)$, where
\[\begin{array}{rcl}
S_\tau &=& (S-\{\top\})\cup \{\tau\}\\
R_\tau &=& (R-\{(x,\top), (\top, y) | (x, \top), (\top, y)\in R\})\cup \{(x, \tau), (\tau, y) | (x, \top), (\top, y)\in R\}.
\end{array}\]
\item $(S_\tau, R_\tau)$ is a traditional network and so we seek and find all of its complete extensions (Caminada labelling) of the form $\lambda_\tau$ for which $\lambda_\tau (\tau) =1$.

If such extensions exist then for each $\lambda_\tau$ let $\lambda_\top$ be the function on $S$ obtained by 
\[\begin{array}{l}
\lambda_\top (x) =\lambda_\tau(x),\mbox{ for } x\neq \tau\\
\lambda_\top(\top) =1.
\end{array}
\]
These will be the option (iv) extensions of Figure \ref{509-TF7}.
\item If $(S_\tau, R_\tau)$ has no exensions in $\lambda$ in which $\lambda (\tau)=1$, then let $\lambda_1\comma \lambda_k$ enumerate all the extensions  of $(S_\tau, R_\tau)$ which give $\tau$ a value  $\neq 1$.

Such extensions exist because we are dealing with a traditional network.   Let us enumerate these extensions for our example (of Figure \ref{509-TF7} with $\tau$ replacing $\top$).

We have two such extesions, $\lambda_1$ and $\lambda_2$:
\[\begin{array}{ll}
\lambda_1: & x=1, a=0, b=1, \tau=0, z=1, y=0\\
\lambda_2: &x=1, a=0, b=1,\tau=0, z=0, y=1.
\end{array}
\]
\item We know that the requirement that $\lambda(\tau)=1$ is toxic, because of the attacks $b\tO \tau$ and $z\tO \tau$.  Our remedy is to disconnect some of these attacks in order to get extensions. If we disconnect them all we certainly get such extensions, but we want to accommodate the attacks as much as we can.

Our possibilities are the following:
\begin{enumerate}
\item Disconnect $z\tO \tau$
\item Disconnect $b\tO \tau$
\item Disconnect both $z\tO \tau$ and $b\tO \tau$
\end{enumerate}
We give preference to disconnecting attacks emanating from points nearer the bottom (away from the top) of the figure.  Thus (b) gets priority over (a).

We also want to minimise the number of changes and so (c) has least priority.  

Disconnecting $b\tO \tau$ gives us the extension 
\[\begin{array}{ll}
\lambda: & x=1, a=0, b=1, \tau=1, z=0 \mbox{ and } y=1.
\end{array}
\]

This yields a complete non-toxic extension for Figure \ref{509-TF7}.
\item Note that we may wish to take into account the value of $x$ under $\lambda$  in our consideration of whether to disconnect an attack of $x\tO \tau$.  $\lambda(x)$ can be 1 or $\half$ and we may decide not to fix $\lambda$ (by disconnecting $x\tO \tau$) if $\lambda(x) =1$.   If we do indeed make this decision then we would not disconnect $b\tO \tau$  in our example and in such a case the network of Figure \ref{509-TF7} will have no extensions!
\end{enumerate}
\end{example}

\begin{remark}\label{509-TR6}
Example \ref{509-TE6} also shows that the result (semantics) we get for this option (iv): non-toxic truth intervention, is different from the result we get from option (ii), the counter attack view. According to option (ii), $\top$ counter attacks all its attackers and so we get the extension $\lambda_c$
\[
\lambda_c(x)=1, \lambda_c(a)=0, \lambda_c(b)=0, \lambda_c(\top)=1,\lambda_c(z)=0 \mbox{ and } \lambda_c(y)=1.
\]
\end{remark}

\begin{remark}\label{509-TR8}
The perceptive reader may ask why are we even considering the idea of using maximal non-toxic extensions?  After all, did we not say that instantiating $(S, R)$ for $y\in S$ with $y=\top$ amounts to looking for extensions in which $y=$ ``in''. So if there are no such extensions, then the straightforward answer is that there are no extensions!  Why suddenly claim $y=\top$ is toxic and let us ignore it?  On the one hand, we are interested in $y$ and want it ``in'' and on the other hand when we cannot do that we throw out of the network with that very same $y$, calling it toxic!

The answer is that there are cases of networks with $y$ where we are not looking at the instantiation $y=\top$ as a search for extensions in which $y=$ ``in''.  We are looking in such a network at $y=\top \in S$ as an instrument of intervention of forcing an ``in'' value of some other node $z$ related to $y$ at the object level!  Since we have a purpose in this case, then  if our instrument does not work we need to find an alternative way.\footnote{The Thesis \cite{509-25} and the paper \cite{509-24}, deal with intervention. Suppose we have a network $(S,R)$, and an element $a$ in $S$. We want to intervene and force $a$ to have value $e \in \{0,\half ,1\}$. We can do that by adding to $(S,R)$ some new point $x$ with the following attack pattern:	
\begin{enumerate}
\item  For forcing $a$ to be ``out'', let $x$ attack $a$.	
\item  For forcing $a$ to be ``undecided'' let $x$ attack $a$ and attack itself.	
\item  For forcing $a$ to be ``in'' (if possible), let $a$ attack $x$ and let $x$ attack $\top$ (we have to add $\top$ as well as $x$).
\end{enumerate}
}

Consider the network $(S, R)$ of Figure \ref{509-TF9}.  Let $(S', R')$ be the network obtained  from $(S,R$) by deleting the nodes $\{e, \top\}$. Consider an intervention into $(S',R')$  intending to enforce the node  $x$ to be ``in". The nodes $\{e,\top\}$ added to $(S', R')$ with $e\tO\top$ are intended to achieve this purpose, in Figure  \ref{509-TF9}.

\begin{figure}
\centering
\setlength{\unitlength}{0.00083333in}
\begingroup\makeatletter\ifx\SetFigFont\undefined%
\gdef\SetFigFont#1#2#3#4#5{%
  \reset@font\fontsize{#1}{#2pt}%
  \fontfamily{#3}\fontseries{#4}\fontshape{#5}%
  \selectfont}%
\fi\endgroup%
{\renewcommand{\dashlinestretch}{30}
\begin{picture}(3969,2641)(0,-10)
\put(1679,432){\makebox(0,0)[b]{\smash{{\SetFigFont{10}{12.0}{\rmdefault}{\mddefault}{\updefault}$S'$}}}}
\put(1010,509){\ellipse{824}{226}}
\path(2449,2062)(2899,1312)
\blacken\path(2811.536,1399.464)(2899.000,1312.000)(2862.985,1430.334)(2811.536,1399.464)
\path(949,627)(949,1012)
\blacken\path(979.000,892.000)(949.000,1012.000)(919.000,892.000)(979.000,892.000)
\path(949,1237)(949,1987)
\blacken\path(979.000,1867.000)(949.000,1987.000)(919.000,1867.000)(979.000,1867.000)
\path(1324,2137)(2299,2137)
\blacken\path(2179.000,2107.000)(2299.000,2137.000)(2179.000,2167.000)(2179.000,2107.000)
\path(2074,2137)(2149,2137)
\blacken\path(2029.000,2107.000)(2149.000,2137.000)(2029.000,2167.000)(2029.000,2107.000)
\path(949,1707)(954,1842)
\blacken\path(979.538,1720.972)(954.000,1842.000)(919.579,1723.193)(979.538,1720.972)
\path(949,797)(949,882)
\blacken\path(979.000,762.000)(949.000,882.000)(919.000,762.000)(979.000,762.000)
\path(2769,1542)(2834,1427)
\blacken\path(2748.836,1516.706)(2834.000,1427.000)(2801.070,1546.229)(2748.836,1516.706)
\put(949,2047){\makebox(0,0)[b]{\smash{{\SetFigFont{10}{12.0}{\rmdefault}{\mddefault}{\updefault}${x}$}}}}
\put(2459,2082){\makebox(0,0)[b]{\smash{{\SetFigFont{10}{12.0}{\rmdefault}{\mddefault}{\updefault}$e$}}}}
\put(2914,1162){\makebox(0,0)[b]{\smash{{\SetFigFont{10}{12.0}{\rmdefault}{\mddefault}{\updefault}$\top$}}}}
\put(3954,2147){\makebox(0,0)[b]{\smash{{\SetFigFont{10}{12.0}{\rmdefault}{\mddefault}{\updefault}$S$}}}}
\put(949,1087){\makebox(0,0)[b]{\smash{{\SetFigFont{10}{12.0}{\rmdefault}{\mddefault}{\updefault}$u$}}}}
\put(1803,1313){\ellipse{3590}{2610}}
\end{picture}
}
\caption{}\label{509-TF9}
\end{figure}

In this figure, since $e$ attacks $\top$ it must be ``out'', not because it attacks $\top$ but because some node with value ``in'' is attacking it.  So $x$ must be ``in'', because it is the only the attacker of $e$.    We have used $\top$ and $e$ in $(S, R)$  to force $x$ to be `in'' by the object level geometry of $(S, R)$. This means that only complete extensions of $(S',R')$ with $x=$ ``in" are acceptable. If, because of this intervention, we cannot have an extension, then we can say that our attempt (intervention) fails, or in our terminology, is toxic, and consider giving it up, and we need to look for alternatives. Whatever reason we had for wanting $x$ to be ``in", must now be serviced by defining another complete extension in some other way. The way we define the extension depends on the purpose. Our purpose in this paper is connected with the soundness of the attack formations to be presented in Appendix C.2. So our definitions lead towards that purpose. 

We are dealing here with a new type of instrument for defining extensions in the case where the present of $\top$ is toxic!
\end{remark}

We are going to need to define priority on the set of elements attacking $\top$. This will be done in terms of their distance from the top of the network. To do this we need to use the notions of Strongly Connected Components.  The next series of definitions (Definition  \ref{509-TD10} to Definition \ref{509-TD12}) deals with this.

\begin{definition}\label{509-TD10}
Let $(S, R)$ be an argumentation network. We define $(S^*, R^*)$ the network of {\em strongly connected components} (SCC) derived from $(S, R)$, following \cite{509-19}.
\begin{enumerate}
\item A subset $E\subseteq S$ is an SCC iff the following holds:
\begin{enumerate}
\item $S\neq \varnothing$
\item For any $x, y\in S$, there exists a sequence $z_1\comma z_{k+1}$ such that $z_1=x, z_{k+1}=y$ and for $1 \leq i \leq k$ we have $z_iRz_{i+1}$ holds.
\item $E$ is maximal w.r.t. property (b).
\end{enumerate}
\item Let $S^*$ be the set of all SCC subsets of $S$. Define $E_1 R^* E_2$ on $S^*$ iff for some $x_1\in E_1$ and $x_2\in E_2$ we have $x_1Rx_2$.
\item For $x\in S$, let $x^*$ be the SCC to which it belongs.
\end{enumerate}
\end{definition}

\begin{lemma}\label{509-TL11}
Let $(S, R)$ be a network and let $(S^*, R^*)$ be its associated SCC network. Then 
\begin{enumerate}
\item Any two distinct SCC sets are disjoint.
\item For any $x\in S$, there is a unique $x^*$ for which it belongs.
\item $R^*$ is well defined on $S^*$ and is acyclic.
\end{enumerate}
\end{lemma}
\begin{proof}
Easy. See \cite{509-19}.
\end{proof}

\begin{definition}\label{509-TD12}
Let $(S, R)$ be a network and let $(S^*, R^*)$ be its derived SCC network. Let $E\in S^*$. We define the notion of ``$E$ is of level $(k,n)$'' as follows:
\begin{enumerate}
\item $E$ is of level (1,1) if there does not exist an $E'\in S^*$ such that $E'R^*E$.   Think of the level index $(k,n)$ as $k$ is the minimal $R^*$ distance from the top nodes in $(S^*, R^*)$ and $n$ is the maximal distance. The top nodes are distance 1.
\item $E$ is of level $(k+1, n+1)$ if $k$ is the minimal $m$ of the level $(m,n)$ of any $E'$ such that $E'R^* E$, and $n$ is the maximal such $n$.
\item Since each $x\in S$ is a member of a unique $E$, we can define a level $(k,n)$ for each $x\in S$. It is the level of the $E$ containing it.
\end{enumerate}
\end{definition}

\begin{example}\label{509-TE13}
Consider Figure \ref{509-TF7}.  Then $\{x\}$ and $\{z,y\}$ are of level (1,1). $\{a,b\}$ is of level (2,2) and $\{\top\}$ is also of level (2,3).
\end{example}

\begin{definition}\label{509-TD14}
Let $(S, R)$ be a $\top$-net, that is $(S, R)$ is an argumentation network with a special $\top \in S$, with $\top$ not attacking itself. We want to define the non-toxic truth intervention semantics for it (option (iv) semantics).  

We are going to give the algorithm for finding all complete extensions for $(S, R)$ in the form of Caminada labellings $\lambda$, $\lambda: S\mapsto \{0, \half 1\}$, with $\lambda (\top) =1$.

\paragraph{Step 1:} Let $\tau$ be a letter disjoint from $\top$.  Let $(S_\tau, R_\tau)$ be $(S(\top/\tau), R(\top/\tau))$, where $A(x/y)$ is the result of substituting $y$ in all occurrences of $x$ in $A(x)$, where $y$ is a completely new letter to $A$.
\paragraph{Step 2:}  $(S_\tau, R_\tau)$ is a traditional argumentation network and has complete extensions. Let $\Lambda$ be the set of all such extensions. This set is non-empty. Let $\Lambda_\tau$ be the subset of all extensions $\lambda\in\Lambda$ such that $\lambda(\tau) =1$. $\Lambda_\tau$ may be empty.

For each $\lambda\in\Lambda_\tau$, let $\lambda_\top$ be defined by $\lambda_\top(x) =\lambda(x)$, for $x\neq \tau, \lambda_\top(\top)=\lambda(\tau)=1$.

\paragraph{Step 3:}  If $\Lambda_\tau \neq\varnothing$ then let $\lambda_\top =\{\lambda_\top |\lambda \in\Lambda_\tau\}$ be declared as the set of all the extensions of $(S, R)$. If $\Lambda_\tau =\varnothing$, then proceed to step 4.
\paragraph{Step 4:}  Let $T_0$ be the set of all $x\in S$ such that $x$ attacks $\top$, i.e.\ $x\tO \top$ is in $R$. Let $T_1$ be the set of all $y\in S$ such that $\top\tO y \in R$.

We want to assume that $T_1=\varnothing$. This is possible to do because we can move to the network $(S_\infty,R_\infty)$ where $S_\infty = S\cup \{\infty\}$, where $\infty$ is a new point such that $\infty\not\in S$ and $R_\infty = (R-\{\top \tO y|y\in T_1\})\cup \{\infty\tO y | y \in T_1\}$.

In other words, we ensure that $y \in T_1$ ends up ``out'' because it is attacked by $\infty$ which is ``in''  (not being attacked by anything) and we disconnect any attacks emanating from $\top$.  Any extension $\lambda$ found for $(S_\infty, R_\infty)$ will yield an extension for $(S, R)$, when restricted to $S$.

So we can assume now that $\top$ does not attack anything. We now proceed to find extensions for $(S, R)$.  Each such $x\in S$ has a unique level $(k, n)$ associated with it as defined in Definition \ref{509-TD12}.

Define a priority $\geq$ ordering on nodes $x, y \in S$ by
\begin{itemize}
\item $x\geq y$ iff $n_x\geq n_y$ and if $n_x=n_y$ then $k_x\geqq k_y$
\end{itemize}
where the level of $x$ is $(k_x, n_x)$ and the level of $y$ is $(n_y, k_y)$.
\paragraph{Step 5:}  We assume that we have $(S, R)$, where $\top \in S, \top$ does not attack anything, and when we look at $(S_\tau, R_\tau)$ then all of its extensions $\lambda\in \Lambda$ satisfy $\lambda (\tau)\neq 1$, i.e.\ $\lambda (\tau)=\half$ or =1.

Rather than declare  that $(S, R)$ has no extnesions, we want to salvage some extensions out of $\Lambda$ by declaring the role of $\top \in S$ in some cases to be toxic!

We can now consider $T_0$.  $T_0$ can be divided for each $\lambda \in \Lambda$, into $T_0 =T_{0,\half}^\lambda\cup T_{0,1}^\lambda\cup T^\lambda_{0,0}$ where
\[\begin{array}{l}
T_{0,\half}^\lambda=\{x\in T_0|\lambda (x) =\half\}\\
T_{0,1}^\lambda=\{x\in T_0|\lambda (x) =1\}\\
T_{0,0}^\lambda=\{x\in T_0|\lambda (x) =0\}\\
\end{array}\]
We know for sure that either $T_{0,1}^\lambda \neq \varnothing$ or if $T_{0,\half}^\lambda \neq\varnothing$.

Let us adopt the policy that if for some $x$ attacking $\top$ we have $\lambda (x) =1$ then we give up on $\lambda$.\footnote{Note that in examples \ref{509-TE6} and \ref{509-TE6a} we did fix the net of Figure \ref{509-TF7}.  According to our present policy, we would not do that and declare the network of Figure \ref{509-TF7} as having no extensions.  On the other hand, if we wish to always have complete extensions to any $\top$-net $(S, R$) then we ignore $T_{\lambda,\half}$ and proceed to fix $\lambda$.}

So we are considering the case where all attackers $x$ of $\top$ have value $\lambda(x)=0$ or $\lambda (x)=\half$.

If we disconnect the attacks emanating from $T_{0,\half}^\lambda$ onto $\top$ we get an extension $\lambda$ for $(S, R_\lambda)$ where
\[
R_\lambda = R-\{x\tO \top | x\in T_{0,\half}^\lambda\}.
\]
Now should we choose $\lambda$ as a non-toxic extension for $(S, R)$?  It is a question of priority.  We have a priority relation $>$ on $T_0$.  Let us extend it to priority on sets $T_{0,\half}^\lambda$ by a lexicographic ordering first on the number of elements of the set $T_{0,\half}^\lambda$ and second on the max on the value $n_x$ of the index $(k_x, n_x)$ of $x\in T^\lambda_{0,\half}$.  The winning $T^\lambda_{0,\half}$ for this priority will yield the $\lambda$s we call the extensions of $(S, R)$.
\end{definition}

\section{}
The following sequence  of Appendices C.1--C.4 contain technical results supporting the claims in Remark \ref{509-BBR21}. The material was postponed to this Appendix because of its technical complexity. 
\subsection{Conjunctive and disjunctive attacks}
In Section 2.1 we introduced Boolean instantiation of argumentation networks. This allows for attacks of the form $(a\wedge b) \tO (c\wedge d)$.  The meaning of this is that if $(a\wedge b) =1$ then $(c\wedge d) =0$. We can write this as $\{a,b\} \tO  \{c,d\}$ and the meaning is that if both $a=b=1$ then one of $c$ or $d$ equals 0.

This requires the study of conjunctive and disjunctive attacks.  This is the task of Appendix C.1.  

We recall concepts from \cite{509-3}.

\begin{definition}\label{509-BD1}{\ }
\begin{enumerate}
\item A conjunctive-disjunctive argumentation network, (CD-network) has the form $(S, \BBR)$ where $S$ is a finite set of arguments and $\BBR \subseteq 2^S \times 2^S$ is a relation between subsets. When $X\BBR Y$ holds between $X, Y \subseteq S$, we say that conjunctive $X$ attacks $Y$ disjunctively, or $X$ CD-attacks $Y$.  Figure \ref{509-BF2} (a), (b), (c) shows our graphical notation for this notion.

We also write $X\tO Y$ for $X\BBR Y, X\tO z$ for $X\BBR\{z\}$ and $w\tO Y$ for $\{w\}\BBR Y$.

\begin{figure}
\centering
\setlength{\unitlength}{0.00083333in}
\begingroup\makeatletter\ifx\SetFigFont\undefined%
\gdef\SetFigFont#1#2#3#4#5{%
  \reset@font\fontsize{#1}{#2pt}%
  \fontfamily{#3}\fontseries{#4}\fontshape{#5}%
  \selectfont}%
\fi\endgroup%
{\renewcommand{\dashlinestretch}{30}
\begin{picture}(3927,3123)(0,-10)
\put(3015,60){\makebox(0,0)[lb]{\smash{{\SetFigFont{10}{12.0}{\rmdefault}{\mddefault}{\updefault}(c) $w\tO Y$}}}}
\path(3465,2835)(3465,1485)
\path(615,1485)(1065,715)
\blacken\path(978.551,803.468)(1065.000,715.000)(1030.353,833.742)(978.551,803.468)
\path(605,1495)(175,725)
\blacken\path(207.316,844.397)(175.000,725.000)(259.701,815.143)(207.316,844.397)
\path(615,1935)(615,1485)
\path(1965,1935)(1965,735)
\blacken\path(1935.000,855.000)(1965.000,735.000)(1995.000,855.000)(1935.000,855.000)
\path(3465,1485)(3015,735)
\blacken\path(3051.015,853.334)(3015.000,735.000)(3102.464,822.464)(3051.015,853.334)
\path(3465,1485)(3915,735)
\blacken\path(3827.536,822.464)(3915.000,735.000)(3878.985,853.334)(3827.536,822.464)
\path(1965,1935)(2415,2835)
\path(1965,1935)(1515,2835)
\path(290,935)(260,860)
\blacken\path(276.713,982.559)(260.000,860.000)(332.421,960.275)(276.713,982.559)
\path(930,940)(990,855)
\blacken\path(896.289,935.736)(990.000,855.000)(945.307,970.337)(896.289,935.736)
\path(1960,975)(1955,890)
\blacken\path(1932.098,1011.555)(1955.000,890.000)(1991.995,1008.031)(1932.098,1011.555)
\path(3150,960)(3090,860)
\blacken\path(3126.015,978.334)(3090.000,860.000)(3177.464,947.464)(3126.015,978.334)
\path(3765,970)(3825,875)
\blacken\path(3735.556,960.439)(3825.000,875.000)(3786.286,992.478)(3735.556,960.439)
\put(615,2235){\makebox(0,0)[b]{\smash{{\SetFigFont{10}{12.0}{\rmdefault}{\mddefault}{\updefault}\ldots}}}}
\put(1965,2235){\makebox(0,0)[b]{\smash{{\SetFigFont{10}{12.0}{\rmdefault}{\mddefault}{\updefault}\ldots}}}}
\put(3465,1110){\makebox(0,0)[b]{\smash{{\SetFigFont{10}{12.0}{\rmdefault}{\mddefault}{\updefault}\ldots}}}}
\put(1965,510){\makebox(0,0)[b]{\smash{{\SetFigFont{10}{12.0}{\rmdefault}{\mddefault}{\updefault}$z$}}}}
\put(3465,2985){\makebox(0,0)[b]{\smash{{\SetFigFont{10}{12.0}{\rmdefault}{\mddefault}{\updefault}$w$}}}}
\put(1965,2835){\makebox(0,0)[b]{\smash{{\SetFigFont{10}{12.0}{\rmdefault}{\mddefault}{\updefault}$x_1\quad\comma\quad x_n$}}}}
\put(3465,510){\makebox(0,0)[b]{\smash{{\SetFigFont{10}{12.0}{\rmdefault}{\mddefault}{\updefault}$y_1\quad\comma\quad y_m$}}}}
\put(615,435){\makebox(0,0)[b]{\smash{{\SetFigFont{10}{12.0}{\rmdefault}{\mddefault}{\updefault}$y_1\quad\comma\quad y_m$}}}}
\put(615,1110){\makebox(0,0)[b]{\smash{{\SetFigFont{10}{12.0}{\rmdefault}{\mddefault}{\updefault}\ldots}}}}
\put(615,2835){\makebox(0,0)[b]{\smash{{\SetFigFont{10}{12.0}{\rmdefault}{\mddefault}{\updefault}$x_1\quad\comma\quad x_n$}}}}
\put(15,60){\makebox(0,0)[lb]{\smash{{\SetFigFont{10}{12.0}{\rmdefault}{\mddefault}{\updefault}(a) $X\tO Y$}}}}
\put(1515,60){\makebox(0,0)[lb]{\smash{{\SetFigFont{10}{12.0}{\rmdefault}{\mddefault}{\updefault}(b) $X\tO z$}}}}
\path(165,2835)(615,1935)(1065,2835)
\end{picture}
}

\caption{}\label{509-BF2}
\end{figure}

\item We require $\BBR$ to satisfy: $X\tO Y$ and $X'\supseteq X$ and $Y'\supseteq Y$ imply $X'\tO Y'$.

\item We say a node $z\in S$ is indirectly attacked by $X\subseteq S$ if for some $Y$ $X\BBR Y$ and $z\in Y$.
\end{enumerate}
\end{definition}

\begin{definition}[Kleene 3 valued logic]\label{509-BD3}
Kleene 3-valued logic has 3 values $\{0, 1, \half\}$ and has the truth table in Figure \ref{509-BF4}.  The language has connectives $\wedge, \vee, \neg, \to$ and atomic propositions. Read the values as 1=``in'', 0 = ``out'', and $\half$ = ``undecided'', see \cite{509-18}.

\begin{figure}
\centering
\renewcommand{\arraystretch}{1.5}
\begin{tabular}{c|c|c|c|c|c}

$p$ & $q$ & $\neg p$& $p\wedge q$ & $p\vee q$ & $p\to q$\\
\hline
1&1&0&1 &1&1\\
\hline
$\half$ & 1 & $\half$ & $\half$ & 1&1\\
\hline
0&1&1&0&1&1\\
\hline
1&$\half$ & 0 & $\half$ & 1 & $\half$\\
\hline
$\half$ &$\half$ &$\half$ &$\half$ &$\half$ &$\half$ \\
\hline
0&$\half$ &1&0&$\half$ &$\half$ \\
\hline
1&0&0&0&1&0\\
\hline
$\half$ &0&$\half$ &0&$\half$ &$\half$ \\
\hline
0&0&1&0&0&1\\
\hline
\end{tabular}

\caption{}\label{509-BF4}
\end{figure}
\end{definition}

\begin{definition}\label{509-BD5}{\ }
\begin{enumerate}
\item Let $(S, \BBR)$ be a CD network as in Definition \ref{509-BD1}.  Let $\lambda: S\mapsto \{0,1,\half\}$ be an assignment of values to the elements of $S$, pretending these are atomic logic propositions.  Let $X\subseteq S$. Extend $\lambda$ onto $X$ by letting $\lambda (X) =\lambda (\bigwedge_{x\in X} x)$.

Thus
\[\begin{array}{l}
\lambda (X) =1 \mbox{ if } \lambda (x) =1\mbox{ for all } x \in X\\
\lambda (X) =0 \mbox{ if for some } x \in X, \lambda (x) =0\\
\lambda (X) =\half, \mbox{ otherwise (i.e.\ } \lambda (x) > 0 \mbox{ for all $x\in X$ and for some } x\in X, \lambda (x) =\half).
\end{array}
\]

\item A CD-extension for $(S, \BBR)$ is a function $\lambda: S\mapsto \{0,1\half\}$ satisfying the following conditions:
\begin{enumerate}
\item If $\{z\}$ is not indirectly attacked by any $X$ then $\lambda (z) =1$.
\item If $\lambda (x) =1$ for all $x\in X$ and $X\tO Y$ holds then for some $y \in Y, \lambda (y)=0$.  (I.e.\ if $\lambda (X) =1$ and $X\tO Y$ then $\lambda (Y) =0$.)
\item Assume $Y$ is such that for every $X$ such that $X\tO Y$ there exists an $x\in X$ such that $\lambda (x) =0$. Then for all $y \in Y$ we have $\lambda (y) =1$.  (I.e.\ if $\lambda (X) =0$ for all the attackers of $Y$ then $\lambda (Y) =1$.)
\item For any $Y \subseteq S$, (d1) and (d2) imply (d3).
\begin{enumerate}
\item [(d1)] For every $X$ such that  $X\tO Y$ we have that for some $x\in X, \lambda (x) < 1$.
\item [(d2)] For some $X$ such that $X\tO Y$ we have that for all $x\in X, \lambda (x) > 0$ and for some $x\in X, \lambda (x) =\half$.
\item [(d3)] For some $y \in Y, y =\half$ and for all $y\in Y, \lambda (y) > 0$.  
\end{enumerate}
(I.e.\ for all the attackers $X$ of $Y$, $\lambda (X) \in \{0,\half\}$ and for at least one attacker $X_0, \lambda (X_0)=\half$ then $\lambda (Y) =\half$.)
\end{enumerate}
\end{enumerate}
\end{definition}

\clearpage

\begin{remark}\label{BR6}{\ }
\begin{enumerate}
\item Let $(S, \BBR)$ be a CD-network as in Definition \ref{509-BD1}.  We ask whether there exist CD-extensions for it.  To help us understand the situation, let us look at a figure which might make us think that the answer is negative. Consider an extension for the network of  Figure \ref{509-BF7}.

\begin{figure}
\centering
\setlength{\unitlength}{0.00083333in}
\begingroup\makeatletter\ifx\SetFigFont\undefined%
\gdef\SetFigFont#1#2#3#4#5{%
  \reset@font\fontsize{#1}{#2pt}%
  \fontfamily{#3}\fontseries{#4}\fontshape{#5}%
  \selectfont}%
\fi\endgroup%
{\renewcommand{\dashlinestretch}{30}
\begin{picture}(1727,2412)(0,-10)
\put(1215,474){\makebox(0,0)[b]{\smash{{\SetFigFont{10}{12.0}{\rmdefault}{\mddefault}{\updefault}$c$}}}}
\path(615,1374)(165,624)
\blacken\path(201.015,742.334)(165.000,624.000)(252.464,711.464)(201.015,742.334)
\path(615,1374)(1065,624)
\blacken\path(977.536,711.464)(1065.000,624.000)(1028.985,742.334)(977.536,711.464)
\path(295,839)(245,754)
\blacken\path(279.984,872.643)(245.000,754.000)(331.700,842.222)(279.984,872.643)
\path(930,844)(980,749)
\blacken\path(897.563,841.218)(980.000,749.000)(950.658,869.163)(897.563,841.218)
\path(1325,144)(1260,239)
\blacken\path(1352.521,156.904)(1260.000,239.000)(1303.003,123.023)(1352.521,156.904)
\path(125,139)(65,239)
\blacken\path(152.464,151.536)(65.000,239.000)(101.015,120.666)(152.464,151.536)
\path(165,399)(167,398)(172,397)
	(180,395)(192,392)(207,388)
	(225,382)(246,376)(268,369)
	(291,360)(314,352)(337,342)
	(360,331)(383,318)(405,304)
	(426,288)(447,269)(465,249)
	(479,230)(489,212)(497,195)
	(503,179)(508,165)(511,154)
	(513,143)(514,134)(515,126)
	(515,118)(515,110)(514,102)
	(513,94)(511,85)(508,76)
	(503,65)(497,54)(489,43)
	(479,33)(465,24)(447,17)
	(428,14)(409,12)(391,12)
	(375,12)(362,13)(349,14)
	(338,16)(327,18)(317,20)
	(305,23)(292,27)(277,33)
	(259,41)(239,51)(216,64)
	(190,80)(165,99)(141,121)
	(120,144)(102,167)(87,191)
	(74,214)(63,236)(53,259)
	(45,281)(38,303)(32,324)
	(27,343)(23,360)(20,374)(15,399)
\blacken\path(67.951,287.214)(15.000,399.000)(9.117,275.447)(67.951,287.214)
\path(1365,399)(1367,398)(1372,397)
	(1380,395)(1392,392)(1407,388)
	(1425,382)(1446,376)(1468,369)
	(1491,360)(1514,352)(1537,342)
	(1560,331)(1583,318)(1605,304)
	(1626,288)(1647,269)(1665,249)
	(1679,230)(1689,212)(1697,195)
	(1703,179)(1708,165)(1711,154)
	(1713,143)(1714,134)(1715,126)
	(1715,118)(1715,110)(1714,102)
	(1713,94)(1711,85)(1708,76)
	(1703,65)(1697,54)(1689,43)
	(1679,33)(1665,24)(1647,17)
	(1628,14)(1609,12)(1591,12)
	(1575,12)(1562,13)(1549,14)
	(1538,16)(1527,18)(1517,20)
	(1505,23)(1492,27)(1477,33)
	(1459,41)(1439,51)(1416,64)
	(1390,80)(1365,99)(1341,121)
	(1320,144)(1302,167)(1287,191)
	(1274,214)(1263,236)(1253,259)
	(1245,281)(1238,303)(1232,324)
	(1227,343)(1223,360)(1220,374)(1215,399)
\blacken\path(1267.951,287.214)(1215.000,399.000)(1209.117,275.447)(1267.951,287.214)
\put(615,2274){\makebox(0,0)[b]{\smash{{\SetFigFont{10}{12.0}{\rmdefault}{\mddefault}{\updefault}$x$}}}}
\put(15,474){\makebox(0,0)[b]{\smash{{\SetFigFont{10}{12.0}{\rmdefault}{\mddefault}{\updefault}$a$}}}}
\path(615,2124)(615,1374)
\end{picture}
}
\caption{}\label{509-BF7}
\end{figure}

 Any such extension $\lambda$ would require that $\lambda (x)=1$ and either $\lambda (a) =0$ or $\lambda (c) =0$. But the fact that also $a\tO a$ and $c\tO c$ forces $\lambda (a) =\lambda (b)=\half$.  The values cannot be in $\{0,1\}$!

The perceptive reader might ask: we are interested in Boolean instantiations of traditional networks. Can Figure \ref{509-BF7}  be obtained as such an instantiation?  If not then the use of CD-networks is an overkill.  The answer is yes, see Figure \ref{509-B7a}

\begin{figure}
\centering
\setlength{\unitlength}{0.00083333in}
\begingroup\makeatletter\ifx\SetFigFont\undefined%
\gdef\SetFigFont#1#2#3#4#5{%
  \reset@font\fontsize{#1}{#2pt}%
  \fontfamily{#3}\fontseries{#4}\fontshape{#5}%
  \selectfont}%
\fi\endgroup%
{\renewcommand{\dashlinestretch}{30}
\begin{picture}(2773,1923)(0,-10)
\put(15,60){\makebox(0,0)[lb]{\smash{{\SetFigFont{10}{12.0}{\rmdefault}{\mddefault}{\updefault}Instantiate $u=a, v=c$ and $y = a\wedge c$}}}}
\path(1365,1035)(1365,960)
\blacken\path(1335.000,1080.000)(1365.000,960.000)(1395.000,1080.000)(1335.000,1080.000)
\path(505,1420)(435,1460)
\blacken\path(554.073,1426.511)(435.000,1460.000)(524.305,1374.416)(554.073,1426.511)
\path(2355,1365)(2270,1395)
\blacken\path(2393.143,1383.351)(2270.000,1395.000)(2373.174,1326.772)(2393.143,1383.351)
\path(315,1710)(317,1711)(320,1714)
	(326,1718)(336,1725)(348,1734)
	(364,1744)(381,1756)(401,1769)
	(423,1783)(445,1796)(469,1809)
	(493,1821)(518,1832)(544,1842)
	(571,1850)(600,1856)(629,1861)
	(660,1862)(690,1860)(719,1855)
	(745,1847)(769,1838)(789,1828)
	(806,1819)(820,1810)(832,1802)
	(842,1794)(851,1786)(859,1779)
	(866,1771)(874,1762)(881,1753)
	(888,1742)(895,1729)(903,1715)
	(909,1698)(914,1678)(916,1657)
	(915,1635)(909,1613)(899,1592)
	(888,1572)(876,1555)(864,1541)
	(852,1528)(841,1517)(830,1508)
	(819,1499)(809,1491)(798,1484)
	(786,1476)(773,1468)(758,1460)
	(741,1451)(721,1442)(699,1432)
	(673,1423)(645,1415)(615,1410)
	(582,1408)(551,1410)(522,1415)
	(495,1423)(471,1433)(449,1445)
	(428,1458)(408,1473)(390,1487)
	(373,1502)(358,1516)(344,1529)
	(333,1540)(315,1560)
\blacken\path(417.575,1490.874)(315.000,1560.000)(372.977,1450.736)(417.575,1490.874)
\path(2160,1650)(2162,1651)(2165,1654)
	(2171,1658)(2181,1665)(2193,1674)
	(2209,1684)(2226,1696)(2246,1709)
	(2268,1723)(2290,1736)(2314,1749)
	(2338,1761)(2363,1772)(2389,1782)
	(2416,1790)(2445,1796)(2474,1801)
	(2505,1802)(2535,1800)(2564,1795)
	(2590,1787)(2614,1778)(2634,1768)
	(2651,1759)(2665,1750)(2677,1742)
	(2687,1734)(2696,1726)(2704,1719)
	(2711,1711)(2719,1702)(2726,1693)
	(2733,1682)(2740,1669)(2748,1655)
	(2754,1638)(2759,1618)(2761,1597)
	(2760,1575)(2754,1553)(2744,1532)
	(2733,1512)(2721,1495)(2709,1481)
	(2697,1468)(2686,1457)(2675,1448)
	(2664,1439)(2654,1431)(2643,1424)
	(2631,1416)(2618,1408)(2603,1400)
	(2586,1391)(2566,1382)(2544,1372)
	(2518,1363)(2490,1355)(2460,1350)
	(2427,1348)(2396,1350)(2367,1355)
	(2340,1363)(2316,1373)(2294,1385)
	(2273,1398)(2253,1413)(2235,1427)
	(2218,1442)(2203,1456)(2189,1469)
	(2178,1480)(2160,1500)
\blacken\path(2262.575,1430.874)(2160.000,1500.000)(2217.977,1390.736)(2262.575,1430.874)
\put(1365,1785){\makebox(0,0)[b]{\smash{{\SetFigFont{10}{12.0}{\rmdefault}{\mddefault}{\updefault}$x$}}}}
\put(1365,585){\makebox(0,0)[b]{\smash{{\SetFigFont{10}{12.0}{\rmdefault}{\mddefault}{\updefault}$y$}}}}
\put(15,1560){\makebox(0,0)[lb]{\smash{{\SetFigFont{10}{12.0}{\rmdefault}{\mddefault}{\updefault}$u$}}}}
\put(1815,1560){\makebox(0,0)[lb]{\smash{{\SetFigFont{10}{12.0}{\rmdefault}{\mddefault}{\updefault}$v$}}}}
\path(1365,1710)(1365,810)
\blacken\path(1335.000,930.000)(1365.000,810.000)(1395.000,930.000)(1335.000,930.000)
\end{picture}
}

\caption{}\label{509-B7a}
\end{figure}

If we write the equations for this figure we get 
\[
\begin{array}{l}
u =1-u\\
v=1-v\\
y=1-x\\
x=1
\end{array}
\]
 Substituting the instantiation we get 
\[\begin{array}{l}
a=1-a\\
c=1-c\\
a\wedge c=0
\end{array}
\]
There is no solution in $\{0, 1,\half\}$.

This line of reasoning, however, contains a fallacy. Before the instantiation, $u$ for example, was attacked only by itself, so the equation for it was $u =1-u$.  After the instantiation of $u=a$ and $y =a\wedge c$, $a$ was attacked both by itslef and indirectly also by $x$. Thereofre a new equation should be written for the new situation of Figure \ref{509-BF7}. We cannot just substitute the instantiation in the old equations. We are not reading Figure \ref{509-BF7} correctly.

We need to modify our point of view.  The following discussion is a conceptual analysis seeking a new point of view:

We begin with a slightly modified point of view which we call the RCD view (the restricted CD view).  We note that $x \tO \{a,c\}$ actually implies $\{x,a\}\tO c$ and $\{x,c\}\tO a$.

If we understand $x\tO \{a,c\}$ as meaning the conjunction of the above two attacks then Figure \ref{509-BF7} becomes Figure \ref{509-BF8}.

\begin{figure}
\centering
\setlength{\unitlength}{0.00083333in}
\begingroup\makeatletter\ifx\SetFigFont\undefined%
\gdef\SetFigFont#1#2#3#4#5{%
  \reset@font\fontsize{#1}{#2pt}%
  \fontfamily{#3}\fontseries{#4}\fontshape{#5}%
  \selectfont}%
\fi\endgroup%
{\renewcommand{\dashlinestretch}{30}
\begin{picture}(2720,1691)(0,-10)
\put(1707,396){\makebox(0,0)[b]{\smash{{\SetFigFont{10}{12.0}{\rmdefault}{\mddefault}{\updefault}$\wedge$}}}}
\put(1716.281,352.365){\arc{592.611}{3.8538}{6.2203}}
\path(57,476)(107,561)
\blacken\path(72.016,442.357)(107.000,561.000)(20.300,472.778)(72.016,442.357)
\path(762,686)(672,661)
\blacken\path(779.593,722.023)(672.000,661.000)(795.651,664.212)(779.593,722.023)
\path(2067,691)(2137,641)
\blacken\path(2021.915,686.337)(2137.000,641.000)(2056.789,735.161)(2021.915,686.337)
\path(2642,1046)(2607,961)
\blacken\path(2624.950,1083.384)(2607.000,961.000)(2680.430,1060.539)(2624.950,1083.384)
\path(1157,1356)(1392,816)
\path(1437,656)(1637,196)
\path(357,826)(358,827)(360,828)
	(363,831)(368,835)(376,841)
	(386,849)(398,859)(414,871)
	(432,886)(453,903)(477,921)
	(504,942)(532,964)(563,987)
	(596,1012)(631,1037)(667,1064)
	(704,1091)(742,1118)(782,1146)
	(822,1174)(864,1201)(906,1229)
	(950,1257)(994,1284)(1039,1311)
	(1086,1338)(1134,1365)(1183,1391)
	(1234,1417)(1287,1442)(1341,1467)
	(1397,1491)(1455,1514)(1513,1537)
	(1573,1557)(1632,1576)(1701,1596)
	(1768,1612)(1831,1625)(1889,1636)
	(1941,1645)(1988,1652)(2030,1657)
	(2066,1660)(2098,1662)(2126,1664)
	(2149,1664)(2170,1664)(2188,1663)
	(2204,1661)(2218,1660)(2232,1657)
	(2245,1655)(2258,1651)(2272,1648)
	(2287,1644)(2304,1639)(2322,1633)
	(2342,1627)(2365,1620)(2390,1611)
	(2418,1601)(2447,1590)(2479,1576)
	(2512,1561)(2545,1543)(2577,1523)
	(2607,1501)(2636,1473)(2660,1443)
	(2679,1411)(2692,1379)(2701,1346)
	(2706,1314)(2708,1281)(2706,1248)
	(2701,1215)(2694,1182)(2685,1149)
	(2675,1116)(2662,1083)(2649,1051)
	(2635,1020)(2621,989)(2607,961)
	(2593,934)(2580,910)(2568,888)
	(2558,870)(2549,855)(2543,844)(2532,826)
\blacken\path(2568.976,944.037)(2532.000,826.000)(2620.172,912.750)(2568.976,944.037)
\path(2457,676)(2456,675)(2454,674)
	(2451,672)(2446,668)(2438,662)
	(2428,654)(2415,645)(2400,634)
	(2381,620)(2360,605)(2336,589)
	(2310,570)(2282,551)(2251,530)
	(2219,508)(2185,486)(2150,463)
	(2113,439)(2075,416)(2036,392)
	(1996,369)(1955,345)(1912,322)
	(1868,299)(1822,276)(1775,254)
	(1726,232)(1675,210)(1621,188)
	(1566,167)(1508,147)(1447,127)
	(1385,109)(1322,92)(1257,76)
	(1185,61)(1115,48)(1048,38)
	(986,30)(929,24)(877,19)
	(831,15)(790,13)(753,12)
	(721,12)(693,12)(668,13)
	(645,14)(625,15)(606,17)
	(588,20)(571,22)(553,25)
	(535,29)(515,33)(494,37)
	(471,42)(446,47)(418,54)
	(387,61)(354,69)(319,79)
	(282,90)(243,103)(204,117)
	(167,133)(132,151)(98,173)
	(70,197)(48,222)(31,248)
	(21,273)(14,299)(12,324)
	(14,350)(18,376)(26,401)
	(36,427)(48,452)(61,478)
	(76,503)(91,527)(107,550)
	(123,572)(139,593)(153,611)
	(166,628)(178,642)(188,653)
	(195,662)(207,676)
\blacken\path(151.683,565.365)(207.000,676.000)(106.127,604.413)(151.683,565.365)
\path(357,826)(359,827)(362,829)
	(369,833)(379,839)(393,847)
	(409,857)(428,868)(449,879)
	(472,892)(495,904)(520,916)
	(545,927)(570,938)(596,948)
	(622,957)(649,965)(677,971)
	(705,975)(732,976)(757,975)
	(779,971)(797,966)(812,961)
	(823,955)(832,949)(839,943)
	(844,937)(848,932)(851,926)
	(853,920)(856,914)(859,907)
	(862,899)(865,890)(869,880)
	(874,868)(878,855)(881,841)
	(882,826)(880,808)(875,791)
	(869,777)(862,765)(856,756)
	(850,749)(844,743)(838,738)
	(832,734)(825,729)(816,723)
	(805,716)(791,707)(775,697)
	(755,686)(732,676)(705,667)
	(679,661)(654,658)(631,657)
	(610,657)(590,659)(570,662)
	(552,665)(537,669)(507,676)
\blacken\path(630.678,677.948)(507.000,676.000)(617.044,619.517)(630.678,677.948)
\path(2532,826)(2531,828)(2528,831)
	(2523,838)(2516,847)(2507,858)
	(2496,871)(2483,885)(2468,899)
	(2453,913)(2435,926)(2415,939)
	(2393,950)(2368,961)(2339,970)
	(2307,976)(2278,979)(2250,980)
	(2225,979)(2203,978)(2184,977)
	(2168,975)(2155,973)(2143,972)
	(2132,970)(2121,967)(2110,965)
	(2099,961)(2085,956)(2071,949)
	(2055,941)(2038,930)(2021,917)
	(2007,901)(1997,884)(1990,866)
	(1987,850)(1987,835)(1988,822)
	(1990,810)(1993,800)(1997,791)
	(2001,782)(2005,774)(2011,765)
	(2017,755)(2023,744)(2032,732)
	(2041,719)(2053,704)(2066,690)
	(2082,676)(2104,662)(2128,653)
	(2151,648)(2173,647)(2195,648)
	(2216,651)(2237,655)(2256,660)
	(2273,665)(2307,676)
\blacken\path(2202.061,610.518)(2307.000,676.000)(2183.592,667.605)(2202.061,610.518)
\put(357,676){\makebox(0,0)[b]{\smash{{\SetFigFont{10}{12.0}{\rmdefault}{\mddefault}{\updefault}$a$}}}}
\put(2457,676){\makebox(0,0)[b]{\smash{{\SetFigFont{10}{12.0}{\rmdefault}{\mddefault}{\updefault}$c$}}}}
\put(1407,676){\makebox(0,0)[b]{\smash{{\SetFigFont{10}{12.0}{\rmdefault}{\mddefault}{\updefault}$x$}}}}
\put(1092,1096){\makebox(0,0)[b]{\smash{{\SetFigFont{10}{12.0}{\rmdefault}{\mddefault}{\updefault}$\wedge$}}}}
\put(1182.000,1206.677){\arc{550.316}{0.9942}{3.1755}}
\end{picture}
}
\caption{}\label{509-BF8}
\end{figure}

Since $x$ is not indirectly attacked we have $\lambda (x) =1 $ and since $a$ and $c$ are each self attacking we get $\lambda (a) =\lambda (c) =\half$.

We would also have according to Definition \ref{509-BD5} that $\lambda (\{x,a\}) =\lambda (\{x,c\}) =\half$.

\item It seems then, that if we adopt the suggestion in (1) above and understand $x\tO \{y_1\comma y_k\}$ from the RCD point of view, as the set of attacks $\{x, y_1\comma y_{i-1},\break  y_{i+1}\comma y_k\}\tO y_i$ for $i =1\comma k$, then all we need to define is the concept of attacks of the form $X\tO z$. Attacks of the form $X\tO Y$ are reducible to the form $X\tO z$.  Such joint attacks can be simulated within traditional Dung networks using auxilliary points, as shown in \cite{509-3}.

The attack formation $\{x_1\comma x_n \} \tO  z$  of Figure \ref{509-BF2}(b) becomes the formation of Figure \ref{509-BF9}.  The auxilliary points are $y_1\comma y_n, y$.  The reader should note that the auxiliary points used must be completely new to the rest of the network and be associated with  $\{x_1\comma  x_n , z\}$ only. Any other attack formation, say $\{a_1\comma  a_m\} \tO b$ will its own completely new and disjoint auxiliary points. If we do not observe this restriction and re-use existing points as auxiliary points we get wrong results, as Remark \ref{509-BR101} shows.

\begin{figure}
\centering
\setlength{\unitlength}{0.00083333in}
\begingroup\makeatletter\ifx\SetFigFont\undefined%
\gdef\SetFigFont#1#2#3#4#5{%
  \reset@font\fontsize{#1}{#2pt}%
  \fontfamily{#3}\fontseries{#4}\fontshape{#5}%
  \selectfont}%
\fi\endgroup%
{\renewcommand{\dashlinestretch}{30}
\begin{picture}(1884,3517)(0,-10)
\put(1092,55){\makebox(0,0)[b]{\smash{{\SetFigFont{10}{12.0}{\rmdefault}{\mddefault}{\updefault}$z$}}}}
\path(1842,3280)(1842,2230)
\blacken\path(1812.000,2350.000)(1842.000,2230.000)(1872.000,2350.000)(1812.000,2350.000)
\path(42,1930)(1092,1330)
\blacken\path(972.927,1363.489)(1092.000,1330.000)(1002.695,1415.584)(972.927,1363.489)
\path(1842,1930)(1092,1330)
\blacken\path(1166.963,1428.389)(1092.000,1330.000)(1204.445,1381.537)(1166.963,1428.389)
\path(1092,1030)(1092,280)
\blacken\path(1062.000,400.000)(1092.000,280.000)(1122.000,400.000)(1062.000,400.000)
\path(1092,505)(1092,430)
\blacken\path(1062.000,550.000)(1092.000,430.000)(1122.000,550.000)(1062.000,550.000)
\path(42,2530)(42,2380)
\blacken\path(12.000,2500.000)(42.000,2380.000)(72.000,2500.000)(12.000,2500.000)
\path(1842,2530)(1842,2380)
\blacken\path(1812.000,2500.000)(1842.000,2380.000)(1872.000,2500.000)(1812.000,2500.000)
\path(1297,1500)(1212,1435)
\blacken\path(1289.099,1531.725)(1212.000,1435.000)(1325.546,1484.063)(1289.099,1531.725)
\path(847,1475)(967,1400)
\blacken\path(849.340,1438.160)(967.000,1400.000)(881.140,1489.040)(849.340,1438.160)
\put(42,3355){\makebox(0,0)[b]{\smash{{\SetFigFont{10}{12.0}{\rmdefault}{\mddefault}{\updefault}$x_1$}}}}
\put(942,3355){\makebox(0,0)[b]{\smash{{\SetFigFont{10}{12.0}{\rmdefault}{\mddefault}{\updefault}\comma}}}}
\put(1842,3355){\makebox(0,0)[b]{\smash{{\SetFigFont{10}{12.0}{\rmdefault}{\mddefault}{\updefault}$x_n$}}}}
\put(42,2005){\makebox(0,0)[b]{\smash{{\SetFigFont{10}{12.0}{\rmdefault}{\mddefault}{\updefault}$y_1$}}}}
\put(1842,2005){\makebox(0,0)[b]{\smash{{\SetFigFont{10}{12.0}{\rmdefault}{\mddefault}{\updefault}$y_n$}}}}
\put(1017,2005){\makebox(0,0)[b]{\smash{{\SetFigFont{10}{12.0}{\rmdefault}{\mddefault}{\updefault}\comma}}}}
\put(1092,1105){\makebox(0,0)[b]{\smash{{\SetFigFont{10}{12.0}{\rmdefault}{\mddefault}{\updefault}$y$}}}}
\path(42,3280)(42,2230)
\blacken\path(12.000,2350.000)(42.000,2230.000)(72.000,2350.000)(12.000,2350.000)
\end{picture}
}
\caption{}\label{509-BF9}
\end{figure}

\item The RCD point of view is not the best we can offer. There is a better one, already mentioned in the 2009 paper \cite{509-3}.  In Section 4.3 of that paper, we discussed what we call ``flow argumentation networks''. To explain it simply for our case, in Figure \ref{509-BF2}(c) the node $w$ transmits attacks to all the nodes $y_1\comma y_m$ (the attack ``flow'' emanating from $w$) but expects only at least one of them to succeed. Thus applying this ``flow'' view to Figure \ref{509-BF7} we have that $x$ attacks $a$ and $x$ attacks $c$ but expects at least one to succeed.  Since $x=1$ and $a=\half$, $a$ will become 0. Similarly $c$ will become 0. Thus we get two extensions 
\[
x=1, a=0, c=\half
\]
and 
\[
x=1,a=\half,c=0
\]

\end{enumerate}
We shall address this point of view, which we call the FCD-view in Appendix C2. More details in Section 4.3 of \cite{509-3}.  
\end{remark}

\begin{remark}\label{509-BR10}
It would be instructive to compare our RCD disjunctive attack notion with the attacks notion of Nielsen and Parsons \cite{509-17} notion of joint and disjunctive attacks. Call it the NP view.  Let us look again at Figure \ref{509-BF2}(c).

We have
\[
w\tO Y =\{y_1\comma y_m\}.
\]
Let us simplify and consider $w\tO \{a,c\}$.  We do our comparison for this case but before that let us summarise our options:

We have the following four notions for $X\tO Y$, where 
\[X =\{x_1\comma x_n\}
\mbox{ and } Y =\{y_1\comma y_m\}.
\]
\begin{enumerate}
\item Gabbay 2009 \cite{509-3}, CD-view:

If $\bigwedge_i \lambda (x_i)=1$ then $\bigvee_i\lambda (y)i)=0$.
 
This view allows us to generate networks like in Figure \ref{509-BF7}.
\item The equational view, which is not the same as the CD-view.  This view is connected with instantiation networks and may not be able to generate networks like Figure \ref{509-BF7}. This view will require us to  solve the equation $(\bigwedge_j) y_j)\leftrightarrow \neg \bigwedge_i x_i$.

\item Gabbay alternative as in (1) above RCD-view:
\\
$w\tO \{a,c\}$ means $\{w, a\}\tO c$ {\em and} $\{w,c\}\tO a$.
\item Nielsen and Parsons' 2007 \cite{509-17}, the NP-view:\\
$w\tO \{a,c\}$ means $w\tO a$ {\em or} $w\tO c$.

(The Nielsen and Parsons' definition of $X\tO Y$ is that for some $y\in Y, X\tO y$).
\end{enumerate}
The implementation of the NP-view for $w\tO \{a,c\}$ is in Figure \ref{509-BF11}. 

We need a semantics which will not allow for $\{y_2, y_4\}$ to be both undecided.

\begin{figure}
\centering
\setlength{\unitlength}{0.00083333in}
\begingroup\makeatletter\ifx\SetFigFont\undefined%
\gdef\SetFigFont#1#2#3#4#5{%
  \reset@font\fontsize{#1}{#2pt}%
  \fontfamily{#3}\fontseries{#4}\fontshape{#5}%
  \selectfont}%
\fi\endgroup%
{\renewcommand{\dashlinestretch}{30}
\begin{picture}(1587,2568)(0,-10)
\put(1535,65){\makebox(0,0)[b]{\smash{{\SetFigFont{10}{12.0}{\rmdefault}{\mddefault}{\updefault}$c$}}}}
\path(795,2355)(1545,1755)
\blacken\path(1432.555,1806.537)(1545.000,1755.000)(1470.037,1853.389)(1432.555,1806.537)
\path(45,1455)(45,1005)
\blacken\path(15.000,1125.000)(45.000,1005.000)(75.000,1125.000)(15.000,1125.000)
\path(1545,1455)(1545,1005)
\blacken\path(1515.000,1125.000)(1545.000,1005.000)(1575.000,1125.000)(1515.000,1125.000)
\path(45,705)(45,255)
\blacken\path(15.000,375.000)(45.000,255.000)(75.000,375.000)(15.000,375.000)
\path(1545,705)(1545,255)
\blacken\path(1515.000,375.000)(1545.000,255.000)(1575.000,375.000)(1515.000,375.000)
\blacken\path(390.000,885.000)(270.000,855.000)(390.000,825.000)(390.000,885.000)
\path(270,855)(1395,855)
\blacken\path(1275.000,825.000)(1395.000,855.000)(1275.000,885.000)(1275.000,825.000)
\blacken\path(540.000,885.000)(420.000,855.000)(540.000,825.000)(540.000,885.000)
\path(420,855)(1245,855)
\blacken\path(1125.000,825.000)(1245.000,855.000)(1125.000,885.000)(1125.000,825.000)
\path(235,1910)(170,1865)
\blacken\path(251.587,1957.971)(170.000,1865.000)(285.739,1908.639)(251.587,1957.971)
\path(1340,1915)(1415,1870)
\blacken\path(1296.666,1906.015)(1415.000,1870.000)(1327.536,1957.464)(1296.666,1906.015)
\path(1550,1235)(1560,1160)
\blacken\path(1514.404,1274.982)(1560.000,1160.000)(1573.877,1282.912)(1514.404,1274.982)
\path(45,1255)(45,1160)
\blacken\path(15.000,1280.000)(45.000,1160.000)(75.000,1280.000)(15.000,1280.000)
\path(45,505)(50,415)
\blacken\path(13.390,533.151)(50.000,415.000)(73.297,536.479)(13.390,533.151)
\path(1540,505)(1545,415)
\blacken\path(1508.390,533.151)(1545.000,415.000)(1568.297,536.479)(1508.390,533.151)
\put(795,2430){\makebox(0,0)[b]{\smash{{\SetFigFont{10}{12.0}{\rmdefault}{\mddefault}{\updefault}$w$}}}}
\put(45,1530){\makebox(0,0)[b]{\smash{{\SetFigFont{10}{12.0}{\rmdefault}{\mddefault}{\updefault}$y_1$}}}}
\put(1545,1530){\makebox(0,0)[b]{\smash{{\SetFigFont{10}{12.0}{\rmdefault}{\mddefault}{\updefault}$y_3$}}}}
\put(45,780){\makebox(0,0)[b]{\smash{{\SetFigFont{10}{12.0}{\rmdefault}{\mddefault}{\updefault}$y_2$}}}}
\put(1545,780){\makebox(0,0)[b]{\smash{{\SetFigFont{10}{12.0}{\rmdefault}{\mddefault}{\updefault}$y_4$}}}}
\put(35,55){\makebox(0,0)[b]{\smash{{\SetFigFont{10}{12.0}{\rmdefault}{\mddefault}{\updefault}$a$}}}}
\path(795,2355)(45,1755)
\blacken\path(119.963,1853.389)(45.000,1755.000)(157.445,1806.537)(119.963,1853.389)
\end{picture}
}

\caption{}\label{509-BF11}
\end{figure}

To complete our full comparison of the above four approaches, consider Figure \ref{509-BF12}.

\begin{figure}
\centering
\setlength{\unitlength}{0.00083333in}
\begingroup\makeatletter\ifx\SetFigFont\undefined%
\gdef\SetFigFont#1#2#3#4#5{%
  \reset@font\fontsize{#1}{#2pt}%
  \fontfamily{#3}\fontseries{#4}\fontshape{#5}%
  \selectfont}%
\fi\endgroup%
{\renewcommand{\dashlinestretch}{30}
\begin{picture}(2535,1727)(0,-10)
\put(1318,164){\makebox(0,0)[b]{\smash{{\SetFigFont{10}{12.0}{\rmdefault}{\mddefault}{\updefault}$\vee$}}}}
\put(1256.457,789.606){\arc{1176.341}{4.0244}{5.4523}}
\path(423,1514)(1323,989)(2073,1514)
\path(1323,989)(1323,389)
\path(193,1264)(258,1374)
\blacken\path(222.780,1255.427)(258.000,1374.000)(171.125,1285.951)(222.780,1255.427)
\path(2303,1284)(2233,1379)
\blacken\path(2328.336,1300.189)(2233.000,1379.000)(2280.032,1264.597)(2328.336,1300.189)
\path(1323,389)(1324,388)(1328,387)
	(1334,384)(1343,380)(1357,375)
	(1374,367)(1395,359)(1419,349)
	(1447,338)(1478,325)(1512,313)
	(1548,300)(1585,286)(1623,274)
	(1662,261)(1702,249)(1742,238)
	(1783,228)(1825,220)(1867,212)
	(1910,206)(1954,202)(1998,201)
	(2043,201)(2089,204)(2134,210)
	(2178,219)(2217,230)(2253,243)
	(2285,257)(2314,271)(2339,285)
	(2361,298)(2380,310)(2396,320)
	(2409,330)(2420,338)(2429,346)
	(2437,353)(2443,359)(2448,365)
	(2453,371)(2457,377)(2461,384)
	(2466,392)(2470,402)(2475,413)
	(2480,426)(2485,442)(2491,461)
	(2497,483)(2503,508)(2510,538)
	(2515,571)(2520,607)(2523,647)
	(2523,689)(2520,733)(2515,777)
	(2507,821)(2497,864)(2486,906)
	(2472,947)(2458,987)(2442,1026)
	(2425,1064)(2407,1101)(2389,1137)
	(2369,1172)(2350,1207)(2329,1241)
	(2309,1274)(2289,1307)(2269,1337)
	(2250,1366)(2232,1393)(2215,1418)
	(2200,1440)(2187,1460)(2175,1476)
	(2166,1489)(2159,1499)(2148,1514)
\blacken\path(2243.156,1434.972)(2148.000,1514.000)(2194.771,1399.490)(2243.156,1434.972)
\path(1323,389)(1322,388)(1319,387)
	(1314,385)(1306,382)(1294,377)
	(1279,371)(1261,363)(1239,355)
	(1214,345)(1186,334)(1155,322)
	(1123,310)(1088,298)(1052,286)
	(1015,274)(977,262)(938,250)
	(899,240)(859,230)(818,221)
	(777,213)(735,207)(692,202)
	(649,199)(604,197)(559,198)
	(513,201)(468,206)(423,214)
	(381,225)(341,237)(305,251)
	(273,264)(245,277)(220,290)
	(199,301)(181,311)(166,321)
	(154,329)(143,336)(134,342)
	(127,348)(121,354)(115,359)
	(110,365)(105,372)(100,380)
	(94,389)(88,400)(82,413)
	(74,428)(66,447)(58,468)
	(49,494)(40,523)(31,556)
	(23,592)(17,632)(13,674)
	(12,718)(14,763)(19,807)
	(26,850)(35,893)(46,935)
	(59,975)(72,1014)(87,1053)
	(104,1091)(121,1128)(138,1164)
	(157,1200)(175,1234)(194,1268)
	(213,1301)(232,1333)(250,1362)
	(267,1390)(284,1416)(298,1438)
	(311,1458)(322,1475)(331,1488)
	(337,1498)(348,1514)
\blacken\path(304.738,1398.119)(348.000,1514.000)(255.295,1432.111)(304.738,1398.119)
\put(423,1589){\makebox(0,0)[b]{\smash{{\SetFigFont{10}{12.0}{\rmdefault}{\mddefault}{\updefault}$a$}}}}
\put(2073,1589){\makebox(0,0)[b]{\smash{{\SetFigFont{10}{12.0}{\rmdefault}{\mddefault}{\updefault}$b$}}}}
\put(1308,1169){\makebox(0,0)[b]{\smash{{\SetFigFont{10}{12.0}{\rmdefault}{\mddefault}{\updefault}$\wedge$}}}}
\put(1295.430,437.935){\arc{858.004}{0.4431}{2.6462}}
\end{picture}
}
\caption{}\label{509-BF12}
\end{figure}

This figure says that $\{a,b\}\tO \{a,b\}$.

The CD approach and the equational approach  offer no stable $\{0,1\}$ extensions.  The RCD approach reduces $\{a,b\}\tO \{a,b\}$ into 
\[
a\wedge b\tO a
\]
and
\[
a\wedge b\tO b
\]
and has no stable $\{0,1\}$ extensions either.

However, the NP approach has two stable $\{0,1\}$ extensions 

\[
a=1,b=0
\]
and
\[
a=0,b=1
\]

This is because if $b\tO a$ or $a\tO b$ then $\{a,b\}\tO \{a,b\}$ holds.  There is more discussion of the NP-approach in \cite{509-3}.
\end{remark}

\begin{remark}\label{509-BR101}
If the joint attacks are done in  two state networks, as defined and discussed in Definition \ref{509-D2} and Lemma \ref{509-L3}, then we might think that it is much simpler to reduce joint attacks to single attacks by using existing points as auxiliary points. The attack 
\[
a\wedge b \tO c
\]
can be reduced to the two attacks
\[\begin{array}{l}
\neg a\tO \neg c\\
\neg b\tO \neg c.
\end{array}
\]
The reason for that equivalence can be seen by looking at the attack equationally. $a\wedge b\tO c$ means that $c\equiv\neg(a\wedge b)$. Therefore $\neg c\equiv a\wedge b$ or equivalently $\neg c\equiv\neg(\neg a)\wedge\neg(\neg b)$ which is the same equation for $\neg a\tO \neg c$ and $\neg b\tO \neg c$.

In a network where for each $x$, $\neg x$  is also present with $x \LtO  \neg x$ (as we have in a 2-state network) then Figures \ref{509-BF102} and \ref{509-BF103} are equivalent.

\begin{figure}
\centering
\setlength{\unitlength}{0.00083333in}
\begingroup\makeatletter\ifx\SetFigFont\undefined%
\gdef\SetFigFont#1#2#3#4#5{%
  \reset@font\fontsize{#1}{#2pt}%
  \fontfamily{#3}\fontseries{#4}\fontshape{#5}%
  \selectfont}%
\fi\endgroup%
{\renewcommand{\dashlinestretch}{30}
\begin{picture}(4005,2322)(0,-10)
\put(2190,2160){\makebox(0,0)[b]{\smash{{\SetFigFont{10}{12.0}{\rmdefault}{\mddefault}{\updefault}$\comma$}}}}
\path(1140,2085)(2190,1185)(3240,2085)
\path(2190,1185)(2190,285)
\blacken\path(2160.000,405.000)(2190.000,285.000)(2220.000,405.000)(2160.000,405.000)
\path(2190,585)(2190,435)
\blacken\path(2160.000,555.000)(2190.000,435.000)(2220.000,555.000)(2160.000,555.000)
\blacken\path(210.000,2265.000)(90.000,2235.000)(210.000,2205.000)(210.000,2265.000)
\path(90,2235)(840,2235)
\blacken\path(720.000,2205.000)(840.000,2235.000)(720.000,2265.000)(720.000,2205.000)
\blacken\path(360.000,2265.000)(240.000,2235.000)(360.000,2205.000)(360.000,2265.000)
\path(240,2235)(690,2235)
\blacken\path(570.000,2205.000)(690.000,2235.000)(570.000,2265.000)(570.000,2205.000)
\blacken\path(1410.000,165.000)(1290.000,135.000)(1410.000,105.000)(1410.000,165.000)
\path(1290,135)(2040,135)
\blacken\path(1920.000,105.000)(2040.000,135.000)(1920.000,165.000)(1920.000,105.000)
\blacken\path(1560.000,165.000)(1440.000,135.000)(1560.000,105.000)(1560.000,165.000)
\path(1440,135)(1890,135)
\blacken\path(1770.000,105.000)(1890.000,135.000)(1770.000,165.000)(1770.000,105.000)
\blacken\path(3360.000,2265.000)(3240.000,2235.000)(3360.000,2205.000)(3360.000,2265.000)
\path(3240,2235)(3990,2235)
\blacken\path(3870.000,2205.000)(3990.000,2235.000)(3870.000,2265.000)(3870.000,2205.000)
\blacken\path(3510.000,2265.000)(3390.000,2235.000)(3510.000,2205.000)(3510.000,2265.000)
\path(3390,2235)(3840,2235)
\blacken\path(3720.000,2205.000)(3840.000,2235.000)(3720.000,2265.000)(3720.000,2205.000)
\put(2190,1260){\makebox(0,0)[b]{\smash{{\SetFigFont{10}{12.0}{\rmdefault}{\mddefault}{\updefault}$\wedge$}}}}
\put(1140,2160){\makebox(0,0)[rb]{\smash{{\SetFigFont{10}{12.0}{\rmdefault}{\mddefault}{\updefault}$x_1$}}}}
\put(15,2160){\makebox(0,0)[rb]{\smash{{\SetFigFont{10}{12.0}{\rmdefault}{\mddefault}{\updefault}$\neg x_1$}}}}
\put(2265,60){\makebox(0,0)[rb]{\smash{{\SetFigFont{10}{12.0}{\rmdefault}{\mddefault}{\updefault}$z$}}}}
\put(1215,60){\makebox(0,0)[rb]{\smash{{\SetFigFont{10}{12.0}{\rmdefault}{\mddefault}{\updefault}$\neg z$}}}}
\put(3240,2160){\makebox(0,0)[rb]{\smash{{\SetFigFont{10}{12.0}{\rmdefault}{\mddefault}{\updefault}$\neg x_n$}}}}
\put(3990,2160){\makebox(0,0)[lb]{\smash{{\SetFigFont{10}{12.0}{\rmdefault}{\mddefault}{\updefault}$x_n$}}}}
\put(2190.000,1110.000){\arc{750.000}{4.0689}{5.3559}}
\end{picture}
}
\caption{}\label{509-BF102}
\end{figure}

\begin{figure}
\centering
\setlength{\unitlength}{0.00083333in}
\begingroup\makeatletter\ifx\SetFigFont\undefined%
\gdef\SetFigFont#1#2#3#4#5{%
  \reset@font\fontsize{#1}{#2pt}%
  \fontfamily{#3}\fontseries{#4}\fontshape{#5}%
  \selectfont}%
\fi\endgroup%
{\renewcommand{\dashlinestretch}{30}
\begin{picture}(4230,1722)(0,-10)
\put(3090,60){\makebox(0,0)[lb]{\smash{{\SetFigFont{10}{12.0}{\rmdefault}{\mddefault}{\updefault}$z$}}}}
\path(2790,1485)(2040,285)
\blacken\path(2078.160,402.660)(2040.000,285.000)(2129.040,370.860)(2078.160,402.660)
\path(1745,515)(1805,425)
\blacken\path(1713.474,508.205)(1805.000,425.000)(1763.397,541.487)(1713.474,508.205)
\path(2175,505)(2120,410)
\blacken\path(2154.162,528.882)(2120.000,410.000)(2206.087,498.820)(2154.162,528.882)
\blacken\path(210.000,1665.000)(90.000,1635.000)(210.000,1605.000)(210.000,1665.000)
\path(90,1635)(765,1635)
\blacken\path(645.000,1605.000)(765.000,1635.000)(645.000,1665.000)(645.000,1605.000)
\blacken\path(360.000,1665.000)(240.000,1635.000)(360.000,1605.000)(360.000,1665.000)
\path(240,1635)(615,1635)
\blacken\path(495.000,1605.000)(615.000,1635.000)(495.000,1665.000)(495.000,1605.000)
\blacken\path(3585.000,1665.000)(3465.000,1635.000)(3585.000,1605.000)(3585.000,1665.000)
\path(3465,1635)(4140,1635)
\blacken\path(4020.000,1605.000)(4140.000,1635.000)(4020.000,1665.000)(4020.000,1605.000)
\blacken\path(3735.000,1665.000)(3615.000,1635.000)(3735.000,1605.000)(3735.000,1665.000)
\path(3615,1635)(3990,1635)
\blacken\path(3870.000,1605.000)(3990.000,1635.000)(3870.000,1665.000)(3870.000,1605.000)
\blacken\path(2460.000,165.000)(2340.000,135.000)(2460.000,105.000)(2460.000,165.000)
\path(2340,135)(3015,135)
\blacken\path(2895.000,105.000)(3015.000,135.000)(2895.000,165.000)(2895.000,105.000)
\blacken\path(2610.000,165.000)(2490.000,135.000)(2610.000,105.000)(2610.000,165.000)
\path(2490,135)(2865,135)
\blacken\path(2745.000,105.000)(2865.000,135.000)(2745.000,165.000)(2745.000,105.000)
\put(1140,1560){\makebox(0,0)[rb]{\smash{{\SetFigFont{10}{12.0}{\rmdefault}{\mddefault}{\updefault}$x_1$}}}}
\put(15,1560){\makebox(0,0)[rb]{\smash{{\SetFigFont{10}{12.0}{\rmdefault}{\mddefault}{\updefault}$\neg x_1$}}}}
\put(2065,1560){\makebox(0,0)[b]{\smash{{\SetFigFont{10}{12.0}{\rmdefault}{\mddefault}{\updefault}$\comma$}}}}
\put(2715,1560){\makebox(0,0)[lb]{\smash{{\SetFigFont{10}{12.0}{\rmdefault}{\mddefault}{\updefault}$\neg x_n$}}}}
\put(4215,1560){\makebox(0,0)[lb]{\smash{{\SetFigFont{10}{12.0}{\rmdefault}{\mddefault}{\updefault}$x_n$}}}}
\put(1890,60){\makebox(0,0)[lb]{\smash{{\SetFigFont{10}{12.0}{\rmdefault}{\mddefault}{\updefault}$\neg z$ }}}}
\path(1140,1485)(1890,285)
\blacken\path(1800.960,370.860)(1890.000,285.000)(1851.840,402.660)(1800.960,370.860)
\end{picture}
}
\caption{}\label{509-BF103}
\end{figure}

Notice the similarity between Figure \ref{509-BF103} and Figure \ref{509-BF9}.  If we let $y_1=\neg x_1\comma\allowbreak y_n=\neg x_n$ and $y=\neg z$ and change the attacks $x_i\tO y_i$ into $y\LtO z$, the two figures become the same. Note that change the atatcks ``$\tO$'' in Figure \ref{509-BF9} into bidirectional attacks ``$\LtO$\'' does not affect the job that Figure \ref{509-BF9} does.  We still have the same input/output relation between $\bigwedge_i x_i$ and $z$, namely $\bigwedge x_i \tO z$.

So we reduce the above attack $\wedge x_i \tO z$ to the single attacks $\neg x_i \tO \neg z, i=1\comma n$.

To do this, however, by utilising the existing points $\{\neg x_i, z\}$ as auxiliary points is a mistake. We must use only new, disjoint set of auxiliary points. Otherwise we get the wrong results.

Consider Figure \ref{509-BF104}. In this figure $\top$ is ``in'', $\neg a=\neg b=$ ``out'', $\neg g=$ ``out'' and since $\{a,b\}$ jointly attack $g$, $g$ must also be out. This contradicts the notation implying that at least one of $\{x, \neg x\}$ must be in.  If we ignore this restriction, we get that both $g$ and $\neg g$ are ``out''.  In this case think of $\neg x$ as just another $x'$.

\begin{figure}
\centering
\setlength{\unitlength}{0.00083333in}
\begingroup\makeatletter\ifx\SetFigFont\undefined%
\gdef\SetFigFont#1#2#3#4#5{%
  \reset@font\fontsize{#1}{#2pt}%
  \fontfamily{#3}\fontseries{#4}\fontshape{#5}%
  \selectfont}%
\fi\endgroup%
{\renewcommand{\dashlinestretch}{30}
\begin{picture}(3011,3181)(0,-10)
\put(1571,468){\makebox(0,0)[lb]{\smash{{\SetFigFont{10}{12.0}{\rmdefault}{\mddefault}{\updefault}$g$}}}}
\blacken\path(2426.000,2523.000)(2396.000,2643.000)(2366.000,2523.000)(2426.000,2523.000)
\path(2396,2643)(2396,1893)
\blacken\path(2366.000,2013.000)(2396.000,1893.000)(2426.000,2013.000)(2366.000,2013.000)
\blacken\path(1226.000,2373.000)(1196.000,2493.000)(1166.000,2373.000)(1226.000,2373.000)
\path(1196,2493)(1196,2043)
\blacken\path(1166.000,2163.000)(1196.000,2043.000)(1226.000,2163.000)(1166.000,2163.000)
\blacken\path(2426.000,2373.000)(2396.000,2493.000)(2366.000,2373.000)(2426.000,2373.000)
\path(2396,2493)(2396,2043)
\blacken\path(2366.000,2163.000)(2396.000,2043.000)(2426.000,2163.000)(2366.000,2163.000)
\path(2996,1143)(2396,1593)
\blacken\path(2510.000,1545.000)(2396.000,1593.000)(2474.000,1497.000)(2510.000,1545.000)
\path(2996,1143)(1196,1593)
\blacken\path(1319.693,1593.000)(1196.000,1593.000)(1305.141,1534.791)(1319.693,1593.000)
\path(2581,1453)(2521,1508)
\blacken\path(2629.730,1449.028)(2521.000,1508.000)(2589.187,1404.798)(2629.730,1449.028)
\path(1431,1533)(1356,1548)
\blacken\path(1479.553,1553.883)(1356.000,1548.000)(1467.786,1495.049)(1479.553,1553.883)
\path(771,208)(701,278)
\blacken\path(807.066,214.360)(701.000,278.000)(764.640,171.934)(807.066,214.360)
\blacken\path(866.000,573.000)(746.000,543.000)(866.000,513.000)(866.000,573.000)
\path(746,543)(1496,543)
\blacken\path(1376.000,513.000)(1496.000,543.000)(1376.000,573.000)(1376.000,513.000)
\blacken\path(1016.000,573.000)(896.000,543.000)(1016.000,513.000)(1016.000,573.000)
\path(896,543)(1346,543)
\blacken\path(1226.000,513.000)(1346.000,543.000)(1226.000,573.000)(1226.000,513.000)
\path(1381,788)(1456,723)
\blacken\path(1345.669,778.921)(1456.000,723.000)(1384.965,824.262)(1345.669,778.921)
\path(2396,2868)(2395,2868)(2393,2869)
	(2388,2870)(2382,2872)(2372,2875)
	(2359,2879)(2342,2883)(2322,2889)
	(2297,2896)(2269,2904)(2237,2913)
	(2201,2923)(2161,2933)(2118,2945)
	(2072,2957)(2023,2970)(1972,2983)
	(1918,2996)(1863,3009)(1806,3023)
	(1748,3036)(1690,3049)(1630,3062)
	(1570,3075)(1510,3086)(1449,3098)
	(1388,3108)(1328,3118)(1267,3126)
	(1206,3134)(1145,3141)(1084,3146)
	(1023,3150)(962,3153)(901,3154)
	(840,3154)(779,3152)(719,3147)
	(659,3141)(601,3133)(544,3122)
	(489,3109)(436,3093)(382,3073)
	(333,3051)(288,3028)(247,3004)
	(211,2980)(179,2957)(151,2935)
	(127,2913)(106,2893)(88,2875)
	(73,2858)(60,2842)(49,2828)
	(40,2815)(33,2803)(27,2792)
	(23,2781)(19,2771)(16,2761)
	(14,2751)(13,2741)(12,2730)
	(12,2719)(12,2706)(12,2691)
	(13,2675)(15,2657)(17,2636)
	(19,2612)(23,2586)(27,2556)
	(33,2523)(40,2485)(48,2444)
	(59,2400)(72,2351)(88,2298)
	(107,2242)(130,2184)(156,2123)
	(184,2067)(214,2010)(247,1953)
	(282,1897)(319,1841)(357,1787)
	(397,1734)(437,1682)(479,1631)
	(521,1581)(564,1533)(607,1485)
	(651,1438)(696,1392)(741,1347)
	(786,1303)(832,1259)(878,1216)
	(924,1173)(970,1132)(1016,1091)
	(1062,1051)(1107,1012)(1151,974)
	(1195,937)(1237,902)(1277,868)
	(1316,836)(1353,806)(1387,778)
	(1419,753)(1448,729)(1474,709)
	(1497,691)(1516,675)(1533,662)
	(1546,651)(1557,643)(1564,637)(1576,628)
\blacken\path(1462.000,676.000)(1576.000,628.000)(1498.000,724.000)(1462.000,676.000)
\path(1196,2868)(1194,2869)(1191,2871)
	(1184,2875)(1173,2881)(1159,2889)
	(1141,2899)(1119,2911)(1095,2924)
	(1067,2938)(1038,2953)(1007,2968)
	(975,2982)(942,2996)(909,3009)
	(876,3020)(842,3030)(809,3039)
	(775,3045)(740,3048)(706,3049)
	(672,3046)(638,3039)(606,3028)
	(579,3014)(555,2998)(535,2982)
	(517,2966)(503,2950)(492,2937)
	(483,2924)(476,2914)(471,2905)
	(467,2897)(464,2890)(462,2884)
	(461,2877)(460,2870)(459,2862)
	(459,2853)(458,2841)(458,2827)
	(457,2809)(456,2787)(455,2761)
	(454,2729)(454,2692)(455,2650)
	(457,2604)(461,2553)(466,2507)
	(473,2462)(480,2418)(488,2376)
	(495,2339)(502,2305)(509,2275)
	(515,2249)(520,2226)(525,2207)
	(529,2190)(533,2176)(536,2163)
	(540,2152)(543,2141)(547,2130)
	(551,2118)(556,2105)(561,2090)
	(568,2072)(576,2051)(586,2026)
	(597,1998)(611,1964)(627,1926)
	(645,1883)(666,1835)(689,1784)
	(714,1729)(741,1673)(769,1617)
	(798,1563)(826,1510)(854,1461)
	(882,1414)(909,1370)(935,1328)
	(961,1289)(986,1251)(1011,1216)
	(1035,1181)(1059,1149)(1083,1118)
	(1106,1088)(1128,1059)(1150,1032)
	(1171,1006)(1190,982)(1209,959)
	(1226,939)(1241,921)(1254,906)
	(1265,893)(1274,883)(1281,875)
	(1285,869)(1289,866)(1290,864)(1291,863)
\path(2996,918)(2996,917)(2996,914)
	(2997,909)(2997,902)(2997,891)
	(2998,877)(2998,860)(2998,840)
	(2997,817)(2996,792)(2994,765)
	(2991,737)(2987,707)(2981,676)
	(2974,645)(2964,613)(2953,581)
	(2938,549)(2921,516)(2900,482)
	(2876,449)(2847,415)(2813,380)
	(2774,345)(2728,310)(2677,276)
	(2621,243)(2571,217)(2520,194)
	(2469,172)(2419,152)(2372,135)
	(2327,119)(2285,106)(2247,94)
	(2211,84)(2179,76)(2150,68)
	(2123,62)(2099,57)(2076,53)
	(2055,49)(2035,46)(2015,43)
	(1995,40)(1974,38)(1952,36)
	(1929,33)(1903,31)(1875,29)
	(1845,26)(1810,24)(1772,21)
	(1731,19)(1685,17)(1635,15)
	(1582,13)(1525,12)(1466,13)
	(1406,15)(1346,18)(1280,24)
	(1217,33)(1159,43)(1105,55)
	(1056,69)(1010,83)(969,98)
	(931,115)(896,132)(864,149)
	(834,167)(806,186)(781,205)
	(757,224)(734,243)(714,262)
	(694,280)(677,298)(661,315)
	(647,331)(635,345)(624,357)
	(616,368)(609,376)(604,383)(596,393)
\blacken\path(694.389,318.037)(596.000,393.000)(647.537,280.555)(694.389,318.037)
\put(1196,1668){\makebox(0,0)[b]{\smash{{\SetFigFont{10}{12.0}{\rmdefault}{\mddefault}{\updefault}$\neg a$}}}}
\put(2396,1668){\makebox(0,0)[b]{\smash{{\SetFigFont{10}{12.0}{\rmdefault}{\mddefault}{\updefault}$\neg b$}}}}
\put(1196,2718){\makebox(0,0)[b]{\smash{{\SetFigFont{10}{12.0}{\rmdefault}{\mddefault}{\updefault}$a$}}}}
\put(2396,2718){\makebox(0,0)[b]{\smash{{\SetFigFont{10}{12.0}{\rmdefault}{\mddefault}{\updefault}$b$}}}}
\put(2996,993){\makebox(0,0)[b]{\smash{{\SetFigFont{10}{12.0}{\rmdefault}{\mddefault}{\updefault}$\top$}}}}
\put(596,468){\makebox(0,0)[b]{\smash{{\SetFigFont{10}{12.0}{\rmdefault}{\mddefault}{\updefault}$\neg g$}}}}
\blacken\path(1226.000,2523.000)(1196.000,2643.000)(1166.000,2523.000)(1226.000,2523.000)
\path(1196,2643)(1196,1893)
\blacken\path(1166.000,2013.000)(1196.000,1893.000)(1226.000,2013.000)(1166.000,2013.000)
\end{picture}
}

\caption{}\label{509-BF104}
\end{figure}

Figure \ref{509-BF105} eliminates the joint attack $\{a,b\}\tO g$ by using $\neg a, \neg b, \neg g$ as auxiliary points, instead of using copletely new points.  What we get is the wrong result.

\begin{figure}
\centering
\setlength{\unitlength}{0.00083333in}
\begingroup\makeatletter\ifx\SetFigFont\undefined%
\gdef\SetFigFont#1#2#3#4#5{%
  \reset@font\fontsize{#1}{#2pt}%
  \fontfamily{#3}\fontseries{#4}\fontshape{#5}%
  \selectfont}%
\fi\endgroup%
{\renewcommand{\dashlinestretch}{30}
\begin{picture}(2692,2856)(0,-10)
\put(1252,468){\makebox(0,0)[b]{\smash{{\SetFigFont{10}{12.0}{\rmdefault}{\mddefault}{\updefault}$g$}}}}
\blacken\path(2107.000,2523.000)(2077.000,2643.000)(2047.000,2523.000)(2107.000,2523.000)
\path(2077,2643)(2077,1893)
\blacken\path(2047.000,2013.000)(2077.000,1893.000)(2107.000,2013.000)(2047.000,2013.000)
\blacken\path(907.000,2373.000)(877.000,2493.000)(847.000,2373.000)(907.000,2373.000)
\path(877,2493)(877,2043)
\blacken\path(847.000,2163.000)(877.000,2043.000)(907.000,2163.000)(847.000,2163.000)
\blacken\path(2107.000,2373.000)(2077.000,2493.000)(2047.000,2373.000)(2107.000,2373.000)
\path(2077,2493)(2077,2043)
\blacken\path(2047.000,2163.000)(2077.000,2043.000)(2107.000,2163.000)(2047.000,2163.000)
\path(2677,1143)(2077,1593)
\blacken\path(2191.000,1545.000)(2077.000,1593.000)(2155.000,1497.000)(2191.000,1545.000)
\path(2677,1143)(877,1593)
\blacken\path(1000.693,1593.000)(877.000,1593.000)(986.141,1534.791)(1000.693,1593.000)
\path(2262,1453)(2202,1508)
\blacken\path(2310.730,1449.028)(2202.000,1508.000)(2270.187,1404.798)(2310.730,1449.028)
\path(1112,1533)(1037,1548)
\blacken\path(1160.553,1553.883)(1037.000,1548.000)(1148.786,1495.049)(1160.553,1553.883)
\path(452,208)(382,278)
\blacken\path(488.066,214.360)(382.000,278.000)(445.640,171.934)(488.066,214.360)
\path(2022,1598)(282,648)
\blacken\path(372.948,731.836)(282.000,648.000)(401.700,679.174)(372.948,731.836)
\path(862,1603)(267,663)
\blacken\path(305.832,780.440)(267.000,663.000)(356.529,748.349)(305.832,780.440)
\path(397,873)(352,798)
\blacken\path(388.015,916.334)(352.000,798.000)(439.464,885.464)(388.015,916.334)
\path(517,773)(412,728)
\blacken\path(510.480,802.845)(412.000,728.000)(534.115,747.696)(510.480,802.845)
\blacken\path(547.000,573.000)(427.000,543.000)(547.000,513.000)(547.000,573.000)
\path(427,543)(1102,543)
\blacken\path(982.000,513.000)(1102.000,543.000)(982.000,573.000)(982.000,513.000)
\blacken\path(697.000,573.000)(577.000,543.000)(697.000,513.000)(697.000,573.000)
\path(577,543)(952,543)
\blacken\path(832.000,513.000)(952.000,543.000)(832.000,573.000)(832.000,513.000)
\path(2677,918)(2677,917)(2677,914)
	(2678,909)(2678,902)(2678,891)
	(2679,877)(2679,860)(2679,840)
	(2678,817)(2677,792)(2675,765)
	(2672,737)(2668,707)(2662,676)
	(2655,645)(2645,613)(2634,581)
	(2619,549)(2602,516)(2581,482)
	(2557,449)(2528,415)(2494,380)
	(2455,345)(2409,310)(2358,276)
	(2302,243)(2252,217)(2201,194)
	(2150,172)(2100,152)(2053,135)
	(2008,119)(1966,106)(1928,94)
	(1892,84)(1860,76)(1831,68)
	(1804,62)(1780,57)(1757,53)
	(1736,49)(1716,46)(1696,43)
	(1676,40)(1655,38)(1633,36)
	(1610,33)(1584,31)(1556,29)
	(1526,26)(1491,24)(1453,21)
	(1412,19)(1366,17)(1316,15)
	(1263,13)(1206,12)(1147,13)
	(1087,15)(1027,18)(961,24)
	(898,33)(840,43)(786,55)
	(737,69)(691,83)(650,98)
	(612,115)(577,132)(545,149)
	(515,167)(487,186)(462,205)
	(438,224)(415,243)(395,262)
	(375,280)(358,298)(342,315)
	(328,331)(316,345)(305,357)
	(297,368)(290,376)(285,383)(277,393)
\blacken\path(375.389,318.037)(277.000,393.000)(328.537,280.555)(375.389,318.037)
\put(877,1668){\makebox(0,0)[b]{\smash{{\SetFigFont{10}{12.0}{\rmdefault}{\mddefault}{\updefault}$\neg a$}}}}
\put(2077,1668){\makebox(0,0)[b]{\smash{{\SetFigFont{10}{12.0}{\rmdefault}{\mddefault}{\updefault}$\neg b$}}}}
\put(877,2718){\makebox(0,0)[b]{\smash{{\SetFigFont{10}{12.0}{\rmdefault}{\mddefault}{\updefault}$a$}}}}
\put(2077,2718){\makebox(0,0)[b]{\smash{{\SetFigFont{10}{12.0}{\rmdefault}{\mddefault}{\updefault}$b$}}}}
\put(2677,993){\makebox(0,0)[b]{\smash{{\SetFigFont{10}{12.0}{\rmdefault}{\mddefault}{\updefault}$\top$}}}}
\put(277,468){\makebox(0,0)[b]{\smash{{\SetFigFont{10}{12.0}{\rmdefault}{\mddefault}{\updefault}$\neg g$}}}}
\blacken\path(907.000,2523.000)(877.000,2643.000)(847.000,2523.000)(907.000,2523.000)
\path(877,2643)(877,1893)
\blacken\path(847.000,2013.000)(877.000,1893.000)(907.000,2013.000)(847.000,2013.000)
\end{picture}
}

\caption{}\label{509-BF105}
\end{figure}

What we get is $\top=$ ``in'', $\neg a=$ ``out'' $=\neg b=\neg g$. $a=b=g=$ in.

Figure \ref{509-BF105} eliminates the joint attacks using auxiliary points which are completely new, as shown in Figure \ref{509-BF9}.

We get the correct result, same as in Figure \ref{509-BF104}.  We hae $\top =a=b=$ ``in''. so $x=y=$ ``out''.  So $z=$ ``in'' and so $g=$ ``out''.

\begin{figure}
\centering
\setlength{\unitlength}{0.00083333in}
\begingroup\makeatletter\ifx\SetFigFont\undefined%
\gdef\SetFigFont#1#2#3#4#5{%
  \reset@font\fontsize{#1}{#2pt}%
  \fontfamily{#3}\fontseries{#4}\fontshape{#5}%
  \selectfont}%
\fi\endgroup%
{\renewcommand{\dashlinestretch}{30}
\begin{picture}(4056,4024)(0,-10)
\put(563,453){\makebox(0,0)[rb]{\smash{{\SetFigFont{10}{12.0}{\rmdefault}{\mddefault}{\updefault}$\neg g$}}}}
\blacken\path(3358.000,3358.000)(3328.000,3478.000)(3298.000,3358.000)(3358.000,3358.000)
\path(3328,3478)(3328,2728)
\blacken\path(3298.000,2848.000)(3328.000,2728.000)(3358.000,2848.000)(3298.000,2848.000)
\blacken\path(2158.000,3208.000)(2128.000,3328.000)(2098.000,3208.000)(2158.000,3208.000)
\path(2128,3328)(2128,2878)
\blacken\path(2098.000,2998.000)(2128.000,2878.000)(2158.000,2998.000)(2098.000,2998.000)
\blacken\path(3358.000,3208.000)(3328.000,3328.000)(3298.000,3208.000)(3358.000,3208.000)
\path(3328,3328)(3328,2878)
\blacken\path(3298.000,2998.000)(3328.000,2878.000)(3358.000,2998.000)(3298.000,2998.000)
\path(3928,1978)(3328,2428)
\blacken\path(3442.000,2380.000)(3328.000,2428.000)(3406.000,2332.000)(3442.000,2380.000)
\path(3928,1978)(2128,2428)
\blacken\path(2251.693,2428.000)(2128.000,2428.000)(2237.141,2369.791)(2251.693,2428.000)
\path(3513,2288)(3453,2343)
\blacken\path(3561.730,2284.028)(3453.000,2343.000)(3521.187,2239.798)(3561.730,2284.028)
\path(2363,2368)(2288,2383)
\blacken\path(2411.553,2388.883)(2288.000,2383.000)(2399.786,2330.049)(2411.553,2388.883)
\path(1378,1003)(1378,628)
\blacken\path(1348.000,748.000)(1378.000,628.000)(1408.000,748.000)(1348.000,748.000)
\path(1378,853)(1378,778)
\blacken\path(1348.000,898.000)(1378.000,778.000)(1408.000,898.000)(1348.000,898.000)
\path(1378,1678)(1378,1228)
\blacken\path(1348.000,1348.000)(1378.000,1228.000)(1408.000,1348.000)(1348.000,1348.000)
\path(1378,1528)(1378,1378)
\blacken\path(1348.000,1498.000)(1378.000,1378.000)(1408.000,1498.000)(1348.000,1498.000)
\path(2093,3473)(1393,1938)
\blacken\path(1415.495,2059.631)(1393.000,1938.000)(1470.086,2034.735)(1415.495,2059.631)
\path(1493,2148)(1468,2068)
\blacken\path(1475.159,2191.486)(1468.000,2068.000)(1532.427,2173.589)(1475.159,2191.486)
\path(1348,2748)(1363,2668)
\blacken\path(1311.399,2780.416)(1363.000,2668.000)(1370.372,2791.473)(1311.399,2780.416)
\path(1743,1118)(1663,1073)
\blacken\path(1752.881,1157.979)(1663.000,1073.000)(1782.297,1105.684)(1752.881,1157.979)
\blacken\path(723.000,1863.000)(603.000,1833.000)(723.000,1803.000)(723.000,1863.000)
\path(603,1833)(1243,1833)
\blacken\path(1123.000,1803.000)(1243.000,1833.000)(1123.000,1863.000)(1123.000,1803.000)
\blacken\path(808.000,1863.000)(688.000,1833.000)(808.000,1803.000)(808.000,1863.000)
\path(688,1833)(1098,1833)
\blacken\path(978.000,1803.000)(1098.000,1833.000)(978.000,1863.000)(978.000,1803.000)
\blacken\path(743.000,1103.000)(623.000,1073.000)(743.000,1043.000)(743.000,1103.000)
\path(623,1073)(1263,1073)
\blacken\path(1143.000,1043.000)(1263.000,1073.000)(1143.000,1103.000)(1143.000,1043.000)
\blacken\path(838.000,1103.000)(718.000,1073.000)(838.000,1043.000)(838.000,1103.000)
\path(718,1073)(1128,1073)
\blacken\path(1008.000,1043.000)(1128.000,1073.000)(1008.000,1103.000)(1008.000,1043.000)
\blacken\path(748.000,2398.000)(628.000,2368.000)(748.000,2338.000)(748.000,2398.000)
\path(628,2368)(1268,2368)
\blacken\path(1148.000,2338.000)(1268.000,2368.000)(1148.000,2398.000)(1148.000,2338.000)
\blacken\path(888.000,2403.000)(768.000,2373.000)(888.000,2343.000)(888.000,2403.000)
\path(768,2373)(1178,2373)
\blacken\path(1058.000,2343.000)(1178.000,2373.000)(1058.000,2403.000)(1058.000,2343.000)
\blacken\path(738.000,533.000)(618.000,503.000)(738.000,473.000)(738.000,533.000)
\path(618,503)(1258,503)
\blacken\path(1138.000,473.000)(1258.000,503.000)(1138.000,533.000)(1138.000,473.000)
\blacken\path(888.000,528.000)(768.000,498.000)(888.000,468.000)(888.000,528.000)
\path(768,498)(1178,498)
\blacken\path(1058.000,468.000)(1178.000,498.000)(1058.000,528.000)(1058.000,468.000)
\path(725,283)(613,336)
\blacken\path(734.300,311.788)(613.000,336.000)(708.636,257.554)(734.300,311.788)
\path(3928,1753)(3928,1752)(3929,1750)
	(3931,1747)(3934,1741)(3937,1734)
	(3942,1723)(3948,1710)(3955,1695)
	(3963,1676)(3971,1655)(3980,1632)
	(3989,1607)(3999,1580)(4008,1550)
	(4016,1520)(4024,1488)(4031,1455)
	(4037,1421)(4041,1387)(4043,1351)
	(4044,1314)(4042,1277)(4037,1238)
	(4030,1198)(4020,1157)(4006,1115)
	(3988,1071)(3966,1025)(3939,978)
	(3906,930)(3869,880)(3826,829)
	(3778,778)(3733,735)(3687,694)
	(3639,654)(3591,616)(3545,580)
	(3501,547)(3458,516)(3418,488)
	(3381,462)(3347,439)(3316,417)
	(3287,398)(3261,380)(3237,364)
	(3215,350)(3194,336)(3175,324)
	(3157,312)(3139,301)(3122,290)
	(3104,280)(3085,270)(3066,260)
	(3044,249)(3021,238)(2995,227)
	(2966,215)(2934,202)(2898,189)
	(2858,175)(2813,161)(2764,145)
	(2710,130)(2650,114)(2586,97)
	(2516,81)(2443,66)(2365,52)
	(2285,39)(2203,28)(2124,20)
	(2046,15)(1970,12)(1895,12)
	(1822,13)(1752,16)(1684,21)
	(1618,27)(1555,35)(1494,45)
	(1434,55)(1377,66)(1321,79)
	(1267,92)(1214,106)(1162,121)
	(1112,137)(1063,153)(1015,170)
	(969,187)(923,204)(879,221)
	(837,238)(796,255)(757,272)
	(720,288)(686,304)(653,318)
	(624,332)(597,345)(572,356)
	(551,367)(533,375)(518,383)
	(506,389)(496,394)(489,398)(478,403)
\blacken\path(599.658,380.655)(478.000,403.000)(574.830,326.033)(599.658,380.655)
\path(1518,2343)(1519,2343)(1522,2343)
	(1526,2342)(1533,2342)(1542,2341)
	(1553,2339)(1567,2337)(1582,2334)
	(1598,2330)(1616,2324)(1633,2317)
	(1652,2308)(1670,2298)(1688,2284)
	(1707,2268)(1725,2248)(1743,2225)
	(1762,2196)(1780,2162)(1797,2123)
	(1813,2078)(1825,2037)(1836,1995)
	(1845,1955)(1853,1916)(1859,1880)
	(1865,1847)(1869,1818)(1873,1792)
	(1876,1768)(1879,1747)(1882,1727)
	(1884,1709)(1886,1691)(1887,1674)
	(1888,1655)(1889,1636)(1890,1615)
	(1890,1591)(1889,1565)(1888,1536)
	(1886,1504)(1882,1468)(1878,1430)
	(1871,1389)(1863,1348)(1853,1308)
	(1840,1268)(1824,1232)(1808,1201)
	(1790,1174)(1772,1152)(1753,1133)
	(1735,1118)(1715,1105)(1696,1095)
	(1676,1087)(1656,1080)(1637,1075)
	(1618,1071)(1600,1068)(1582,1066)
	(1567,1065)(1553,1064)(1542,1063)
	(1532,1063)(1518,1063)
\blacken\path(1638.000,1093.000)(1518.000,1063.000)(1638.000,1033.000)(1638.000,1093.000)
\path(3303,3713)(3302,3713)(3299,3714)
	(3295,3716)(3287,3719)(3277,3723)
	(3263,3728)(3245,3735)(3223,3743)
	(3198,3753)(3168,3764)(3135,3776)
	(3098,3789)(3057,3803)(3014,3818)
	(2968,3833)(2920,3849)(2870,3864)
	(2818,3880)(2765,3896)(2712,3910)
	(2657,3925)(2602,3938)(2547,3951)
	(2491,3962)(2435,3972)(2379,3981)
	(2323,3988)(2267,3993)(2211,3996)
	(2154,3997)(2098,3996)(2042,3992)
	(1985,3985)(1929,3975)(1873,3962)
	(1818,3946)(1765,3925)(1713,3901)
	(1663,3873)(1614,3839)(1570,3802)
	(1530,3763)(1494,3721)(1463,3678)
	(1435,3633)(1411,3587)(1391,3540)
	(1374,3493)(1360,3445)(1348,3396)
	(1339,3347)(1332,3297)(1328,3247)
	(1324,3197)(1323,3146)(1322,3096)
	(1323,3045)(1326,2995)(1328,2946)
	(1332,2898)(1336,2851)(1341,2806)
	(1346,2764)(1351,2724)(1356,2686)
	(1361,2653)(1365,2622)(1369,2596)
	(1373,2573)(1376,2555)(1378,2540)
	(1380,2529)(1383,2513)
\blacken\path(1331.399,2625.416)(1383.000,2513.000)(1390.372,2636.473)(1331.399,2625.416)
\put(2128,2503){\makebox(0,0)[b]{\smash{{\SetFigFont{10}{12.0}{\rmdefault}{\mddefault}{\updefault}$\neg a$}}}}
\put(3328,2503){\makebox(0,0)[b]{\smash{{\SetFigFont{10}{12.0}{\rmdefault}{\mddefault}{\updefault}$\neg b$}}}}
\put(2128,3553){\makebox(0,0)[b]{\smash{{\SetFigFont{10}{12.0}{\rmdefault}{\mddefault}{\updefault}$a$}}}}
\put(3328,3553){\makebox(0,0)[b]{\smash{{\SetFigFont{10}{12.0}{\rmdefault}{\mddefault}{\updefault}$b$}}}}
\put(3928,1828){\makebox(0,0)[b]{\smash{{\SetFigFont{10}{12.0}{\rmdefault}{\mddefault}{\updefault}$\top$}}}}
\put(1368,1028){\makebox(0,0)[b]{\smash{{\SetFigFont{10}{12.0}{\rmdefault}{\mddefault}{\updefault}$z$}}}}
\put(1383,2298){\makebox(0,0)[b]{\smash{{\SetFigFont{10}{12.0}{\rmdefault}{\mddefault}{\updefault}$y$}}}}
\put(1373,1768){\makebox(0,0)[b]{\smash{{\SetFigFont{10}{12.0}{\rmdefault}{\mddefault}{\updefault}$x$}}}}
\put(1378,453){\makebox(0,0)[b]{\smash{{\SetFigFont{10}{12.0}{\rmdefault}{\mddefault}{\updefault}$g$}}}}
\put(573,2303){\makebox(0,0)[rb]{\smash{{\SetFigFont{10}{12.0}{\rmdefault}{\mddefault}{\updefault}$\neg y$}}}}
\put(553,1793){\makebox(0,0)[rb]{\smash{{\SetFigFont{10}{12.0}{\rmdefault}{\mddefault}{\updefault}$\neg x$}}}}
\put(588,1028){\makebox(0,0)[rb]{\smash{{\SetFigFont{10}{12.0}{\rmdefault}{\mddefault}{\updefault}$\neg z$}}}}
\blacken\path(2158.000,3358.000)(2128.000,3478.000)(2098.000,3358.000)(2158.000,3358.000)
\path(2128,3478)(2128,2728)
\blacken\path(2098.000,2848.000)(2128.000,2728.000)(2158.000,2848.000)(2098.000,2848.000)
\end{picture}
}

\caption{}\label{509-BF106}
\end{figure}

We note that it is for this reason that we introduce in Apendix C for the Boolean attack formations.  These encapsulate the new auxiliary points inside the formation.

Note also that the presence of the $\{w, \neg w\}$ pairs can allow us to possibly economise and use less auxiliary points, as Figure \ref{509-BF107} shows. We economise by letting $\neg x =a$ and $\neg y =b$. 

There is no advantage, however, in economising. what is important is that joint attacks can be eliminated in a systematic way.

\begin{figure}
\centering
\setlength{\unitlength}{0.00083333in}
\begingroup\makeatletter\ifx\SetFigFont\undefined%
\gdef\SetFigFont#1#2#3#4#5{%
  \reset@font\fontsize{#1}{#2pt}%
  \fontfamily{#3}\fontseries{#4}\fontshape{#5}%
  \selectfont}%
\fi\endgroup%
{\renewcommand{\dashlinestretch}{30}
\begin{picture}(3527,3230)(0,-10)
\put(70,492){\makebox(0,0)[rb]{\smash{{\SetFigFont{10}{12.0}{\rmdefault}{\mddefault}{\updefault}$\neg g$}}}}
\blacken\path(2645.000,2897.000)(2615.000,3017.000)(2585.000,2897.000)(2645.000,2897.000)
\path(2615,3017)(2615,2267)
\blacken\path(2585.000,2387.000)(2615.000,2267.000)(2645.000,2387.000)(2585.000,2387.000)
\blacken\path(1445.000,2747.000)(1415.000,2867.000)(1385.000,2747.000)(1445.000,2747.000)
\path(1415,2867)(1415,2417)
\blacken\path(1385.000,2537.000)(1415.000,2417.000)(1445.000,2537.000)(1385.000,2537.000)
\blacken\path(2645.000,2747.000)(2615.000,2867.000)(2585.000,2747.000)(2645.000,2747.000)
\path(2615,2867)(2615,2417)
\blacken\path(2585.000,2537.000)(2615.000,2417.000)(2645.000,2537.000)(2585.000,2537.000)
\path(3215,1517)(2615,1967)
\blacken\path(2729.000,1919.000)(2615.000,1967.000)(2693.000,1871.000)(2729.000,1919.000)
\path(3215,1517)(1415,1967)
\blacken\path(1538.693,1967.000)(1415.000,1967.000)(1524.141,1908.791)(1538.693,1967.000)
\path(2800,1827)(2740,1882)
\blacken\path(2848.730,1823.028)(2740.000,1882.000)(2808.187,1778.798)(2848.730,1823.028)
\path(1650,1907)(1575,1922)
\blacken\path(1698.553,1927.883)(1575.000,1922.000)(1686.786,1869.049)(1698.553,1927.883)
\path(2615,1967)(815,1517)
\blacken\path(924.141,1575.209)(815.000,1517.000)(938.693,1517.000)(924.141,1575.209)
\path(815,1292)(815,692)
\blacken\path(785.000,812.000)(815.000,692.000)(845.000,812.000)(785.000,812.000)
\path(815,917)(815,842)
\blacken\path(785.000,962.000)(815.000,842.000)(845.000,962.000)(785.000,962.000)
\path(1110,1587)(945,1547)
\blacken\path(1054.554,1604.428)(945.000,1547.000)(1068.690,1546.116)(1054.554,1604.428)
\path(1415,1967)(815,1517)
\blacken\path(893.000,1613.000)(815.000,1517.000)(929.000,1565.000)(893.000,1613.000)
\path(1015,1672)(930,1617)
\blacken\path(1014.451,1707.377)(930.000,1617.000)(1047.046,1657.003)(1014.451,1707.377)
\blacken\path(195.000,1462.000)(75.000,1432.000)(195.000,1402.000)(195.000,1462.000)
\path(75,1432)(670,1432)
\blacken\path(550.000,1402.000)(670.000,1432.000)(550.000,1462.000)(550.000,1402.000)
\blacken\path(305.000,1467.000)(185.000,1437.000)(305.000,1407.000)(305.000,1467.000)
\path(185,1437)(575,1437)
\blacken\path(455.000,1407.000)(575.000,1437.000)(455.000,1467.000)(455.000,1407.000)
\blacken\path(215.000,577.000)(95.000,547.000)(215.000,517.000)(215.000,577.000)
\path(95,547)(690,547)
\blacken\path(570.000,517.000)(690.000,547.000)(570.000,577.000)(570.000,517.000)
\blacken\path(335.000,577.000)(215.000,547.000)(335.000,517.000)(335.000,577.000)
\path(215,547)(605,547)
\blacken\path(485.000,517.000)(605.000,547.000)(485.000,577.000)(485.000,517.000)
\path(220,232)(160,297)
\blacken\path(263.438,229.172)(160.000,297.000)(219.350,188.475)(263.438,229.172)
\path(3215,1292)(3216,1292)(3217,1290)
	(3220,1288)(3225,1285)(3231,1280)
	(3240,1274)(3250,1266)(3262,1256)
	(3277,1245)(3293,1232)(3310,1217)
	(3329,1201)(3348,1184)(3368,1165)
	(3388,1145)(3408,1125)(3427,1103)
	(3445,1081)(3461,1058)(3476,1034)
	(3489,1010)(3500,986)(3508,960)
	(3513,934)(3515,907)(3514,880)
	(3508,851)(3498,822)(3483,791)
	(3462,759)(3436,726)(3403,693)
	(3364,658)(3318,623)(3265,587)
	(3219,559)(3170,532)(3120,506)
	(3068,481)(3016,457)(2965,434)
	(2914,413)(2865,392)(2817,373)
	(2771,356)(2727,339)(2685,324)
	(2645,310)(2607,297)(2571,285)
	(2536,274)(2504,264)(2472,254)
	(2441,245)(2412,236)(2383,228)
	(2354,220)(2325,212)(2297,204)
	(2268,197)(2238,189)(2207,182)
	(2175,174)(2141,166)(2105,158)
	(2068,150)(2028,141)(1985,132)
	(1940,123)(1892,114)(1841,104)
	(1787,94)(1729,84)(1669,74)
	(1606,64)(1540,54)(1472,45)
	(1402,36)(1331,29)(1260,22)
	(1190,17)(1103,13)(1020,12)
	(941,13)(868,17)(800,22)
	(737,30)(678,39)(624,50)
	(574,62)(528,75)(486,89)
	(446,104)(409,120)(374,137)
	(342,155)(312,173)(284,191)
	(257,210)(232,228)(209,247)
	(188,265)(169,283)(151,299)
	(135,315)(120,330)(108,343)
	(97,354)(88,364)(81,373)
	(75,379)(71,384)(65,392)
\blacken\path(161.000,314.000)(65.000,392.000)(113.000,278.000)(161.000,314.000)
\put(1415,2042){\makebox(0,0)[b]{\smash{{\SetFigFont{10}{12.0}{\rmdefault}{\mddefault}{\updefault}$\neg a$}}}}
\put(2615,2042){\makebox(0,0)[b]{\smash{{\SetFigFont{10}{12.0}{\rmdefault}{\mddefault}{\updefault}$\neg b$}}}}
\put(1415,3092){\makebox(0,0)[b]{\smash{{\SetFigFont{10}{12.0}{\rmdefault}{\mddefault}{\updefault}$a$}}}}
\put(2615,3092){\makebox(0,0)[b]{\smash{{\SetFigFont{10}{12.0}{\rmdefault}{\mddefault}{\updefault}$b$}}}}
\put(3215,1367){\makebox(0,0)[b]{\smash{{\SetFigFont{10}{12.0}{\rmdefault}{\mddefault}{\updefault}$\top$}}}}
\put(805,1372){\makebox(0,0)[b]{\smash{{\SetFigFont{10}{12.0}{\rmdefault}{\mddefault}{\updefault}$z$}}}}
\put(820,497){\makebox(0,0)[b]{\smash{{\SetFigFont{10}{12.0}{\rmdefault}{\mddefault}{\updefault}$g$}}}}
\put(15,1377){\makebox(0,0)[rb]{\smash{{\SetFigFont{10}{12.0}{\rmdefault}{\mddefault}{\updefault}$\neg z$}}}}
\blacken\path(1445.000,2897.000)(1415.000,3017.000)(1385.000,2897.000)(1445.000,2897.000)
\path(1415,3017)(1415,2267)
\blacken\path(1385.000,2387.000)(1415.000,2267.000)(1445.000,2387.000)(1385.000,2387.000)
\end{picture}
}

\caption{}\label{509-BF107}
\end{figure}
\end{remark}

\subsection{Boolean attack formations (BAF)}
In this appendix, we generalise the  instantiation sequence of the form of Figure \ref{509-BBF3} where $x_i, y_j$ are atomic arguments and $\Psi(a_1\comma a_n)$ is a Boolean formula in the arguments $\{a_1\comma a_n\}$.  What we see in this figure is a substitution of some  complex argumentation entity for the node $z$. Traditional abstract argumentation networks know how to handle attacks on atomic nodes $z$, they do not know how to deal with attacks on Boolean formulas. This Appendix C.2, replaces the formulas by attack formations and defines how to handle them. We thus define the notion of a Boolean Attack Formation with input and output nodes which can be substituted for nodes $z$. All attacks on $z$ go into the input point of the formation and all attacks from $z$ emanate from the output point of the formation replacing $z$.

We begin with the special case of  Boolean attack formations designed to represent conjunction of formulas. Then we define formations to represent negation of formulas. Since negations and conjunctions can generate any formula of propositional classical logic, we will have attack formation representation for any classical propositional logic formula. The general definition, therefore,  allows for general input/output formations.

\begin{figure}
\centering
\setlength{\unitlength}{0.00083333in}
\begingroup\makeatletter\ifx\SetFigFont\undefined%
\gdef\SetFigFont#1#2#3#4#5{%
  \reset@font\fontsize{#1}{#2pt}%
  \fontfamily{#3}\fontseries{#4}\fontshape{#5}%
  \selectfont}%
\fi\endgroup%
{\renewcommand{\dashlinestretch}{30}
\begin{picture}(2490,2922)(0,-10)
\put(1270,60){\makebox(0,0)[b]{\smash{{\SetFigFont{10}{12.0}{\rmdefault}{\mddefault}{\updefault}$z$ is instantiated as $I(z)=\Psi$.}}}}
\path(145,2685)(1195,2160)
\blacken\path(1074.252,2186.833)(1195.000,2160.000)(1101.085,2240.498)(1074.252,2186.833)
\path(2245,2685)(1345,2160)
\blacken\path(1433.537,2246.378)(1345.000,2160.000)(1463.770,2194.551)(1433.537,2246.378)
\path(1195,1260)(145,810)
\blacken\path(243.480,884.845)(145.000,810.000)(267.115,829.696)(243.480,884.845)
\path(1345,1260)(2245,810)
\blacken\path(2124.252,836.833)(2245.000,810.000)(2151.085,890.498)(2124.252,836.833)
\path(960,2285)(1055,2250)
\blacken\path(932.028,2263.334)(1055.000,2250.000)(952.770,2319.635)(932.028,2263.334)
\path(1555,2285)(1485,2250)
\blacken\path(1578.915,2330.498)(1485.000,2250.000)(1605.748,2276.833)(1578.915,2330.498)
\path(1995,930)(2105,870)
\blacken\path(1985.287,901.125)(2105.000,870.000)(2014.018,953.799)(1985.287,901.125)
\path(400,920)(285,870)
\blacken\path(383.087,945.359)(285.000,870.000)(407.010,890.335)(383.087,945.359)
\put(145,2760){\makebox(0,0)[b]{\smash{{\SetFigFont{10}{12.0}{\rmdefault}{\mddefault}{\updefault}$x_1$}}}}
\put(2245,2760){\makebox(0,0)[b]{\smash{{\SetFigFont{10}{12.0}{\rmdefault}{\mddefault}{\updefault}$x_m$}}}}
\put(1195,2760){\makebox(0,0)[b]{\smash{{\SetFigFont{10}{12.0}{\rmdefault}{\mddefault}{\updefault}\ldots}}}}
\put(1270,1710){\makebox(0,0)[b]{\smash{{\SetFigFont{10}{12.0}{\rmdefault}{\mddefault}{\updefault}$z=\Psi(a_1\comma a_n)$}}}}
\put(2245,585){\makebox(0,0)[b]{\smash{{\SetFigFont{10}{12.0}{\rmdefault}{\mddefault}{\updefault}$y_k$}}}}
\put(145,585){\makebox(0,0)[b]{\smash{{\SetFigFont{10}{12.0}{\rmdefault}{\mddefault}{\updefault}$y_1$}}}}
\put(1245,1732){\ellipse{2474}{824}}
\end{picture}
}

\caption{}\label{509-BBF3}
\end{figure}

\begin{definition}\label{509-BBD4}
A Boolean attack formation $\BBB\BBF$ has the form 
\[
\BBB\BBF = (S_1\cup S_2, R, \Psi(S_2))
\]
where $S_2=\{a_1\comma a_n\}$ and $\Psi(S_2)$ is a Boolean formula of Kleene 3 valued logic in the variables $\{a_1\comma a_n\}$.  We also write $\BBB\BBF=\BBB\BBF(a_1\comma a_n)$.

The following holds:
\begin{enumerate}
\item $S_1$ is a set disjoint from $S_2$ containing two special nodes among other additional nodes. There are {\bf in} (an input node) and {\bf out} (output node).  $S_1$ is referred to as the set of auxiliary nodes. When several attack formation are involved we always assume that their sets of auxiliary node, including the in and out nodes  are pairwise   disjoint.
\item We describe $(S_1\cup S_2, R, \Psi)$ schematically in Figure \ref{509-BBF5}.

\begin{figure}
\centering
\setlength{\unitlength}{0.00083333in}
\begingroup\makeatletter\ifx\SetFigFont\undefined%
\gdef\SetFigFont#1#2#3#4#5{%
  \reset@font\fontsize{#1}{#2pt}%
  \fontfamily{#3}\fontseries{#4}\fontshape{#5}%
  \selectfont}%
\fi\endgroup%
{\renewcommand{\dashlinestretch}{30}
\begin{picture}(1824,2889)(0,-10)
\put(912,2037){\makebox(0,0)[b]{\smash{{\SetFigFont{10}{12.0}{\rmdefault}{\mddefault}{\updefault}{\bf in}}}}}
\path(912,462)(912,12)
\blacken\path(882.000,132.000)(912.000,12.000)(942.000,132.000)(882.000,132.000)
\path(912,2037)(912,1662)
\blacken\path(882.000,1782.000)(912.000,1662.000)(942.000,1782.000)(882.000,1782.000)
\path(912,1287)(912,912)
\blacken\path(882.000,1032.000)(912.000,912.000)(942.000,1032.000)(882.000,1032.000)
\path(912,2637)(912,2562)
\blacken\path(882.000,2682.000)(912.000,2562.000)(942.000,2682.000)(882.000,2682.000)
\path(912,1887)(912,1812)
\blacken\path(882.000,1932.000)(912.000,1812.000)(942.000,1932.000)(882.000,1932.000)
\path(912,1137)(912,1062)
\blacken\path(882.000,1182.000)(912.000,1062.000)(942.000,1182.000)(882.000,1182.000)
\path(912,237)(912,162)
\blacken\path(882.000,282.000)(912.000,162.000)(942.000,282.000)(882.000,282.000)
\path(912,2412)(1812,1362)(912,462)
\path(912,2412)(12,1362)
\path(12,1362)(912,462)
\put(912,1362){\makebox(0,0)[b]{\smash{{\SetFigFont{10}{12.0}{\rmdefault}{\mddefault}{\updefault}$\Psi(a_1\comma a_n)$}}}}
\put(912,687){\makebox(0,0)[b]{\smash{{\SetFigFont{10}{12.0}{\rmdefault}{\mddefault}{\updefault}{\bf out}}}}}
\path(912,2862)(912,2412)
\blacken\path(882.000,2532.000)(912.000,2412.000)(942.000,2532.000)(882.000,2532.000)
\end{picture}
}
\caption{}\label{509-BBF5}
\end{figure}

\item Let $(S, R)$ be an argumentation network and assume that $\BBB\BBF(a_1\comma a_n)$ is a subnetwork of $(S, R)$.  We say that $\BBB\BBF$ is legitimately embedded in $(S, R)$ if the following holds:
\begin{enumerate}
\item The only elements of $\BBB\BBF$ attacked from outside $\BBB\BBF$ (by elements of $S$, which may include the elements $a_i$ themeselves as attackers from $S$) are $a_1\comma a_n$ and {\bf in} (of $\BBB\BBF$).
\item {\bf out} of $\BBB\BBF$ attacks only elements outside $\BBB\BBF$.
\end{enumerate}
The elements of $\BBB\BBF$ which are not in $\{a_1\comma a_n\}$ appear only in $\BBB\BBF$ (so {\bf in} and {\bf out} are labelled {\bf in}$(\BBB\BBF)$ and {\bf out}$(\BBB\BBF)$).
\item The following must hold for $\BBB\BBF(a_1\comma a_n)$ when embedded legitimately in any $(S, R)$, see Figure \ref{509-BBF6}.

\begin{figure}
\centering
\setlength{\unitlength}{0.00083333in}
\begingroup\makeatletter\ifx\SetFigFont\undefined%
\gdef\SetFigFont#1#2#3#4#5{%
  \reset@font\fontsize{#1}{#2pt}%
  \fontfamily{#3}\fontseries{#4}\fontshape{#5}%
  \selectfont}%
\fi\endgroup%
{\renewcommand{\dashlinestretch}{30}
\begin{picture}(3035,4183)(0,-10)
\put(3020,1082){\makebox(0,0)[lb]{\smash{{\SetFigFont{10}{12.0}{\rmdefault}{\mddefault}{\updefault}$(S, R)$}}}}
\path(1500,3532)(1500,3007)
\blacken\path(1470.000,3127.000)(1500.000,3007.000)(1530.000,3127.000)(1470.000,3127.000)
\path(1500,3307)(1500,3157)
\blacken\path(1470.000,3277.000)(1500.000,3157.000)(1530.000,3277.000)(1470.000,3277.000)
\path(1500,1282)(1500,532)
\blacken\path(1470.000,652.000)(1500.000,532.000)(1530.000,652.000)(1470.000,652.000)
\path(1500,832)(1500,682)
\blacken\path(1470.000,802.000)(1500.000,682.000)(1530.000,802.000)(1470.000,802.000)
\path(1500,3007)(450,2032)(1500,982)
	(2550,2032)(1500,3007)
\path(2700,2632)(1650,2257)
\blacken\path(1752.919,2325.613)(1650.000,2257.000)(1773.099,2269.108)(1752.919,2325.613)
\path(1895,2347)(1810,2312)
\blacken\path(1909.539,2385.430)(1810.000,2312.000)(1932.384,2329.950)(1909.539,2385.430)
\put(1500,3607){\makebox(0,0)[b]{\smash{{\SetFigFont{10}{12.0}{\rmdefault}{\mddefault}{\updefault}$\alpha$}}}}
\put(1500,2107){\makebox(0,0)[b]{\smash{{\SetFigFont{10}{12.0}{\rmdefault}{\mddefault}{\updefault}$a_i$}}}}
\put(1500,1282){\makebox(0,0)[b]{\smash{{\SetFigFont{10}{12.0}{\rmdefault}{\mddefault}{\updefault}{\bf out}}}}}
\put(1500,382){\makebox(0,0)[b]{\smash{{\SetFigFont{10}{12.0}{\rmdefault}{\mddefault}{\updefault}$\beta$}}}}
\put(2700,2707){\makebox(0,0)[b]{\smash{{\SetFigFont{10}{12.0}{\rmdefault}{\mddefault}{\updefault}$e$}}}}
\put(1500,2782){\makebox(0,0)[b]{\smash{{\SetFigFont{10}{12.0}{\rmdefault}{\mddefault}{\updefault}{\bf in}}}}}
\put(2105,1347){\makebox(0,0)[lb]{\smash{{\SetFigFont{10}{12.0}{\rmdefault}{\mddefault}{\updefault}$\BBB\BBF$}}}}
\put(1473,2084){\ellipse{2930}{4154}}
\end{picture}
}

\caption{}\label{509-BBF6}
\end{figure}

Let $\lambda: S\mapsto \{0,1,\half\}$ be any extension of $(S, R)$.  Then 
\begin{enumerate}
\item $\lambda({\bf out}) =\Psi (\lambda (a_1),\allowbreak\ldots,\lambda(a_n))$.
\item If none of $\{a_1\comma a_n\}$ are attacked from outside $\BBB\BBF$ or if all outside attackers $z$ of $a_1\comma a_n$ are out, i.e. have $\lambda (z)=0$, then $\lambda ({\bf in}) =\Psi (\lambda (a_1),\allowbreak\ldots,\lambda (a_n))$.
\end{enumerate}
\item If $\Psi (\lambda (a_i)) =1$ then $\lambda ({\bf in}) =1$, even if some $a_i$ are attacked from outside $\BBB\BBF$.  It could be the case, however, that even though $\lambda ({\bf in}) =1$, we have $\Psi(\lambda (a_i)) \neq 1$ because of outside attacks on $\{a_i\}$.
\end{enumerate}
\end{definition}

\begin{example}\label{509-BBE7}
We show that we can find a $\BBB\BBF$ for every formula of classical propositional logic. We show this by finding a $\BBB\BBF$ for atomic $d$, for negation $\neg d$ and for conjunctions $a\wedge c$. The $\BBB\BBF$ for arbitrary formulas can be done by legitimate substitutions of $\BBB\BBF$s.

Consider the network of Figure \ref{509-BBF8}.

\begin{figure}
\centering
\setlength{\unitlength}{0.00083333in}
\begingroup\makeatletter\ifx\SetFigFont\undefined%
\gdef\SetFigFont#1#2#3#4#5{%
  \reset@font\fontsize{#1}{#2pt}%
  \fontfamily{#3}\fontseries{#4}\fontshape{#5}%
  \selectfont}%
\fi\endgroup%
{\renewcommand{\dashlinestretch}{30}
\begin{picture}(1830,226)(0,-10)
\path(615,139)(765,139)
\blacken\path(645.000,109.000)(765.000,139.000)(645.000,169.000)(645.000,109.000)
\path(1065,139)(1515,139)
\blacken\path(1395.000,109.000)(1515.000,139.000)(1395.000,169.000)(1395.000,109.000)
\path(1440,139)(1665,139)
\blacken\path(1545.000,109.000)(1665.000,139.000)(1545.000,169.000)(1545.000,109.000)
\path(165,139)(615,139)
\blacken\path(495.000,109.000)(615.000,139.000)(495.000,169.000)(495.000,109.000)
\put(915,64){\makebox(0,0)[b]{\smash{{\SetFigFont{10}{12.0}{\rmdefault}{\mddefault}{\updefault}$d$}}}}
\put(1815,64){\makebox(0,0)[b]{\smash{{\SetFigFont{10}{12.0}{\rmdefault}{\mddefault}{\updefault}$\beta$}}}}
\put(15,64){\makebox(0,0)[b]{\smash{{\SetFigFont{10}{12.0}{\rmdefault}{\mddefault}{\updefault}$\alpha$}}}}
\end{picture}
}
\caption{}\label{509-BBF8}
\end{figure}

This network can be replaced by Figure \ref{509-BBF9}.

\begin{figure}
\centering\setlength{\unitlength}{0.00083333in}
\begingroup\makeatletter\ifx\SetFigFont\undefined%
\gdef\SetFigFont#1#2#3#4#5{%
  \reset@font\fontsize{#1}{#2pt}%
  \fontfamily{#3}\fontseries{#4}\fontshape{#5}%
  \selectfont}%
\fi\endgroup%
{\renewcommand{\dashlinestretch}{30}
\begin{picture}(3812,240)(0,-10)
\put(3744,60){\makebox(0,0)[b]{\smash{{\SetFigFont{10}{14.4}{\rmdefault}{\mddefault}{\updefault}$\beta$}}}}
\path(1419,135)(1794,135)
\blacken\path(1674.000,105.000)(1794.000,135.000)(1674.000,165.000)(1674.000,105.000)
\path(144,135)(519,135)
\blacken\path(399.000,105.000)(519.000,135.000)(399.000,165.000)(399.000,105.000)
\path(2019,135)(2394,135)
\blacken\path(2274.000,105.000)(2394.000,135.000)(2274.000,165.000)(2274.000,105.000)
\path(2544,135)(2919,135)
\blacken\path(2799.000,105.000)(2919.000,135.000)(2799.000,165.000)(2799.000,105.000)
\path(3294,135)(3669,135)
\blacken\path(3549.000,105.000)(3669.000,135.000)(3549.000,165.000)(3549.000,105.000)
\path(294,135)(369,135)
\blacken\path(249.000,105.000)(369.000,135.000)(249.000,165.000)(249.000,105.000)
\path(969,135)(1044,135)
\blacken\path(924.000,105.000)(1044.000,135.000)(924.000,165.000)(924.000,105.000)
\path(1569,135)(1644,135)
\blacken\path(1524.000,105.000)(1644.000,135.000)(1524.000,165.000)(1524.000,105.000)
\path(2169,135)(2244,135)
\blacken\path(2124.000,105.000)(2244.000,135.000)(2124.000,165.000)(2124.000,105.000)
\path(2694,135)(2769,135)
\blacken\path(2649.000,105.000)(2769.000,135.000)(2649.000,165.000)(2649.000,105.000)
\path(3444,135)(3519,135)
\blacken\path(3399.000,105.000)(3519.000,135.000)(3399.000,165.000)(3399.000,105.000)
\put(69,60){\makebox(0,0)[b]{\smash{{\SetFigFont{10}{14.4}{\rmdefault}{\mddefault}{\updefault}$\alpha$}}}}
\put(669,60){\makebox(0,0)[b]{\smash{{\SetFigFont{10}{14.4}{\rmdefault}{\mddefault}{\updefault}{\bf in}}}}}
\put(1269,60){\makebox(0,0)[b]{\smash{{\SetFigFont{10}{14.4}{\rmdefault}{\mddefault}{\updefault}$x$}}}}
\put(1869,60){\makebox(0,0)[b]{\smash{{\SetFigFont{10}{14.4}{\rmdefault}{\mddefault}{\updefault}$d$}}}}
\put(2469,60){\makebox(0,0)[b]{\smash{{\SetFigFont{10}{14.4}{\rmdefault}{\mddefault}{\updefault}$y$}}}}
\put(3069,60){\makebox(0,0)[b]{\smash{{\SetFigFont{10}{14.4}{\rmdefault}{\mddefault}{\updefault}{\bf out}}}}}
\path(819,135)(1194,135)
\blacken\path(1074.000,105.000)(1194.000,135.000)(1074.000,165.000)(1074.000,105.000)
\end{picture}
}

\caption{}\label{509-BBF9}
\end{figure}

The attack formation of Figure \ref{509-BBF10} is involved where $x, y$ are auxiliary points and $\Psi(d)= d$.

\begin{figure}
\centering
\setlength{\unitlength}{0.00083333in}
\begingroup\makeatletter\ifx\SetFigFont\undefined%
\gdef\SetFigFont#1#2#3#4#5{%
  \reset@font\fontsize{#1}{#2pt}%
  \fontfamily{#3}\fontseries{#4}\fontshape{#5}%
  \selectfont}%
\fi\endgroup%
{\renewcommand{\dashlinestretch}{30}
\begin{picture}(3174,4792)(0,-10)
\put(1512,880){\makebox(0,0)[b]{\smash{{\SetFigFont{10}{12.0}{\rmdefault}{\mddefault}{\updefault}{\bf out}}}}}
\path(1512,3805)(1512,3280)
\blacken\path(1482.000,3400.000)(1512.000,3280.000)(1542.000,3400.000)(1482.000,3400.000)
\path(1512,3055)(1512,2530)
\blacken\path(1482.000,2650.000)(1512.000,2530.000)(1542.000,2650.000)(1482.000,2650.000)
\path(1512,2305)(1512,1780)
\blacken\path(1482.000,1900.000)(1512.000,1780.000)(1542.000,1900.000)(1482.000,1900.000)
\path(1512,1555)(1512,1030)
\blacken\path(1482.000,1150.000)(1512.000,1030.000)(1542.000,1150.000)(1482.000,1150.000)
\path(1512,805)(1512,280)
\blacken\path(1482.000,400.000)(1512.000,280.000)(1542.000,400.000)(1482.000,400.000)
\path(1512,4405)(12,2080)(1512,580)
\path(1512,4405)(3162,2080)(1512,580)
\path(1512,4255)(1512,4180)
\blacken\path(1482.000,4300.000)(1512.000,4180.000)(1542.000,4300.000)(1482.000,4300.000)
\path(1512,3505)(1512,3430)
\blacken\path(1482.000,3550.000)(1512.000,3430.000)(1542.000,3550.000)(1482.000,3550.000)
\path(1512,2755)(1512,2680)
\blacken\path(1482.000,2800.000)(1512.000,2680.000)(1542.000,2800.000)(1482.000,2800.000)
\path(1512,2005)(1512,1705)
\blacken\path(1482.000,1825.000)(1512.000,1705.000)(1542.000,1825.000)(1482.000,1825.000)
\path(1512,1255)(1512,1180)
\blacken\path(1482.000,1300.000)(1512.000,1180.000)(1542.000,1300.000)(1482.000,1300.000)
\path(1512,505)(1512,430)
\blacken\path(1482.000,550.000)(1512.000,430.000)(1542.000,550.000)(1482.000,550.000)
\put(1512,4630){\makebox(0,0)[b]{\smash{{\SetFigFont{10}{12.0}{\rmdefault}{\mddefault}{\updefault}$\alpha$}}}}
\put(1512,3880){\makebox(0,0)[b]{\smash{{\SetFigFont{10}{12.0}{\rmdefault}{\mddefault}{\updefault}{\bf in}}}}}
\put(1512,3130){\makebox(0,0)[b]{\smash{{\SetFigFont{10}{12.0}{\rmdefault}{\mddefault}{\updefault}$x$}}}}
\put(1512,2380){\makebox(0,0)[b]{\smash{{\SetFigFont{10}{12.0}{\rmdefault}{\mddefault}{\updefault}$d$}}}}
\put(1512,1630){\makebox(0,0)[b]{\smash{{\SetFigFont{10}{12.0}{\rmdefault}{\mddefault}{\updefault}$y$}}}}
\put(1512,55){\makebox(0,0)[b]{\smash{{\SetFigFont{10}{12.0}{\rmdefault}{\mddefault}{\updefault}$\beta$}}}}
\path(1512,4555)(1512,4030)
\blacken\path(1482.000,4150.000)(1512.000,4030.000)(1542.000,4150.000)(1482.000,4150.000)
\end{picture}
}

\caption{}\label{509-BBF10}
\end{figure}

The $\BBB\BBF$ for $\neg d$ can be obtained from the $\BBB\BBF$ of $d$, (i.e. from Figure \ref{509-BBF9}) by deleting the node $x$, i.e. allowing  $\Bi\Bn$ to attack $d$ directly.

The above two cases are   simple but consider the more complex case of conjunction of Figure \ref{509-BBF11}.

\begin{figure}
\centering
\setlength{\unitlength}{0.00083333in}
\begingroup\makeatletter\ifx\SetFigFont\undefined%
\gdef\SetFigFont#1#2#3#4#5{%
  \reset@font\fontsize{#1}{#2pt}%
  \fontfamily{#3}\fontseries{#4}\fontshape{#5}%
  \selectfont}%
\fi\endgroup%
{\renewcommand{\dashlinestretch}{30}
\begin{picture}(1830,2992)(0,-10)
\put(915,55){\makebox(0,0)[b]{\smash{{\SetFigFont{10}{12.0}{\rmdefault}{\mddefault}{\updefault}$\beta$}}}}
\path(915,2080)(1665,1630)
\blacken\path(1546.666,1666.015)(1665.000,1630.000)(1577.536,1717.464)(1546.666,1666.015)
\path(1665,1630)(915,1180)(915,280)
\blacken\path(885.000,400.000)(915.000,280.000)(945.000,400.000)(885.000,400.000)
\path(165,1630)(915,1180)
\path(390,1780)(315,1705)
\blacken\path(378.640,1811.066)(315.000,1705.000)(421.066,1768.640)(378.640,1811.066)
\path(1440,1780)(1515,1705)
\blacken\path(1408.934,1768.640)(1515.000,1705.000)(1451.360,1811.066)(1408.934,1768.640)
\path(915,505)(915,430)
\blacken\path(885.000,550.000)(915.000,430.000)(945.000,550.000)(885.000,550.000)
\put(915,2830){\makebox(0,0)[b]{\smash{{\SetFigFont{10}{12.0}{\rmdefault}{\mddefault}{\updefault}$\alpha$}}}}
\put(15,1555){\makebox(0,0)[b]{\smash{{\SetFigFont{10}{12.0}{\rmdefault}{\mddefault}{\updefault}$a$}}}}
\put(1815,1555){\makebox(0,0)[b]{\smash{{\SetFigFont{10}{12.0}{\rmdefault}{\mddefault}{\updefault}$c$}}}}
\path(915,2755)(915,2080)(165,1630)
\blacken\path(252.464,1717.464)(165.000,1630.000)(283.334,1666.015)(252.464,1717.464)
\end{picture}
}

\caption{}\label{509-BBF11}
\end{figure}

In this figure, $\alpha$ disjunctively attacks $\{a,c\}$ and $\{a,c\}$ conjunctively attacks $\beta$.  Figure \ref{509-BBF11} is the same in meaning as Figure \ref{509-BBF8} instantiated by $d =a\wedge c$, namely Figure \ref{509-BBF11a}.

\begin{figure}
\centering
\setlength{\unitlength}{0.00083333in}
\begingroup\makeatletter\ifx\SetFigFont\undefined%
\gdef\SetFigFont#1#2#3#4#5{%
  \reset@font\fontsize{#1}{#2pt}%
  \fontfamily{#3}\fontseries{#4}\fontshape{#5}%
  \selectfont}%
\fi\endgroup%
{\renewcommand{\dashlinestretch}{30}
\begin{picture}(2580,226)(0,-10)
\path(2115,139)(2265,139)
\blacken\path(2145.000,109.000)(2265.000,139.000)(2145.000,169.000)(2145.000,109.000)
\path(165,139)(990,139)
\blacken\path(870.000,109.000)(990.000,139.000)(870.000,169.000)(870.000,109.000)
\path(765,139)(840,139)
\blacken\path(720.000,109.000)(840.000,139.000)(720.000,169.000)(720.000,109.000)
\path(1740,139)(2415,139)
\blacken\path(2295.000,109.000)(2415.000,139.000)(2295.000,169.000)(2295.000,109.000)
\put(1365,64){\makebox(0,0)[b]{\smash{{\SetFigFont{10}{12.0}{\rmdefault}{\mddefault}{\updefault}$(a\wedge c)$}}}}
\put(15,64){\makebox(0,0)[b]{\smash{{\SetFigFont{10}{12.0}{\rmdefault}{\mddefault}{\updefault}$\alpha$}}}}
\put(2565,64){\makebox(0,0)[b]{\smash{{\SetFigFont{10}{12.0}{\rmdefault}{\mddefault}{\updefault}$\beta$}}}}
\end{picture}
}

\caption{}\label{509-BBF11a}
\end{figure}

In this case it is not so simple to find a Boolean attack formation to do the same job.  Namely we want Figure \ref{509-BBF13} to be realised as a traditional network.

\begin{figure}
\centering
\setlength{\unitlength}{0.00083333in}
\begingroup\makeatletter\ifx\SetFigFont\undefined%
\gdef\SetFigFont#1#2#3#4#5{%
  \reset@font\fontsize{#1}{#2pt}%
  \fontfamily{#3}\fontseries{#4}\fontshape{#5}%
  \selectfont}%
\fi\endgroup%
{\renewcommand{\dashlinestretch}{30}
\begin{picture}(2424,3592)(0,-10)
\put(1212,55){\makebox(0,0)[b]{\smash{{\SetFigFont{10}{12.0}{\rmdefault}{\mddefault}{\updefault}$\beta$}}}}
\path(1212,3130)(1212,2830)
\blacken\path(1182.000,2950.000)(1212.000,2830.000)(1242.000,2950.000)(1182.000,2950.000)
\path(1212,2530)(1212,2155)
\blacken\path(1182.000,2275.000)(1212.000,2155.000)(1242.000,2275.000)(1182.000,2275.000)
\path(1212,2380)(1212,2305)
\blacken\path(1182.000,2425.000)(1212.000,2305.000)(1242.000,2425.000)(1182.000,2425.000)
\path(1212,1855)(1212,1480)
\blacken\path(1182.000,1600.000)(1212.000,1480.000)(1242.000,1600.000)(1182.000,1600.000)
\path(1212,1705)(1212,1630)
\blacken\path(1182.000,1750.000)(1212.000,1630.000)(1242.000,1750.000)(1182.000,1750.000)
\path(1212,1180)(1212,280)
\blacken\path(1182.000,400.000)(1212.000,280.000)(1242.000,400.000)(1182.000,400.000)
\path(1212,580)(1212,430)
\blacken\path(1182.000,550.000)(1212.000,430.000)(1242.000,550.000)(1182.000,550.000)
\path(1212,2905)(12,1855)(1212,430)
	(2412,1855)(1212,2905)
\put(1212,3430){\makebox(0,0)[b]{\smash{{\SetFigFont{10}{12.0}{\rmdefault}{\mddefault}{\updefault}$\alpha$}}}}
\put(1212,2605){\makebox(0,0)[b]{\smash{{\SetFigFont{10}{12.0}{\rmdefault}{\mddefault}{\updefault}{\bf in}}}}}
\put(1212,1930){\makebox(0,0)[b]{\smash{{\SetFigFont{10}{12.0}{\rmdefault}{\mddefault}{\updefault}$(a\wedge c)$}}}}
\put(1212,1255){\makebox(0,0)[b]{\smash{{\SetFigFont{10}{12.0}{\rmdefault}{\mddefault}{\updefault}{\bf out}}}}}
\path(1212,3355)(1212,2905)
\blacken\path(1182.000,3025.000)(1212.000,2905.000)(1242.000,3025.000)(1182.000,3025.000)
\end{picture}
}

\caption{}\label{509-BBF13}
\end{figure}

The network of Figure \ref{509-BBF14} can do the job. It is a $\top$-net. We assume we are dealing with option (iv) of Appendix B, adopting the non-toxic approach. 

\begin{figure}
\centering
\setlength{\unitlength}{0.00083333in}
\begingroup\makeatletter\ifx\SetFigFont\undefined%
\gdef\SetFigFont#1#2#3#4#5{%
  \reset@font\fontsize{#1}{#2pt}%
  \fontfamily{#3}\fontseries{#4}\fontshape{#5}%
  \selectfont}%
\fi\endgroup%
{\renewcommand{\dashlinestretch}{30}
\begin{picture}(5394,7792)(0,-10)
\put(455,7030){\makebox(0,0)[b]{\smash{{\SetFigFont{10}{12.0}{\rmdefault}{\mddefault}{\updefault}$\BBB\BBF$}}}}
\path(2555,7555)(2555,7180)
\blacken\path(2525.000,7300.000)(2555.000,7180.000)(2585.000,7300.000)(2525.000,7300.000)
\path(2555,7405)(2555,7330)
\blacken\path(2525.000,7450.000)(2555.000,7330.000)(2585.000,7450.000)(2525.000,7450.000)
\path(2555,6880)(2555,6505)
\blacken\path(2525.000,6625.000)(2555.000,6505.000)(2585.000,6625.000)(2525.000,6625.000)
\path(2555,6730)(2555,6655)
\blacken\path(2525.000,6775.000)(2555.000,6655.000)(2585.000,6775.000)(2525.000,6775.000)
\path(2480,6880)(1205,5605)
\blacken\path(1268.640,5711.066)(1205.000,5605.000)(1311.066,5668.640)(1268.640,5711.066)
\path(2555,6205)(1280,5605)
\blacken\path(1375.804,5683.240)(1280.000,5605.000)(1401.352,5628.951)(1375.804,5683.240)
\path(2555,6205)(3455,6130)
\blacken\path(3332.923,6110.069)(3455.000,6130.000)(3337.906,6169.862)(3332.923,6110.069)
\path(3605,5980)(3305,5605)
\blacken\path(3356.537,5717.445)(3305.000,5605.000)(3403.389,5679.963)(3356.537,5717.445)
\blacken\path(3561.319,5518.217)(3455.000,5455.000)(3578.560,5460.747)(3561.319,5518.217)
\path(3455,5455)(4205,5680)
\blacken\path(4098.681,5616.783)(4205.000,5680.000)(4081.440,5674.253)(4098.681,5616.783)
\blacken\path(4081.671,5914.487)(4205.000,5905.000)(4100.645,5971.408)(4081.671,5914.487)
\path(4205,5905)(3755,6055)
\blacken\path(3878.329,6045.513)(3755.000,6055.000)(3859.355,5988.592)(3878.329,6045.513)
\path(4205,5680)(4655,5305)
\blacken\path(4543.608,5358.775)(4655.000,5305.000)(4582.019,5404.869)(4543.608,5358.775)
\path(2555,4855)(3230,5380)
\blacken\path(3153.696,5282.647)(3230.000,5380.000)(3116.860,5330.008)(3153.696,5282.647)
\path(1205,5380)(2555,4855)
\blacken\path(2432.286,4870.533)(2555.000,4855.000)(2454.033,4926.454)(2432.286,4870.533)
\path(1655,4330)(2555,4630)
\blacken\path(2450.645,4563.592)(2555.000,4630.000)(2431.671,4620.513)(2450.645,4563.592)
\path(3455,4330)(2555,4630)
\blacken\path(2678.329,4620.513)(2555.000,4630.000)(2659.355,4563.592)(2678.329,4620.513)
\path(1655,3730)(1655,4105)
\blacken\path(1685.000,3985.000)(1655.000,4105.000)(1625.000,3985.000)(1685.000,3985.000)
\path(1655,3130)(1655,3505)
\blacken\path(1685.000,3385.000)(1655.000,3505.000)(1625.000,3385.000)(1685.000,3385.000)
\path(3530,3130)(3530,3505)
\blacken\path(3560.000,3385.000)(3530.000,3505.000)(3500.000,3385.000)(3560.000,3385.000)
\path(3530,3730)(3530,4105)
\blacken\path(3560.000,3985.000)(3530.000,4105.000)(3500.000,3985.000)(3560.000,3985.000)
\path(3530,2455)(3530,2905)
\blacken\path(3560.000,2785.000)(3530.000,2905.000)(3500.000,2785.000)(3560.000,2785.000)
\path(3530,1780)(3530,2230)
\blacken\path(3560.000,2110.000)(3530.000,2230.000)(3500.000,2110.000)(3560.000,2110.000)
\path(1655,1780)(1655,2155)
\blacken\path(1685.000,2035.000)(1655.000,2155.000)(1625.000,2035.000)(1685.000,2035.000)
\path(1655,2380)(1655,2905)
\blacken\path(1685.000,2785.000)(1655.000,2905.000)(1625.000,2785.000)(1685.000,2785.000)
\path(3530,1555)(2555,1180)
\blacken\path(2656.232,1251.078)(2555.000,1180.000)(2677.771,1195.077)(2656.232,1251.078)
\path(1655,1555)(2555,1180)
\blacken\path(2432.692,1198.462)(2555.000,1180.000)(2455.769,1253.846)(2432.692,1198.462)
\path(2248,4965)(2405,4905)
\blacken\path(2282.197,4919.815)(2405.000,4905.000)(2303.616,4975.861)(2282.197,4919.815)
\path(3050,5243)(3110,5280)
\blacken\path(3023.606,5191.478)(3110.000,5280.000)(2992.113,5242.548)(3023.606,5191.478)
\path(3447,5775)(3395,5708)
\blacken\path(3444.875,5821.192)(3395.000,5708.000)(3492.274,5784.405)(3444.875,5821.192)
\path(3942,5993)(3882,6008)
\blacken\path(4005.693,6008.000)(3882.000,6008.000)(3991.141,5949.791)(4005.693,6008.000)
\path(4167,5910)(4227,5903)
\blacken\path(4104.332,5887.108)(4227.000,5903.000)(4111.285,5946.704)(4104.332,5887.108)
\path(3942,5595)(4055,5648)
\blacken\path(3959.096,5569.882)(4055.000,5648.000)(3933.617,5624.204)(3959.096,5569.882)
\path(3710,5528)(3620,5498)
\blacken\path(3724.355,5564.408)(3620.000,5498.000)(3743.329,5507.487)(3724.355,5564.408)
\path(4460,5475)(4565,5408)
\blacken\path(4447.703,5447.260)(4565.000,5408.000)(4479.977,5497.840)(4447.703,5447.260)
\path(2802,4545)(2720,4575)
\blacken\path(2843.002,4561.944)(2720.000,4575.000)(2822.387,4505.597)(2843.002,4561.944)
\path(2262,4538)(2412,4583)
\blacken\path(2305.681,4519.783)(2412.000,4583.000)(2288.440,4577.253)(2305.681,4519.783)
\path(3522,3877)(3522,3967)
\blacken\path(3552.000,3847.000)(3522.000,3967.000)(3492.000,3847.000)(3552.000,3847.000)
\path(1655,3855)(1655,3952)
\blacken\path(1685.000,3832.000)(1655.000,3952.000)(1625.000,3832.000)(1685.000,3832.000)
\path(1647,3277)(1662,3367)
\blacken\path(1671.864,3243.701)(1662.000,3367.000)(1612.680,3253.565)(1671.864,3243.701)
\path(1355,3277)(1445,3210)
\blacken\path(1330.830,3257.593)(1445.000,3210.000)(1366.658,3305.721)(1330.830,3257.593)
\path(3522,3255)(3530,3360)
\blacken\path(3550.797,3238.068)(3530.000,3360.000)(3490.970,3242.626)(3550.797,3238.068)
\path(3822,3247)(3732,3202)
\blacken\path(3825.915,3282.498)(3732.000,3202.000)(3852.748,3228.833)(3825.915,3282.498)
\path(3522,2647)(3492,2767)
\blacken\path(3550.209,2657.859)(3492.000,2767.000)(3492.000,2643.307)(3550.209,2657.859)
\path(3530,1965)(3530,2077)
\blacken\path(3560.000,1957.000)(3530.000,2077.000)(3500.000,1957.000)(3560.000,1957.000)
\path(3305,1882)(3410,1830)
\blacken\path(3289.151,1856.372)(3410.000,1830.000)(3315.779,1910.139)(3289.151,1856.372)
\path(1887,1882)(1812,1830)
\blacken\path(1893.522,1923.027)(1812.000,1830.000)(1927.709,1873.720)(1893.522,1923.027)
\path(1662,1912)(1662,2010)
\blacken\path(1692.000,1890.000)(1662.000,2010.000)(1632.000,1890.000)(1692.000,1890.000)
\path(1647,2655)(1647,2760)
\blacken\path(1677.000,2640.000)(1647.000,2760.000)(1617.000,2640.000)(1677.000,2640.000)
\path(2300,1283)(2412,1245)
\blacken\path(2288.724,1255.146)(2412.000,1245.000)(2308.001,1311.965)(2288.724,1255.146)
\path(2810,1283)(2705,1245)
\blacken\path(2807.629,1314.046)(2705.000,1245.000)(2828.047,1257.627)(2807.629,1314.046)
\path(2555,652)(2555,225)
\blacken\path(2525.000,345.000)(2555.000,225.000)(2585.000,345.000)(2525.000,345.000)
\path(2555,465)(2555,375)
\blacken\path(2525.000,495.000)(2555.000,375.000)(2585.000,495.000)(2525.000,495.000)
\path(2555,7255)(20,5820)(12,2302)
	(2555,655)(5382,2355)(5382,5820)(2630,7255)
\path(1380,5790)(1305,5710)
\blacken\path(1365.187,5818.063)(1305.000,5710.000)(1408.959,5777.026)(1365.187,5818.063)
\path(1510,5710)(1415,5675)
\blacken\path(1517.230,5744.635)(1415.000,5675.000)(1537.972,5688.334)(1517.230,5744.635)
\path(3200,6155)(3300,6150)
\blacken\path(3178.652,6126.030)(3300.000,6150.000)(3181.648,6185.955)(3178.652,6126.030)
\path(885,5345)(965,5390)
\blacken\path(875.119,5305.021)(965.000,5390.000)(845.703,5357.316)(875.119,5305.021)
\path(2560,975)(2560,650)
\path(1655,3505)(1656,3504)(1658,3502)
	(1662,3497)(1668,3490)(1676,3480)
	(1687,3467)(1701,3451)(1717,3431)
	(1736,3408)(1757,3382)(1781,3353)
	(1806,3321)(1832,3287)(1860,3251)
	(1889,3214)(1918,3175)(1946,3134)
	(1975,3093)(2004,3051)(2031,3009)
	(2058,2965)(2084,2921)(2109,2877)
	(2133,2831)(2155,2784)(2176,2737)
	(2195,2688)(2211,2638)(2226,2588)
	(2238,2536)(2247,2484)(2253,2431)
	(2255,2380)(2252,2324)(2245,2271)
	(2234,2222)(2219,2177)(2200,2136)
	(2180,2099)(2156,2065)(2131,2034)
	(2104,2006)(2076,1981)(2046,1957)
	(2015,1936)(1983,1916)(1950,1898)
	(1917,1881)(1885,1865)(1853,1851)
	(1822,1838)(1792,1827)(1765,1816)
	(1740,1808)(1718,1800)(1700,1794)
	(1685,1789)(1673,1785)(1655,1780)
\blacken\path(1762.593,1841.023)(1655.000,1780.000)(1778.651,1783.212)(1762.593,1841.023)
\path(3455,3505)(3454,3504)(3453,3502)
	(3450,3497)(3445,3490)(3439,3480)
	(3430,3467)(3419,3451)(3407,3431)
	(3392,3408)(3375,3382)(3357,3353)
	(3338,3321)(3317,3287)(3295,3251)
	(3273,3214)(3251,3175)(3229,3134)
	(3207,3093)(3185,3051)(3164,3009)
	(3143,2965)(3123,2921)(3105,2877)
	(3087,2831)(3071,2784)(3055,2737)
	(3042,2688)(3030,2638)(3020,2588)
	(3012,2536)(3007,2484)(3004,2431)
	(3005,2380)(3010,2324)(3018,2271)
	(3030,2222)(3044,2177)(3061,2136)
	(3080,2099)(3101,2065)(3123,2034)
	(3147,2006)(3171,1981)(3197,1957)
	(3224,1936)(3251,1916)(3279,1898)
	(3307,1881)(3335,1865)(3363,1851)
	(3389,1838)(3414,1827)(3437,1816)
	(3458,1808)(3477,1800)(3492,1794)
	(3505,1789)(3515,1785)(3530,1780)
\blacken\path(3406.671,1789.487)(3530.000,1780.000)(3425.645,1846.408)(3406.671,1789.487)
\path(2630,7255)(2631,7255)(2632,7254)
	(2636,7252)(2641,7249)(2648,7245)
	(2658,7240)(2671,7233)(2687,7225)
	(2706,7214)(2729,7202)(2755,7187)
	(2785,7171)(2818,7152)(2855,7132)
	(2896,7109)(2940,7084)(2987,7057)
	(3038,7029)(3091,6998)(3147,6966)
	(3205,6932)(3265,6897)(3327,6860)
	(3390,6822)(3455,6782)(3521,6742)
	(3587,6700)(3654,6657)(3722,6614)
	(3789,6569)(3857,6524)(3924,6478)
	(3991,6432)(4057,6384)(4123,6336)
	(4188,6287)(4252,6237)(4315,6186)
	(4377,6135)(4438,6082)(4497,6029)
	(4556,5974)(4613,5918)(4668,5861)
	(4722,5803)(4774,5744)(4825,5683)
	(4873,5621)(4919,5557)(4963,5492)
	(5004,5425)(5042,5357)(5077,5288)
	(5109,5218)(5137,5147)(5161,5076)
	(5180,5005)(5196,4923)(5206,4842)
	(5210,4763)(5208,4686)(5201,4612)
	(5189,4540)(5172,4470)(5151,4403)
	(5127,4339)(5098,4277)(5066,4218)
	(5031,4161)(4994,4106)(4953,4052)
	(4910,4001)(4865,3951)(4818,3903)
	(4768,3856)(4717,3811)(4665,3767)
	(4611,3724)(4556,3682)(4499,3641)
	(4443,3602)(4385,3564)(4328,3527)
	(4271,3491)(4214,3456)(4158,3423)
	(4104,3391)(4051,3361)(3999,3333)
	(3950,3306)(3904,3281)(3861,3258)
	(3820,3237)(3783,3218)(3750,3201)
	(3720,3186)(3694,3173)(3672,3162)
	(3653,3153)(3638,3146)(3626,3140)
	(3618,3136)(3605,3130)
\blacken\path(3701.383,3207.526)(3605.000,3130.000)(3726.527,3153.048)(3701.383,3207.526)
\path(2555,7255)(2554,7255)(2553,7254)
	(2549,7252)(2544,7249)(2537,7245)
	(2527,7240)(2515,7232)(2499,7223)
	(2480,7213)(2457,7200)(2431,7185)
	(2402,7168)(2369,7149)(2332,7128)
	(2292,7104)(2249,7079)(2203,7052)
	(2154,7023)(2103,6992)(2049,6959)
	(1993,6924)(1936,6889)(1876,6851)
	(1816,6813)(1754,6773)(1692,6732)
	(1629,6690)(1566,6647)(1503,6603)
	(1439,6559)(1376,6513)(1314,6467)
	(1251,6420)(1190,6371)(1129,6323)
	(1069,6273)(1010,6222)(952,6170)
	(895,6117)(839,6063)(784,6008)
	(731,5951)(679,5893)(629,5833)
	(580,5772)(533,5709)(488,5645)
	(446,5579)(405,5512)(368,5442)
	(333,5372)(302,5300)(274,5227)
	(250,5154)(230,5080)(213,4998)
	(202,4917)(197,4837)(197,4759)
	(201,4684)(210,4611)(223,4540)
	(240,4472)(260,4406)(284,4343)
	(310,4282)(340,4223)(372,4166)
	(406,4111)(443,4057)(482,4005)
	(523,3955)(565,3906)(610,3858)
	(655,3812)(702,3766)(750,3722)
	(799,3679)(848,3637)(898,3596)
	(948,3557)(998,3519)(1047,3482)
	(1096,3446)(1144,3412)(1190,3380)
	(1235,3349)(1278,3320)(1318,3294)
	(1356,3269)(1392,3246)(1424,3225)
	(1453,3207)(1479,3191)(1502,3177)
	(1521,3165)(1538,3155)(1551,3147)
	(1561,3141)(1569,3137)(1580,3130)
\blacken\path(1462.654,3169.115)(1580.000,3130.000)(1494.867,3219.735)(1462.654,3169.115)
\path(1080,5565)(1078,5567)(1073,5571)
	(1064,5577)(1052,5586)(1037,5597)
	(1018,5610)(998,5623)(976,5636)
	(953,5648)(929,5659)(903,5668)
	(876,5675)(847,5679)(816,5679)
	(785,5675)(758,5667)(734,5658)
	(712,5647)(694,5636)(679,5625)
	(666,5615)(655,5605)(646,5596)
	(637,5587)(630,5578)(622,5569)
	(614,5558)(606,5546)(598,5532)
	(591,5516)(584,5498)(580,5479)
	(580,5460)(584,5443)(591,5428)
	(600,5415)(609,5403)(619,5393)
	(629,5385)(638,5378)(647,5373)
	(655,5368)(664,5363)(673,5359)
	(682,5354)(693,5350)(706,5346)
	(720,5342)(736,5338)(755,5334)
	(777,5331)(800,5329)(825,5330)
	(856,5335)(886,5343)(913,5354)
	(937,5367)(960,5381)(981,5396)
	(1000,5413)(1019,5429)(1035,5446)
	(1050,5461)(1063,5475)(1085,5500)
\blacken\path(1028.246,5390.096)(1085.000,5500.000)(983.203,5429.733)(1028.246,5390.096)
\put(2555,7630){\makebox(0,0)[b]{\smash{{\SetFigFont{10}{12.0}{\rmdefault}{\mddefault}{\updefault}$\alpha$}}}}
\put(2555,6955){\makebox(0,0)[b]{\smash{{\SetFigFont{10}{12.0}{\rmdefault}{\mddefault}{\updefault}$\bar{\alpha}$}}}}
\put(2555,6280){\makebox(0,0)[b]{\smash{{\SetFigFont{10}{12.0}{\rmdefault}{\mddefault}{\updefault}$\bar{\bar{\alpha}}$}}}}
\put(1205,5455){\makebox(0,0)[b]{\smash{{\SetFigFont{10}{12.0}{\rmdefault}{\mddefault}{\updefault}$v$}}}}
\put(3305,5455){\makebox(0,0)[b]{\smash{{\SetFigFont{10}{12.0}{\rmdefault}{\mddefault}{\updefault}$\bar{x}$}}}}
\put(3605,6055){\makebox(0,0)[b]{\smash{{\SetFigFont{10}{12.0}{\rmdefault}{\mddefault}{\updefault}$x$}}}}
\put(4205,5755){\makebox(0,0)[b]{\smash{{\SetFigFont{10}{12.0}{\rmdefault}{\mddefault}{\updefault}$e$}}}}
\put(4655,5155){\makebox(0,0)[b]{\smash{{\SetFigFont{10}{12.0}{\rmdefault}{\mddefault}{\updefault}$\top$}}}}
\put(2555,4705){\makebox(0,0)[b]{\smash{{\SetFigFont{10}{12.0}{\rmdefault}{\mddefault}{\updefault}$u$}}}}
\put(1655,4180){\makebox(0,0)[b]{\smash{{\SetFigFont{10}{12.0}{\rmdefault}{\mddefault}{\updefault}$\bar{a}_1$}}}}
\put(2555,4180){\makebox(0,0)[b]{\smash{{\SetFigFont{10}{12.0}{\rmdefault}{\mddefault}{\updefault}\ldots}}}}
\put(3530,4180){\makebox(0,0)[b]{\smash{{\SetFigFont{10}{12.0}{\rmdefault}{\mddefault}{\updefault}$\bar{a}_n$}}}}
\put(1655,3580){\makebox(0,0)[b]{\smash{{\SetFigFont{10}{12.0}{\rmdefault}{\mddefault}{\updefault}$a_1$}}}}
\put(3530,3580){\makebox(0,0)[b]{\smash{{\SetFigFont{10}{12.0}{\rmdefault}{\mddefault}{\updefault}$a_n$}}}}
\put(1655,2980){\makebox(0,0)[b]{\smash{{\SetFigFont{10}{12.0}{\rmdefault}{\mddefault}{\updefault}$\bar{b}_1$}}}}
\put(2555,2980){\makebox(0,0)[b]{\smash{{\SetFigFont{10}{12.0}{\rmdefault}{\mddefault}{\updefault}$\ldots$}}}}
\put(3530,2980){\makebox(0,0)[b]{\smash{{\SetFigFont{10}{12.0}{\rmdefault}{\mddefault}{\updefault}$\bar{b}_n$}}}}
\put(3530,2305){\makebox(0,0)[b]{\smash{{\SetFigFont{10}{12.0}{\rmdefault}{\mddefault}{\updefault}$b_n$}}}}
\put(1655,2230){\makebox(0,0)[b]{\smash{{\SetFigFont{10}{12.0}{\rmdefault}{\mddefault}{\updefault}$b_1$}}}}
\put(3530,1630){\makebox(0,0)[b]{\smash{{\SetFigFont{10}{12.0}{\rmdefault}{\mddefault}{\updefault}$a_n$}}}}
\put(1655,1630){\makebox(0,0)[b]{\smash{{\SetFigFont{10}{12.0}{\rmdefault}{\mddefault}{\updefault}$a'_1$}}}}
\put(2555,1030){\makebox(0,0)[b]{\smash{{\SetFigFont{10}{12.0}{\rmdefault}{\mddefault}{\updefault}$w$}}}}
\put(2555,55){\makebox(0,0)[b]{\smash{{\SetFigFont{10}{12.0}{\rmdefault}{\mddefault}{\updefault}$\beta$}}}}
\put(4655,5193){\ellipse{750}{224}}
\end{picture}
}
\caption{}\label{509-BBF14}
\end{figure}
\end{example}

\begin{lemma}\label{509-BBL15}
Figure \ref{509-BBF14} implements in $\top$-net non-toxic approach with auxiliary points the concept of disjunctive attacks.
\end{lemma}

\begin{proof}
We need to show the following for any $\lambda$.
\begin{enumerate}
\item If $\lambda (\alpha)=1$, then for some $i$, $\lambda (a_i) =0$ and $\lambda (w) =0$.
\item If $\lambda (\alpha)=0$, then unless attacked from outside the diamond, all $\lambda (a_i)=1$ and $\lambda (w)=1$.
\item If $\lambda (w)=\half$, then $\lambda (a_i)<1$. They are all 0 or $\half$ and if some $a_i$ is not successfully attacked from the outside then $\lambda (a_i) =\half$. Thus $\lambda (w)=\half$, unless some $\lambda (a_i)=0$ in which case $\lambda (w) =0$.  

Let us examine each case.
\paragraph{Case 1:}  Assume $\lambda (\alpha)=1$. Hence 
$\lambda (\bar{\alpha})=0$ and 
$\lambda (\bar{\bar{\alpha}})=1$.  
Hence $\lambda (v) =0$ and 
$\lambda (x) =0$. We know that 
$\lambda (e)$ must be 0, so one of its attackers must be 1.\footnote{Note that in option  (i) for the semantics, an attacker of $\top$ must be ``out" not by virtue of it attacking $\top$ but because it must have an attacker which is ``in".}  It cannot be $x$ so it must be $\bar{x}$.  If this is the case then 
$\lambda (u) =0$ and then for some 
$i, \lambda (\bar{a}_i)=1$.  

Note that since 
$\lambda (\bar{\alpha})=0$, the loop 
$a_i\tO a'_i\tO b_i\tO \bar{b}_i\tO a_i$ is not broken. So $a_i$ is attacked and so either the loop is resolved by 
$\lambda (a_i)=\lambda (b_i) =1$ or by 
$\lambda (a_i)=\lambda (b_i)=0$.

However, since $\lambda (\bar{a}_i) =1$ we must have $\lambda (a_i)=0$.
In which case $\lambda (a'_i)=1$ and $\lambda (w)=0$.
\paragraph{Case 2:}  Assume $\lambda (\alpha) =0$. Then $\lambda (\bar{\alpha})=1, \lambda(\bar{\bar{\alpha}}) =0$.  Hence $\lambda (v) =0$ and also $\lambda (x) =1$ because $\lambda(e)$ has to be 0. Therefore $\lambda (\bar{x})=0$. $\lambda(u)$ can be anything. It needs not be 1 or $\half$ because $\lambda (x)=1$ forces $\lambda (\bar{x})=0$.

Now since $\lambda (\bar{\alpha})=1$, we get $\lambda (\bar{b}_i)=0$. Thus $a_1\comma a_n$ are not attacked from any point in the inside. So unless $a_i$ is attacked from the outside, we have $\lambda (a_i)=1$. So if all $\lambda (a_i)=1$ for all $i$, we get all $\lambda (a'_i)=0$ for all $i$ and so $\lambda (w) =1$.
\paragraph{Case 3:}  Assume $\lambda (\alpha) =\half$, then $\lambda (\bar{\alpha}) =\lambda (\bar{\bar{\alpha}}) =\lambda (v) =\lambda (x) =\half$.  Since $\lambda (e)$ must be 0, we look for an attacker of $e$ whose $\lambda$ value is 1.  The two candidates are $x$ and $\bar{x}$.  $\lambda (x) =\half$ implies $\lambda (\bar{x})\neq 1$. Thus there is no way forward and it is obvious that the set $\{\top\}$ is toxic for the case of the input $\lambda (\alpha)=\half$.  We need to ignore it.  See Appendix B, especially Definition \ref{509-TD14}.

In this case we get that $\lambda (\alpha) =\lambda (\bar{\alpha})=\lambda(\bar{\bar{\alpha}})=\lambda (v) =\half$.

We do not care what values the set $\{x,\bar{x}, e\}$ gets, for example it can get the values $\lambda (e)=1, \lambda (x) =\lambda (\bar{x}) =0$.  The question is what is $\lambda (u)$?  It cannot be 1 because $\lambda (v) =\half$.  Can it be 0?  If $\lambda (u)=0$, then for some $i, \lambda (\bar{a}_i)=1$. Then in this case $\lambda (a_i)=0$. Then $\lambda (a'_i)=1, \lambda (b_i)=0$, but then $\lambda (\bar{b}_i)=\half$ because it is attacked by $\bar{\alpha}$ and $\lambda (\bar{\alpha}) =\half$. So how can $\lambda (a_i)=1$?

So the last possibility is $\lambda (u) =\half$. Then $\lambda (\bar{a}_i)=\half$ or 0 and for at least one $i$ $\lambda (\bar{a}_i)=\half$.

If $\lambda (\bar{a}_i)=0$, then $\lambda (a_i)=1$ so $\lambda (a'_i) =0$, $\lambda (b_i)=1$ and  $\lambda (\bar{b}_i) =0$. We know that for at least one $i$, we have $\lambda (\bar{a}_i)=\half$. So $\lambda (a'_i)=\half$.

Thus from the attackers of $w, \{a'_i\}$, at least one has value $\half$ and the rest of the values are $\half$ or 0. So $\lambda (w) =\half$.
\end{enumerate}
\end{proof}

\subsection{Instantiating with Boolean attack formations}
This appendix prepares the technical ground for intantiating with Boolean formulas. We first turn any such formula into a Boolean attack formation (see Example \ref{509-BBE7}) and then instantiate  with the resulting formations. To achieve that we need to define the notion of instantiation with $\BBB\BBF$s.

We need to agree on some diagrammatic conventions. Consider Figure \ref{509-BBF16}:

\begin{figure}
\centering
\[
a\tO b\tO a\tO b'
\]
\caption{}\label{509-BBF16}
\end{figure}

In this figure the letter $a$ appears twice.  Is this a misprint and the second $a$ should be $a'$, or is the intention actually as depicted in Figure \ref{509-BBF16x}?

\begin{figure}
\centering
\setlength{\unitlength}{0.00083333in}
\begingroup\makeatletter\ifx\SetFigFont\undefined%
\gdef\SetFigFont#1#2#3#4#5{%
  \reset@font\fontsize{#1}{#2pt}%
  \fontfamily{#3}\fontseries{#4}\fontshape{#5}%
  \selectfont}%
\fi\endgroup%
{\renewcommand{\dashlinestretch}{30}
\begin{picture}(1257,1582)(0,-10)
\put(42,55){\makebox(0,0)[b]{\smash{{\SetFigFont{10}{12.0}{\rmdefault}{\mddefault}{\updefault}$b'$}}}}
\path(42,505)(42,355)
\blacken\path(12.000,475.000)(42.000,355.000)(72.000,475.000)(12.000,475.000)
\path(232,915)(172,935)
\blacken\path(295.329,925.513)(172.000,935.000)(276.355,868.592)(295.329,925.513)
\path(1102,1445)(1142,1395)
\blacken\path(1043.611,1469.963)(1142.000,1395.000)(1090.463,1507.445)(1043.611,1469.963)
\path(42,1255)(43,1256)(46,1258)
	(51,1261)(59,1266)(69,1273)
	(83,1282)(100,1293)(121,1305)
	(144,1320)(169,1335)(197,1352)
	(227,1369)(258,1386)(290,1404)
	(323,1421)(357,1438)(391,1455)
	(427,1471)(463,1486)(500,1500)
	(539,1512)(579,1524)(620,1534)
	(662,1543)(705,1549)(749,1554)
	(792,1555)(841,1553)(886,1547)
	(927,1539)(965,1528)(998,1514)
	(1028,1500)(1055,1483)(1079,1465)
	(1102,1447)(1122,1427)(1141,1407)
	(1158,1386)(1173,1366)(1187,1346)
	(1200,1327)(1211,1310)(1220,1295)
	(1227,1282)(1233,1271)(1242,1255)
\blacken\path(1157.021,1344.881)(1242.000,1255.000)(1209.316,1374.297)(1157.021,1344.881)
\path(1242,1030)(1241,1030)(1237,1028)
	(1231,1026)(1222,1023)(1209,1019)
	(1192,1014)(1171,1007)(1148,1000)
	(1121,992)(1091,983)(1060,974)
	(1026,964)(992,955)(956,946)
	(919,936)(882,928)(843,919)
	(804,911)(763,904)(721,898)
	(677,892)(632,887)(586,883)
	(539,881)(492,880)(434,881)
	(381,885)(334,891)(292,899)
	(256,908)(224,918)(195,929)
	(170,940)(147,952)(126,964)
	(108,977)(92,988)(77,999)
	(66,1009)(57,1017)(42,1030)
\blacken\path(152.331,974.079)(42.000,1030.000)(113.035,928.738)(152.331,974.079)
\put(42,1105){\makebox(0,0)[b]{\smash{{\SetFigFont{10}{12.0}{\rmdefault}{\mddefault}{\updefault}$a$}}}}
\put(1242,1105){\makebox(0,0)[b]{\smash{{\SetFigFont{10}{12.0}{\rmdefault}{\mddefault}{\updefault}$b$}}}}
\path(42,1030)(42,205)
\blacken\path(12.000,325.000)(42.000,205.000)(72.000,325.000)(12.000,325.000)
\end{picture}
}
\caption{}\label{509-BBF16x}
\end{figure}

The answer is that sometimes figures can be very complex and for the purpose of simplification, a letter can be repeated. To be on the safe side a repeated letter can be encircled like in Figure \ref{509-BBF16c}.

\begin{figure}
\centering
\setlength{\unitlength}{0.00083333in}
\begingroup\makeatletter\ifx\SetFigFont\undefined%
\gdef\SetFigFont#1#2#3#4#5{%
  \reset@font\fontsize{#1}{#2pt}%
  \fontfamily{#3}\fontseries{#4}\fontshape{#5}%
  \selectfont}%
\fi\endgroup%
{\renewcommand{\dashlinestretch}{30}
\begin{picture}(2859,324)(0,-10)
\put(2844,105){\makebox(0,0)[b]{\smash{{\SetFigFont{10}{12.0}{\rmdefault}{\mddefault}{\updefault}$b'$}}}}
\put(1944,166){\ellipse{272}{272}}
\path(294,180)(894,180)
\blacken\path(774.000,150.000)(894.000,180.000)(774.000,210.000)(774.000,150.000)
\path(1194,180)(1719,180)
\blacken\path(1599.000,150.000)(1719.000,180.000)(1599.000,210.000)(1599.000,150.000)
\path(2094,180)(2694,180)
\blacken\path(2574.000,150.000)(2694.000,180.000)(2574.000,210.000)(2574.000,150.000)
\path(2469,180)(2544,180)
\blacken\path(2424.000,150.000)(2544.000,180.000)(2424.000,210.000)(2424.000,150.000)
\path(1494,180)(1569,180)
\blacken\path(1449.000,150.000)(1569.000,180.000)(1449.000,210.000)(1449.000,150.000)
\path(669,180)(744,180)
\blacken\path(624.000,150.000)(744.000,180.000)(624.000,210.000)(624.000,150.000)
\put(144,105){\makebox(0,0)[b]{\smash{{\SetFigFont{10}{12.0}{\rmdefault}{\mddefault}{\updefault}$a$}}}}
\put(1044,105){\makebox(0,0)[b]{\smash{{\SetFigFont{10}{12.0}{\rmdefault}{\mddefault}{\updefault}$b$}}}}
\put(1944,105){\makebox(0,0)[b]{\smash{{\SetFigFont{10}{12.0}{\rmdefault}{\mddefault}{\updefault}$a$}}}}
\put(144,143){\ellipse{272}{272}}
\end{picture}
}
\caption{}\label{509-BBF16c}
\end{figure}

The need for repetition comes from the notion of $\BBB\BBF$. In Definition \ref{509-BBD4}, we see that any $\BBB\BBF$ has two types of variables, the auxiliary variblaes $S^{\BBB\BBF}_1$, unique to the $\BBB\BBF$ containing it and disjoint from any other $\BBB\BBF'$, and the variables $S^{\BBB\BBF}_2$, which may be attacked by other $\BBB\BBF$s which may share some of these ariables in $S^{\BBB\BBF'}_2$. Consider, for example, Figure \ref{509-BBF17}.

\begin{figure}
\centering
\setlength{\unitlength}{0.00083333in}
\begingroup\makeatletter\ifx\SetFigFont\undefined%
\gdef\SetFigFont#1#2#3#4#5{%
  \reset@font\fontsize{#1}{#2pt}%
  \fontfamily{#3}\fontseries{#4}\fontshape{#5}%
  \selectfont}%
\fi\endgroup%
{\renewcommand{\dashlinestretch}{30}
\begin{picture}(1530,1993)(0,-10)
\put(765,1255){\makebox(0,0)[b]{\smash{{\SetFigFont{10}{12.0}{\rmdefault}{\mddefault}{\updefault}$\wedge$}}}}
\path(15,1780)(765,1180)(1515,1780)
\path(765,1180)(765,280)
\blacken\path(735.000,400.000)(765.000,280.000)(795.000,400.000)(735.000,400.000)
\path(765,1180)(1515,280)
\blacken\path(1415.131,352.981)(1515.000,280.000)(1461.225,391.392)(1415.131,352.981)
\path(765,1180)(15,280)
\blacken\path(68.775,391.392)(15.000,280.000)(114.869,352.981)(68.775,391.392)
\path(170,465)(115,405)
\blacken\path(173.972,513.730)(115.000,405.000)(218.202,473.187)(173.972,513.730)
\path(760,520)(775,430)
\blacken\path(725.680,543.435)(775.000,430.000)(784.864,553.299)(725.680,543.435)
\path(1355,475)(1410,400)
\blacken\path(1314.844,479.028)(1410.000,400.000)(1363.229,514.510)(1314.844,479.028)
\put(15,1855){\makebox(0,0)[b]{\smash{{\SetFigFont{10}{12.0}{\rmdefault}{\mddefault}{\updefault}$a$}}}}
\put(1515,1855){\makebox(0,0)[b]{\smash{{\SetFigFont{10}{12.0}{\rmdefault}{\mddefault}{\updefault}$c$}}}}
\put(15,55){\makebox(0,0)[b]{\smash{{\SetFigFont{10}{12.0}{\rmdefault}{\mddefault}{\updefault}$a$}}}}
\put(765,55){\makebox(0,0)[b]{\smash{{\SetFigFont{10}{12.0}{\rmdefault}{\mddefault}{\updefault}$b$}}}}
\put(1515,55){\makebox(0,0)[b]{\smash{{\SetFigFont{10}{12.0}{\rmdefault}{\mddefault}{\updefault}$c$}}}}
\put(765.000,1236.250){\arc{487.500}{3.5364}{5.8884}}
\end{picture}
}
\caption{}\label{509-BBF17}
\end{figure}

In this figure, $\{a,c\}$ mount jointly a disjunctive attack on $\{a,b,c\}$. See Appendix C1.  If we do not want repetition of nodes, we write Figure \ref{509-BBF17a}.

\begin{figure}
\centering
\setlength{\unitlength}{0.00083333in}
\begingroup\makeatletter\ifx\SetFigFont\undefined%
\gdef\SetFigFont#1#2#3#4#5{%
  \reset@font\fontsize{#1}{#2pt}%
  \fontfamily{#3}\fontseries{#4}\fontshape{#5}%
  \selectfont}%
\fi\endgroup%
{\renewcommand{\dashlinestretch}{30}
\begin{picture}(2222,1468)(0,-10)
\put(1085,55){\makebox(0,0)[b]{\smash{{\SetFigFont{10}{12.0}{\rmdefault}{\mddefault}{\updefault}$b$}}}}
\path(1085,730)(1085,280)
\blacken\path(1055.000,400.000)(1085.000,280.000)(1115.000,400.000)(1055.000,400.000)
\path(1085,505)(1085,430)
\blacken\path(1055.000,550.000)(1085.000,430.000)(1115.000,550.000)(1055.000,550.000)
\path(2100,1105)(2040,1160)
\blacken\path(2148.730,1101.028)(2040.000,1160.000)(2108.187,1056.798)(2148.730,1101.028)
\path(55,1025)(100,1115)
\blacken\path(73.167,994.252)(100.000,1115.000)(19.502,1021.085)(73.167,994.252)
\path(1085,730)(1086,730)(1090,728)
	(1096,726)(1105,723)(1118,719)
	(1135,713)(1155,707)(1179,699)
	(1207,690)(1237,681)(1270,671)
	(1304,661)(1340,650)(1378,640)
	(1416,630)(1454,620)(1494,611)
	(1533,603)(1573,595)(1614,589)
	(1655,583)(1697,578)(1740,575)
	(1783,574)(1826,574)(1869,576)
	(1910,580)(1952,587)(1989,596)
	(2022,605)(2050,615)(2074,625)
	(2094,635)(2110,644)(2122,652)
	(2132,659)(2140,667)(2146,673)
	(2150,680)(2154,686)(2157,693)
	(2160,700)(2163,708)(2167,717)
	(2171,727)(2175,739)(2181,753)
	(2187,769)(2193,787)(2199,808)
	(2204,830)(2208,854)(2210,880)
	(2207,914)(2199,947)(2187,978)
	(2172,1008)(2155,1035)(2135,1061)
	(2113,1086)(2090,1110)(2066,1133)
	(2041,1154)(2017,1175)(1993,1193)
	(1972,1210)(1952,1224)(1937,1236)(1910,1255)
\blacken\path(2025.402,1210.475)(1910.000,1255.000)(1990.872,1161.407)(2025.402,1210.475)
\path(1085,730)(1083,729)(1080,728)
	(1073,726)(1063,723)(1049,718)
	(1031,712)(1009,705)(984,697)
	(955,687)(924,678)(890,668)
	(855,657)(819,647)(782,637)
	(744,627)(706,618)(667,610)
	(627,602)(587,595)(546,589)
	(504,584)(462,580)(419,578)
	(376,578)(335,580)(293,584)
	(255,590)(223,597)(195,604)
	(173,610)(155,617)(141,622)
	(131,627)(123,631)(117,635)
	(113,639)(110,643)(107,647)
	(105,651)(102,657)(98,664)
	(93,672)(87,683)(80,696)
	(71,712)(62,731)(52,753)
	(43,778)(35,805)(30,840)
	(28,874)(31,909)(36,942)
	(44,974)(53,1005)(65,1036)
	(77,1065)(91,1094)(105,1121)
	(120,1148)(134,1172)(147,1195)
	(159,1214)(168,1229)(185,1255)
\blacken\path(144.439,1138.146)(185.000,1255.000)(94.221,1170.981)(144.439,1138.146)
\put(185,1330){\makebox(0,0)[b]{\smash{{\SetFigFont{10}{12.0}{\rmdefault}{\mddefault}{\updefault}$a$}}}}
\put(1835,1330){\makebox(0,0)[b]{\smash{{\SetFigFont{10}{12.0}{\rmdefault}{\mddefault}{\updefault}$c$}}}}
\path(185,1255)(1085,730)(1835,1255)
\end{picture}
}
\caption{}\label{509-BBF17a}
\end{figure}

However, if we want to implement Figure \ref{509-BBF17a} using Boolean attack formations as discussed in Appendix C2, we get Figure \ref{509-BBF18}, with auxiliary nodes $a^*_1, c^*_1, z^*_1, w^*_1$ and the auxiliary nodes of Figure \ref{509-BBF14} for $\{a,b,c\}$.

\begin{figure}
\centering
\setlength{\unitlength}{0.00083333in}
\begingroup\makeatletter\ifx\SetFigFont\undefined%
\gdef\SetFigFont#1#2#3#4#5{%
  \reset@font\fontsize{#1}{#2pt}%
  \fontfamily{#3}\fontseries{#4}\fontshape{#5}%
  \selectfont}%
\fi\endgroup%
{\renewcommand{\dashlinestretch}{30}
\begin{picture}(1584,4350)(0,-10)
\put(792,312){\makebox(0,0)[b]{\smash{{\SetFigFont{10}{12.0}{\rmdefault}{\mddefault}{\updefault}{\bf out}}}}}
\path(1542,4212)(1542,3762)
\blacken\path(1512.000,3882.000)(1542.000,3762.000)(1572.000,3882.000)(1512.000,3882.000)
\path(1542,3537)(792,3162)
\blacken\path(885.915,3242.498)(792.000,3162.000)(912.748,3188.833)(885.915,3242.498)
\path(42,3537)(792,3162)
\blacken\path(671.252,3188.833)(792.000,3162.000)(698.085,3242.498)(671.252,3188.833)
\path(792,2862)(792,2262)
\blacken\path(762.000,2382.000)(792.000,2262.000)(822.000,2382.000)(762.000,2382.000)
\path(792,2037)(792,1512)
\blacken\path(762.000,1632.000)(792.000,1512.000)(822.000,1632.000)(762.000,1632.000)
\path(792,1512)(1242,762)(792,12)
	(342,762)(792,1512)
\put(42,4212){\makebox(0,0)[b]{\smash{{\SetFigFont{10}{12.0}{\rmdefault}{\mddefault}{\updefault}$a$}}}}
\put(1542,4212){\makebox(0,0)[b]{\smash{{\SetFigFont{10}{12.0}{\rmdefault}{\mddefault}{\updefault}$c$}}}}
\put(1542,3612){\makebox(0,0)[b]{\smash{{\SetFigFont{10}{12.0}{\rmdefault}{\mddefault}{\updefault}$c^*_1$}}}}
\put(42,3612){\makebox(0,0)[b]{\smash{{\SetFigFont{10}{12.0}{\rmdefault}{\mddefault}{\updefault}$a^*_1$}}}}
\put(792,2937){\makebox(0,0)[b]{\smash{{\SetFigFont{10}{12.0}{\rmdefault}{\mddefault}{\updefault}$z^*_1$}}}}
\put(792,2112){\makebox(0,0)[b]{\smash{{\SetFigFont{10}{12.0}{\rmdefault}{\mddefault}{\updefault}$w^*_1$}}}}
\put(792,1062){\makebox(0,0)[b]{\smash{{\SetFigFont{10}{12.0}{\rmdefault}{\mddefault}{\updefault}{\bf  in}}}}}
\put(792,867){\makebox(0,0)[b]{\smash{{\SetFigFont{8}{9.6}{\rmdefault}{\mddefault}{\updefault}Figure \ref{509-BBF14}}}}}
\put(792,612){\makebox(0,0)[b]{\smash{{\SetFigFont{8}{9.6}{\rmdefault}{\mddefault}{\updefault}for $\{a,c,b\}$}}}}
\path(42,4137)(42,3762)
\blacken\path(12.000,3882.000)(42.000,3762.000)(72.000,3882.000)(12.000,3882.000)
\end{picture}
}

\caption{}\label{509-BBF18}
\end{figure}

So be mindful that the following definitions allow for repetions in the figures.

\begin{definition}\label{509-BBD19}{\ }
\begin{enumerate}
\item Let $\BBB\BBF_1$ and $\BBB\BBF_2$ be two Boolean attack formations with {\bf in}$_1$ and {\bf out}$_1$ and {\bf in}$_2$ and {\bf out}$_2$.  We assume that the auxiliary nodes of these formations are disjoint. Then Figure \ref{509-BBF20} describes the resulting formation for the attack of $\BBB\BBF_1$ on $\BBB\BBF_2$.

\begin{figure}
\centering
\setlength{\unitlength}{0.00083333in}
\begingroup\makeatletter\ifx\SetFigFont\undefined%
\gdef\SetFigFont#1#2#3#4#5{%
  \reset@font\fontsize{#1}{#2pt}%
  \fontfamily{#3}\fontseries{#4}\fontshape{#5}%
  \selectfont}%
\fi\endgroup%
{\renewcommand{\dashlinestretch}{30}
\begin{picture}(946,3939)(0,-10)
\put(462,837){\makebox(0,0)[b]{\smash{{\SetFigFont{10}{12.0}{\rmdefault}{\mddefault}{\updefault}${\bf out}_2$}}}}
\path(484,1949)(34,1199)(484,449)
	(934,1199)(484,1949)
\path(462,2637)(462,1587)
\blacken\path(432.000,1707.000)(462.000,1587.000)(492.000,1707.000)(432.000,1707.000)
\path(482,687)(482,12)
\blacken\path(452.000,132.000)(482.000,12.000)(512.000,132.000)(452.000,132.000)
\path(482,242)(482,167)
\blacken\path(452.000,287.000)(482.000,167.000)(512.000,287.000)(452.000,287.000)
\path(462,1802)(462,1717)
\blacken\path(432.000,1837.000)(462.000,1717.000)(492.000,1837.000)(432.000,1837.000)
\put(462,3462){\makebox(0,0)[b]{\smash{{\SetFigFont{10}{12.0}{\rmdefault}{\mddefault}{\updefault}${\bf in}_1$}}}}
\put(462,2712){\makebox(0,0)[b]{\smash{{\SetFigFont{10}{12.0}{\rmdefault}{\mddefault}{\updefault}${\bf out}_1$}}}}
\put(462,3087){\makebox(0,0)[b]{\smash{{\SetFigFont{10}{12.0}{\rmdefault}{\mddefault}{\updefault}\ldots}}}}
\put(462,1137){\makebox(0,0)[b]{\smash{{\SetFigFont{10}{12.0}{\rmdefault}{\mddefault}{\updefault}\ldots}}}}
\put(462,1437){\makebox(0,0)[b]{\smash{{\SetFigFont{10}{12.0}{\rmdefault}{\mddefault}{\updefault}${\bf in}_2$}}}}
\path(462,3912)(12,3162)(462,2412)
	(912,3162)(462,3912)
\end{picture}
}
\caption{}\label{509-BBF20}
\end{figure}

\item Let $(S, R)$ be an arumentation  network and let $I$ be a function associating with each $x\in S$, a Boolean attack formation $I(x) =\BBB\BBF_x$, with {\bf in}$_x, {\bf out}_x, S_x, R_x$. Assume that all the auxiliary nodes sets of $\BBB\BBF_x$ for $x\in S$ are pairwise disjoint and disjoint from $S$.  Then the result of the instantiation $(S_I, R_I)$ is the following:
\[\begin{array}{lcl}
S_I &=& \bigcup_{x\in S} S_x\\
R_I&=& \bigcup_{x\in S}R_x \cup \{{\bf out}_x\tO {\bf in}_y | x\tO y \in R\}.
\end{array}
\]
\item The semantics for such instantiation  is taken to be option (iv) (the non-toxic truth intervention) semantics of Appendix B.
\end{enumerate}
\end{definition}

\subsection{Instantiating with classical propositional wffs using Boolean attack formations}
In this Appendix C.4, we backup the comments made in Remark \ref{509-BBR21}.  We are given a finite instantiated network $(S, R, I)$.  $I$ is an instantiation into classical propositional logic and we can assume that all wffs involved are built up from the set of atoms $Q = \{q_1\comma q_n\}$ using $\wedge$ and $\neg$.  We want to identify the Caminada extensions for $(S, R)$, arising from the instantiation $I$.  We know that for each $x \in S$, the wff $I(x) =\Phi_x$ is a formula of a logic which has models \Bm. We know that if we go through all the models $\Bm$ and define 
\[
\lambda_\Bm (x) ={\rm def.}~
 \Bm(\Phi_x)
\] 
then the legitimate Caminada extension of 
$(S, R)$ among the set 
$\{\lambda_\Bm\}$ will be the set of all extensions of 
$(S, R, I)$.  We can do that, however, we do not want to use the logic, we want to use purely syntactical means.  By the results in Appendix C.2, every formula $\Phi$ of classical propositional logic can be represented by a Boolean attack formation 
$\BBB\BBF(\Phi)$. 
We can look at the instantiation 
$I^*$ where 
$I^* (x) =\BBB\BBF(\Phi_x)$ for 
$x\in S$ and consider the network  
$(S^*, R^*) = (S, R, I^*)$ 
as defined according to Definition \ref{509-BBD19} of Appendix C.3.  

The network $(S^*, R^*)$ has extensions $\mu$ (the non-toxic extensions).  

These extensions give values in $\{0, \half, 1\}$ to the atoms of $Q$.  These define a model $\Bm (\mu)$ to the formulas $\Phi_x, x\in S$. We look at $\{\lambda_{\Bm (\mu)}\}$ and the legitimate Caminada extensions from among all the extensions of $(S, R, I)$.

\begin{remark}\label{509-BC41}
Note that as a byproduct of our process of Remark \ref{509-BBR21} and Appendix C.4, we get argumentation semantics for classical propositioal logic. Any $\Phi$ is represented by $\BBB\BBF(\Phi)$, whose extensions give the models of $\Phi$.
\end{remark}

\section{Comparing Boolean instantiation with abstract dialectical framework (ADF).}
We first introduce ADF from the original brilliant paper of Brewka and Woltran \cite{509-2,509-34}, and then analyse ADF and explain why it is not suitable for us.  See also \cite{509-36}.

Brewka and Woltran read condition $(\sharp 1$ atomic) above (i.e the view: ``$ y_1\comma y_n$  attack $x$",  see Figure \ref{509-2F6}) not as an attack of $y_i$ on $x$ but as an acceptance condition, relating the $\{0,1\}$ values of the $y_i$ to the $\{0,1\}$ value of $x$. The condition, according to the ADF view,  is 
\begin{center}
$x=1$ iff all $y_i=0$, (namely we take $I_{(\sharp 1 \mbox{ atomic})}(x)$ as $\bigwedge\neg y_i$).
\end{center}
  This is a brilliant shift in point of view and we now can generalise and put forward different conditions, say $I(x)=C_x(y_1\comma y_n)$.  We can now interpret $I(x)$ as the Brewka--Woltran condition for the acceptance of $x$. To achieve that we must restrict the variables in $I(x)$ in our network $(S, R, I)$ to be the set $Y_x$ of all the attackers in $(S, R)$ of the node $x$.

So our networks will look like $(S, R, I)$, where $I(x) =\Phi_x\{Y_x\}$.

Now we can regard our instantiation networks with this additional restriction possibly  as an ADF and we might think that our problems are solved. We have given a meaning to $(S, R, I)$, at least for the case of this additional restriction.

The answer is that this is not the meaning we want and need. We cannot get the general case of instantiations which we are studying, even with this additional restriction. Consider the simple two node network with nodes $x$ and $y$ and with $x$ attacking $y$. Since $x$ is not attacked, ADF can give it a truth value only, say value \Bt,  and since $x$ is the only attacker of $y$, ADF can give it a wff  $\Phi (x)$ with $x$ as the only variable of $\Phi$. ADF will write the equalities  
\[\begin{array}{l}
x=\Bt\\
 y= \Phi (x).
\end{array}
\]
While in comparison, we would write the equalities
\[\begin{array}{l}
x= \Bt\\
\Phi(x) = \neg  \Bt
\end{array}
\]
Note however, that we can get ADF as a special case of our instatiation if we take an arbitrary network with nodes $S$ and no attacks whatsoever, i.e.\ $R=\emptyset$.

If we substitute now for each $x$ in $S$ a wff $\Phi_x$, we get an ADF.

The fact is that ADF is not doing argumentation but is really doing logic programming. We are departing from the principle $(\sharp 1)$ and we are not regarding $I(y_1)\comma I(y_m)$ as attacking $I(x)$.

If we take the ADF view then we are not doing argumentation any more. We are not really instantiating $x$ to be $I(x)$. We can no longer connect with, e.g., ASPIC+ or any other instantiation papers. We are playing a different game!

This is not a criticism of ADF; all we are saying is that the ADF  approach is not compatible with our approach. We think highly enough of ADF to mention it here and say that we cannot use it. Note for example   that with the ADF restriction, $R$ is redundant. This is a warning sign for us. All we really have in ADF  are logic programming clauses of the form 
\[
x \mbox{ if } I(x)
\]
and we have 
\[
xRy\mbox{ iff $y$ appears in } I(x).
\]

Thus in ADF $R$ is not an attack relation but an ``occur" relation. Brewka and Woltran are careful; they call it ``link" relation.

Our view about ADF becomes clearer when we consider it in the context of the equational approach \cite{509-1,509-4,509-6,509-7,509-36}. 

Given $(S, R)$, we consider the network as generating equations of the form $Eq_{\max}$:
\[
x=1-\max(Y_x)
\]
where $x\in S$ and $Y_x =\{y|yRx\}$.

The solution of these equations in $[0,1]$ yield all the Dung extensions, where with each \Bf\ we associate a Caminada labelling $\lambda (\Bf)$ as follows:
\[\begin{array}{l}
\lambda (\Bf) = \mbox{ in, if } \Bf(x)=1\\
\lambda (\Bf) =\mbox{ out, if } \Bf(x) =0\\
\lambda(\Bf)=\mbox{undecided, if } 0 < \Bf(x) < 1.
\end{array}
\]

Let us now depart from the above equations in two ways.
\begin{enumerate}
\item Let us have possibly different continuous functions $h_x(Y_x)$ associated with each $x\in S$.
\item Let $h_x(Y_x)$ be a Boolean function in $Y_x$
\end{enumerate}

The solutions $\Bf$ of the above system of equation (call it $Eq_B$) would yield us a Boolean ADF. See \cite[p. 804]{509-34} and \cite{509-36}.

The reader should note the generality of the equational approach. We can choose any family of continuous functions. We can choose T-norms which generalise the classical connectives and using T-norms study numerical ADF (T-norm). However, this is not the place to elaborate, but we leave this to a future paper.  

Let us at this point quote from paper \cite{509-34}:

\begin{quote}
Begin quote from \cite{509-34}.

\noindent 
{\bf Definition 1.} An {\em abstract dialectical framework} is a tuple $D = (S, L, C)$ where 
\begin{itemize}
\item $S$ is a set of statements (positions, nodes),
\item $L \subseteq S \times S$ is a set of links,
\item $C =\{C_s\}_{s\in S}$ is a set of total functions $C_s: 2^{par(s)}\to \{\Bt,\Bf\}$, one for each statement $s$. $C_s$ is called acceptance condition of $S$
\end{itemize}

In many cases it is convenient to represent acceptance conditions as propositional formulas. For this reason we will frequently use a logical representation of ADFs $(S, L, C)$ where $C$ is a collection $\{\varphi_s\}_{s\in S}$ is called acceptance condition of $s$.

In many cases it is convenient to represent acceptance conditions as propositional formulas. For this reason we will frequently use a lgoical representation of ADFs $(S, L, C)$ where $C$ is a collection $\{\varphi_s\}_{s\in S}$ of propositional formulas.

Moreover, unless specified differently we will tacitly assume that the acceptance formulas specify the parents a node depends on implicitly. It is then not necessary to give the links in the graph explicitly.  We thus can represent an ADF $D$ as a tuple $(S, C)$ where $S$ and $C$ are as above and $L$ is implicitly given as $(ab)\in L$ iff $a$ appears in $\varphi_b$.\footnote{When presenting examples we will use a notation where acceptance conditions are written in square brackets behind nodes, e.g.\ $c [\neg (a\wedge b)]$ denotes a node $a$ which is jointly attacked by nodes $a$ and $b$, that is, each attacker alone is insufficient to defeat $c$.}

The different semantics of ADFs over statements $S$ are based on the notion of a model. A two-valued interpretation $v$ --- a mapping from statements to the truth balues true and false --- is a {\em two-valued model} ({\em model}, if clear from the context) of an ADF $(S, C)$ whenever for all statements $s \in S$ we have $v(s) =v(\varphi_s)$, that is $v$ maps exactly those statements to true whose acceptance conditions are satisfied under $v$.  Our analysis in this paper will be based on a straightforward generalization of two-valued interpretations for ADFs to Kleene's strong three-valued logic \cite{509-18}.\footnote{A comparable treatment for AFs was given by the labellings of \cite{509-Caminada2006}.  We use standard notation and terminology from mathematical logic.}  A three-valued interpretation is a mapping $v: S\to \{\Bt, \Bf, \Bu\}$ that assigns one of he truth values true (\Bt), false \Bf) or unknown (\Bu) to each statement.  Interpretations can easily be extended to assign truth values to propositional formulas over the statements: negation switches \Bt\ and \Bf, and leaves \Bu\ unchanged; a conjunction is \Bt\ if both conjuncts are \Bt, it is \Bf\ if some conjunct is \Bf\
 and it is \Bu\ otherwise; disjunction is dual. It is also straightforward to generalize the notion of a model: a three=valued interpretation is a model whenever for all statements $s\in S$ we have $v(s)\neq \Bu$ implies $v(s) =v(\varphi_s)$.

The three truth values are partially ordered by $\leq_i$ according to their information content: we have $
Bu <_i\Bt$ and $\Bu <_i\Bf$ and no other pair in $<_i$, which intuitively means that the classical truth valuecontain more information than the truth value unknown. The pair $(\{\Bt, \Bf, \Bu\} \leq_i)$ forms a complete meet-semilattice\footnote{A complete meet-semilattice is such that every non-empty finite subset has a greatest lower bound, the meet; and every non-empty directed subset has a least upper bound. A subset is directed if any two of its elements have an upper bound in the set.} with the meet operation $\sqcap$. This meet can be read as {\em consensus} and assigns $\Bt\sqcap \Bt =\Bt, \Bf\sqcap \Bf=\Bf$, and returns \Bu\ otherwise.

End quote from \cite{509-34}.
\end{quote}

Let us conclude this Appendix and quote here from the Brewka--Woltran paper. They are aware that they are really doing logic programming but, I assume for their own reasons, they still present their paper as an argumentation paper.  See also \cite{509-9,509-10,509-12}.  We present three quotations from \cite{509-2}.

\begin{quote}
Begin quote 1:\\
Vocabulary for the quotation below: $s$ is a node. $Par(s)$ are its attackers, it is $Y_s$ in our notation. $C_s$ is our $I(s)$.

\medskip\noindent
Definition 1. An abstract dialectical framework is a tuple
$D = (S, L, C)$ where
\begin{itemize}
\item  $S$ is a set of statements (positions, nodes),
\item $L \subseteq S\times S $ is a set of links,
\item $C = \{C_s \}_{s\in S}$ is a set of total functions $C_s: 2^{par(s)}\to \{\mbox{in, out}\}$,  one for each statement $s$. $C_s$ is called acceptance condition of $s$.
\end{itemize}

\medskip\noindent 
Begin quote 2:\\
Definition 2.   Let $D = (S, L, C)$ be an ADF. $M \subseteq S$ is
called conflict-free (in $D$) if for all $s \in M$ we have $C_s (M \cap
par (s)) =$ in. Moreover, $M \subseteq S$ is a model of $D$ if $M$
is conflict-free and for each $s in S, C_s(M\cap par(s)) =$ in implies $s\in M$.

  In other words, $M\subseteq S$ is a model of $D = (S, L, C)$ if for all $s \in S$ we have $s \in M$ iff $C_s (M \cap par (s)) =$ in.

We say $M$ is a minimal model if there is no model $M$
which is a proper subset of $M$.

\medskip\noindent
Begin quote 3:\\
It is not difficult to verify that, when the acceptance condition of each node s is represented as a propositional formula $F (s)$, a model is just a propositional model of the set of formulas
\[
\{s\equiv F(s)|s\in S\}.
\]
End quotes from \cite{509-2}.
\end{quote}

\begin{remark}\label{509-ADF1}
To further highlight the fact that the idea of Instantiation is   different from and is  orthogonal to the ADF idea, let us define the notion of abstract Instantiated Dialectical Frameworks.

An abstract dialectical framework is a tuple $DI = (S, L, C, I)$ 
 where
\begin{itemize}
\item $S$ is a set of statements (positions, nodes),
\item $L \subseteq S\times S$ is a set of links,
\item $C = \{C_s\}, s \in S$   is a set of total functions $C_s: 2^{par(s)} \to \{\mbox{in, out}\}$,
one for each statement $s$. $C_s$  is called acceptance condition of
s.
\item For each $s \in S, I(s)$ is a wff of classical propositional logic. We consider I(s) as a function from $2^{par(s)}  to  \{in, out\}$.
\end{itemize}

Following quote 3 above , we understand by a model of the system any solution to the equations for $s \in S$, 
\[
I(s) = C_s( y/I(y)| y \mbox{ a parent of } s).
\]

This system of equations is not guaranteed a solution is the real interval $[0,1]$, while the system of Quote 3, does always have a solution in $[0,1]$, though not always in $\{0, 1/2 , 1\}$.
\end{remark}

\section{}
\subsection{Comparing abstract instantiation with the ASPIC approach}
When we talk in this paper about ``instantiation in argumentation'', we must compare what we are doing with the well known school of ``instantiated argumentation networks'' and the ASPIC movement \cite{509-13,509-30}.  What is the connection between what this paper is doing and ASPIC?  The answer is that they are similar but different. We first explain the difference in principle and then give examples.

\paragraph{1.  The abstract instantiation of the current paper.}
In this paper we start with an abstract argumentation frame $(S, R)$ and in parallel with a closed logical theory $\Delta$.\footnote{By a closed theory we mean a theory containing whatever it proves. This can also be defined not only for monotonic theories but also  for non-monotonic theories which satisfy the restricted monotonicity rule, namely:
\[
\Delta \proves \Phi  \mbox{ and } \Delta + \Phi \proves \Psi ==> \Delta \proves \Psi.
\]} It is important to note that $\Delta$ need not be a defeasible theory. In fact our main case studies are monotonic logics; classical propositional logic, monadic predicate logic and modal logic S5. We can also use a defeasible theory if we want.  We substitute wffs of $\Delta$ into $S$ via a function $I: S\mapsto \Delta$. In this set up the network $(S, R)$ and the attack structure $R$ of the argumenation frame is primary and it is retained and is influenced by the substitution function $I$. Our main task is to give the system $(S, R, I)$ proper meaning.  This substitution may cause problems and we need to find a theoretical remedy. 

For the purpose of comparing with ASPIC and with other instantiation approaches, such as \cite{509-8,509-26,509-27,509-29}, let us offer another   way of looking at the same problem. It is to say that we have a basic consistent closed theory $\Delta$ and in parallel we have an abstract schema $(S,R)$ of attack relation. By substituting, using $I$,  formulas from  $\Delta$ for the elements of $S$, we form an indexed collection of formulas $\Gamma =\{I(x) | x\in S\}$ of wffs from $\Delta$ containing an abstract unspecified conflict, as  recorded by $R$. If this abstract conflict could be expressed in the language of $\Delta$ as a theory $\Delta_R$, then $\Delta \cup \Delta_R$ would have been a defeasible theory or a monotonic consistent theory or a monotonic inconsistent theory, all depending on the underlying logic. We ask the question: is it possible to retain consistency of $\Gamma$ and yet satisfy the constraints $R$?  

To motivate this point of view of constraints, and to be able to have a working case study to use in comparing with ASPIC, let us give an example.

\begin{example}\label{509-EE0}
Let $\Delta$ be a theory governing a birthday party.  Let $S$ be a set of people and let $I(x)$ for $x\in S$ be ``$x$ is invited to the party''.  Suppose everything is consistent together, namely $\Delta\cup \{I(x) |x\in S\}$ is consistent.  We can have the party as required by $\Delta$ and invite all the people of $S$. Bear in mind that $\Delta$ may contain requirements which may affect the people invited.

We now add the constraints $R$ of the form $xRy$ meaning  ``If you invite $x$ you cannot invite $y$''.   

OK, our problem now is, whom do we invite and still satisfy the constraints?  In other words, we are looking for an extension (in fact a maximal preferred extension) of $(S,R,I)$. 
\end{example}

\paragraph{2.  Instantiated argumentation: ASPIC.}
In comparison ASPIC does something different.  It will look at all the arguments (proofs) generated by a defeasible theory $\Delta$.  Since $\Delta$ is defeasible, it might defeasibly prove both  an $x$ and its negation $\neg x$. This can be viewed as  giving rise to an attack relation (several possible notions of attack relations) among the set of all proofs from the theory $\Delta$. We need to resolve the conflict. ASPIC will use argumentation theory  and {\em construct} an instantiated network $(S, R)$ accordingly.  Then ASPIC would apply the machinery of finding extensions and {\em expect} the resulting wffs of the extensions to be consistent in the logic of $\Delta$.

When this does not happen, ASPIC offers postulates to {\em force} $\Delta$ to behave as to ensure that the resulting extensions are consistent.

Thus the ASPIC  process has three stages:	
\begin{enumerate}
\item  {\bf The Input}\\  Input a defeasible theory $\Delta$.
\item {\bf The System}\\ Construct an argumentation network $(S,R)$ from $\Delta$. Call the process of construction the `ASPIC construction". This is ASPIC's way of doing it. 	
\item {\bf The Output}\\ Output complete extension of the network $(S,R)$ constructed in 2 and  expect these  complete extensions  to be consistent in the logic of $\Delta$.
\end{enumerate}
When the output 3 does not meet expectations, ASPIC restricts the input by putting postulates on $\Delta$. ASPIC does not try and improve or change the construction in point 2. (We shall offer a different way of constructing a network  in Appendix E2).

\paragraph{3.  Summary comparison.}
As you can see, although our paper and ASPIC deal with similar (lego pieces) components, they do not do the same things.  Several examples would be helpful.

\begin{example}\label{509-EE1}
We look at Example 6 from the Caminada and Amgoud paper \cite{509-16}.

A defeasible theory $\Delta$ is given with strict rules $\CS$ and defeasible rules $\CD$ where 
\[
\begin{array}{l}
\CS =\{a, d, g, b\wedge c\wedge e \wedge f\to \neg g\}\\
\CD =\{a\To b, b\To c, e\To f, d\To e\}
\end{array}
\]
and where ``$\to$'' is strict implication and $\To$ is defeasible implication. There are two modus ponens rules, one for each implication.

Following the ASPIC instantiation idea, we look at all possible proofs of atoms from $\Delta$.  These are (notation: [\ldots]):

\begin{itemlist}{PPPPPP}
\item [$\Pi$1.] $[a]$
\item [$\Pi$2.] $[d]$
\item [$\Pi$3.] $[g]$
\item [$\Pi$4.] $[a,a\To b]$, (proving $b$)
\item [$\Pi$5.] $[d, d\To e]$, (proving $e$)
\item [$\Pi$6.] $[a, a\To b, b\To c]$, (proving $c$)
\item [$\Pi$7.] $[d, d\To e, e\To f]$, (proving $f$)
\item [$\Pi$8.] $[4, 5, 6, 7, b\wedge c\wedge e \wedge f\to \neg g]$, (proving $\neg g$)
\end{itemlist}

The paper \cite{509-16} uses the proofs $\Pi 1$--$\Pi 7$ to construct an argumentation   network 
\[\begin{array}{l}
S_\Delta =\{\Pi 1, \Pi 2, \Pi 3, \Pi 4, \Pi 5, \Pi 6, \Pi 7\}\\
R_\Delta =\{\Pi R\Pi'\mbox{ if } \Pi \mbox{ proves $x$ and $\Pi'$ proves }\neg x\}.
\end{array}
\]
The way \cite{509-16} defines $S_\Delta$ it turns out that the proof $\Pi 8$ cannot be used and is not a member of $S_\Delta$. This is because they string the data without using conjunctions, using only ``linear'' sequences of the form $x_1, x_1\rightsquigarrow x_2, x_2\rightsquigarrow x_3\comma$ where ``$\rightsquigarrow$'' is either ``$\to$'' or ``$\To$'' and where each $x_i$ is atomic.  This way of constructing $(S_\Delta, R_\Delta)$ causes problems.

If we collect the atoms proved in the proofs of $S_\Delta$ and which are in the ground extension, we get $\Gamma =\{a,d,g,b,e,c,f\}$. Using the logic of $\Delta$ and the rule $b\wedge c\wedge e\wedge f\to \neg g$, we get that $\Gamma$ is contradictory.  Thus taking extensions may result in inconsistent sets.

So \cite{509-16} offers postulates on theories $\Delta$, hoping to ensure there will be no problems.

We are not here to evaluate \cite{509-16} or the ASPIC approach. We are just comparing their idea of ``instantiation'' with ours.

\begin{itemize}
\item ASPIC starts with $\Delta$ and constructs $(S_\Delta, R_\Delta)$.
\item We start with $(S, R)$ and substitute proofs from $\Delta$.
\end{itemize}
We could end up with similar difficulties. Suppose we start with 
\[
S=\{x_1\comma x_8\}, R=\varnothing
\]
and $I(x_i)=\Pi i$. We end up with 
\[
S_I=\{\Pi_1 \comma \Pi_8\}.
\]

The question is, in our methodology, how do we define $R_I$?  Since we allow for joint attacks, we can express $\Pi_8$. Let us play the game the ASPIC way. We are given $\Delta$, let us ask, since we have powerful machinery in this paper, if we wanted to construct $(S_\Delta, R_\Delta)$, how would we have done it?\footnote{Recall the discussion in Section 1.2. ASPIC restricts the input, i.e. using (r1). We are going to change the system, the way we construct the network, i.e. we are using (r2).}

Our answer is that we would construct a 2-state network with joint attacks.  Figure \ref{509-EF2} illustrates the network we get.

\begin{figure}
\centering
\setlength{\unitlength}{0.00083333in}
\begingroup\makeatletter\ifx\SetFigFont\undefined%
\gdef\SetFigFont#1#2#3#4#5{%
  \reset@font\fontsize{#1}{#2pt}%
  \fontfamily{#3}\fontseries{#4}\fontshape{#5}%
  \selectfont}%
\fi\endgroup%
{\renewcommand{\dashlinestretch}{30}
\begin{picture}(5127,2922)(0,-10)
\put(782,60){\makebox(0,0)[b]{\smash{{\SetFigFont{10}{12.0}{\rmdefault}{\mddefault}{\updefault}$\neg g$}}}}
\path(1812,1785)(2547,2395)
\blacken\path(2473.818,2295.278)(2547.000,2395.000)(2435.500,2341.449)(2473.818,2295.278)
\path(2712,1785)(3542,2430)
\blacken\path(3465.655,2332.678)(3542.000,2430.000)(3428.839,2380.055)(3465.655,2332.678)
\path(4662,1635)(5112,960)
\blacken\path(5020.474,1043.205)(5112.000,960.000)(5070.397,1076.487)(5020.474,1043.205)
\path(757,2095)(787,1995)
\blacken\path(723.783,2101.319)(787.000,1995.000)(781.253,2118.560)(723.783,2101.319)
\path(307,900)(377,875)
\blacken\path(253.901,887.108)(377.000,875.000)(274.081,943.613)(253.901,887.108)
\path(322,240)(387,200)
\blacken\path(269.078,237.342)(387.000,200.000)(300.524,288.441)(269.078,237.342)
\path(2077,185)(1957,150)
\blacken\path(2063.800,212.400)(1957.000,150.000)(2080.600,154.800)(2063.800,212.400)
\path(4982,1155)(5017,1080)
\blacken\path(4939.068,1176.055)(5017.000,1080.000)(4993.439,1201.429)(4939.068,1176.055)
\path(2352,2240)(2422,2295)
\blacken\path(2346.176,2197.272)(2422.000,2295.000)(2309.107,2244.451)(2346.176,2197.272)
\path(3337,2270)(3432,2335)
\blacken\path(3349.904,2242.479)(3432.000,2335.000)(3316.023,2291.997)(3349.904,2242.479)
\path(4642,1615)(4917,1215)
\path(2702,1810)(3287,2265)
\path(1812,1815)(2297,2220)
\blacken\path(1002.000,1790.000)(882.000,1760.000)(1002.000,1730.000)(1002.000,1790.000)
\path(882,1760)(1527,1760)
\blacken\path(1407.000,1730.000)(1527.000,1760.000)(1407.000,1790.000)(1407.000,1730.000)
\blacken\path(1002.000,895.000)(882.000,865.000)(1002.000,835.000)(1002.000,895.000)
\path(882,865)(1527,865)
\blacken\path(1407.000,835.000)(1527.000,865.000)(1407.000,895.000)(1407.000,835.000)
\blacken\path(1012.000,145.000)(892.000,115.000)(1012.000,85.000)(1012.000,145.000)
\path(892,115)(1537,115)
\blacken\path(1417.000,85.000)(1537.000,115.000)(1417.000,145.000)(1417.000,85.000)
\blacken\path(1111.595,1786.576)(992.000,1755.000)(1112.384,1726.581)(1111.595,1786.576)
\path(992,1755)(1372,1760)(1407,1760)
\blacken\path(1287.000,1730.000)(1407.000,1760.000)(1287.000,1790.000)(1287.000,1730.000)
\blacken\path(1132.000,890.000)(1012.000,860.000)(1132.000,830.000)(1132.000,890.000)
\path(1012,860)(1407,860)
\blacken\path(1287.000,830.000)(1407.000,860.000)(1287.000,890.000)(1287.000,830.000)
\blacken\path(1122.000,140.000)(1002.000,110.000)(1122.000,80.000)(1122.000,140.000)
\path(1002,110)(1437,110)
\blacken\path(1317.000,80.000)(1437.000,110.000)(1317.000,140.000)(1317.000,80.000)
\blacken\path(3648.042,2300.265)(3617.000,2420.000)(3588.045,2299.744)(3648.042,2300.265)
\path(3617,2420)(3622,1845)
\blacken\path(3590.958,1964.735)(3622.000,1845.000)(3650.955,1965.256)(3590.958,1964.735)
\blacken\path(2585.042,2285.265)(2554.000,2405.000)(2525.045,2284.744)(2585.042,2285.265)
\path(2554,2405)(2559,1830)
\blacken\path(2527.958,1949.735)(2559.000,1830.000)(2587.955,1950.256)(2527.958,1949.735)
\blacken\path(2588.459,2169.170)(2562.000,2290.000)(2528.485,2170.934)(2588.459,2169.170)
\path(2562,2290)(2552,1950)
\blacken\path(2525.541,2070.830)(2552.000,1950.000)(2585.515,2069.066)(2525.541,2070.830)
\blacken\path(3647.000,2210.000)(3617.000,2330.000)(3587.000,2210.000)(3647.000,2210.000)
\path(3617,2330)(3617,1945)
\blacken\path(3587.000,2065.000)(3617.000,1945.000)(3647.000,2065.000)(3587.000,2065.000)
\blacken\path(4707.000,2290.000)(4677.000,2410.000)(4647.000,2290.000)(4707.000,2290.000)
\path(4677,2410)(4677,1860)
\blacken\path(4647.000,1980.000)(4677.000,1860.000)(4707.000,1980.000)(4647.000,1980.000)
\blacken\path(4707.000,2200.000)(4677.000,2320.000)(4647.000,2200.000)(4707.000,2200.000)
\path(4677,2320)(4677,1965)
\blacken\path(4647.000,2085.000)(4677.000,1965.000)(4707.000,2085.000)(4647.000,2085.000)
\blacken\path(4442.000,890.000)(4322.000,860.000)(4442.000,830.000)(4442.000,890.000)
\path(4322,860)(4887,860)
\blacken\path(4767.000,830.000)(4887.000,860.000)(4767.000,890.000)(4767.000,830.000)
\blacken\path(4562.422,888.258)(4442.000,860.000)(4561.553,828.264)(4562.422,888.258)
\path(4442,860)(4787,855)
\blacken\path(4666.578,826.742)(4787.000,855.000)(4667.447,886.736)(4666.578,826.742)
\path(537,2685)(537,2684)(536,2682)
	(534,2677)(531,2670)(527,2660)
	(522,2646)(516,2629)(508,2608)
	(499,2584)(488,2555)(477,2524)
	(464,2489)(450,2451)(436,2411)
	(421,2368)(405,2324)(389,2278)
	(373,2230)(357,2182)(341,2133)
	(325,2083)(309,2033)(294,1983)
	(279,1932)(265,1881)(251,1830)
	(238,1779)(225,1727)(214,1675)
	(203,1622)(193,1570)(184,1517)
	(177,1464)(171,1411)(166,1359)
	(163,1309)(162,1260)(163,1206)
	(167,1157)(174,1112)(182,1073)
	(192,1039)(204,1009)(217,983)
	(231,961)(247,943)(263,927)
	(280,914)(298,904)(316,896)
	(335,889)(354,884)(373,880)
	(392,878)(411,877)(429,876)
	(447,876)(463,877)(478,878)
	(492,879)(504,880)(514,881)
	(522,882)(528,883)(537,885)
\blacken\path(426.365,829.683)(537.000,885.000)(413.350,888.254)(426.365,829.683)
\path(537,2685)(537,2684)(536,2682)
	(534,2678)(531,2672)(527,2664)
	(521,2652)(514,2637)(506,2619)
	(496,2598)(484,2573)(471,2544)
	(457,2513)(441,2478)(424,2440)
	(406,2400)(387,2358)(368,2313)
	(348,2267)(328,2219)(307,2169)
	(287,2119)(266,2068)(246,2016)
	(226,1963)(207,1909)(187,1855)
	(169,1801)(151,1745)(134,1689)
	(117,1632)(101,1575)(86,1516)
	(73,1457)(60,1397)(48,1336)
	(38,1273)(29,1211)(22,1148)
	(17,1084)(13,1022)(12,960)
	(14,887)(19,818)(27,753)
	(37,694)(50,639)(65,589)
	(81,543)(100,502)(119,464)
	(140,430)(161,398)(184,369)
	(208,343)(232,319)(257,296)
	(282,276)(307,257)(333,240)
	(357,224)(382,209)(405,196)
	(427,184)(448,174)(466,165)
	(483,158)(497,151)(509,146)
	(519,142)(526,139)(537,135)
\blacken\path(413.972,147.815)(537.000,135.000)(434.477,204.203)(413.972,147.815)
\path(4662,1635)(4661,1634)(4660,1633)
	(4657,1631)(4653,1627)(4647,1621)
	(4640,1614)(4629,1605)(4617,1594)
	(4602,1580)(4585,1564)(4566,1547)
	(4544,1527)(4520,1506)(4494,1483)
	(4466,1459)(4437,1433)(4405,1406)
	(4373,1378)(4339,1349)(4303,1320)
	(4266,1289)(4229,1259)(4189,1228)
	(4149,1196)(4107,1164)(4064,1132)
	(4019,1099)(3973,1065)(3925,1031)
	(3874,997)(3822,961)(3766,925)
	(3709,889)(3649,851)(3586,813)
	(3520,775)(3453,736)(3383,698)
	(3312,660)(3237,621)(3163,584)
	(3090,549)(3019,516)(2950,485)
	(2884,457)(2821,430)(2759,406)
	(2700,383)(2643,361)(2588,342)
	(2535,323)(2483,306)(2433,290)
	(2384,275)(2336,261)(2290,248)
	(2245,236)(2201,224)(2159,213)
	(2118,203)(2079,194)(2042,185)
	(2008,177)(1975,169)(1946,163)
	(1919,157)(1896,152)(1875,148)
	(1858,144)(1844,141)(1833,139)
	(1824,137)(1812,135)
\blacken\path(1925.435,184.320)(1812.000,135.000)(1935.299,125.136)(1925.435,184.320)
\path(4057,870)(3212,615)
\path(3602,1665)(3603,1664)(3604,1661)
	(3606,1656)(3609,1648)(3613,1637)
	(3618,1623)(3623,1606)(3629,1587)
	(3635,1566)(3641,1543)(3646,1519)
	(3651,1493)(3654,1465)(3656,1436)
	(3657,1405)(3655,1371)(3652,1335)
	(3645,1296)(3636,1253)(3623,1208)
	(3607,1160)(3591,1119)(3573,1079)
	(3554,1041)(3534,1005)(3514,970)
	(3493,938)(3472,907)(3451,878)
	(3430,850)(3409,823)(3387,797)
	(3366,772)(3345,748)(3324,726)
	(3305,705)(3286,685)(3269,667)
	(3253,651)(3240,638)(3229,627)
	(3220,618)(3214,612)(3210,608)
	(3208,606)(3207,605)
\path(2567,1670)(2568,1669)(2571,1666)
	(2576,1662)(2584,1655)(2595,1646)
	(2609,1633)(2625,1618)(2645,1600)
	(2666,1580)(2689,1558)(2714,1534)
	(2739,1508)(2765,1481)(2790,1453)
	(2816,1423)(2841,1393)(2865,1361)
	(2889,1327)(2911,1292)(2933,1254)
	(2954,1215)(2973,1173)(2990,1129)
	(3005,1082)(3017,1035)(3025,988)
	(3030,942)(3032,899)(3030,858)
	(3026,820)(3020,784)(3012,750)
	(3003,718)(2992,688)(2979,660)
	(2966,633)(2952,606)(2938,582)
	(2923,558)(2909,536)(2895,516)
	(2881,498)(2869,482)(2859,468)
	(2850,457)(2843,448)(2838,442)
	(2835,438)(2833,436)(2832,435)
\path(1827,870)(1828,870)(1830,870)
	(1834,871)(1840,872)(1850,873)
	(1862,874)(1878,876)(1897,879)
	(1920,882)(1947,885)(1978,889)
	(2012,894)(2050,899)(2091,905)
	(2135,911)(2182,918)(2231,926)
	(2283,934)(2336,942)(2391,951)
	(2447,960)(2504,970)(2563,981)
	(2622,992)(2681,1004)(2742,1016)
	(2802,1029)(2864,1042)(2925,1057)
	(2988,1072)(3051,1088)(3114,1105)
	(3178,1124)(3243,1143)(3308,1164)
	(3375,1186)(3441,1209)(3509,1234)
	(3576,1261)(3644,1288)(3711,1318)
	(3777,1348)(3842,1380)(3917,1420)
	(3988,1460)(4054,1500)(4114,1540)
	(4169,1579)(4219,1618)(4264,1655)
	(4304,1692)(4341,1728)(4373,1763)
	(4403,1797)(4429,1831)(4453,1864)
	(4474,1897)(4493,1929)(4510,1960)
	(4526,1991)(4540,2022)(4552,2051)
	(4563,2080)(4572,2107)(4580,2133)
	(4588,2157)(4594,2180)(4599,2200)
	(4604,2218)(4607,2234)(4610,2248)
	(4612,2259)(4614,2268)(4615,2275)(4617,2285)
\blacken\path(4622.883,2161.447)(4617.000,2285.000)(4564.049,2173.214)(4622.883,2161.447)
\path(1807,900)(1808,900)(1810,900)
	(1814,901)(1820,901)(1829,902)
	(1842,904)(1857,905)(1877,907)
	(1900,910)(1926,913)(1957,917)
	(1991,921)(2028,925)(2069,930)
	(2113,936)(2159,942)(2208,948)
	(2259,956)(2312,963)(2367,971)
	(2423,980)(2480,989)(2538,999)
	(2597,1010)(2656,1021)(2716,1033)
	(2777,1045)(2838,1059)(2900,1073)
	(2962,1088)(3025,1104)(3088,1121)
	(3153,1140)(3218,1160)(3283,1181)
	(3350,1204)(3417,1229)(3485,1255)
	(3553,1282)(3621,1312)(3689,1343)
	(3756,1376)(3822,1410)(3895,1451)
	(3964,1492)(4028,1534)(4087,1575)
	(4142,1616)(4192,1656)(4238,1696)
	(4279,1735)(4317,1773)(4351,1811)
	(4383,1847)(4411,1884)(4437,1920)
	(4460,1955)(4481,1990)(4501,2024)
	(4518,2058)(4534,2091)(4549,2124)
	(4562,2156)(4574,2187)(4585,2216)
	(4594,2245)(4603,2271)(4610,2296)
	(4617,2319)(4622,2340)(4627,2359)
	(4631,2375)(4634,2388)(4637,2400)
	(4639,2408)(4640,2415)(4642,2425)
\blacken\path(4647.883,2301.447)(4642.000,2425.000)(4589.049,2313.214)(4647.883,2301.447)
\put(2562,2460){\makebox(0,0)[b]{\smash{{\SetFigFont{10}{12.0}{\rmdefault}{\mddefault}{\updefault}$\neg b$}}}}
\put(3612,2460){\makebox(0,0)[b]{\smash{{\SetFigFont{10}{12.0}{\rmdefault}{\mddefault}{\updefault}$\neg c$}}}}
\put(4662,2460){\makebox(0,0)[b]{\smash{{\SetFigFont{10}{12.0}{\rmdefault}{\mddefault}{\updefault}$\neg e$}}}}
\put(4662,1710){\makebox(0,0)[b]{\smash{{\SetFigFont{10}{12.0}{\rmdefault}{\mddefault}{\updefault}$e$}}}}
\put(3612,1710){\makebox(0,0)[b]{\smash{{\SetFigFont{10}{12.0}{\rmdefault}{\mddefault}{\updefault}$c$}}}}
\put(2562,1710){\makebox(0,0)[b]{\smash{{\SetFigFont{10}{12.0}{\rmdefault}{\mddefault}{\updefault}$b$}}}}
\put(537,2760){\makebox(0,0)[b]{\smash{{\SetFigFont{10}{12.0}{\rmdefault}{\mddefault}{\updefault}$\top$}}}}
\put(1662,1710){\makebox(0,0)[b]{\smash{{\SetFigFont{10}{12.0}{\rmdefault}{\mddefault}{\updefault}$a$}}}}
\put(1662,810){\makebox(0,0)[b]{\smash{{\SetFigFont{10}{12.0}{\rmdefault}{\mddefault}{\updefault}$d$}}}}
\put(4212,810){\makebox(0,0)[b]{\smash{{\SetFigFont{10}{12.0}{\rmdefault}{\mddefault}{\updefault}$f$}}}}
\put(5112,810){\makebox(0,0)[b]{\smash{{\SetFigFont{10}{12.0}{\rmdefault}{\mddefault}{\updefault}$\neg f$}}}}
\put(1662,60){\makebox(0,0)[b]{\smash{{\SetFigFont{10}{12.0}{\rmdefault}{\mddefault}{\updefault}$g$}}}}
\put(697,1725){\makebox(0,0)[b]{\smash{{\SetFigFont{10}{12.0}{\rmdefault}{\mddefault}{\updefault}$\neg a $}}}}
\put(757,810){\makebox(0,0)[b]{\smash{{\SetFigFont{10}{12.0}{\rmdefault}{\mddefault}{\updefault}$\neg d$}}}}
\path(537,2685)(837,1860)
\blacken\path(767.797,1962.523)(837.000,1860.000)(824.185,1983.028)(767.797,1962.523)
\end{picture}
}
\caption{}\label{509-EF2}
\end{figure}

We need two notions of attack, strict attack $x\tO y$ (same as $a\to \neg y$ and defeasible attack $x \To\!\!\!\!\To y$ (same as $x \To\!\!\!\!\To \neg y$). We have $\neg\neg x=x$ and $\neg x \LtO x$.

We need joint attacks as shown in figure \ref{509-EF2}. We also need to use $\top$, $\top$ is {\em truth}, helping us describe the strict facts. 

Note that in view of Remark \ref{509-BR101}, Figure \ref{509-EF2} is equivalent to Figure \ref{509-EF2a}. This figure contains node to node attacks only.

\begin{figure}
\centering
\setlength{\unitlength}{0.00083333in}
\begingroup\makeatletter\ifx\SetFigFont\undefined%
\gdef\SetFigFont#1#2#3#4#5{%
  \reset@font\fontsize{#1}{#2pt}%
  \fontfamily{#3}\fontseries{#4}\fontshape{#5}%
  \selectfont}%
\fi\endgroup%
{\renewcommand{\dashlinestretch}{30}
\begin{picture}(5321,3222)(0,-10)
\put(956,1110){\makebox(0,0)[b]{\smash{{\SetFigFont{10}{12.0}{\rmdefault}{\mddefault}{\updefault}$\neg d$}}}}
\blacken\path(2576.000,165.000)(2456.000,135.000)(2576.000,105.000)(2576.000,165.000)
\path(2456,135)(2981,135)
\blacken\path(2861.000,105.000)(2981.000,135.000)(2861.000,165.000)(2861.000,105.000)
\blacken\path(1676.000,765.000)(1556.000,735.000)(1676.000,705.000)(1676.000,765.000)
\path(1556,735)(2306,735)
\blacken\path(2186.000,705.000)(2306.000,735.000)(2186.000,765.000)(2186.000,705.000)
\blacken\path(1826.000,765.000)(1706.000,735.000)(1826.000,705.000)(1826.000,765.000)
\path(1706,735)(2156,735)
\blacken\path(2036.000,705.000)(2156.000,735.000)(2036.000,765.000)(2036.000,705.000)
\blacken\path(1376.000,1215.000)(1256.000,1185.000)(1376.000,1155.000)(1376.000,1215.000)
\path(1256,1185)(2006,1185)
\blacken\path(1886.000,1155.000)(2006.000,1185.000)(1886.000,1215.000)(1886.000,1155.000)
\blacken\path(1451.000,1215.000)(1331.000,1185.000)(1451.000,1155.000)(1451.000,1215.000)
\path(1331,1185)(1931,1185)
\blacken\path(1811.000,1155.000)(1931.000,1185.000)(1811.000,1215.000)(1811.000,1155.000)
\blacken\path(4526.000,1215.000)(4406.000,1185.000)(4526.000,1155.000)(4526.000,1215.000)
\path(4406,1185)(5156,1185)
\blacken\path(5036.000,1155.000)(5156.000,1185.000)(5036.000,1215.000)(5036.000,1155.000)
\blacken\path(4676.000,1215.000)(4556.000,1185.000)(4676.000,1155.000)(4676.000,1215.000)
\path(4556,1185)(5081,1185)
\blacken\path(4961.000,1155.000)(5081.000,1185.000)(4961.000,1215.000)(4961.000,1155.000)
\blacken\path(4736.000,2640.000)(4706.000,2760.000)(4676.000,2640.000)(4736.000,2640.000)
\path(4706,2760)(4706,2010)
\blacken\path(4676.000,2130.000)(4706.000,2010.000)(4736.000,2130.000)(4676.000,2130.000)
\blacken\path(4736.000,2565.000)(4706.000,2685.000)(4676.000,2565.000)(4736.000,2565.000)
\path(4706,2685)(4706,2160)
\blacken\path(4676.000,2280.000)(4706.000,2160.000)(4736.000,2280.000)(4676.000,2280.000)
\blacken\path(3836.000,2640.000)(3806.000,2760.000)(3776.000,2640.000)(3836.000,2640.000)
\path(3806,2760)(3806,2010)
\blacken\path(3776.000,2130.000)(3806.000,2010.000)(3836.000,2130.000)(3776.000,2130.000)
\blacken\path(3836.000,2565.000)(3806.000,2685.000)(3776.000,2565.000)(3836.000,2565.000)
\path(3806,2685)(3806,2160)
\blacken\path(3776.000,2280.000)(3806.000,2160.000)(3836.000,2280.000)(3776.000,2280.000)
\blacken\path(2786.000,2640.000)(2756.000,2760.000)(2726.000,2640.000)(2786.000,2640.000)
\path(2756,2760)(2756,2010)
\blacken\path(2726.000,2130.000)(2756.000,2010.000)(2786.000,2130.000)(2726.000,2130.000)
\blacken\path(2786.000,2565.000)(2756.000,2685.000)(2726.000,2565.000)(2786.000,2565.000)
\path(2756,2685)(2756,2085)
\blacken\path(2726.000,2205.000)(2756.000,2085.000)(2786.000,2205.000)(2726.000,2205.000)
\blacken\path(1526.000,2565.000)(1406.000,2535.000)(1526.000,2505.000)(1526.000,2565.000)
\path(1406,2535)(2006,2535)
\blacken\path(1886.000,2505.000)(2006.000,2535.000)(1886.000,2565.000)(1886.000,2505.000)
\blacken\path(1601.000,2565.000)(1481.000,2535.000)(1601.000,2505.000)(1601.000,2565.000)
\path(1481,2535)(1931,2535)
\blacken\path(1811.000,2505.000)(1931.000,2535.000)(1811.000,2565.000)(1811.000,2505.000)
\path(806,2985)(1181,2610)
\blacken\path(1074.934,2673.640)(1181.000,2610.000)(1117.360,2716.066)(1074.934,2673.640)
\path(806,2985)(956,1260)
\blacken\path(915.717,1376.950)(956.000,1260.000)(975.492,1382.148)(915.717,1376.950)
\path(1001,2790)(1081,2710)
\blacken\path(974.934,2773.640)(1081.000,2710.000)(1017.360,2816.066)(974.934,2773.640)
\path(941,1510)(941,1415)
\blacken\path(911.000,1535.000)(941.000,1415.000)(971.000,1535.000)(911.000,1535.000)
\path(1446,165)(1551,150)
\blacken\path(1427.963,137.272)(1551.000,150.000)(1436.449,196.669)(1427.963,137.272)
\path(2356,1050)(2396,965)
\blacken\path(2317.760,1060.804)(2396.000,965.000)(2372.049,1086.352)(2317.760,1060.804)
\path(2576,1070)(2531,975)
\blacken\path(2555.258,1096.291)(2531.000,975.000)(2609.482,1070.606)(2555.258,1096.291)
\path(2766,930)(2631,905)
\blacken\path(2743.531,956.349)(2631.000,905.000)(2754.457,897.352)(2743.531,956.349)
\path(2776,2040)(3751,2765)
\blacken\path(3672.606,2669.322)(3751.000,2765.000)(3636.803,2717.470)(3672.606,2669.322)
\path(3536,2605)(3626,2675)
\blacken\path(3549.696,2577.647)(3626.000,2675.000)(3512.860,2625.008)(3549.696,2577.647)
\path(4731,1800)(5266,1250)
\blacken\path(5160.824,1315.100)(5266.000,1250.000)(5203.833,1356.936)(5160.824,1315.100)
\path(5111,1410)(5176,1360)
\blacken\path(5062.594,1409.387)(5176.000,1360.000)(5099.176,1456.944)(5062.594,1409.387)
\path(4516,2435)(4596,2575)
\blacken\path(4562.511,2455.927)(4596.000,2575.000)(4510.416,2485.695)(4562.511,2455.927)
\path(4711,1780)(5036,1435)
\path(2781,2085)(3496,2610)
\path(2451,625)(3196,195)
\blacken\path(3077.073,229.004)(3196.000,195.000)(3107.066,280.969)(3077.073,229.004)
\path(3001,310)(3056,270)
\blacken\path(2941.307,316.319)(3056.000,270.000)(2976.597,364.843)(2941.307,316.319)
\path(806,2985)(805,2984)(804,2982)
	(801,2979)(797,2973)(791,2965)
	(784,2954)(773,2940)(761,2923)
	(746,2903)(729,2879)(710,2851)
	(689,2820)(665,2786)(639,2749)
	(612,2709)(583,2666)(553,2621)
	(522,2573)(491,2524)(458,2473)
	(426,2420)(393,2366)(361,2311)
	(329,2256)(298,2199)(268,2142)
	(238,2085)(210,2026)(183,1968)
	(157,1909)(133,1849)(110,1789)
	(90,1728)(71,1667)(55,1605)
	(40,1542)(29,1478)(20,1414)
	(14,1349)(12,1284)(12,1218)
	(17,1153)(26,1087)(39,1023)
	(56,960)(79,897)(106,836)
	(137,779)(171,725)(209,675)
	(249,629)(290,586)(334,546)
	(380,509)(427,476)(476,444)
	(526,416)(577,389)(629,365)
	(682,343)(735,322)(790,303)
	(845,286)(901,270)(957,255)
	(1013,241)(1069,229)(1124,217)
	(1178,207)(1232,197)(1284,188)
	(1334,180)(1382,173)(1428,167)
	(1470,161)(1510,156)(1546,152)
	(1578,148)(1607,145)(1631,142)
	(1652,140)(1668,138)(1682,137)
	(1692,136)(1706,135)
\blacken\path(1584.168,113.626)(1706.000,135.000)(1588.442,173.473)(1584.168,113.626)
\path(2751,2760)(2750,2759)(2749,2757)
	(2746,2753)(2741,2748)(2734,2739)
	(2725,2729)(2714,2715)(2701,2699)
	(2686,2680)(2669,2658)(2651,2634)
	(2631,2608)(2610,2580)(2588,2550)
	(2566,2519)(2544,2486)(2522,2452)
	(2499,2417)(2477,2380)(2456,2342)
	(2434,2303)(2414,2262)(2393,2219)
	(2374,2174)(2355,2127)(2337,2078)
	(2319,2026)(2304,1971)(2289,1914)
	(2276,1855)(2266,1795)(2258,1731)
	(2252,1668)(2249,1608)(2249,1550)
	(2251,1495)(2255,1444)(2260,1394)
	(2267,1348)(2276,1303)(2285,1261)
	(2296,1220)(2307,1180)(2319,1143)
	(2332,1106)(2345,1071)(2359,1038)
	(2372,1006)(2385,977)(2397,949)
	(2409,924)(2419,902)(2429,883)
	(2437,867)(2443,855)(2448,845)(2456,830)
\blacken\path(2373.059,921.765)(2456.000,830.000)(2426.000,950.000)(2373.059,921.765)
\path(3811,2745)(3810,2744)(3809,2742)
	(3806,2739)(3802,2734)(3796,2726)
	(3788,2716)(3778,2703)(3766,2687)
	(3751,2668)(3735,2646)(3716,2621)
	(3695,2594)(3672,2564)(3647,2531)
	(3621,2497)(3593,2461)(3564,2424)
	(3535,2385)(3504,2345)(3473,2304)
	(3441,2262)(3409,2220)(3376,2177)
	(3344,2134)(3310,2090)(3277,2046)
	(3244,2001)(3210,1956)(3175,1910)
	(3141,1863)(3106,1816)(3071,1769)
	(3035,1720)(2999,1671)(2963,1622)
	(2927,1572)(2891,1522)(2855,1473)
	(2821,1425)(2767,1349)(2719,1280)
	(2677,1219)(2641,1166)(2610,1119)
	(2584,1079)(2562,1044)(2543,1014)
	(2528,987)(2515,965)(2504,945)
	(2495,927)(2488,912)(2482,899)
	(2477,888)(2474,879)(2471,871)
	(2469,866)(2466,855)
\blacken\path(2468.631,978.665)(2466.000,855.000)(2526.517,962.878)(2468.631,978.665)
\path(4676,2745)(4676,2744)(4676,2742)
	(4676,2739)(4676,2734)(4676,2726)
	(4676,2716)(4675,2703)(4675,2687)
	(4674,2669)(4673,2648)(4671,2624)
	(4669,2598)(4666,2570)(4663,2540)
	(4658,2508)(4653,2475)(4647,2441)
	(4640,2405)(4631,2369)(4621,2332)
	(4609,2294)(4596,2255)(4580,2215)
	(4563,2174)(4543,2132)(4520,2089)
	(4494,2045)(4465,1999)(4432,1952)
	(4395,1904)(4354,1854)(4308,1803)
	(4259,1751)(4204,1698)(4146,1645)
	(4093,1600)(4039,1557)(3984,1515)
	(3928,1475)(3871,1437)(3815,1400)
	(3760,1366)(3704,1333)(3649,1302)
	(3595,1272)(3541,1244)(3487,1218)
	(3434,1192)(3381,1168)(3329,1145)
	(3277,1122)(3225,1101)(3174,1080)
	(3123,1061)(3073,1042)(3023,1023)
	(2975,1006)(2928,989)(2882,973)
	(2838,958)(2795,944)(2755,931)
	(2718,919)(2683,908)(2652,898)
	(2624,889)(2599,881)(2577,875)
	(2559,869)(2544,865)(2532,861)
	(2524,859)(2511,855)
\blacken\path(2616.871,918.964)(2511.000,855.000)(2634.516,861.617)(2616.871,918.964)
\path(2216,1150)(2217,1150)(2219,1150)
	(2224,1150)(2231,1150)(2240,1151)
	(2253,1151)(2270,1152)(2290,1153)
	(2314,1154)(2341,1156)(2372,1158)
	(2406,1160)(2443,1163)(2483,1166)
	(2525,1170)(2569,1174)(2615,1179)
	(2663,1185)(2711,1191)(2761,1198)
	(2812,1207)(2864,1216)(2916,1226)
	(2970,1237)(3024,1250)(3079,1264)
	(3135,1280)(3193,1298)(3251,1318)
	(3311,1340)(3373,1364)(3435,1391)
	(3499,1420)(3564,1452)(3630,1487)
	(3696,1525)(3761,1565)(3824,1607)
	(3885,1650)(3943,1694)(3997,1739)
	(4048,1783)(4096,1827)(4140,1870)
	(4182,1913)(4220,1956)(4257,1998)
	(4290,2040)(4322,2081)(4351,2122)
	(4379,2162)(4405,2202)(4430,2242)
	(4454,2282)(4476,2320)(4496,2359)
	(4516,2396)(4535,2432)(4552,2467)
	(4568,2501)(4583,2532)(4596,2562)
	(4608,2589)(4619,2614)(4629,2636)
	(4637,2655)(4643,2671)(4649,2685)
	(4653,2695)(4656,2703)(4661,2715)
\blacken\path(4642.538,2592.692)(4661.000,2715.000)(4587.154,2615.769)(4642.538,2592.692)
\path(2191,1155)(2192,1155)(2194,1155)
	(2199,1155)(2206,1155)(2215,1156)
	(2228,1156)(2245,1157)(2265,1158)
	(2289,1159)(2316,1161)(2347,1163)
	(2381,1165)(2418,1168)(2458,1171)
	(2500,1175)(2544,1179)(2590,1184)
	(2638,1190)(2686,1196)(2736,1203)
	(2787,1212)(2839,1221)(2891,1231)
	(2945,1242)(2999,1255)(3054,1269)
	(3110,1285)(3168,1303)(3226,1323)
	(3286,1345)(3348,1369)(3410,1396)
	(3474,1425)(3539,1457)(3605,1492)
	(3671,1530)(3736,1570)(3799,1612)
	(3860,1655)(3918,1699)(3972,1744)
	(4023,1788)(4071,1832)(4115,1875)
	(4157,1918)(4195,1961)(4232,2003)
	(4265,2045)(4297,2086)(4326,2127)
	(4354,2167)(4380,2207)(4405,2247)
	(4429,2287)(4451,2325)(4471,2364)
	(4491,2401)(4510,2437)(4527,2472)
	(4543,2506)(4558,2537)(4571,2567)
	(4583,2594)(4594,2619)(4604,2641)
	(4612,2660)(4618,2676)(4624,2690)
	(4628,2700)(4631,2708)(4636,2720)
\blacken\path(4617.538,2597.692)(4636.000,2720.000)(4562.154,2620.769)(4617.538,2597.692)
\put(806,3060){\makebox(0,0)[b]{\smash{{\SetFigFont{10}{12.0}{\rmdefault}{\mddefault}{\updefault}$\top$}}}}
\put(1181,2460){\makebox(0,0)[b]{\smash{{\SetFigFont{10}{12.0}{\rmdefault}{\mddefault}{\updefault}$\neg a$}}}}
\put(2081,2460){\makebox(0,0)[b]{\smash{{\SetFigFont{10}{12.0}{\rmdefault}{\mddefault}{\updefault}$a$}}}}
\put(2756,2835){\makebox(0,0)[b]{\smash{{\SetFigFont{10}{12.0}{\rmdefault}{\mddefault}{\updefault}$\neg b$}}}}
\put(3806,2835){\makebox(0,0)[b]{\smash{{\SetFigFont{10}{12.0}{\rmdefault}{\mddefault}{\updefault}$\neg c$}}}}
\put(4706,2835){\makebox(0,0)[b]{\smash{{\SetFigFont{10}{12.0}{\rmdefault}{\mddefault}{\updefault}$\neg e$}}}}
\put(4706,1860){\makebox(0,0)[b]{\smash{{\SetFigFont{10}{12.0}{\rmdefault}{\mddefault}{\updefault}$e$}}}}
\put(3806,1860){\makebox(0,0)[b]{\smash{{\SetFigFont{10}{12.0}{\rmdefault}{\mddefault}{\updefault}$c$}}}}
\put(2756,1860){\makebox(0,0)[b]{\smash{{\SetFigFont{10}{12.0}{\rmdefault}{\mddefault}{\updefault}$b$}}}}
\put(2081,1110){\makebox(0,0)[b]{\smash{{\SetFigFont{10}{12.0}{\rmdefault}{\mddefault}{\updefault}$d$}}}}
\put(5306,1110){\makebox(0,0)[b]{\smash{{\SetFigFont{10}{12.0}{\rmdefault}{\mddefault}{\updefault}$\neg f$}}}}
\put(4256,1110){\makebox(0,0)[b]{\smash{{\SetFigFont{10}{12.0}{\rmdefault}{\mddefault}{\updefault}$f$}}}}
\put(2456,660){\makebox(0,0)[b]{\smash{{\SetFigFont{10}{12.0}{\rmdefault}{\mddefault}{\updefault}$z$}}}}
\put(1256,660){\makebox(0,0)[b]{\smash{{\SetFigFont{10}{12.0}{\rmdefault}{\mddefault}{\updefault}$\neg z$}}}}
\put(2006,60){\makebox(0,0)[b]{\smash{{\SetFigFont{10}{12.0}{\rmdefault}{\mddefault}{\updefault}$\neg g$}}}}
\put(3206,60){\makebox(0,0)[b]{\smash{{\SetFigFont{10}{12.0}{\rmdefault}{\mddefault}{\updefault}$g$}}}}
\blacken\path(2426.000,165.000)(2306.000,135.000)(2426.000,105.000)(2426.000,165.000)
\path(2306,135)(3056,135)
\blacken\path(2936.000,105.000)(3056.000,135.000)(2936.000,165.000)(2936.000,105.000)
\end{picture}
}

$b\wedge c \wedge e\wedge f\tO g$ is replaced by $\neg b\tO z$ and $\neg c\tO z$ and $\neg e \tO z$ and $\neg f\tO z$ and $z\tO g$.
\caption{}\label{509-EF2a}
\end{figure}

We need to say how we compute extensions for the argumentation network of Figure \ref{509-EF2} arising from the defeasible theory $\Delta$.  Being a defeasible theory we should be able to resolve the conflict of the atom $g$ both supported and attacked by its neighbours. Intuitively $g$ has a direct strict proof from $\top$, but $\neg g$ is derived defeasibly using 3 separate defeasible rules, two of which are chained. So $\neg g$ has lower preference ``value'' than $g$ and so $g$ is accepted. We need to define the priority values on the atoms $x, \neg x, y, \neg y$, etc., and use these in calculating extensions.  Our network has two attack relations and so the priority values and the extensions must be defined geometrically in terms of the graphs.  This is done in Appendix E2.
\end{example}

\begin{example}\label{509-EE1a}
Let us look at another example from \cite{509-16}.  This is example 4 from their paper.  The vocabulary is:
\begin{quote}
{\em wr}: John wears a ring\\
{\em m}: John is married\\
{\em hw}: John has a wife\\
{\em b}: John is a bachelor\\
{\em go}: John often goes out until late with friends.
\end{quote}

The arguments are obtained from a database with strict rules $\CS$ and defeasible rules $\CD$.
\[\begin{array}{l}
\CS =\{\top\to wr, \top \to go, b\to \neg hw, m\to hw\}\\
\CD =\{wr\To m, go \To b\}.
\end{array}
\]
Caminada and Amgoud form the following arguments from the data $(\CS, \CD)$.
\[\begin{array}{l}
A_1: [\top\to wr]\\
A_2: [\top \to go]\\
A_3: [A_1\To m]\\
A_4: [A_2\To b]\\
A_5: [A_3\to hw]\\
A_6: [A_4\to \neg hw]
\end{array}
\]

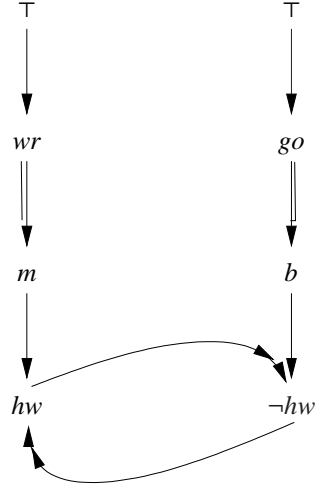
\begin{figure}
\centering
\setlength{\unitlength}{0.00083333in}
\begingroup\makeatletter\ifx\SetFigFont\undefined%
\gdef\SetFigFont#1#2#3#4#5{%
  \reset@font\fontsize{#1}{#2pt}%
  \fontfamily{#3}\fontseries{#4}\fontshape{#5}%
  \selectfont}%
\fi\endgroup%
{\renewcommand{\dashlinestretch}{30}
\begin{picture}(1739,3079)(0,-10)
\put(1697,442){\makebox(0,0)[b]{\smash{{\SetFigFont{10}{12.0}{\rmdefault}{\mddefault}{\updefault}$\neg hw$}}}}
\path(1697,2842)(1697,2317)
\blacken\path(1667.000,2437.000)(1697.000,2317.000)(1727.000,2437.000)(1667.000,2437.000)
\path(47,2017)(47,1492)
\blacken\path(17.000,1612.000)(47.000,1492.000)(77.000,1612.000)(17.000,1612.000)
\path(1697,2017)(1697,1492)
\blacken\path(1667.000,1612.000)(1697.000,1492.000)(1727.000,1612.000)(1667.000,1612.000)
\path(47,1192)(47,667)
\blacken\path(17.000,787.000)(47.000,667.000)(77.000,787.000)(17.000,787.000)
\path(1697,1192)(1697,667)
\blacken\path(1667.000,787.000)(1697.000,667.000)(1727.000,787.000)(1667.000,787.000)
\path(12,2017)(17,1652)
\path(1717,2012)(1722,1647)(1692,1647)
\path(1497,832)(1582,782)
\blacken\path(1463.357,816.984)(1582.000,782.000)(1493.778,868.700)(1463.357,816.984)
\path(117,137)(82,222)
\blacken\path(155.430,122.461)(82.000,222.000)(99.950,99.616)(155.430,122.461)
\path(77,612)(78,612)(81,614)
	(86,616)(94,619)(106,623)
	(121,629)(140,636)(163,645)
	(189,655)(218,666)(251,678)
	(286,690)(323,704)(362,718)
	(403,732)(445,746)(488,760)
	(531,774)(576,787)(621,800)
	(666,813)(712,825)(760,836)
	(807,847)(856,857)(906,866)
	(957,874)(1008,881)(1060,886)
	(1112,890)(1162,892)(1221,892)
	(1275,889)(1324,883)(1368,875)
	(1407,866)(1441,855)(1471,843)
	(1497,830)(1521,816)(1542,802)
	(1560,786)(1577,770)(1592,754)
	(1605,738)(1617,722)(1627,707)
	(1635,693)(1643,681)(1649,670)
	(1653,661)(1657,653)(1662,642)
\blacken\path(1585.033,738.830)(1662.000,642.000)(1639.655,763.658)(1585.033,738.830)
\path(1712,387)(1711,386)(1708,385)
	(1703,383)(1694,379)(1683,374)
	(1667,367)(1648,359)(1625,349)
	(1597,337)(1567,324)(1532,310)
	(1495,294)(1455,278)(1414,261)
	(1370,244)(1325,226)(1278,208)
	(1231,190)(1183,173)(1135,155)
	(1086,139)(1037,123)(987,107)
	(937,92)(887,78)(836,65)
	(784,53)(732,42)(679,33)
	(626,25)(574,18)(522,14)
	(472,12)(419,13)(370,16)
	(326,22)(287,30)(253,40)
	(223,52)(196,65)(174,79)
	(154,95)(137,111)(123,129)
	(110,147)(100,165)(91,184)
	(83,203)(77,222)(72,241)
	(68,259)(65,277)(63,293)
	(61,307)(59,320)(58,331)
	(58,340)(57,346)(57,357)
\blacken\path(87.000,237.000)(57.000,357.000)(27.000,237.000)(87.000,237.000)
\put(47,2917){\makebox(0,0)[b]{\smash{{\SetFigFont{10}{12.0}{\rmdefault}{\mddefault}{\updefault}$\top$}}}}
\put(1697,2917){\makebox(0,0)[b]{\smash{{\SetFigFont{10}{12.0}{\rmdefault}{\mddefault}{\updefault}$\top$}}}}
\put(47,2092){\makebox(0,0)[b]{\smash{{\SetFigFont{10}{12.0}{\rmdefault}{\mddefault}{\updefault}$wr$}}}}
\put(1697,2092){\makebox(0,0)[b]{\smash{{\SetFigFont{10}{12.0}{\rmdefault}{\mddefault}{\updefault}$go$}}}}
\put(47,1267){\makebox(0,0)[b]{\smash{{\SetFigFont{10}{12.0}{\rmdefault}{\mddefault}{\updefault}$m$}}}}
\put(1697,1267){\makebox(0,0)[b]{\smash{{\SetFigFont{10}{12.0}{\rmdefault}{\mddefault}{\updefault}$b$}}}}
\put(47,442){\makebox(0,0)[b]{\smash{{\SetFigFont{10}{12.0}{\rmdefault}{\mddefault}{\updefault}$hw$}}}}
\path(47,2842)(47,2317)
\blacken\path(17.000,2437.000)(47.000,2317.000)(77.000,2437.000)(17.000,2437.000)
\end{picture}
}
\caption{Representation of Caminada and Amgoud example}\label{509-EF1b}
\end{figure}

We form (in our notation) the following network of Figure \ref{509-EF1b}.

The correspondence in this case between the arguments of Caminada and Amgoud and the arguments in Figure \ref{509-EF1b} is to paths in this figure as follows:

\[\begin{array}{l}
A_1: (\top\to wr)\\
A_2: (\top \to go)\\
A_3: (\top\to wf\To m)\\
A_4: (\top \to go\To b)\\
A_5: (\top \to wr\To m \to hw)\\
A_6: (\top \to go \To b\to \neg hw)
\end{array}\]

In the above set of arguments, only $A_5$ and $A_6$ attack each other and so we get the extension $\{A_1, A_2, A_3, A_4\}$. The {\em output} of this extension are the argument heads, namely $\{wr, go, m, b\}$.

Caminada and Amgoud's approach corresponds to our considering the network of Figure \ref{509-EF1b}, with attack relation $\tO$, and we thus get the extension $\{wr, go, m, b\}$. Caminada and Amgoud proceed to close this extenson under the strict rules and they get inconsistency. They consider this a problem and offer to remedy the problem by systematically adding, with every strict rule $x\to y$, its contrapositive rule $\neg y \to \neg x$.  This allows them essentially to also have that $m$ attacks $b$ and $b$ attacks $m$, and so the extension becomes only $\{wr, go\}$, which is consistenly closed under the strict rules.

We translate Figure \ref{509-EF1b} into Figure \ref{509-EF1c},  and so have no problems with it..
\begin{figure}
\centering
\setlength{\unitlength}{0.00083333in}
\begingroup\makeatletter\ifx\SetFigFont\undefined%
\gdef\SetFigFont#1#2#3#4#5{%
  \reset@font\fontsize{#1}{#2pt}%
  \fontfamily{#3}\fontseries{#4}\fontshape{#5}%
  \selectfont}%
\fi\endgroup%
{\renewcommand{\dashlinestretch}{30}
\begin{picture}(4002,2547)(0,-10)
\put(3012,960){\makebox(0,0)[rb]{\smash{{\SetFigFont{10}{12.0}{\rmdefault}{\mddefault}{\updefault}$\neg b$}}}}
\path(1962,2310)(3012,1935)
\blacken\path(2888.901,1947.108)(3012.000,1935.000)(2909.081,2003.613)(2888.901,1947.108)
\path(912,1635)(162,1185)
\blacken\path(249.464,1272.464)(162.000,1185.000)(280.334,1221.015)(249.464,1272.464)
\path(1402,910)(1537,220)
\blacken\path(1484.517,332.007)(1537.000,220.000)(1543.400,343.527)(1484.517,332.007)
\path(3987,1635)(3087,1185)
\blacken\path(3180.915,1265.498)(3087.000,1185.000)(3207.748,1211.833)(3180.915,1265.498)
\path(3762,885)(2412,285)
\blacken\path(2509.473,361.151)(2412.000,285.000)(2533.842,306.322)(2509.473,361.151)
\path(382,1320)(297,1275)
\blacken\path(389.018,1357.660)(297.000,1275.000)(417.091,1304.633)(389.018,1357.660)
\path(2782,2015)(2862,1990)
\blacken\path(2738.514,1997.159)(2862.000,1990.000)(2756.411,2054.427)(2738.514,1997.159)
\path(3297,1295)(3222,1260)
\blacken\path(3318.055,1337.932)(3222.000,1260.000)(3343.429,1283.561)(3318.055,1337.932)
\path(2637,385)(2552,355)
\blacken\path(2655.174,423.228)(2552.000,355.000)(2675.143,366.649)(2655.174,423.228)
\path(447,1325)(927,1600)
\path(3382,1315)(3987,1600)
\blacken\path(357.000,1815.000)(237.000,1785.000)(357.000,1755.000)(357.000,1815.000)
\path(237,1785)(912,1785)
\blacken\path(792.000,1755.000)(912.000,1785.000)(792.000,1815.000)(792.000,1755.000)
\blacken\path(507.000,1815.000)(387.000,1785.000)(507.000,1755.000)(507.000,1815.000)
\path(387,1785)(762,1785)
\blacken\path(642.000,1755.000)(762.000,1785.000)(642.000,1815.000)(642.000,1755.000)
\blacken\path(807.000,1065.000)(687.000,1035.000)(807.000,1005.000)(807.000,1065.000)
\path(687,1035)(1062,1035)
\blacken\path(942.000,1005.000)(1062.000,1035.000)(942.000,1065.000)(942.000,1005.000)
\blacken\path(1932.000,165.000)(1812.000,135.000)(1932.000,105.000)(1932.000,165.000)
\path(1812,135)(2187,135)
\blacken\path(2067.000,105.000)(2187.000,135.000)(2067.000,165.000)(2067.000,105.000)
\blacken\path(1782.000,165.000)(1662.000,135.000)(1782.000,105.000)(1782.000,165.000)
\path(1662,135)(2337,135)
\blacken\path(2217.000,105.000)(2337.000,135.000)(2217.000,165.000)(2217.000,105.000)
\blacken\path(3207.000,1815.000)(3087.000,1785.000)(3207.000,1755.000)(3207.000,1815.000)
\path(3087,1785)(3762,1785)
\blacken\path(3642.000,1755.000)(3762.000,1785.000)(3642.000,1815.000)(3642.000,1755.000)
\blacken\path(657.000,1065.000)(537.000,1035.000)(657.000,1005.000)(657.000,1065.000)
\path(537,1035)(1212,1035)
\blacken\path(1092.000,1005.000)(1212.000,1035.000)(1092.000,1065.000)(1092.000,1005.000)
\blacken\path(3357.000,1815.000)(3237.000,1785.000)(3357.000,1755.000)(3357.000,1815.000)
\path(3237,1785)(3612,1785)
\blacken\path(3492.000,1755.000)(3612.000,1785.000)(3492.000,1815.000)(3492.000,1755.000)
\blacken\path(3207.000,1065.000)(3087.000,1035.000)(3207.000,1005.000)(3207.000,1065.000)
\path(3087,1035)(3762,1035)
\blacken\path(3642.000,1005.000)(3762.000,1035.000)(3642.000,1065.000)(3642.000,1005.000)
\blacken\path(3357.000,1065.000)(3237.000,1035.000)(3357.000,1005.000)(3357.000,1065.000)
\path(3237,1035)(3612,1035)
\blacken\path(3492.000,1005.000)(3612.000,1035.000)(3492.000,1065.000)(3492.000,1005.000)
\path(267,1920)(167,1895)
\blacken\path(276.141,1953.209)(167.000,1895.000)(290.693,1895.000)(276.141,1953.209)
\path(1492,475)(1507,375)
\blacken\path(1459.531,489.222)(1507.000,375.000)(1518.867,498.123)(1459.531,489.222)
\put(1962,2385){\makebox(0,0)[b]{\smash{{\SetFigFont{10}{12.0}{\rmdefault}{\mddefault}{\updefault}$\top$}}}}
\put(912,1710){\makebox(0,0)[lb]{\smash{{\SetFigFont{10}{12.0}{\rmdefault}{\mddefault}{\updefault}$wr$}}}}
\put(162,1710){\makebox(0,0)[rb]{\smash{{\SetFigFont{10}{12.0}{\rmdefault}{\mddefault}{\updefault}$\neg wr$}}}}
\put(162,960){\makebox(0,0)[lb]{\smash{{\SetFigFont{10}{12.0}{\rmdefault}{\mddefault}{\updefault}$\neg m$}}}}
\put(1287,960){\makebox(0,0)[lb]{\smash{{\SetFigFont{10}{12.0}{\rmdefault}{\mddefault}{\updefault}$m$}}}}
\put(1512,60){\makebox(0,0)[b]{\smash{{\SetFigFont{10}{12.0}{\rmdefault}{\mddefault}{\updefault}$\neg hw$}}}}
\put(2412,60){\makebox(0,0)[lb]{\smash{{\SetFigFont{10}{12.0}{\rmdefault}{\mddefault}{\updefault}$hw$}}}}
\put(3987,1710){\makebox(0,0)[b]{\smash{{\SetFigFont{10}{12.0}{\rmdefault}{\mddefault}{\updefault}$go$}}}}
\put(3087,1710){\makebox(0,0)[rb]{\smash{{\SetFigFont{10}{12.0}{\rmdefault}{\mddefault}{\updefault}$\neg go$}}}}
\put(3912,960){\makebox(0,0)[b]{\smash{{\SetFigFont{10}{12.0}{\rmdefault}{\mddefault}{\updefault}$b$}}}}
\path(1962,2310)(12,1860)
\blacken\path(122.181,1916.215)(12.000,1860.000)(135.673,1857.751)(122.181,1916.215)
\end{picture}
}
\caption{}\label{509-EF1c}
\end{figure}
\end{example}

\subsection{Defining complete extensions for two-state two-attacks abstract argumentation networks}

\begin{definition}\label{509-ED3}
A 2-state bipolar 2-attack network has the form $(S\cup S^\urcorner \cup \{\top\}, R_1, R_2)$ where $S$ is a set of atomic letters, $S^\urcorner =\{\neg x|x\in S\}$, $\top\not\in S$ and $R_1$ and $R_2$ are subsets of $(S\cup S\urcorner \cup \{\top\})^2$. We write $x\tO y$ for $(x, y)\in R_1$ and $x  \To\!\!\!\!\To y$ for $(x,y)\in R_2$.

We consider the elements of $S^\urcorner$ as negations of the elements of $S$. $\top$ is truth.

The following holds (compare with Definition \ref{509-D2}):
\begin{itemize}
\item $\neg \exists x (x\tO \top$ or $x\To\!\!\!\!\To \top)$
\item $\forall x( x\tO \neg x$ and $\neg x \tO x)$.
\end{itemize}
$R_1$ is called the strict attacks and $R_2$ is called the defeasible attacks.
\end{definition}

\begin{example}\label{509-EE4}
Consider the network of Figure \ref{509-EF2a}. This is a 2-state two-attack bipolar network. We need to define a process for finding extensions for it. The difficult part in the definition of such a notion is to deal with cases where the attacks ``$\tO$'' and $``\To\!\!\!\!\To$'' disagree.  By calling ``$\tO$'' strict and ``$\To\!\!\!\!\To$'' defeasible  we are giving ``$\tO$'' priority over ``$\To\!\!\!\!\To$''. But we may have cases where ``$\To\!\!\!\!\To$'' attacks both $x$ and $\neg x$. We might give priority to either $x$ or $\neg x$ depending on the geometrical location of $x, \neg x$ relative to their attacking ``$\To\!\!\!\!\To$'' ancestors.

Let us see how to calculate the ground extension by propagating the attacks from $\top$. We record a defeasible index $\BD(x)$ for every node $x$, by counting the maximal number of ``$\To$'' participating in the chain of attacks leading up to it.

The following is the progression, step by step.  1= ``in'', 0= ``out'', $\half$=``undecided''.

\begin{enumerate}
\item $\top =1$ \hfill $\BD(\top =0)$
\item $\neg g =0, g=1$,\hfill from 1\\
$\BD(g) =\BD(\neg g)=0$
\item $\neg d =0, d=1$, \hfill from 1\\
$\BD(d) =\BD(\neg d)=0$
\item $\neg a=0, a=1$, \hfill from 1\\
$\BD(a) =\BD(\neg a) =0$.
\item $\neg e=0, e=1$, \hfill from 3\\
$\BD(e)=\BD(\neg e)=1$
\item $\neg b=0, b=1$, \hfill from 4\\
$\BD(b) =\BD(\neg b) =1$
\item $\neg c =0 c=1$, \hfill from 6\\
$\BD(c) =\BD(\neg c) =2$
\item $\neg f=0. f=1$\hfill from 5\\
$\BD(f) =\BD(\neg f) =2$
\item $\neg g=1, g=0$\hfill from 5, 6, 7, 8\\
$\BD(g) =\BD(\neg g) =2$.
\end{enumerate}
We see that 2 contradicts 9.  Since the \BD\ index of 2 is lower than that of 9, it wins.

Therefore our ground extension for the network of Figure \ref{509-EF2} is 
\[
\{\top, g, d, a, e, b, c, f\}.
\]
The reader can see that by associating with a theory $\Delta$ the network of Figure \ref{509-EF2} we have none of the problem mentioned by the ASPIC group \cite{509-16}.

Furthermore, if one does not want to deal with joint attacks (see \cite{509-3} and see \cite{509-17} for arguments in favour of joint attacks), one can use additional auxiliary points and eliminate them as in Figure \ref{509-EF2a}.
\end{example}

\begin{example}\label{509-EE4a}
Let us calculate the ground extension of Figure \ref{509-EF1c} in steps:
\begin{enumerate}
\item $\top =1$. $\BD(\top)=0$
\item $\neg go =0, go =1, \BD(\neg go)=\BD(go)=0$, from 1.
\item $\neg wr =0, wr =1, \BD(\neg wr) =\BD(wr)=0$, from 1.
\item $\neg m =0, m=1, \BD(\neg m) =\BD(m) =1$, from 3.
\item $\neg b =0, b=1, \BD(\neg b) =\BD(b) =1$, from 2
\item $\neg hw =0$. $\BD(\neg hw) =1$, from 4.
\item $hw =0, \BD(hw) =1$, from 5.
\end{enumerate}
If we insist on the sum of the values of any $\{w, \neg w\}$ to be 1 we need to give $hw$ and $\neg hw$ values $\half$.
\end{example}

\begin{example}\label{509-EE5}
We want to discuss options for defining the index $\BD(x)$ for index $x$. Consider the network in Figure \ref{509-EF6}.

\begin{figure}
\centering
\setlength{\unitlength}{0.00083333in}
\begingroup\makeatletter\ifx\SetFigFont\undefined%
\gdef\SetFigFont#1#2#3#4#5{%
  \reset@font\fontsize{#1}{#2pt}%
  \fontfamily{#3}\fontseries{#4}\fontshape{#5}%
  \selectfont}%
\fi\endgroup%
{\renewcommand{\dashlinestretch}{30}
\begin{picture}(3480,1792)(0,-10)
\put(3465,955){\makebox(0,0)[b]{\smash{{\SetFigFont{10}{12.0}{\rmdefault}{\mddefault}{\updefault}$e$}}}}
\path(1065,880)(1065,205)
\blacken\path(1035.000,325.000)(1065.000,205.000)(1095.000,325.000)(1035.000,325.000)
\path(1815,1555)(1140,205)
\blacken\path(1166.833,325.748)(1140.000,205.000)(1220.498,298.915)(1166.833,325.748)
\path(1815,1555)(1065,1105)
\blacken\path(1152.464,1192.464)(1065.000,1105.000)(1183.334,1141.015)(1152.464,1192.464)
\path(1815,1555)(2415,1105)
\blacken\path(2301.000,1153.000)(2415.000,1105.000)(2337.000,1201.000)(2301.000,1153.000)
\path(1815,1555)(3465,1105)
\blacken\path(3341.335,1107.631)(3465.000,1105.000)(3357.122,1165.517)(3341.335,1107.631)
\path(2415,955)(1215,205)
\blacken\path(1300.860,294.040)(1215.000,205.000)(1332.660,243.160)(1300.860,294.040)
\path(3465,880)(1365,205)
\blacken\path(1470.063,270.282)(1365.000,205.000)(1488.424,213.160)(1470.063,270.282)
\path(325,1060)(770,1060)
\path(1785,1500)(1235,1165)
\path(1035,885)(1035,355)
\path(1830,1530)(1230,330)
\path(2370,955)(1330,305)
\path(1505,285)(3460,895)
\path(1815,1515)(2270,1165)
\path(1885,1520)(3300,1110)
\put(1815,1630){\makebox(0,0)[b]{\smash{{\SetFigFont{10}{12.0}{\rmdefault}{\mddefault}{\updefault}$\top$}}}}
\put(1065,955){\makebox(0,0)[b]{\smash{{\SetFigFont{10}{12.0}{\rmdefault}{\mddefault}{\updefault}$a$}}}}
\put(1065,55){\makebox(0,0)[b]{\smash{{\SetFigFont{10}{12.0}{\rmdefault}{\mddefault}{\updefault}$c$}}}}
\put(2415,955){\makebox(0,0)[b]{\smash{{\SetFigFont{10}{12.0}{\rmdefault}{\mddefault}{\updefault}$d$}}}}
\put(15,955){\makebox(0,0)[b]{\smash{{\SetFigFont{10}{12.0}{\rmdefault}{\mddefault}{\updefault}$b$}}}}
\blacken\path(285.000,1060.000)(165.000,1030.000)(285.000,1000.000)(285.000,1060.000)
\path(165,1030)(915,1030)
\blacken\path(795.000,1000.000)(915.000,1030.000)(795.000,1060.000)(795.000,1000.000)
\end{picture}
}

\caption{}\label{509-EF6}
\end{figure}

The node $c$ is supported by several chains.  For example
\begin{enumerate}
\item $\top \To a \To c$
\item $\top \To a \To b \To a\To c$
\item $\top \To d \To c$
\item $\top \To e \To c$
\item $\top \To c$
\end{enumerate}
We can also loop and get
\[
(6, k) ~~\top \To a\To b\To a\comma \mbox{ loop $k$ times } \To a \To b \To a \To c.
\]

\paragraph{Questions.}
What index do we give to c?

The first question is what to do with (6,k). We have an infinite number of paths leading from $\top$ to $c$. The second question is that we can get to $c$ using several parallel paths through $a, d$ and $e$. Does this increase its priority over $f$?

The thrid question is how do we record the totality of inidices/paths which characterise $c$?

We think it is reasonable to do the following:
\begin{enumerate}
\item Limit the paths with loops to $n$ number of repetitions, where $n$ is the number of elements in the network. Recall that we are dealing with finite networks.  Thus $(6,k)$ becomes $(6,1)\comma (6,6)$.
\item Record for each $x$ the number of different paths of each length.  Thus the index of node $c$ becomes the table in Figure \ref{509-EF7}.

\begin{figure}
\centering

\begin{tabular}{c|c}
length & number  of paths\\
\hline
2&1\\
3&3\\
5&1\\
7&1\\
9&1\\
11&1\\
13&1\\
15&1
\end{tabular}

\caption{}\label{509-EF7}
\end{figure}

\end{enumerate}

We leave it to the defeasible logical theory to decide which $x$ with table $\BD(x)$ is preferable to which $y$ with table $\BD(y)$. For example we can take the number of shortest paths to be the index.
\end{example}

\begin{definition}\label{509-ED8}
Let $(S\cup S^\urcorner \cup \{\top\}, R_1, R_2\}$ be a finite 2-state 2-attack bipolar network. Let $k$ be the number of elements in the network.
\begin{enumerate}
\item By a legitimate path from $x$ to $\top$ we mean a sequence 
\[
\Pi_x=(x_1=x, x_2\comma x_n=\top)
\]
such that for each $1 < i \leq n$ we have $x_iRx_{i-1}$, where $R$ is either $R_1= ~\tO$ or $R_2=  ~\To\!\!\!\!\To$ and such that no node $y$ appears in $\Pi_x$ more than $k$ times.
\item The defeasible value of $\Pi_x$ (denoted by $\BD(\Pi_x)$) is the number of points $x_i, 1 < i \leq n$ such that $x_i   \To\!\!\!\!\To x_{i-1}$.
\item The defeasible value of $x$ denoted by $\BD(x)$ is the pair $(\BD_1(x), \BD_2(x))$ where $\BD_1(x) =\min \{\BD(\Pi_x)\}$ and $\BD_2(x)=$ the number of paths $\Pi_x$ such that $\BD(\Pi_x) =\BD_1(x)$.
\item Order the nodes $x$ according to the lexicographic ordering of $\BD(x)$. We get three possibilities for two nodes $x,y$.
\begin{enumerate}
\item $\BD_1(x) < \BD_1 (y)$
\item $\BD_1(x) =\BD_1(y)$ and $\BD_2(x)> \BD_2(y)$
\item $\BD_1(x) =\BD_1(y)$ and $\BD_2(x)=\BD_(y)$.
\end{enumerate}
Possibilities (a) and (b) are considered $x > y$ ($x$ is stronger than $y$). Possibility (c) is considered $x \approx y$ ($x$ is indifferent to $y$).
\end{enumerate}
\end{definition}

\begin{definition}\label{509-ED9}
Let $(S\cup S^\urcorner \cup \{\top\}, \tO,  \To\!\!\!\!\To)$ be a finite 2-state, 2-attack bipolar network. Assume that each node $x$ already has an index $\BD(x)$ defined and that a priority ordering $x > y$ and $x\approx y$ have been defined. We define the notion of a Caminada-Gabbay labelling function 
\[
\lambda: S \cup S^\urcorner \cup \{\top\}\mapsto \{0,\half,1\}
\]
being a legitimate complete extension for the network.

 $\lambda$ is defined relative to the relations $>$ and $\approx$.

$\lambda$ must satisfy the following
\begin{enumerate}
\item $\lambda (\top)=1$.
\item Each pair $\{x,\neg x\}$ in the network is part of the following geometrical constellation of Figure \ref{509-EF10}.

\begin{figure}
\centering
\setlength{\unitlength}{0.00083333in}
\begingroup\makeatletter\ifx\SetFigFont\undefined%
\gdef\SetFigFont#1#2#3#4#5{%
  \reset@font\fontsize{#1}{#2pt}%
  \fontfamily{#3}\fontseries{#4}\fontshape{#5}%
  \selectfont}%
\fi\endgroup%
{\renewcommand{\dashlinestretch}{30}
\begin{picture}(4080,2022)(0,-10)
\put(1515,1860){\makebox(0,0)[b]{\smash{{\SetFigFont{10}{12.0}{\rmdefault}{\mddefault}{\updefault}$\comma$}}}}
\path(2865,1785)(2865,210)
\blacken\path(2835.000,330.000)(2865.000,210.000)(2895.000,330.000)(2835.000,330.000)
\path(3465,1785)(2940,210)
\blacken\path(2949.487,333.329)(2940.000,210.000)(3006.408,314.355)(2949.487,333.329)
\path(3990,1785)(3015,210)
\blacken\path(3052.655,327.822)(3015.000,210.000)(3103.671,296.241)(3052.655,327.822)
\path(615,1785)(615,210)
\blacken\path(585.000,330.000)(615.000,210.000)(645.000,330.000)(585.000,330.000)
\path(15,1785)(540,210)
\blacken\path(473.592,314.355)(540.000,210.000)(530.513,333.329)(473.592,314.355)
\path(1140,1785)(690,210)
\blacken\path(694.121,333.625)(690.000,210.000)(751.812,317.141)(694.121,333.625)
\path(1740,1785)(765,210)
\blacken\path(802.655,327.822)(765.000,210.000)(853.671,296.241)(802.655,327.822)
\blacken\path(1035.000,165.000)(915.000,135.000)(1035.000,105.000)(1035.000,165.000)
\path(915,135)(2715,135)
\blacken\path(2595.000,105.000)(2715.000,135.000)(2595.000,165.000)(2595.000,105.000)
\blacken\path(1185.000,165.000)(1065.000,135.000)(1185.000,105.000)(1185.000,165.000)
\path(1065,135)(2565,135)
\blacken\path(2445.000,105.000)(2565.000,135.000)(2445.000,165.000)(2445.000,105.000)
\path(465,450)(490,350)
\blacken\path(431.791,459.141)(490.000,350.000)(490.000,473.693)(431.791,459.141)
\path(615,465)(620,365)
\blacken\path(584.045,483.352)(620.000,365.000)(643.970,486.348)(584.045,483.352)
\path(2710,455)(2735,355)
\blacken\path(2676.791,464.141)(2735.000,355.000)(2735.000,478.693)(2676.791,464.141)
\path(2865,460)(2865,360)
\blacken\path(2835.000,480.000)(2865.000,360.000)(2895.000,480.000)(2835.000,480.000)
\path(3435,1790)(2960,365)
\blacken\path(2969.487,488.329)(2960.000,365.000)(3026.408,469.355)(2969.487,488.329)
\path(3955,1790)(3075,365)
\blacken\path(3112.526,482.863)(3075.000,365.000)(3163.577,451.338)(3112.526,482.863)
\path(1105,1780)(710,365)
\blacken\path(713.369,488.647)(710.000,365.000)(771.160,472.515)(713.369,488.647)
\path(1695,1795)(825,365)
\blacken\path(861.741,483.110)(825.000,365.000)(913.000,451.925)(861.741,483.110)
\put(615,60){\makebox(0,0)[b]{\smash{{\SetFigFont{10}{12.0}{\rmdefault}{\mddefault}{\updefault}$\neg x$}}}}
\put(2265,1860){\makebox(0,0)[b]{\smash{{\SetFigFont{10}{12.0}{\rmdefault}{\mddefault}{\updefault}$c_1 $}}}}
\put(2865,1860){\makebox(0,0)[b]{\smash{{\SetFigFont{10}{12.0}{\rmdefault}{\mddefault}{\updefault}$c_2$}}}}
\put(2565,1860){\makebox(0,0)[b]{\smash{{\SetFigFont{10}{12.0}{\rmdefault}{\mddefault}{\updefault}$\comma$}}}}
\put(3465,1860){\makebox(0,0)[b]{\smash{{\SetFigFont{10}{12.0}{\rmdefault}{\mddefault}{\updefault}$d_1$}}}}
\put(4065,1860){\makebox(0,0)[b]{\smash{{\SetFigFont{10}{12.0}{\rmdefault}{\mddefault}{\updefault}$d_2$}}}}
\put(3765,1860){\makebox(0,0)[b]{\smash{{\SetFigFont{10}{12.0}{\rmdefault}{\mddefault}{\updefault}$\comma$}}}}
\put(2865,60){\makebox(0,0)[b]{\smash{{\SetFigFont{10}{12.0}{\rmdefault}{\mddefault}{\updefault}$x$}}}}
\put(15,1860){\makebox(0,0)[b]{\smash{{\SetFigFont{10}{12.0}{\rmdefault}{\mddefault}{\updefault}$a_1$}}}}
\put(615,1860){\makebox(0,0)[b]{\smash{{\SetFigFont{10}{12.0}{\rmdefault}{\mddefault}{\updefault}$a_m$}}}}
\put(315,1860){\makebox(0,0)[b]{\smash{{\SetFigFont{10}{12.0}{\rmdefault}{\mddefault}{\updefault}$\comma$}}}}
\put(1140,1860){\makebox(0,0)[b]{\smash{{\SetFigFont{10}{12.0}{\rmdefault}{\mddefault}{\updefault}$b_1$}}}}
\put(1815,1860){\makebox(0,0)[b]{\smash{{\SetFigFont{10}{12.0}{\rmdefault}{\mddefault}{\updefault}$b_k$}}}}
\path(2265,1785)(2790,210)
\blacken\path(2723.592,314.355)(2790.000,210.000)(2780.513,333.329)(2723.592,314.355)
\end{picture}
}
\caption{}\label{509-EF10}
\end{figure}

There may be no $a_i$ and/or no $b_j$ and/or no $c_i$ and/or no $d_j$.

We assume that the $\lambda$ function is known for $\{a$s, $b$s, $c$s, $d$s\} and we indicate by case analysis what the values $\lambda (\neg x)$ and $\lambda (x)$ should be.
\item $\lambda (x) + \lambda (\neg x)=1$.
\item It is not the case that for some $i$ and some $j$, 
\[
\lambda (a_i) =\lambda (c_j)=1
\]
(If this happens then $\lambda$ is not legitimate.)

For networks coming from consistent defeasible theories $\Delta$, this will not happen because it means that using strict rules we have $a_i\vdash \neg x$ and $c_j\vdash x$ and both $\Delta \vdash a_i$ and $\Delta\vdash c_j$.
\item If for some $i, \lambda (a_i)=1$ and for all $j, \lambda (c_j) < 1$ then $\lambda (\neg x) =0$ and $\lambda(x)=1$.
\item If for some $i, \lambda (c_i)=1$ and for all $j, \lambda (a_j)<1$ then $\lambda (\neg x) =1$ and $\lambda (x) =0$.
\item Assume that the values $\lambda (a_j),\lambda (c_i)$ are $<1$ for all $a_j$ and $c_i$.  If for at least one of $\{a_j,c_i\}, \lambda$ gives value $\half$ then $\lambda(x)=\lambda(\neg x) =\half$.
\item Assume $\lambda$ gives value 0 to all $\{a_i,c_j\}$, and assume that $x >\neg x$, then 
\begin{enumerate}
\item If for some $d_j, \lambda (d_j) =1$ then $\lambda(x)=0$ and $\lambda(\neg x) =1$.
\item If for all $d_j$, $\lambda (d_j) < 1$ and for some $d_j~ \lambda (d_j)=\half$ then $\lambda (x) =\lambda (\neg x)=\half$. 
\end{enumerate}
\item Assume for all $\{a$s, $c$s, $d$s\} $\lambda$ gives value 0 and $x>\neg x$ then
\begin{enumerate}
\item If for some $b_j, \lambda(b_j) =1$ then $\lambda(\neg x) =0$ and $\lambda(x)=1$.
\item If for all $b_j, \lambda (b_j) < 1$ and for some $b_j, \lambda (b_j)=\half$ then $\lambda (x) =\lambda(\neg x) =\half$.
\item If for all $j, \lambda (b_j)=0$ (i.e.\ the case is that none of $\neg x, x$ is attacked in any way) then $\lambda (x)=\lambda(\neg x) =\half$.
\end{enumerate}
\item The mirror case of (7)--(8) for $\neg x > x$. Take the mirror case analysis.
\item If $x\approx \neg x$ and all $\lambda  (a$s$)=\lambda(c$s$)=0$ and $\lambda (b_j)=1$ for some $j$ and all $\lambda (d$s$)<1$ then $\lambda (\neg x)=0$ and $\lambda (x)=1$.
\item If all $\lambda (a{\rm s}) =\lambda (c{\rm s})=0$ and all $\lambda (b{\rm s})<1$ and some $\lambda (d{\rm s})=1$ then $\lambda (\neg x) =1$ and $\lambda (x)=0$.
\item If all $\lambda (a{\rm s})=\lambda (c{\rm s})=0$ and either for some $b_j$ and $d_i$ $\lambda (b_j) =\lambda (d_i)=1$ or all $\lambda (b{\rm s}), \lambda (d{\rm s})<1$ then $\lambda (\neg x) =\lambda (x)=\half$.
\end{enumerate}
\end{definition}

\begin{definition}\label{509-ED11}
Consider a rule of the form $\alpha =[\bigwedge_i x_i \rightsquigarrow z]$ where ``$\rightsquigarrow$'' is either ``$\tO$'' or ``$ \To\!\!\!\!\To$'', and $x_i$ and $z$ are literals of the form either $b$ or $\neg b$, with $b$ atomic. We agree that ``$\neg\neg b$'' is ``$b$''.

We translate $\alpha$ into an attack formation $\Delta_\alpha$ in the language of 2-state 2-attack networks as in Figure \ref{509-BF9}.  We use the auxiliary points $S_\alpha =\{y_1(\alpha)\comma y_n(\alpha), y(\alpha)\}$
\[
\Delta_\alpha =\{x_i\rightsquigarrow\!\!\to \neg y_k (\alpha)|i=1\comma n\}\cup \{y_i(\alpha)\rightsquigarrow\!\!\to \neg y(\alpha),y(\alpha)\rightsquigarrow\!\!\to \neg z\}.
\]
Where ``$\rightsquigarrow\!\!\to$'' is ``$\tO$'' if ``$\rightsquigarrow$'' is ``$\to$'' and ``$\rightsquigarrow\!\!\to$'' is ``$ \To\!\!\!\!\To$'' if ``$\rightsquigarrow$'' is ``$\To$''.

The auxiliary points of $S_\alpha, S_\beta$ are all pairwise disjoint for $\alpha$ different from $\beta$ and disjoint from the literals of $\Delta$.
\end{definition}

\begin{definition}\label{ED12}
Let $\Delta$ be a defeasible theory based on the set of atoms $S$, containing the literals $S\cup S^\urcorner$ with $S^\urcorner =\{\neg x|x\in S\}$, with $\bot, \top\not\in S\cup S^\urcorner$.

Let $\CD$ be the set of defeasible rules and $\CS$ be the set of strict rules. Assume the language of $\Delta$ has strict implication $\to$ and defeasible implication $\To$.  The rules of $\Delta$ have the form 
\[
\bigwedge_i x_i\to y
\]
 or 
\[
\bigwedge_j y_j\To z
\]
where $\{x_i, y_i, y, z\}$ are literals, i.e.\ have the form $b$ or $\neg b$, $b$ atomic letter.

We allow for rules of the form $\to y$ or $\To z$ meaning $y$ or $z$ are assumptions.

We define the associated network $N(\Delta)$ for $\Delta$ as follows:
\begin{enumerate}
\item The set of nodes of $N(\Delta)$ is $S\cup S^\urcorner\cup \{\top\}\cup\bigcup_\alpha (S_\alpha\cup S^\urcorner_\alpha)$ where $\alpha$ runs over all rules of the form $\bigwedge^k_{i=1} x_i\rightsquigarrow z$ with $k \geq 2$. (I.e.\ joint rules) and ``$\rightsquigarrow$'' is either ``$\to$''  or ``$\To$''.

The attack relation of $N(\Delta)$ is as follows:
\[\begin{array}{l}
\{b \twoheadleftarrow\!\!\tO\neg b|b \mbox{ atom of }
 N(\Delta\}\cup \{\top \tO \neg b | c \TO b \mbox{ in } \Delta\}\cup \{\top \To\!\!\!\!\To \neg b | c \To b\mbox{ in }\Delta\}\cup \bigcup_\alpha\\
~~~\Delta_\alpha\cup \{ x \rightsquigarrow\!\!\to \neg y | x \rightsquigarrow y \mbox{ in } \Delta\}
\end{array}\]
where $\alpha$ runs over all joint rules $\alpha$ in $\Delta$.

Note that $N(\Delta)$ satisfies that for every $b\neq\top$ either $b$ or $\neg b$ is attacked (using $\rightsquigarrow\!\!\to$) by some node.
\end{enumerate}

\end{definition}

\section{Discussion of papers of Arieli and Strasser \cite{509-26,509-27} and the book of Besnard and Hunter \cite{509-29}}

Our purpose here is to compare our work with that of Besnard and Hunter and in parallel, with that of Arieli and Strasser.  We first want to make two comments:
\begin{enumerate}
\item To set the scene for the comarison we need to start with my 1999 paper \cite{509-31}, which contains the relevant machinery.
\item Whatever criticism I have here of Arieli and Strasser \cite{509-26,509-27}, it must be borne in mind that these papers are preliminary conference papers and not definitive versions, like, e.g.\ Besnard and Hunter \cite{509-29}.
\end{enumerate}

So let us start. My paper \cite{509-31} and the later chapter 7 of our mootgraph \cite{509-32} dealt with what I called {\em compromise revision} of databases. We explain by example.

\begin{example}\label{509-AFE1}
This example continues Example \ref{509-EE0}, with a view of illustrating the idea of compromise revision.

Let $\Delta$ be a theory governing a birthday party. Add to $\Delta$ the following additional statements (t1)--(t5)
\begin{itemlist}{(ttt-5}
\item [(t1)] $a$ = invite Agnes
\item [(t2)] $b$ = invite Bertha
\item [(t3)] $a \to \neg b$
\item [(t4)]$b \to \neg a$
\item [(t5)]$a\wedge b \to$ invite Caterina (let $c$ = invite Caterina).
\end{itemlist}
Then our database is $\Gamma$:
\[
\Gamma =\Delta \cup \{a,b,a\to \neg b, b\to \neg a, a\wedge b \to c\}.
\]
$\Gamma$ is not consistent, but each of its items makes initial sense. Agnes and Bertha may be old aunties who do not talk to each other because of something that happened 30 years ago. Caterina may be an old friend of each one of them whose presence might ``mitigate'' the friction.  After some deliberation a decision was made not to risk inviting these two warring aunties. This means that we regard $\Gamma$ as an inconsistent theory in need of belief revision. The obvious revision is to delete either (t1) or (t2), i.e.\ not invite one of the aunties. If we do that then we need not invite Caterina (i.e.\ $c$ would not be provable).

The compromise revision problem is the following:
\begin{itemize}
\item Given an inconsistent $\Gamma$ and $X\in \Gamma$ such that $\Gamma -X$ is consistent, and given a $Z$ such that $\Gamma\vdash Z$ but $\Gamma -X \not\vdash Z$, then compromise revision would like to include $Z$ in the revised theory in case the revised theory is $\Gamma-X$.
\end{itemize}

The problem is that if $~\Gamma$ is inconsistent, then $\Gamma$ can prove everything, including 
\begin{center}
$V$ = invite Vladimir Putin
\end{center}
So we need to be careful and ``control'' what $\Gamma$ exactly proves. For that we use the discipline of labelled deductive systems \cite{509-33}.  We label every step in the syntactic proof of any wff $Z$.

The LDS rule of Modus Ponens is 
\[
\frac{\beta: A; \gamma: A\to B}{(\gamma,\beta): B}
\]

Thus we have that $\Gamma$ proves
\[
(t_1, t_2, t_5): c
\]
If we can prove $V$, it will be by some label $\alpha$, $\alpha: V$ which can be recognised.

So to prove $V$ by virtue of $\Gamma$ being inconsistent, we will need to first prove some $x\wedge\neg x$ and then use the axiom $x\wedge\neg x\to V$. This will all be recorded in $\alpha$ and we can recognise it and not include $V$ in the revised theory, i.e.\ invite Vladimir Putin. By comparison, if for example Putin is a relative of Agnes and $\Delta$ says something about relatives, we may have a more direct labelled proof of $V$.
\end{example}

The background considerations in Example \ref{509-AFE1} show that we have precise LDS machinery to trace proofs. Thus the Besnard and Hunter notion of argument of the form $(\Delta, A)$ where $\Delta$ is a minimal theory such that $\Delta\vdash A$, can be refined to be $t_\Delta: A$ where $t$ is a label showing how $A$ is proved from $\Delta$. There may be several such proofs in which case there would be several such labels. So given an inconsistent theory $\Gamma$, we can look at all labelled proofs $t_i: A_i$ from $\Gamma$ and define an attack relation among them in a much refined way, taking into account exactly how each formula is proved.

Arieli and Strasser use a Gentzen formulation of the logic and use progressions of Gentzen sequents as their attacking elements. I have reservations about the very idea of using Gentzen systems as the machinery for attack in the context  of argumentation networks. I think that Arieli and Strasser's impressive system is more at home with meta-level theories of belief revision, rather than abstract  argumentation. However, this is not the place to to discuss and evaluate their system, especially since, the current publication is at a conference and we yet to wait for the Journal expanded version. It is enough to say that for the purpose of comparison with the current paper, given $(S, R)$ then  if we consider instantiations of the form $(S,R,I)$, where $I$ is an instantiation into labelled formulas of some labelled deductive system, i.e.\  $I: S \mapsto$ LDS, then such a system would  generalise and include both approaches; the Arieli--Strasser approach and the Besnard--Hunter approach.  This, however, is the subject for a new paper. 

We concldue by looking at Figure \ref{509-AFF2}, which explains the situation of Example \ref{509-AFE1}.

\begin{figure}
\centering
\setlength{\unitlength}{0.00083333in}
\begingroup\makeatletter\ifx\SetFigFont\undefined%
\gdef\SetFigFont#1#2#3#4#5{%
  \reset@font\fontsize{#1}{#2pt}%
  \fontfamily{#3}\fontseries{#4}\fontshape{#5}%
  \selectfont}%
\fi\endgroup%
{\renewcommand{\dashlinestretch}{30}
\begin{picture}(1980,3072)(0,-10)
\put(1242,885){\makebox(0,0)[lb]{\smash{{\SetFigFont{10}{12.0}{\rmdefault}{\mddefault}{\updefault}joint attack}}}}
\path(792,2835)(1542,2385)
\blacken\path(1423.666,2421.015)(1542.000,2385.000)(1454.536,2472.464)(1423.666,2421.015)
\path(277,2530)(177,2470)
\blacken\path(264.464,2557.464)(177.000,2470.000)(295.334,2506.015)(264.464,2557.464)
\path(1307,2530)(1402,2455)
\blacken\path(1289.225,2505.811)(1402.000,2455.000)(1326.403,2552.904)(1289.225,2505.811)
\path(42,2085)(42,1635)
\blacken\path(12.000,1755.000)(42.000,1635.000)(72.000,1755.000)(12.000,1755.000)
\path(1542,2085)(1542,1635)
\blacken\path(1512.000,1755.000)(1542.000,1635.000)(1572.000,1755.000)(1512.000,1755.000)
\path(1542,1935)(1542,1785)
\blacken\path(1512.000,1905.000)(1542.000,1785.000)(1572.000,1905.000)(1512.000,1905.000)
\path(42,1935)(42,1785)
\blacken\path(12.000,1905.000)(42.000,1785.000)(72.000,1905.000)(12.000,1905.000)
\blacken\path(312.000,1515.000)(192.000,1485.000)(312.000,1455.000)(312.000,1515.000)
\path(192,1485)(1392,1485)
\blacken\path(1272.000,1455.000)(1392.000,1485.000)(1272.000,1515.000)(1272.000,1455.000)
\blacken\path(462.000,1515.000)(342.000,1485.000)(462.000,1455.000)(462.000,1515.000)
\path(342,1485)(1242,1485)
\blacken\path(1122.000,1455.000)(1242.000,1485.000)(1122.000,1515.000)(1122.000,1455.000)
\path(42,1335)(792,885)
\path(1542,1335)(792,885)
\path(792,885)(792,285)
\blacken\path(762.000,405.000)(792.000,285.000)(822.000,405.000)(762.000,405.000)
\path(792,510)(792,435)
\blacken\path(762.000,555.000)(792.000,435.000)(822.000,555.000)(762.000,555.000)
\blacken\path(1137.000,165.000)(1017.000,135.000)(1137.000,105.000)(1137.000,165.000)
\path(1017,135)(1692,135)
\blacken\path(1572.000,105.000)(1692.000,135.000)(1572.000,165.000)(1572.000,105.000)
\blacken\path(1287.000,165.000)(1167.000,135.000)(1287.000,105.000)(1287.000,165.000)
\path(1167,135)(1542,135)
\blacken\path(1422.000,105.000)(1542.000,135.000)(1422.000,165.000)(1422.000,105.000)
\path(567,1035)(569,1037)(572,1041)
	(577,1048)(586,1058)(596,1071)
	(609,1085)(623,1101)(639,1116)
	(656,1131)(674,1145)(693,1157)
	(715,1168)(738,1177)(764,1183)
	(792,1185)(820,1183)(846,1177)
	(869,1168)(891,1157)(910,1145)
	(928,1131)(945,1116)(961,1101)
	(975,1085)(988,1071)(998,1058)
	(1007,1048)(1012,1041)(1015,1037)(1017,1035)
\put(792,2910){\makebox(0,0)[b]{\smash{{\SetFigFont{10}{12.0}{\rmdefault}{\mddefault}{\updefault}$\top$}}}}
\put(42,2160){\makebox(0,0)[b]{\smash{{\SetFigFont{10}{12.0}{\rmdefault}{\mddefault}{\updefault}$\neg a$}}}}
\put(1542,2160){\makebox(0,0)[b]{\smash{{\SetFigFont{10}{12.0}{\rmdefault}{\mddefault}{\updefault}$\neg b$}}}}
\put(42,1410){\makebox(0,0)[b]{\smash{{\SetFigFont{10}{12.0}{\rmdefault}{\mddefault}{\updefault}$a$}}}}
\put(1542,1410){\makebox(0,0)[b]{\smash{{\SetFigFont{10}{12.0}{\rmdefault}{\mddefault}{\updefault}$b$}}}}
\put(792,960){\makebox(0,0)[b]{\smash{{\SetFigFont{10}{12.0}{\rmdefault}{\mddefault}{\updefault}$\wedge$}}}}
\put(792,60){\makebox(0,0)[b]{\smash{{\SetFigFont{10}{12.0}{\rmdefault}{\mddefault}{\updefault}$\neg c$}}}}
\put(1767,60){\makebox(0,0)[lb]{\smash{{\SetFigFont{10}{12.0}{\rmdefault}{\mddefault}{\updefault}$c$}}}}
\path(792,2835)(42,2385)
\blacken\path(129.464,2472.464)(42.000,2385.000)(160.334,2421.015)(129.464,2472.464)
\end{picture}
}

\caption{}\label{509-AFF2}
\end{figure}

The preferred extensions for this figure are the following 
\begin{enumerate}
\item $\neg b =\neg a=0, a=1, b=0, \neg c=1, c=0$
\item $\neg b=\neg a=0, a=1, b=0, \neg c=0, c=1$
\item Same as (1), with $a=0, b=1$
\item Same as (2), with $a=0, b=1$.
\end{enumerate}

\begin{example}\label{509-AFE3}
Let us do Example 17 from Arieli and Strasser \cite{509-27}.  This is to show how simple the labelled approach is compared with the Gentzen formulation.  Gentzen systems were invented to prove the consistency of arithmetic. It is risky to take off the shelf tool designed for one purpose and apply it to another purpose; such a move requires proper justification.

The data of this example is:
\begin{itemlist}{(ttt-5}
\item [(t1)] $m$
\item [(t2)] $a$ 
\item [(t3)] $m \to \bigcirc\neg f$
\item [(t4)]$m\wedge a\to \bigcirc f$.
\end{itemlist}

The meaning of the normative $\bigcirc f$ and $\bigcirc \neg f$ is not important here. It is sufficient to note that they attack each other. We can derive:
\begin{enumerate}
\item $(t_3, t_1):\bigcirc \neg f$
\item $(t_1, t_2): m\wedge a$
\item $(t_3, (t_1, t_2)):\bigcirc f$
\end{enumerate}

(1) and (3) attack each other. They have labels telling us how they were proved and one can define an attack relation sensitive of the labels.  In LDS we call this ``flattening''. See \cite{509-33}.\footnote{We may have in LDS that  we can prove $t_i : X$ and also $s_j : \neg X$, yielding multiple bilateral attacks between $X$ and $\neg X$ from different proofs. The Flattening process decides, based on $\{t_i, s_j\}$ whether $X$ or $\neg X$ has the upper hand.}

In fact, the attack relation can be a relation $R$ on labels. So $S$ can be the labels and $R\subseteq S\times S$.  This is OK since the labels contain the information of the proofs, including the proved formulas. In fact, in my book on LDS \cite{509-33}, I use many times the formulas themselves as their own labels. So $\bigcirc\neg f$ is labelled by ($1^*$) and $\bigcirc f$ by ($3^*$), where 
\begin{itemlist}{3333}
\item [($1^*$)] $(m \to \bigcirc\neg f, m): \bigcirc\neg f$
\item [($3^*$)] $(m\wedge a \to \bigcirc f, (m, a)): \bigcirc f$
\end{itemlist}

Now compare this with the Gentzen formulation in Figure 3 of \cite{509-27}. We reproduce it as Figure \ref{509-AFF4} (the horseshoe is classical implication and the double arrow is the main symbol for the Gentzen sequent).

\begin{figure}
\centering
\setlength{\unitlength}{0.00083333in}
\begingroup\makeatletter\ifx\SetFigFont\undefined%
\gdef\SetFigFont#1#2#3#4#5{%
  \reset@font\fontsize{#1}{#2pt}%
  \fontfamily{#3}\fontseries{#4}\fontshape{#5}%
  \selectfont}%
\fi\endgroup%
{\renewcommand{\dashlinestretch}{30}
\begin{picture}(2130,2443)(0,-10)
\put(15,55){\makebox(0,0)[b]{\smash{{\SetFigFont{10}{12.0}{\rmdefault}{\mddefault}{\updefault}$A$}}}}
\path(765,2230)(2115,1705)
\path(2115,1555)(1365,1105)
\path(1290,880)(615,205)
\path(1290,880)(1365,205)
\dashline{60.000}(1965,2380)(915,2380)
\blacken\path(1035.000,2410.000)(915.000,2380.000)(1035.000,2350.000)(1035.000,2410.000)
\dashline{60.000}(1965,2380)(90,1105)
\blacken\path(172.362,1197.285)(90.000,1105.000)(206.100,1147.669)(172.362,1197.285)
\dashline{60.000}(1965,2380)(15,205)
\blacken\path(72.768,314.375)(15.000,205.000)(117.442,274.322)(72.768,314.375)
\put(765,2305){\makebox(0,0)[b]{\smash{{\SetFigFont{10}{12.0}{\rmdefault}{\mddefault}{\updefault}$I$}}}}
\put(2115,2305){\makebox(0,0)[b]{\smash{{\SetFigFont{10}{12.0}{\rmdefault}{\mddefault}{\updefault}$H$}}}}
\put(2115,1555){\makebox(0,0)[b]{\smash{{\SetFigFont{10}{12.0}{\rmdefault}{\mddefault}{\updefault}$G$}}}}
\put(2115,55){\makebox(0,0)[b]{\smash{{\SetFigFont{10}{12.0}{\rmdefault}{\mddefault}{\updefault}$D$}}}}
\put(15,955){\makebox(0,0)[b]{\smash{{\SetFigFont{10}{12.0}{\rmdefault}{\mddefault}{\updefault}$E$}}}}
\put(1365,955){\makebox(0,0)[b]{\smash{{\SetFigFont{10}{12.0}{\rmdefault}{\mddefault}{\updefault}$F$}}}}
\put(1365,55){\makebox(0,0)[b]{\smash{{\SetFigFont{10}{12.0}{\rmdefault}{\mddefault}{\updefault}$C$}}}}
\put(615,55){\makebox(0,0)[b]{\smash{{\SetFigFont{10}{12.0}{\rmdefault}{\mddefault}{\updefault}$B$}}}}
\path(765,2230)(15,1105)
\end{picture}
}

\[
\begin{array}{rcl}
\hat{A} &=& m\supset \bigcirc\neg f\To m \supset \bigcirc \neg f\\
\hat{B} &=& m\To m\\
\hat {C} &=& a\To a\\
\hat{D} &=& (m\wedge a)\supset\bigcirc f\To (m\wedge a)\supset \bigcirc f\\
\hat{E}&=& m,m\supset \bigcirc \neg f\To \bigcirc\neg f\\
\hat{F} &=& m, a\To m\wedge a\\
\hat{G} &=& m, a,(m\wedge a)\supset \bigcirc f \To \bigcirc f\\
\hat{H} &=& m, a,(m\wedge a) \supset\bigcirc f\To \neg (m\supset \bigcirc\neg f)\\
\hat{I} &=& m, a, m \supset\bigcirc \neg f, (m\wedge a)\supset \bigcirc f \To \bigcirc \bot.
\end{array}
\]
\caption{}\label{509-AFF4}
\end{figure}

The problem is not so much the complexity of the representation. The problem is that proofs in a Gentzen system do not flow with the implication, while the human argument does follow the implication.  Arieli and Strasser are just using Gentzen as a meta-level deductive machine!

Figure \ref{509-AFF5} gives the argumentation network the way we build it. We simplified $m\wedge a\to \bigcirc f$ as $(m, a)\to \bigcirc f$, so we avoid conjunction. There is nothing special to this move.

You can see that paths indicate chains of implications, very intuitive the way we think of it!

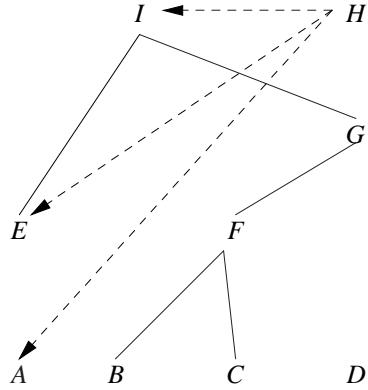
\begin{figure}
\centering
\setlength{\unitlength}{0.00083333in}
\begingroup\makeatletter\ifx\SetFigFont\undefined%
\gdef\SetFigFont#1#2#3#4#5{%
  \reset@font\fontsize{#1}{#2pt}%
  \fontfamily{#3}\fontseries{#4}\fontshape{#5}%
  \selectfont}%
\fi\endgroup%
{\renewcommand{\dashlinestretch}{30}
\begin{picture}(3078,3072)(0,-10)
\put(1215,60){\makebox(0,0)[rb]{\smash{{\SetFigFont{10}{12.0}{\rmdefault}{\mddefault}{\updefault}$\bigcirc f$}}}}
\path(1890,2835)(2640,2385)
\blacken\path(2521.666,2421.015)(2640.000,2385.000)(2552.536,2472.464)(2521.666,2421.015)
\path(1375,2530)(1275,2470)
\blacken\path(1362.464,2557.464)(1275.000,2470.000)(1393.334,2506.015)(1362.464,2557.464)
\path(2405,2530)(2500,2455)
\blacken\path(2387.225,2505.811)(2500.000,2455.000)(2424.403,2552.904)(2387.225,2505.811)
\path(1140,2085)(1140,1485)
\blacken\path(1110.000,1605.000)(1140.000,1485.000)(1170.000,1605.000)(1110.000,1605.000)
\path(2640,2085)(2640,1560)
\blacken\path(2610.000,1680.000)(2640.000,1560.000)(2670.000,1680.000)(2610.000,1680.000)
\path(2640,1935)(2640,1635)
\blacken\path(2610.000,1755.000)(2640.000,1635.000)(2670.000,1755.000)(2610.000,1755.000)
\path(1140,1935)(1140,1635)
\blacken\path(1110.000,1755.000)(1140.000,1635.000)(1170.000,1755.000)(1110.000,1755.000)
\path(1140,1335)(1890,885)
\path(2640,1335)(1890,885)
\path(1890,885)(1890,285)
\blacken\path(1860.000,405.000)(1890.000,285.000)(1920.000,405.000)(1860.000,405.000)
\path(1890,510)(1890,435)
\blacken\path(1860.000,555.000)(1890.000,435.000)(1920.000,555.000)(1860.000,555.000)
\path(1140,1335)(240,1035)
\blacken\path(344.355,1101.408)(240.000,1035.000)(363.329,1044.487)(344.355,1101.408)
\blacken\path(210.000,990.000)(90.000,960.000)(210.000,930.000)(210.000,990.000)
\path(90,960)(465,960)
\blacken\path(345.000,930.000)(465.000,960.000)(345.000,990.000)(345.000,930.000)
\blacken\path(135.000,990.000)(15.000,960.000)(135.000,930.000)(135.000,990.000)
\path(15,960)(540,960)
\blacken\path(420.000,930.000)(540.000,960.000)(420.000,990.000)(420.000,930.000)
\blacken\path(1770.000,105.000)(1890.000,135.000)(1770.000,165.000)(1770.000,105.000)
\path(1890,135)(1290,135)
\blacken\path(1410.000,165.000)(1290.000,135.000)(1410.000,105.000)(1410.000,165.000)
\blacken\path(1620.000,105.000)(1740.000,135.000)(1620.000,165.000)(1620.000,105.000)
\path(1740,135)(1365,135)
\blacken\path(1485.000,165.000)(1365.000,135.000)(1485.000,105.000)(1485.000,165.000)
\blacken\path(774.502,794.670)(690.000,885.000)(722.052,765.532)(774.502,794.670)
\path(690,885)(1065,210)
\blacken\path(980.498,300.330)(1065.000,210.000)(1032.948,329.468)(980.498,300.330)
\blacken\path(852.464,647.536)(765.000,735.000)(801.015,616.666)(852.464,647.536)
\path(765,735)(990,360)
\blacken\path(902.536,447.464)(990.000,360.000)(953.985,478.334)(902.536,447.464)
\path(1140,1935)(1140,2160)
\blacken\path(1170.000,2040.000)(1140.000,2160.000)(1110.000,2040.000)(1170.000,2040.000)
\path(1140,1935)(1140,2010)
\blacken\path(1170.000,1890.000)(1140.000,2010.000)(1110.000,1890.000)(1170.000,1890.000)
\path(2640,2010)(2640,2160)
\blacken\path(2670.000,2040.000)(2640.000,2160.000)(2610.000,2040.000)(2670.000,2040.000)
\path(2640,1935)(2640,2010)
\blacken\path(2670.000,1890.000)(2640.000,2010.000)(2610.000,1890.000)(2670.000,1890.000)
\path(1665,1035)(1667,1037)(1670,1041)
	(1675,1048)(1684,1058)(1694,1071)
	(1707,1085)(1721,1101)(1737,1116)
	(1754,1131)(1772,1145)(1791,1157)
	(1813,1168)(1836,1177)(1862,1183)
	(1890,1185)(1918,1183)(1944,1177)
	(1967,1168)(1989,1157)(2008,1145)
	(2026,1131)(2043,1116)(2059,1101)
	(2073,1085)(2086,1071)(2096,1058)
	(2105,1048)(2110,1041)(2113,1037)(2115,1035)
\put(1890,2910){\makebox(0,0)[b]{\smash{{\SetFigFont{10}{12.0}{\rmdefault}{\mddefault}{\updefault}$\top$}}}}
\put(1140,2160){\makebox(0,0)[b]{\smash{{\SetFigFont{10}{12.0}{\rmdefault}{\mddefault}{\updefault}$\neg m$}}}}
\put(2640,2160){\makebox(0,0)[b]{\smash{{\SetFigFont{10}{12.0}{\rmdefault}{\mddefault}{\updefault}$\neg a$}}}}
\put(1140,1410){\makebox(0,0)[b]{\smash{{\SetFigFont{10}{12.0}{\rmdefault}{\mddefault}{\updefault}$m$}}}}
\put(2640,1410){\makebox(0,0)[b]{\smash{{\SetFigFont{10}{12.0}{\rmdefault}{\mddefault}{\updefault}$a$}}}}
\put(1890,960){\makebox(0,0)[b]{\smash{{\SetFigFont{10}{12.0}{\rmdefault}{\mddefault}{\updefault}$\wedge$}}}}
\put(2340,885){\makebox(0,0)[lb]{\smash{{\SetFigFont{10}{12.0}{\rmdefault}{\mddefault}{\updefault}joint attack}}}}
\put(15,885){\makebox(0,0)[rb]{\smash{{\SetFigFont{10}{12.0}{\rmdefault}{\mddefault}{\updefault}$\neg\bigcirc \neg f$}}}}
\put(615,885){\makebox(0,0)[lb]{\smash{{\SetFigFont{10}{12.0}{\rmdefault}{\mddefault}{\updefault}$\bigcirc\neg f$}}}}
\put(1890,60){\makebox(0,0)[lb]{\smash{{\SetFigFont{10}{12.0}{\rmdefault}{\mddefault}{\updefault}$\neg\bigcirc f$}}}}
\path(1890,2835)(1140,2385)
\blacken\path(1227.464,2472.464)(1140.000,2385.000)(1258.334,2421.015)(1227.464,2472.464)
\end{picture}
}

\caption{}\label{509-AFF5}
\end{figure}
\end{example}

We conclude this Appendix by quoting  a response from C. Strasser and O. Arieli.

\begin{quote}
Begin quote.

Thanks for referring to our work in the above-mentioned paper. Below are some comments and a response to several issues.
\begin{enumerate}
\item  First, since Besnard and Hunter's (BH) formalism is mentioned in the same appendix, let us
emphasize the differences between our approach and theirs. According to BH, an argument
is a pair $\jset{\Gamma,\psi}$, where $\Gamma$ is a subset-minimal consistent set of propositional formulas that
entails according to classical logic the propositional formula  $\psi$. In our approach none of
these is assumed: languages other than the propositional one may be used, $\Gamma$ need not
be consistent nor minimal, and the underlying logic need not be classical logic. Another
difference is that we enhance the calculus of the base logic by elimination rules (see Item 3
below).
\item  The general view on arguments, as indicated in the previous item, may serve as a justification
for our choice to incorporate sequents in our framework: once an underlying logic $\CL$
is fixed, an argument $\jset{\Gamma,\psi}$  in our sense is an indication that   logically follows, according
to $\CL$, from $\Gamma$. If a sound and complete sequent calculus $\CC$ exists for $\CL$ it serves as a
syntactical tool for constructing complex arguments from simpler ones. Such a mechanism
must accompany, either implicitly or explicitly, any structural (logic-based) argumentation
system, so we are not sure that we follow the criticism in this case. Moreover, to some
extent (we still have to check this more carefully), and in the notations of your paper, a
labeled formula $t_\Delta: \psi$   may be associated with the sequent $\Delta\To \psi$   (the way we encode
arguments in COMMA'14), or with a concrete proof of a sequent $\Delta\To \psi$  (the way we
encode arguments in DEON'14 --- see also Item 6 below).
\item  We do not agree with the claim that our use of sequent-based Gentzen-style systems is
purely a `meta-level deductive machine', as in addition to deductions, aimed at systematically
constructing arguments, we also have sequent-based rules for {\em eliminating} arguments.
Thus, sequents are not only meta-leveled deducible objects, but they are syntactical entities
that may be retracted as well. This has two implications. First, the (enhanced) calculus
does not only produce the arguments for the argumentation framework but also the attacks.
Second, this tighter link between the calculus of the base logic and argumentation
frameworks (as compared to the BH-approach) allows to define a machinery for automated
deduction on the basis of dynamic proofs (see our COMMA'14 paper and the literature on
adaptive logics). It follows that derivations are more complicated structures than `ordinary'
proofs in Gentzen-type systems, which also allow for non-monotonic reasoning.
\item  Example F.1: It is noted that `the obvious revision is to delete either (t1) or (t2)'. One may
argue that eliminating both (t3) and (t4), thus inviting both aunties as well as Caterina,
is also a plausible revision. This view, which is more in-line with a paraconsistent view of
1 the state of affairs, in which all the assertions in a theory are treated uniformly, may be
supported by our setting, depending on the pre-defined logic and the attack rules.
\item  We are not sure that we understand what do you mean by `flowing with implication'.
Whatever this property may be, undesirable properties may be lifted by modifying the
corresponding proof system (and maybe also by changing the underlying logic). As we
indicated before, this is fully supported by our approach.
\item  Please note that in the DEON'14 paper an argument is the whole proof of $\psi$ from $\Gamma$ (and
so attacks may be on subproofs). This is similar to the way the ASPIC system views
arguments. In the COMMA'14 paper we adopted a simpler view, in which an argument is
simply a sequent (or, alternatively, only the `top sequent' of a proof). Both representations
of arguments seem to be different than the way that argument are encoded in your Dunglike
digraphs, where vertices are propositions. In view of this it is also difficult to directly
compare the two approaches in terms of representational complexity/transparency as you
seem to do at the end of appendix.
\end{enumerate}
End quote.
\end{quote}

\end{document}